\pdfoutput=1
\documentclass[11pt]{article}

\usepackage[margin=1in]{geometry}
\usepackage[numbers,compress]{natbib}

\usepackage[utf8]{inputenc} 
\usepackage[T1]{fontenc}    
\usepackage{url}            
\usepackage{booktabs}       
\usepackage{amsfonts}       
\usepackage{nicefrac}       
\usepackage{microtype}      
\usepackage{xcolor}         
\usepackage{enumitem}
\usepackage{amsmath, amssymb, amsthm, bm}
\usepackage{tabularx}
\newtheorem{assumption}{Assumption}
\usepackage{titletoc}
\usepackage[colorlinks=true, linkcolor=blue,citecolor=blue]{hyperref}       
\newcommand{\LineComment}[1]{\Statex $\triangleright$ \textit{#1}}

\newcommand{\Tr}{\mathrm{Tr}}
\newcommand{\cDde}{\bar{\mathcal{D}}_z^{(\infty)}(\varepsilon,\epsilon)}
\newcommand{\cDd}{\bar{\mathcal{D}}_z^{(\infty)}(\epsilon)}
\newcommand{\cDtwo}{\mathcal{D}_{z}^{(2)}(\epsilon)}
\newcommand{\cDinf}{\mathcal{D}_{z}^{(\infty)}(\epsilon)}

\newcommand{\cDuvtwo}{\mathcal{D}_{uv}^{(2)}(\epsilon)}
\newcommand{\cDuvinf}{\mathcal{D}_{uv}^{(\infty)}(\epsilon)}
\newcommand{\cDuvde}{\bar{\mathcal{D}}_{uv}^{(\infty)}(\varepsilon,\epsilon)}
\newcommand{\cDuvd}{\bar{\mathcal{D}}_{uv}^{(\infty)}(\epsilon)}

\newcommand{\prect}{_{\alpha}^{\natural}}
\newcommand{\Fn}{\mathrm{F}}
\newcommand{\Fnw}{2,\mathrm{F}}
\newcommand{\twinf}{2\to\infty}
\newcommand{\twinfw}{2,\infty}
\newcommand{\disttwo}{dist_{2}}
\newcommand{\distinf}{dist_{\infty}}
\let\hat\widehat
\let\tilde\widetilde

\newcommand{\ba}{\bm{a}}
\newcommand{\bb}{\bm{b}}

\newcommand{\bd}{\bm{d}}
\newcommand{\be}{\bm{e}}

\newcommand{\bh}{\bm{h}}

\newcommand{\bu}{\bm{u}}
\newcommand{\bv}{\bm{v}}
\newcommand{\bw}{\bm{w}}
\newcommand{\bx}{\bm{x}}
\newcommand{\by}{\bm{y}}
\newcommand{\bz}{\bm{z}}

\newcommand{\bA}{\bm{A}}
\newcommand{\bB}{\bm{B}}
\newcommand{\bC}{\bm{C}}
\newcommand{\bD}{\bm{D}}
\newcommand{\bE}{\bm{E}}

\newcommand{\bG}{\bm{G}}
\newcommand{\bH}{\bm{H}}
\newcommand{\bI}{\bm{I}}

\newcommand{\bL}{\bm{L}}
\newcommand{\bM}{\bm{M}}

\newcommand{\bP}{\bm{P}}
\newcommand{\bQ}{\bm{Q}}
\newcommand{\bR}{\bm{R}}
\newcommand{\bS}{\bm{S}}

\newcommand{\bU}{\bm{U}}
\newcommand{\bV}{\bm{V}}
\newcommand{\bW}{\bm{W}}
\newcommand{\bX}{\bm{X}}
\newcommand{\bY}{\bm{Y}}
\newcommand{\bZ}{\bm{Z}}

\newcommand{\cA}{\mathcal{A}}

\newcommand{\cD}{\mathcal{D}}
\newcommand{\cE}{\mathcal{E}}
\newcommand{\cF}{\mathcal{F}}
\newcommand{\cG}{\mathcal{G}}
\newcommand{\cH}{\mathcal{H}}

\newcommand{\cL}{\mathcal{L}}

\newcommand{\cN}{\mathcal{N}}
\newcommand{\cO}{\mathcal{O}}
\newcommand{\cP}{\mathcal{P}}

\newcommand{\cR}{\mathcal{R}}
\newcommand{\cS}{{\mathcal{S}}}

\newcommand{\EE}{\mathbb{E}}

\newcommand{\NN}{\mathbb{N}}
\newcommand{\PP}{\mathbb{P}}

\newcommand{\RR}{\mathbb{R}}
\newcommand{\SSS}{\mathbb{S}}

\newcommand{\btheta}{\bm{\theta}}
\newcommand{\bvartheta}{\bm{\vartheta}}

\newcommand{\bomega}{\bm{\omega}}

\newcommand{\bDelta}{\bm{\Delta}}

\newcommand{\bLambda}{\bm{\Lambda}}

\newcommand{\bSigma}{\bm{\Sigma}}

\newcommand{\diag}{{\rm diag}}

\newcommand{\argmin}{\mathop{\mathrm{argmin}}}

\newcommand{\tr}{\mathop{\mathrm{tr}}}
\newcommand*{\zero}{{\bm 0}}

\def\T{{ \intercal }}
\ifx\BlackBox\undefined
\newcommand{\BlackBox}{\rule{1.5ex}{1.5ex}}  
\fi

\ifx\QED\undefined
\def\QED{~\rule[-1pt]{5pt}{5pt}\par\medskip}
\fi

\ifx\proof\undefined
\newenvironment{proof}{\par\noindent{\bf Proof\ }}{\hfill\BlackBox\\ [2mm]}
\fi

\ifx\theorem\undefined
\newtheorem{theorem}{Theorem}
\fi
\ifx\example\undefined
\newtheorem{example}{Example}
\fi
\ifx\property\undefined
\newtheorem{property}{Property}
\fi
\ifx\lemma\undefined
\newtheorem{lemma}{Lemma}
\fi
\ifx\proposition\undefined
\newtheorem{proposition}{Proposition}
\fi
\ifx\remark\undefined
\newtheorem{remark}{Remark}
\fi
\ifx\corollary\undefined
\newtheorem{corollary}{Corollary}
\fi
\ifx\definition\undefined
\newtheorem{definition}{Definition}
\fi
\ifx\conjecture\undefined
\newtheorem{conjecture}{Conjecture}
\fi
\ifx\fact\undefined
\newtheorem{fact}{Fact}
\fi
\ifx\claim\undefined
\newtheorem{claim}{Claim}
\fi
\ifx\cond\undefined
\newtheorem{cond}{Condition}
\fi

\title{Convexity in Disguise:  A Theoretical Framework for Nonconvex Low-Rank Matrix Estimation}

\author{
Chengyu Cui and Gongjun Xu\\
Department of Statistics, University of Michigan}
\date{}

\begin{document}

\maketitle

\begin{abstract}
  Nonconvex methods have emerged as a dominant approach for low-rank matrix estimation, a problem that arises widely in machine learning and AI for learning and representing high-dimensional data. 
Existing analyses for these methods often require additional regularization to mitigate nonconvexity, even though such regularization is often unnecessary in practice. Moreover, most analyses rely on problem-specific arguments that are difficult to generalize to more complex settings.
In this paper, we develop a theoretical framework for studying nonconvex procedures across a broad class of low-rank matrix estimation problems. 
Rather than focusing on a specific model, we reveal a fundamental mechanism that explains why nonconvex procedures can behave well in low-rank estimation. 
Our key device is a {\it benign regularizer} that does not alter the original update rule, but yields an equivalent locally strongly convex formulation of the algorithm. 
This perspective uncovers a disguised convexity inherent in the nonconvex procedure and provides a new route to theoretical guarantees for nonconvex low-rank matrix estimation. 
\end{abstract}

\section{Introduction}
Low-rank matrix estimation arises in a wide range of problems in statistics, machine learning, and information processing~\citep{candes2010power,chi2019nonconvex,davenport2016overview}. 
Recently, low-rank structure has also become a central tool in modern machine learning, such as parameter-efficient adaptation of large models~\citep{hu2022lora,zhang2023adalora,zhang2025lora}, efficient Transformer attention via low-rank approximations~\citep{wang2020linformer,xiong2021nystromformer}, memory-efficient training and model compression~\citep{wang2024svd,zhao2024galore}, and implicit low-rank regularization in deep learning~\citep{arora2019implicit}.
The growing use of low-rank structure in increasingly complex models motivates a theoretical framework to understand the optimization geometry underlying low-rank estimation.

 In this paper, we study the problem of estimating a rank-$r$ matrix $\bX^*\in\RR^{n\times q}$ by minimizing a general loss function $\cL:\RR^{n\times q}\to\RR$, where $r$ is typically much smaller than $\min\{n,q\}$. Direct optimization of $\cL(\bX)$ subject to the rank constraint on $\bX$ is often computationally challenging, which is NP-hard in several important settings~\citep{recht2010guaranteed}. A common approach is to parameterize $\bX$ as $\bX = \bU\bV^\T$, where $\bU\in\RR^{n\times r}$ and $\bV\in\RR^{q\times r}$, often referred to as the Burer--Monteiro factorization~\citep{burer2003nonlinear}. Then the estimation problem can be reformulated as the following nonconvex problem:
\begin{equation}(\hat{\bU},\hat\bV) \; =\; {\argmin}_{(\bU,\bV)\in\RR^{n\times r}\times\RR^{q\times r}}\, \cL(\bU\bV^\T).\label{uv_case}
\end{equation}
A common approach to solving this problem is gradient descent, which updates $(\bU^t,\bV^t)$ at step $t$ by
\begin{equation}
    \bU^{t+1} = \bU^t-q^{-1}\eta^t\nabla_{\bU}\cL(\bU^t(\bV^t)^\T), \quad\text{and}\quad \bV^{t+1} = \bV^t-n^{-1}\eta^t\nabla_{\bV}\cL(\bU^t(\bV^t)^\T),\label{eq_one_step_update_uv}
\end{equation}where $\eta^t$ is the learning rate, $\nabla_{\bU}\cL$ is the gradient of $\cL$ w.r.t. $\bU$, and $n^{-1},q^{-1}$ are scaling factors.

The primary challenge in analyzing \eqref{eq_one_step_update_uv} is the nonconvexity of $\cL(\bU\bV^\T)$. In particular, regardless of the specific form or structural properties of $\cL(\cdot)$, one fundamental difficulty that leads to nonconvexity remains: the factorization is non-identifiable, since $(\bU,\bV)$ and $(\bU\bG,\bV\bG^{-\T})$ yield identical loss for any invertible $\bG\in\RR^{r\times r}$~\citep{cui2025identifiability}. To address this issue, many existing works 
introduce additional regularization or constraints to~\eqref{eq_one_step_update_uv}. Popular strategies include truncating selected measurements or updates to control outliers or incoherence~\citep{cherapanamjeri2017nearly,keshavan2010matrix,li2020non}; adding penalty terms~\citep{chen2020nonconvex,sun2016guaranteed,tu2016low,zheng2016convergence,wang2017unified}; 
and projecting the iterates onto certain sets based on prior knowledge~\citep{chen2015fast,zheng2016convergence,ma2020universal}.
While these modifications provide convergence guarantees in various models, they are often unnecessary in practice and introduce additional tuning parameters. More importantly, much of the existing analysis relies crucially on these regularization mechanisms, which makes it difficult to understand the geometry of the problem and generalize to complex settings.


Recently, a growing body of work has focused on the original regularization-free scheme~\eqref{eq_one_step_update_uv}, motivated not only by its strong empirical performance despite the nonconvexity of $\cL$, but also by its potential to shed light on the underlying geometry of the problem~\citep{arora2019implicit,du2018algorithmic,li2020global,ma2021beyond,tong2021accelerating}. 
For example, \citet{ma2021beyond} showed that, in noiseless low-rank matrix sensing and under suitable initialization, gradient descent converges linearly without an explicit balancing penalty $\|\bU^\T\bU - \bV^\T\bV\|_{\Fn}$. 
Similarly, \citet{li2020global} studied unregularized low-rank matrix recovery and established benign global landscape properties in the exact recovery setting. 
However, these analyses rely heavily on problem-specific structures and are therefore difficult to extend. Specifically, they are largely limited to noiseless settings and often rely on the squared loss, whose gradient has a closed-form expression that can be decomposed into simpler components and analyzed separately.
Moreover, the resulting algebraic derivations may obscure the geometry of the problem, limiting insight into the underlying mechanism.



In this paper, we develop a theoretical framework for the nonconvex procedure \eqref{eq_one_step_update_uv} under a broad class of loss functions $\cL(\cdot)$. Our analysis offers new insight into the underlying geometry and provides a new route to establishing convergence guarantees. Our contributions are summarized as follows.

\begin{enumerate}[leftmargin  = 0.6cm]
    \item
    We develop a novel framework for analyzing nonconvex
    procedures in low-rank matrix estimation. Our key finding is that \eqref{eq_one_step_update_uv} can be mapped to mirror a gradient update on a locally strongly convex objective. This is achieved via a {\it benign regularizer} that leaves the original update unchanged while augmenting the objective to be locally strongly convex, under which error contraction can be established.  Remarkably, because the update itself is not changed, this mirrored convex formulation suggests that the nonconvex procedure already contains an implicit convex structure, with our benign regularizer acting to reveal this structural feature. Our analysis provides a new way to understand the geometry underlying the problem.
    \item
   We establish convergence guarantees for the nonconvex procedure over a general class of loss functions under localized regularity conditions. A key feature of our theory is that these regularity conditions are imposed only on small neighborhoods around the true parameter, rather than globally, thereby enabling convergence analysis for general loss functions. This localization makes the conditions applicable across a broad range of problems, but also necessitates delicate control of the entire trajectory. 
   Our convex formulation offers a new route to establish iterative contraction bounds along the full algorithmic path under suitable initialization, thereby ensuring that the iterates remain within the local region throughout the analysis.

\item 
We establish convergence guarantees for estimation under statistical noise. Quantifying how noise reshapes the optimization landscape of regularization-free nonconvex procedures remains a major challenge. Under our general loss framework, we identify key quantities that characterize how statistical noise perturbs the estimation problem, and prove linear convergence guarantees in its presence. We demonstrate the applicability of our theory through two representative examples for which such convergence guarantees were previously unavailable.

\end{enumerate}


The rest of the paper is organized as follows.
Section~\ref{sec_general} introduces our framework. Section~\ref{sec_noise} describes how the framework handles statistical noise. Section~\ref{sec_conclude} concludes. In the appendices, we provide additional technical details and proofs of all theoretical results.

\medskip
\noindent{\bf Notation.} For any integer $n$, let $[n] = \{1, \dots, n\}$. For any $a,b\in\RR$, let $a\vee b = \max(a,b)$ and $a\wedge b = \min(a,b)$. For matrix $\bM\in\RR^{n\times q}$, we denote by $\|\bM\|_{\Fn}$, $\|\bM\|$, and $\|\bM\|_{2\to\infty}$ the Frobenius norm, the operator norm, and the two-to-infinity norm, respectively. 
The $k$th largest singular value of a matrix $\bM$ is denoted by $\sigma_k(\bM)$. For a symmetric matrix $\bM\in \RR^{n \times n}$, denote by $\lambda_{\min}(\bM)$ its minimal eigenvalue. We write $\bI_n$ for the $n\times n$ identity matrix and $\zero$ for zero vector or zero matrix when the dimension is clear from the context. For any positive integer $r$, define $\cO^r = \{\bR\in\RR^{r\times r}:\bR\bR^\T = \bI_r\}$ and $GL(r) = \{\bG\in\RR^{r\times r}:\bG\text{ is invertible}\}$. 



\section{Optimization without Noise}\label{sec_general}
\subsection{A Warm-up under Special Symmetric Model}\label{subsec_noiseless_Z}
We start with a special symmetric {model} in which $n=q$ and $\bX$ is symmetric, namely, $\bU = \bV$. This setting exhibits nonconvexity similar to that of the asymmetric model~\citep{ma2018implicit}. Thus, we use it as the starting point for introducing our main ideas. At the same time, the technical development still differs in meaningful ways from that of the asymmetric model, so the symmetric problem is also of independent interest.
To avoid ambiguity, we use $\bZ\in\RR^{n\times r}$ to denote the factorization replacing $\bU=\bV$. In this section, we consider the noiseless case where the true parameter is a stationary point of $\cL(\cdot)$, i.e., $\nabla_{\bX}\cL(\bZ^*(\bZ^*)^\T) = \zero$, with the noisy setting presented in Section~\ref{sec_noise}. The gradient update is then given as
\begin{equation}
    \bZ^{t+1} \; = \; \bZ^t - n^{-1}\eta^t\nabla_{\bZ}\cL(\bZ^t\bZ^t{}^\T) \; = \; \bZ^t - 2n^{-1}\eta^t\cG(\bZ^t)\bZ^t,\label{eq_one_step_update}
\end{equation}where $\cG(\bZ) : = \nabla_{\bX}\cL(\bZ\bZ^\T)$ and the second equality follows by chain rule.
Under the symmetric model, the factorization is identifiable only up to orthogonal transformation: for any $\bR\in\cO^r$, $\bZ$ and $\bZ\bR$ yield the same loss $\cL(\bZ\bZ^\T)$. To measure error relative to the true factorization $\bZ^*$, we adopt the standard alignment argument where for any $\bZ$, we let
\begin{equation}
    \bR^* = \;{\argmin}_{\bR\in\cO^{r}}\|\bZ\bR - \bZ^*\|_{\Fn},\label{eq_optimal_align_Z}
\end{equation}
and define the $\ell_2$ and $\ell_\infty$ distances between $\bZ$ and $\bZ^*$ by
\begin{equation*}
\disttwo(\bZ,\bZ^*) := \|\bZ\bR^*- \bZ^*\|_{\Fn},\qquad \distinf(\bZ,\bZ^*) := \|\bZ \bR^* - \bZ^*\|_{\twinf}.
\end{equation*}
Equation~\eqref{eq_optimal_align_Z} is often referred to as Wahba’s problem~\citep{wahba1965least}, or the orthogonal Procrustes problem~\citep{schonemann1966generalized}. 

Establishing contraction for \eqref{eq_one_step_update} is very challenging. Standard convex optimization theory does not apply, because one always has $\lambda_{\min}\big[\nabla_{\bz}^2\{\cL(\bZ\bZ^\T)\}|_{\bZ = \bZ^*}\big] = 0$ for $\bz =: \mathrm{vec}(\bZ^\T)$. This reflects the intrinsic nonconvexity caused by the nonlinear map $\bZ\mapsto\bZ\bZ^\T$. To address this, prior
works typically introduce penalties~\citep{keshavan2010matrix,sun2016guaranteed} or use additional projection steps~\citep{chen2015fast,ma2020universal}. A different approach is taken by \citet{ma2018implicit} for matrix completion, who constructs a region with restricted directions in which the bilinear Hessian admits a suitable lower bound, and then shows that the iterates remain in this region, with the estimation error lying along these directions throughout the updates. This approach is problem-specific and difficult to extend beyond linear models.

We now introduce our method.
Our idea is to introduce a benign regularizer term whose gradient vanishes across the iterates $\{\bZ^t\}_{t\ge 0}$, yet the augmented objective becomes (strongly) convex. At first sight, such a benign regularizer seems impossible to construct, since this favorable property must hold uniformly along the entire trajectory of the iterates. The key observation comes from the error metric. To measure the error of each iterate $\bZ^t$, one needs the optimal alignment $\bR_t^* := \argmin_{\bR\in\cO^r}\|\bZ^t\bR - \bZ^*\|_{\Fn}^2$. By Theorem 2 in \citet{ten1977orthogonal}, $\bZ^t\bR_t^*$ always satisfies that  $(\bZ^*)^\T\bZ^t\bR^*_t = (\bZ^t\bR^*_t)^\T\bZ^*$ is positive semidefinite. Motivated by this property, we propose the following benign regularizer for the symmetric model:
\begin{equation*}
p^*_{\lambda}(\bZ) \; = \; \tfrac{\lambda n^2}4 \|n^{-1}(\bZ^*)^\T\bZ - n^{-1}\bZ^\T\bZ^*\|_{\Fn}^2, \quad \lambda>0
\end{equation*} 
where $\lambda n^2/4$ is a scaling factor with $\lambda$ specified later. The regularizer has two key properties
\begin{enumerate}[label=(\roman*),leftmargin=0.8cm]
    \item $p^*_{\lambda}(\bZ)$ is benign for the gradient update: $\nabla_{\bZ}\, p_{\lambda}^*(\bZ^t\bR^*_t) = \zero,$ $\forall t\in\NN, \lambda>0$. Therefore, after multiplying \eqref{eq_one_step_update} on the right by $\bR_t^*$, the gradient update can be equivalently written as
\begin{equation}
    \bZ^{t+1}\bR_t^* = \bZ^t\bR_t^* - \Big\{2\frac{\eta^t}{n}\cG(\bZ^t)\bZ^t\Big\}\bR_t^* = \bZ^t\bR_t^* - \frac{\eta^t}{n}\Big\{2\cG(\bZ^t\bR_t^*)\bZ^t\bR_t^* + \nabla_{\bZ}p^*_{\lambda}(\bZ^t\bR_t^*)\Big\},\label{alg_aligned_gcd}
\end{equation}where the last equality uses $\cG(\bZ^t\bR_t^*) = \cG(\bZ^t)$, which holds by definition.  
\item $h_{\lambda}^*(\bZ) := \cL(\bZ\bZ^\T) + p_{\lambda}^*(\bZ)$ is strongly convex within a local region around $\bZ^*$ for suitable $\lambda$ (See Lemma~\ref{lemma_general_convex_1} below). Consequently, \eqref{alg_aligned_gcd} yields 
\begin{equation}\label{eq_contraction_gcd_expand}
    \disttwo(\bZ^{t+1}, \bZ^*) \; \le \; \|\bZ^{t+1}\bR_t^* - \bZ^*\|_{\Fn} \; = \; \|\bZ^t\bR_t^* - \tfrac{\eta^t}{n}\nabla_{\bZ}h_{\lambda}^*(\bZ^t\bR_t^*) - \bZ^*\|_{\Fn}.
\end{equation}
This implies that the original update $\bZ^t\mapsto \bZ^{t+1}$ admits the equivalent form $\bZ^t\bR_t^*\mapsto \bZ^{t+1}\bR_t^*$, which contracts under gradient descent on the strongly convex objective $h_{\lambda}^*(\cdot)$ with the next-step error dominating $\disttwo(\bZ^{t+1}, \bZ^*)$. In other words, the original iterates evolve like the gradient descent on a locally strongly convex objective in the sense of error contraction.
\end{enumerate}

While it may seem surprising that adding $p_{\lambda}^*(\cdot)$ yields a strongly convex landscape, we emphasize that this is not an artificial structure. Rather, the formulation reveals an implicit structural feature already present in the original update. Specifically, $p_{\lambda}^*(\cdot)$ is constructed from the first-order condition for \eqref{eq_optimal_align_Z}, which is the standard error metric under this setup. Our innovation is to use optimal alignment not merely for post-hoc error measurement, but to construct a regularizer that benignly augments the loss. Moreover, it should not be concerning that the aligned update \eqref{alg_aligned_gcd} involves the true parameter $\bZ^*$, since this update is introduced only as an analytical device rather than an implemented algorithm. In fact, $\bZ^*$ already enters the analysis through $\bR^*_t$ used to define the $\ell_2$ and $\ell_{\infty}$ distances.


Now we show that $h_{\lambda}^*(\bZ) = \cL(\bZ\bZ^\T) + p_{\lambda}^*(\bZ)$ exhibits local strong convexity for any $\cL(\cdot)$ with the local restricted isometry property (RIP).  To this end, define the local regions of interest
\begin{equation*}
    \cDtwo \!=\! \Big\{\bZ:\frac{\disttwo(\bZ,\bZ^*)}{\|\bZ^*\|_{\Fn}}\le \epsilon\Big\};\quad\cDinf \!=\! \Big\{\bZ:\frac{\disttwo(\bZ,\bZ^*)}{\|\bZ^*\|_{\Fn}}\le \epsilon\text{, } \frac{\distinf(\bZ,\bZ^*)}{\|\bZ^*\|_{\twinf}}\le \epsilon\Big\}.
\end{equation*}
We allow the local radius $\epsilon$ to vanish with the problem size. 
The following regularity conditions are required only within this possibly small local region, allowing our theory to accommodate complex settings where the global behavior may be unrestricted and complicated.
\begin{assumption}\label{assump_rsc}
    There exist some $\epsilon$ and $\beta \ge \alpha>0$ such that, for either $\cD=\cDtwo$ or $\cD=\cDinf$, the following holds: for every $\bZ\in\cD$ and every $\bW\in\RR^{n\times r}$,
    \begin{equation}
        \alpha\|\cP_{\bZ}(\bW)\|_{\Fn}^2 \, \le \, \nabla^2_{\bX} \cL(\bZ\bZ^\T)[\cP_{\bZ}(\bW), \cP_{\bZ}(\bW)] \, \le \, \beta\|\cP_{\bZ}(\bW)\|_{\Fn}^2,
    \end{equation}  where $\cP_{\bZ}(\bW) := \bZ\bW^\T + \bW\bZ^\T$ and $\nabla_{\bX}^2 \cL(\bZ\bZ^\T)[\cdot,\cdot]$ denotes the Hessian bilinear form.
\end{assumption}
\begin{assumption}\label{assump_gradient_lips_1}
There exist some $\epsilon$ and a constant $L_{2}$ such that for all $\bZ_1,\bZ_2\in\cD$, $\|\cG(\bZ_1)-\cG(\bZ_2)\| \le L_{2}\|\bZ_1\bZ_1^\T-\bZ_2\bZ_2^\T\|_{\Fn}$, where $\cD=\cDtwo$ or $\cD=\cDinf$.
\end{assumption}
\begin{remark}
    The restricted isometry property in Assumption~\ref{assump_rsc} is a widely accepted condition in low-rank estimation problems~\citep{bhojanapalli2016global,candes2012exact,ge2017no,li2020global,recht2010guaranteed}. Our condition is general enough to accommodate a broad range of settings, as (i) it is required only locally within the neighborhood $\cDtwo$ or $\cDinf$ instead of a global one; (ii) $\beta$ and $\alpha$ are allowed to be of different orders. Such a local RIP can thus be easily satisfied in problems such as matrix sensing~\citep{ma2021beyond,park2018finding,tu2016low} and matrix completion~\citep{chen2015fast,ge2016matrix,zheng2016convergence}. Moreover, it is close to {\it minimal} for our analysis: without it, the original constrained problem $\argmin_{\mathrm{rank}(\bX)\le r}\cL(\bX)$ may fail to be convex even in a neighborhood of the true parameter $\bX^*$, representing an obstruction inherent to the loss function rather than the factorization.

\medskip
\noindent Furthermore, Assumption~\ref{assump_gradient_lips_1} imposes a local Lipschitz condition on $\cG(\bZ)$. This is a mild regularity requirement, ensuring that $\|\cG(\bZ)\|$ remains small and well controlled within a neighborhood of $\bZ^*$.
\end{remark}

\noindent The following result demonstrates the local strong convexity of $ h_{\lambda}^*(\cdot)$ under Assumptions~\ref{assump_rsc} and~\ref{assump_gradient_lips_1}.
\begin{lemma}\label{lemma_general_convex_1}
    Under Assumptions~\ref{assump_rsc} and~\ref{assump_gradient_lips_1}, and with $\lambda = \alpha$, for any $\bZ\in\cDtwo$ and $\bR^*$ from \eqref{eq_optimal_align_Z},
    \begin{equation*}
        \lambda_{\min}\big\{n^{-1}\nabla_{\bz}^2 h_{\alpha}^*(\bZ\bR^*)\big\}\;\ge \;\alpha\sigma_r(n^{-1/2}\bZ^*)^2 \,-\, 4n^{-1}(\epsilon + 2L_2)\epsilon\|\bZ^*\|_{\Fn}^2.
    \end{equation*}
\end{lemma}
\noindent For simplicity, we write $h_\alpha^*(\cdot)$ and $p_\alpha^*(\cdot)$ throughout the rest of the paper unless otherwise specified.

Local region $\cDinf$ is smaller compared with $\cDtwo$. When Assumptions~\ref{assump_rsc} and~\ref{assump_gradient_lips_1} hold only for $\cD = \cDinf$, we need to show that the iterates remain inside $\cDinf$, which requires additional regularity conditions. To streamline the presentation and due to space limit, we defer these conditions, Assumptions~\ref{assump_row_cross_curvature_sym} and~\ref{assump_gradient_lips}, and their discussion to Appendix~\ref{sec_ell_infty_requirement}. We now state our result.

\begin{theorem}[{\bf Symmetric and Noiseless}]\label{thm_general_theory}
     Let $\bX^* = \bZ^*(\bZ^*)^\T$, $\sigma_{\min} = \sigma_{r}(\bX^*)/n$, and $\kappa = \sigma_{1}(\bX^*)/\sigma_{r}(\bX^*)$. Let the iterates $\{\bZ^t\}_{t\ge 0}$ be generated by \eqref{eq_one_step_update} with step size $\eta^t = \eta = \{10(\alpha + \beta)\kappa\sigma_{\min}\}^{-1}$ and $\rho := 1 - \eta\alpha\sigma_{\min}/4$.
    
     \noindent\underline{\textbf{\textit{$\ell_2$-error contraction.}}} Suppose Assumptions~\ref{assump_rsc} and~\ref{assump_gradient_lips_1} hold with $\cD = \cDtwo$. Assume that the initialization $\bZ^0$ satisfies \begin{equation}
         {\disttwo(\bZ^0,\bZ^*)}{} \le \phi_n\|\bZ^*\|_{\Fn}\quad\text{ with }\quad\phi_n\le \tfrac{\epsilon}{2}\wedge c_0\tfrac{\alpha}{\kappa\sqrt{r}}\label{eq_sym_init_Z}
    \end{equation} for some sequence $\phi_n$ and some sufficiently small constant $c_0 > 0$. Then for all $t \in \mathbb{N}^+$,
\begin{equation}
    \disttwo(\bZ^t,\bZ^*) \; \le \; \rho^t\phi_n\|\bZ^*\|_{\Fn} .\label{item_alg_gcd_l2}
\end{equation}
\noindent\underline{\textbf{\textit{ $\ell_{\infty}$-error contraction.}}} Suppose Assumptions~\ref{assump_rsc},~\ref{assump_gradient_lips_1},~\ref{assump_row_cross_curvature_sym}, and~\ref{assump_gradient_lips} hold with $\cD  = \cDinf$. If $\bZ^0$ satisfies \eqref{eq_sym_init_Z} and
    \begin{equation}
        {\distinf(\bZ^0,\bZ^*)}\le \psi_n \|\bZ^*\|_{\twinf}\quad \text{ with }\quad\psi_n\le \tfrac{\epsilon}{2}\;\text{ and } \;\tfrac{\beta}{\alpha}\kappa^{3/2}\sqrt{r}\tfrac{\phi_n}{\psi_n}\le c_0,\label{eq_sym_init_Z_uniform}
    \end{equation}
    for some sequence $\psi_n$ and sufficiently small constant $c_0$, then for all $t\in\NN^{+}$,  \eqref{item_alg_gcd_l2} continues to hold, and in addition
\begin{equation}
    \distinf(\bZ^t,\bZ^*) \; \le \; \rho^t\psi_n\|\bZ^*\|_{\twinf}.\label{item_alg_gcd_linf}
\end{equation}
\end{theorem}
Theorem~\ref{thm_general_theory} establishes linear convergence of gradient descent under a general loss $\cL(\cdot)$, in both $\ell_2$ and $\ell_{\infty}$ errors. When $\alpha\asymp \beta$ and $\kappa\asymp 1$, the step size satisfies $\eta\asymp 1$, and the iterates converge linearly at a constant rate $\rho$. Notably, the contraction gap $1-\rho$ depends on $\alpha/\beta$ and $\kappa^{-1}$, reflecting the effects of the local curvature $\cL(\cdot)$ and the conditioning of $\bX^*$on the problem geometry, respectively.

We remark that $\phi_n$ and $\psi_n$ need not equal the exact initialization error, and may be chosen more conservatively, provided they satisfy the required scaling conditions. This perspective is consistent with the nonconvex low-rank literature, where convergence is typically established under an initialization condition~\citep{candes2015phase,ma2018implicit,ma2021beyond,netrapalli2013phase,zheng2016convergence}. Initialization methods are then obtained in a model-specific manner, for example via spectral method~\citep{candes2015phase,netrapalli2013phase,ma2018implicit}, nuclear norm minimization~\citep{wang2017unified}, and universal singular value thresholding~\citep{chatterjee2015matrix,ma2020universal}. Since initialization is problem-specific and has been studied extensively in the literature, our focus here is on the gradient descent scheme, provided that the initialization lies in the local basin.


\subsection{Asymmetric Model}\label{subsec_noiseless_uv}
We now turn to the asymmetric model \eqref{uv_case} with gradient update~\eqref{eq_one_step_update_uv}, where $n$ is not necessarily equal to $q$ and $\bU,\bV$ are allowed to differ. Motivated by the previous analysis, our goal is to construct a benign regularizer tied to the error metric, and to show that the gradient update \eqref{eq_one_step_update_uv} admits an equivalent reformulation as a strongly convex problem. Let the true parameters be $(\bU^*,\bV^*)$ satisfying $\nabla_{\bX}\cL(\bU^*(\bV^*)^\T) = \zero$. We begin with the identification issue. Note that for any invertible transformation $\bG\in GL(r)$, $(\bU,\bV)$ and $(\bU\bG,\bV\bG^{-\T})$ yield the same loss $\cL(\bU\bV^\T)$. We then let
\begin{equation*}\bG^* \;=\; {\argmin}_{\bG\in {GL}(r)}\, n^{-1}\|\bU\bG-\bU^*\|_{\Fn}^2 \,+\, q^{-1}\|\bV\bG^{-\T} - \bV^*\|_{\Fn}^2,\end{equation*} 
whenever the minimizer exists, and define the distances between $(\bU,\bV)$ and $(\bU^*,\bV^*)$ by
\begin{align*}
\disttwo\big\{(\bU,\bV),(\bU^*,\bV^*)\big\} := \,&\, \big(n^{-1}\|\bU\bG^*-\bU^*\|_{\Fn}^2 + q^{-1}\|\bV(\bG^*)^{-\T} - \bV^*\|_{\Fn}^2\big)^{1/2},\\
\distinf\big\{(\bU,\bV),(\bU^*,\bV^*)\big\} :=\, & \, \max\big\{ \|\bU\bG^*-\bU^*\|_{\twinf} , \|\bV(\bG^*)^{-\T} - \bV^*\|_{\twinf}\big\}.
\end{align*}
Here, we use $n^{-1}$ and $q^{-1}$ to scale the Frobenius norm to make the subsequent derivation simpler.
By Lemma~\ref{lemma_id_rec} in Appendix~\ref{Sec_prelmi}, whenever $\bG^*$ exists, it holds that $p\prect(\bU\bG^*,\bV(\bG^*)^{-\T})=0$ where $p\prect(\cdot,\cdot)$ is the benign regularizer under the asymmetric model, given as
\begin{equation*}p\prect(\bU,\bV) :=  \tfrac{\alpha nq}{4}\|n^{-1}(\bU - \bU^*)^\T \bU - q^{-1}\bV^\T (\bV - \bV^*)\|_{\Fn}^2,\;\alpha>0.\end{equation*}
Analogous to the symmetric model, $\alpha$ here is chosen to match the lower curvature bound in the asymmetric case, as specified later in Assumption~\ref{assump_rect_rip_weighted}.
Following the same idea, we define $h\prect(\bU,\bV) := \cL(\bU\bV^\T)+p\prect(\bU,\bV)$ and reformulate \eqref{eq_one_step_update_uv} as a gradient update on $h\prect(\bU,\bV) $, which is shown to be strongly convex in a local region after proper scaling of $n,q$ (see  Lemma~\ref{lemma_general_convex_asym} in Appendix~\ref{Sec_prelmi}).

For each $t\ge 0$, set $\bG_t^* = \argmin_{\bG\in {GL}(r)} n^{-1}\|\bU^t\bG - \bU^*\|_{\Fn}^2  + q^{-1}\|\bV^t\bG^{-\T} -\bV^*\|_{\Fn}^2$. For now, we assume $\bG_t^*$ exists; its existence will be established later in the proof. Let 
\begin{equation*}
\bLambda_{t}^* = (\bG_t^*)^\T\bG_t^*; \qquad \tilde\bU^t = \bU^t\bG_t^*;\qquad \tilde\bV^t = \bV^t(\bG_t^*)^{-\T}.\end{equation*} 
Denote $\cG(\bU,\bV) = \nabla_{\bX}\cL(\bU\bV^\T)\in\RR^{n\times q}$. Then obviously, $\cG(\bU,\bV) = \cG(\bU\bG_t^*,\bV(\bG_t^*)^{-\T})$. 
By the chain rule, one then has 
$$\nabla_{\bU}\cL(\bU^t, \bV^t) \, = \, \nabla_{\bU}\cL(\tilde\bU^t,\tilde\bV^t) (\bG_t^*)^\T\quad\text{ and }\quad
\nabla_{\bV}\cL(\bU^t, \bV^t) \, = \, \nabla_{\bV}\cL(\tilde\bU^t,\tilde\bV^t)(\bG_t^*)^{-\T}.$$
Moreover, Lemma~\ref{lemma_id_rec} gives $\nabla_{\bU}p\prect(\tilde\bU^t,\!\tilde\bV^t) \!=\! \zero$ and $\nabla_{\bV}p\prect(\tilde\bU^t,\!\tilde\bV^t)\! =\! \zero$. Thus by multiplying by $\bG_t^*$ and $(\bG_t^*)^{-\T}$ and adding the corresponding gradients of $p\prect(\cdot)$, the update~\eqref{eq_one_step_update_uv} can be written as
\begin{equation}\label{eq_aligned_gradient_update_ub}\begin{aligned}
    \bU^{t+1}\bG_t^* =\,&\,\bU^t\bG_t^*-\frac{\eta^t}{q}\nabla_{\bU}\cL(\bU^t,\bV^t)\bG_t^*\\=&\,
    \tilde\bU^t-\frac{\eta^t}{q}\nabla_{\bU}h\prect(\tilde\bU^t,\tilde\bV^t) - \frac{\eta^t}{q}\nabla_{\bU}h\prect(\tilde\bU^t,\tilde\bV^t)(\bLambda_t^* - \bI_r),\\
    \bV^{t+1}(\bG_t^*)^{-\T} = \,&\, \bV^t(\bG_t^*)^{-\T}-\frac{\eta^t}{n}\nabla_{\bV}h\prect(\bU^t\bG_t^*,\bV^t(\bG_t^*)^{-\T})(\bLambda_t^*)^{-1}\\ = &\,
    \tilde\bV^t - \frac{\eta^t}{n}\nabla_{\bV}h\prect(\tilde\bU^t,\tilde\bV^t) - \frac{\eta^t}{n}\nabla_{\bV}h\prect(\tilde\bU^t,\tilde\bV^t)\big\{(\bLambda_t^*)^{-1} - \bI_r\big\}.
\end{aligned}\end{equation}
That is, after alignment, the update consists of a gradient step on the augmented objective $h\prect(\cdot)$, together with an additional perturbation arising from the error $\bLambda^*_t - \bI_r$ and $(\bLambda^*_t)^{-\T} - \bI_r$. This perturbation is specific to the asymmetric model, because, unlike in the symmetric case, the alignment matrix $\bG_t^*$ need not be orthogonal. To address this, we impose the balancing condition 
\begin{equation*}
    n^{-1}(\bU^*)^\T\bU^* = \, q^{-1}(\bV^*)^\T\bV^*.
\end{equation*}
This entails no loss of generality, since if $\bX^*$ has singular value decomposition $\bU^{\star}\bSigma^{\star}(\bV^{\star})^\T$, we may take $\bU^* = \bU^{\star}(\bSigma^{\star})^{1/2}$ and $\bV^* = \bV^{\star}(\bSigma^{\star})^{1/2}$. 
Theoretically, we show that, provided that the initialization $(\bU^0,\bV^0)$ is suitably balanced, i.e., $\bG_0^*$ is close to some orthogonal matrix, each $\bG_t^*$ is guaranteed to exist and remains close to the same orthogonal matrix throughout the iterations. As a result, $\bLambda_t^* \approx \bI_r$, so the asymmetric update \eqref{eq_aligned_gradient_update_ub} also closely matches a gradient step on the strongly convex objective $h\prect(\cdot)$. Similar to the symmetric model, in \eqref{eq_aligned_gradient_update_ub}, the original update rule is preserved, and the regularizer $p\prect(\cdot)$ acts simply as the device to make explicit the structural feature.



Now we introduce two local regions as the asymmetric versions of $\cDtwo$ and $\cDinf$. 
\begin{align*}
\cDuvtwo := &\,\left\{ (\bU,\bV): \disttwo\big\{(\bU,\bV), (\bU^*,\bV^*)\big\}\le \epsilon\,\tau_* \right\};\\ 
\cDuvinf := &\,\left\{ (\bU,\bV): \disttwo\big\{(\bU,\bV), (\bU^*,\bV^*)\big\}\le \epsilon\, \tau_*,\; \distinf\big\{(\bU,\bV),(\bU^*,\bV^*)\big\}\le \epsilon\, \omega_* \right\},
\end{align*}where $\tau_*{}^2 := ({n^{-1}\|\bU^*\|_{\Fn}^2+q^{-1}\|\bV^*\|_{\Fn}^2})/{2}$ and $\omega_*:={\|\bU^*\|_{2\to\infty}}\vee{\|\bV^*\|_{2\to\infty}}$. The following are asymmetric analogues of Assumptions~\ref{assump_rsc} and~\ref{assump_gradient_lips_1}, respectively.
\begin{assumption}
\label{assump_rect_rip_weighted}
There exist $\epsilon$ and $\beta\ge \alpha>0$ such that, for either $\cD = \cDuvtwo$ or $\cD = \cDuvinf$, the following holds: for every $(\bU,\bV)\in\cD$ and every $(\bL,\bR)\in\RR^{n\times r}\times\RR^{q\times r}$,
\begin{equation*}
\alpha\|\bU\bR^\T+\bL\bV^\T\|_{\Fn}^2\le\nabla_{\bX}^2\cL(\bU\bV^\T)\big[\cP_{(\bU,\bV)}(\bR,\bL),\cP_{(\bU,\bV)}(\bR,\bL)\big]\le \beta{\|\bU\bR^\T+\bL\bV^\T\|_{\Fn}^2},
\end{equation*}where $\cP_{(\bU,\bV)}(\bR,\bL) = \bU\bR^\T + \bL\bV^\T$.
\end{assumption}

\begin{assumption}\label{assump_rect_lip_1}
    There exist a constant $L_2$ such that for all $(\bU_1,\bV_1),(\bU_2,\bV_2)\in\cD$, $\big\|\cG(\bU_1,\bV_1)-\cG(\bU_2,\bV_2)\big\| \le L_2\big\|\bU_1\bV_1{}^\T-\bU_2\bV_2{}^\T\big\|_{\Fn}$ where $\cD=\cDuvtwo$ or $\cD=\cDuvinf$.
\end{assumption}

To establish $\ell_{\infty}$ error contraction, we require two additional technical conditions, Assumptions~\ref{assump_row_cross_curvature_uv} and~\ref{assump_rect_lip}, which are presented in Appendix~\ref{sec_ell_infty_requirement} due to space limit. We now state the result.
\begin{theorem}[{\bf Asymmetric and Noiseless}]\label{thm_general_theory_uv}
    Let $\bX^* = \bU^*(\bV^*)^\T$ and $\kappa = \sigma_{1}(\bX^*)/\sigma_{r}(\bX^*)$. Assume that $(\bU^*,\bV^*)$ are balanced:
    \begin{equation*}
    n^{-1}(\bU^*)^\T\bU^* = \, q^{-1}(\bV^*)^\T\bV^*,\;\text{ with }\; \sigma_{\min}:={\sigma_r(\bX^*)}/{\sqrt{nq}} = {\sigma_r(\bU^*)^2}/{n} = {\sigma_r(\bV^*)^2}/{q}.
\end{equation*}
    Let $\{\bU^t,\bV^t\}_{t\ge 0}$ be generated by \eqref{eq_one_step_update_uv} with $\eta^t = \eta = \{10(\alpha + \beta)\kappa\sigma_{\min}\}^{-1}$ and $\rho := 1 - \eta\alpha\sigma_{\min}/4$.\\
    \noindent\underline{\textbf{\textit{ $\ell_2$-error contraction.}}} Suppose Assumptions~\ref{assump_rect_rip_weighted}--\ref{assump_rect_lip_1} hold with $\cD = \cDuvtwo$. Assume that there is some orthogonal matrix $\bR^0$ such that the initialization $(\bU^0,\bV^0)$ satisfies
\begin{equation}
     {\big(n^{-1}\|\bU^0\bR^0 - \bU^*\|_{\Fn}^2+ q^{-1}\| \bV^0\bR^0 - \bV^*\|_{\Fn}^2\big)^{1/2}}\le \;\tfrac{1}{2}\phi_{nq}{\tau_*} \; \text{ with } \; \phi_{nq}  \le \tfrac{\epsilon}{2}\wedge c_0\tfrac{\alpha(\alpha+\kappa)}{\beta\kappa^2\sqrt{r\kappa}},\label{eq_init_consis_main_R0}
\end{equation} for some sequence $\phi_{nq}$ and some sufficiently small constant $c_0>0$.
Then for all $t \in \mathbb{N}^+$, we have
\begin{equation}
    \disttwo\{(\bU^t,\bV^t),(\bU^*,\bV^*)\} \; \le \;   \rho^{t}\phi_{nq}\tau_*.\label{item_alg_gcd_l2_uv}
\end{equation}
\noindent\underline{\textbf{\textit{ $\ell_{\infty}$-error contraction.}}}  Suppose Assumptions~\ref{assump_rect_rip_weighted},~\ref{assump_rect_lip_1},~\ref{assump_row_cross_curvature_uv}, and~\ref{assump_rect_lip} hold with $\cD = \cDuvinf$. If $(\bU^0,\bV^0)$ satisfies \eqref{eq_init_consis_main_R0} and
\begin{equation}\label{eq_sym_init_Z_uniform_uv}
    {\|\bU^0\bR^0 - \bU^*\|_{\twinf}\vee \|\bV^0\bR^0 - \bV^*\|_{\twinf}}\le\; \tfrac{1}{2}\psi_{nq}{\omega_*}\;\text{ with }\;\psi_{nq} \le \tfrac{\epsilon}{2}\;,\; \tfrac{\beta}{\alpha}\kappa^{3/2}\sqrt{r}\tfrac{\phi_{nq}}{\psi_{nq}}\le c_0,
\end{equation}
for some sequence $\psi_{nq}$ and some sufficiently small constant $c_0$, then for all $t\in\NN^{+}$, \eqref{item_alg_gcd_l2_uv} continues to hold, and in addition,
\begin{equation}
    \label{item_alg_gcd_linf_uv} 
            \distinf\{(\bU^t,\bV^t),(\bU^*,\bV^*)\} \; \le \; \rho^{t}\psi_{nq}\omega_*.
\end{equation}
\end{theorem}
Theorem~\ref{thm_general_theory_uv} establishes the linear convergence of gradient descent under the asymmetric model. Compared with Theorem~\ref{thm_general_theory} for the symmetric model, the initialization requirement is slightly stronger, as the analysis must additionally ensure that $\bLambda_t^* - \bI_r$ is small along the entire trajectory. Our requirement on $\kappa$ aligns with the results in the matrix sensing literature~\citep{ma2021beyond}. At the same time, our setting is more general, as it allows for a broad class of loss functions $\cL(\cdot)$, for which the local curvature parameters $\alpha$ and $\beta$ may differ in order.



\section{Optimization under Noise}\label{sec_noise}

We now consider the estimation problem under a stochastic objective $\cL(\cdot)$, starting again from the symmetric model. Let $(\Omega,\mathcal F,\PP)$ be a probability space carrying the randomness in the data or noise. For each $\omega\in\Omega$, the loss is redefined as $\cL(\cdot\,;\,\omega):\RR^{n\times n}\to\RR$.  Let $\cG(\bZ;\omega):=\nabla_{\bX}\cL(\bZ\bZ^\T;\omega)\in\RR^{n\times n}$ with $\bar\cG(\bZ):=\EE\,\cG(\bZ;\omega)$ and $ \tilde\cG(\bZ;\omega):=\cG(\bZ;\omega)-\bar\cG(\bZ)$,
where the expectation is taken with respect to $\PP$. The true parameter satisfies $\bar\cG(\bZ^*)=\zero$. When the dependence on $\omega$ is not essential, we write $\cL(\bX)$, $\cG(\bZ)$, and $\tilde\cG(\bZ)$ in place of $\cL(\bX;\omega)$, $\cG(\bZ;\omega)$, and $\tilde\cG(\bZ;\omega)$.

We introduce the following quantities to quantify the noise. For problem size $n$, tolerance parameter $\delta\in(0,1)$, and local region $\cD$, let $\Delta_2(n,\delta)$, $\Delta_{\infty}(n,\delta)$, and $\bar\Delta_{\infty}(n,\delta)$ be deterministic quantities such that, with probability at least $1\!-\!\delta$, the following hold for either $\cD \!=\! \cDtwo$ or $\cDinf$:
\begin{subequations}\label{assump_gradient_noise}\begin{align}
        {\sup}_{\bZ\in\cD}\; n^{-1}{\|\tilde\cG(\bZ)\|} &\quad\le \quad \Delta_2(n,\delta),\label{assump_gradient_noise_1}\\ \tfrac{\|\tilde\cG(\bZ^*)\bZ^*\|_{\twinf}}{\sqrt{n}\|\bZ^*\|} \; \le \; \Delta_{\infty}(n,\delta),&\,\quad {\sup}_{\bZ\in\cD}\;\tfrac{\|\tilde\cG(\bZ)\|_{\infty\to 1}}{n} \; \le \; { \bar\Delta_{\infty}(n,\delta)}.\label{assump_gradient_noise_2}
    \end{align}\end{subequations}
    Here, $\|\bM\|_{\infty\to 1} := \max_{1\le i\le n}\sum_{j=1}^q |M_{ij}|$ denotes maximum $\ell_1$ row-sum of $\bM = (m_{ij})_{n\times q}$.
The quantity $\Delta_2(n,\delta)$ controls the overall noise in the operator norm, as enforced by \eqref{assump_gradient_noise_1}. In many linear problems~\citep{park2018finding,wang2017unified,zheng2016convergence}, $\tilde\cG(\bZ)$ is independent of $\bZ$ and $\Delta_2(n,\delta)$ can be readily derived by standard random matrix bounds~\citep{bandeira2016sharp}.  For~\eqref{assump_gradient_noise_2}, $\Delta_{\infty}(n,\delta)$ controls the row-wise noise along the factor $\bZ^*$. This quantity is intrinsic to $\ell_{\infty}$ analysis by noting that $2\tilde\cG(\bZ^*)\bZ^*$ equals the gradient noise at truth: $\nabla_{\bZ}\cL (\bZ^*)- \EE\nabla_{\bZ}\cL(\bZ^*)$. 
For $\bar\Delta_{\infty}(n,\delta)$ in \eqref{assump_gradient_noise_2}, it captures worst-case $\ell_{\infty}$ control and is used to bound terms such as \(\|\tilde\cG(\bZ)(\bZ^t-\bZ^*)\|_{\twinf}\). In general, without such structural information on the loss, the directions of $\bZ^t-\bZ^*$ are difficult to characterize uniformly, which motivates imposing a worst-case bound in the general theory. For models with additional structure, this condition can often be relaxed with problem-specific arguments; see Example~\ref{example_2}. As a benchmark, consider the loss $\cL(\bX) = \|\bX - \bX^* + \bE\|_{\Fn}^2/(2n)$, where the entries of $\bE$ are independent mean-zero sub-Gaussian random variables with sub-Gaussian norm $\sigma$, and as proved in Appendix~\ref{supp_sec_prove_example0}, one may take
\begin{equation}
    \Delta_2(n,\delta)\asymp\sigma\tfrac{\sqrt{n + \log(1/\delta)}}{n};\quad
    \Delta_{\infty}(n,\delta)\asymp\sigma\sqrt{\tfrac{r+\log(n/\delta)}{n}};\quad
    \bar\Delta_{\infty}(n,\delta)\asymp\sigma+\sigma\sqrt{\tfrac{\log(n/\delta)}{n}}.\label{example0}
\end{equation}




\noindent Next, we present the convergence results for the {\it symmetric model} under statistical noise.    
\begin{theorem}[{\bf Symmetric and Noisy}]\label{thm_general_theory_noisy_Z}
Let $\bX^* = \bZ^*(\bZ^*)^\T$, $\sigma_{\min} = \sigma_{r}(\bX^*)/n$ and $\kappa = \sigma_{1}(\bX^*)/\sigma_{r}(\bX^*)$.  Let the iterates $\{\bZ^t\}_{t\ge 0}$ be generated by \eqref{eq_one_step_update} with step size $\eta^t = \eta = \{10(\alpha + \beta)\kappa\sigma_{\min}\}^{-1}$ and $\rho = 1-\eta\alpha\sigma_{\min}/4$. Assume the local radius is taken to satisfy $\epsilon\le c_0{\alpha}/{(\kappa\sqrt r)}$ for some sufficiently small constant $c_0$.

    \noindent\underline{\textbf{\textit{ $\ell_2$-error contraction.}}} Suppose Assumptions~\ref{assump_rsc} and~\ref{assump_gradient_lips_1} hold with $\cD=\cDtwo$ and $\bar\cG(\bZ)$ in place of $\cG(\bZ)$. Suppose \eqref{assump_gradient_noise_1} holds with $\cD = \cDtwo$, and assume for a sufficiently small constant $c_0$,
    \begin{equation*}
        {\Delta_{2}(n,\delta)}/{(\alpha\sigma_{\min})}\le c_0\epsilon. 
    \end{equation*} Suppose the initialization $\bZ^0$ satisfies \eqref{eq_sym_init_Z}. 
    Then, for some constant $C$, with probability at least $1-\delta$, 
     \begin{equation}\label{item_alg_gcd_l2_noise}
        \disttwo(\bZ^t,\bZ^*)\; \le\; \big\{ \rho^{t}\phi_n + C\tfrac{\Delta_2(n,\delta)}{\alpha\sigma_{\min}}\big\}\ \|\bZ^*\|_{\Fn}.
    \end{equation}
    \noindent\underline{\textbf{\textit{ $\ell_{\infty}$-error contraction.}}} Suppose Assumptions~\ref{assump_rsc},~\ref{assump_gradient_lips_1},~\ref{assump_row_cross_curvature_sym}, and~\ref{assump_gradient_lips} hold with $\cD = \cDinf$ and $\bar\cG(\bZ)$ in place of $\cG(\bZ)$. 
    Suppose \eqref{assump_gradient_noise} holds with $\cD = \cDinf$, where 
\begin{equation*}
   \big(\tfrac{\beta}{\alpha} + \sqrt{r}\big)\kappa\sqrt{r}\tfrac{\Delta_{\infty}(n,\delta)}{\alpha\sigma_{\min}}
    \le c_0\epsilon, \quad
    \big(\tfrac{\beta}{\alpha}\kappa\sqrt{\kappa r} + r\kappa\big)\tfrac{\Delta_2(n,\delta)}{\Delta_{\infty}(n,\delta)}
    \le  c_0, \quad
    \tfrac{ \bar\Delta_{\infty}(n,\delta)}{\alpha\sigma_{\min}}\le \tfrac{1}{4},
\end{equation*}for a sufficiently small constant $c_0>0$. Suppose $\bZ^0$ satisfies~\eqref{eq_sym_init_Z} and~\eqref{eq_sym_init_Z_uniform}. Then, for some constant $C$, with probability at least $1-\delta$, \eqref{item_alg_gcd_l2_noise} continues to hold and
 \begin{equation}
        \label{item_alg_gcd_linf_noise}
        \distinf(\bZ^t,\bZ^*) \; \le\; \big\{\rho^t\psi_n + C\tfrac{\Delta_{\infty}(n,\delta)}{\alpha\sigma_{\min}} \big\}\ \|\bZ^*\|_{\twinf}.
        \end{equation}
        \end{theorem}

Both the $\ell_2$ and $\ell_{\infty}$ errors comprise a linearly decaying term and a non-decaying term. The former is exactly the algorithmic error already seen in the noiseless setting (Theorem~\ref{thm_general_theory}), while the latter captures the statistical error induced by noise. Specifically, if $\hat\bZ$ is some local minimizer of $\cL(\cdot)$, Appendix~\ref{supp_prove_thm_general_theory_noisy_Z} shows that, for some constant $C$, with probability at least $1-\delta$,
\begin{equation*}
    {\disttwo(\hat\bZ,\bZ^*)} \le C\tfrac{\Delta_2(n,\delta)}{\alpha\sigma_{\min}}\|\bZ^*\|_{\Fn}\quad\text{and}\quad{\distinf(\hat\bZ, \bZ^*)}\le C\tfrac{\Delta_{\infty}(n,\delta)}{\alpha\sigma_{\min}}\|\bZ^*\|_{\twinf}.
\end{equation*}

A key step in the proof is to use $\hat\bZ\hat\bR$ as the contraction target for $\hat\bR = \argmin_{\bR\in\cO^r}\|\hat\bZ\bR - \bZ^*\|_{\Fn}$. We show that $\hat\bZ\hat\bR$ is uniquely defined whenever $\hat\bZ$ lies in the local region and satisfies the first-order condition $\nabla_{\bZ}\cL(\hat\bZ\hat\bZ^\T)=\zero$. This is a highly nontrivial result because $\hat\bZ$ is obtained from a constrained local minimization over $\cDtwo$ or $\cDinf$. It provides a valid local target around which the iterates contract in the noisy setting. The proof of Theorem~\ref{thm_general_theory_noisy_Z} then follows by adapting the argument of Theorem~\ref{thm_general_theory} to this new target. Thus, although Theorem~\ref{thm_general_theory} can be viewed as a special case of Theorem~\ref{thm_general_theory_noisy_Z}, we present it separately to isolate the optimization aspect of the result. It also serves as the conceptual and technical foundation for the noisy analysis.

We next turn to the {\it asymmetric model} under statistical noise. We adopt the same stochastic notation as in the symmetric case, with the loss now given by $\cL(\cdot\,;\,\omega):\RR^{n\times q}\to\RR$. The first-order derivative is $\cG(\bU,\bV;\omega):=\nabla_{\bX}\cL(\bU\bV^\T;\omega)\in\RR^{n\times q}$, with $\bar\cG(\bU,\bV)$ and $\tilde\cG(\bU,\bV;\omega)$ denoting the corresponding mean and perturbation terms. The true parameter now satisfies $\bar\cG(\bU^*,\bV^*)=\zero$. As before, we suppress the dependence on $\omega$ when it is not essential. For a given problem size $(n,q)$, tolerance parameter $\delta\in(0,1)$, and local region $\cD$, let $\Delta_2(n,q,\delta)$, $\Delta_{\infty}(n,q,\delta)$, and $\bar\Delta_{\infty}(n,q,\delta)$ be deterministic quantities such that, with probability at least $1-\delta$, the following hold:
\begin{subequations}\label{assump_gradient_noise_rect}\begin{align}
        &{\sup}_{(\bU,\bV)\in\cD}\; (nq)^{-1/2}{\|\tilde\cG(\bU,\bV)\|} \le \Delta_2(n,q,\delta),\label{assump_gradient_noise_rect_1}\\ 
        &\tfrac{\|\tilde\cG(\bU^*,\bV^*)\bV^*\|_{\twinf}}{\sqrt{q}\|\bV^*\|}\vee\tfrac{\|\tilde\cG(\bU^*,\bV^*)^\T\bU^* \|_{2\to\infty}}{\sqrt{n}\|\bU^*\|} \le \Delta_\infty(n,q,\delta),\label{assump_gradient_noise_rect_2}\\
        &{\sup}_{(\bU,\bV)\in\cD}\; q^{-1}\|\tilde\cG(\bU,\bV)\|_{\infty\to 1} \vee n^{-1}\|\tilde\cG(\bU,\bV)^\T\|_{\infty\to 1}\le\bar\Delta_\infty(n,q,\delta).\label{assump_gradient_noise_rect_3}
    \end{align}\end{subequations}
The convergence results for the asymmetric model with statistical noise are stated as follows.
\begin{theorem}[{\bf Asymmetric and Noisy}]\label{thm_general_theory_UV_noise}
Let $\bX^* = \bU^*(\bV^*)^\T$ and $\kappa = \sigma_{1}(\bX^*)/\sigma_{r}(\bX^*)$. Assume that $(\bU^*,\bV^*)$ are balanced:
    \begin{equation*}
    n^{-1}(\bU^*)^\T\bU^* = q^{-1}(\bV^*)^\T\bV^*,\;\text{ with }\; \sigma_{\min}:={\sigma_r(\bX^*)}/{\sqrt{nq}} = {\sigma_r(\bU^*)^2}/{n} = {\sigma_r(\bV^*)^2}/{q}.
\end{equation*}
Let $\{\bU^t,\bV^t\}_{t\ge 0}$ be generated by \eqref{eq_one_step_update_uv} with $\eta^t = \eta = \{10(\alpha + \beta)\kappa\sigma_{\min}\}^{-1}$ and $\rho = 1-\eta\alpha\sigma_{\min}/4$.  Assume the local radius is taken to satisfy $\epsilon\le c_0{\alpha}/{(\kappa\sqrt r)}$ for some sufficiently small constant $c_0$.

\noindent\underline{\textbf{\textit{$\ell_2$-error contraction.}}} Suppose Assumptions~\ref{assump_rect_rip_weighted},~\ref{assump_rect_lip_1} hold with $\cD=\cDuvtwo$ and $\bar\cG(\bU,\bV)$ in place of $\cG(\bU,\bV)$. Suppose \eqref{assump_gradient_noise_rect_1} holds with $\cD = \cDuvtwo$, where the following holds
    \begin{equation*}
        {\Delta_{2}(n,q,\delta)}/{(\alpha\sigma_{\min})}\le  c_0\epsilon, 
    \end{equation*}for a sufficiently small constant $c_0$. Suppose initialization $(\bU^0,\bV^0)$ satisfies \eqref{eq_init_consis_main_R0}. 
    Then, for some constant $C$, with probability at least $1-\delta$, 
    \begin{equation}\label{item_alg_gcd_l2_uv_noise}
    \disttwo\{(\bU^t,\bV^t),(\bU^*,\bV^*)\}\; \le\;  \big\{\rho^{t}\phi_{nq} + C\tfrac{\Delta_2(n,q,\delta)}{\alpha\sigma_{\min}}\big\}\, \tau_* .
    \end{equation}
    \noindent\underline{\textbf{\textit{$\ell_{\infty}$-error contraction.}}} Suppose Assumptions~\ref{assump_rect_rip_weighted},~\ref{assump_rect_lip_1},~\ref{assump_row_cross_curvature_uv}, and~\ref{assump_rect_lip} hold with $\cD = \cDuvinf$ and $\bar\cG(\bU,\bV)$ in place of $\cG(\bU,\bV)$. Suppose \eqref{assump_gradient_noise_rect} holds with $\cD = \cDuvinf$, where 
\begin{equation}
    \big(\tfrac{\beta}{\alpha} + \sqrt{r}\big)\kappa\sqrt{r} \tfrac{\Delta_{\infty}(n,q,\delta)}{\alpha\sigma_{\min}}
    \le c_0\epsilon,\quad
    \big(\tfrac{\beta}{\alpha}\kappa\sqrt{\kappa r} + {r\kappa}\big) \tfrac{\Delta_2(n,q,\delta)}{\Delta_{\infty}(n,q,\delta)}
    \le c_0,\quad
    \tfrac{ \bar\Delta_{\infty}(n,q,\delta)}{\alpha\sigma_{\min}}\le \tfrac{1}{4},
\end{equation}for a sufficiently small constant $c_0$. Suppose $(\bU^0,\bV^0)$ satisfies \eqref{eq_init_consis_main_R0} and~\eqref{eq_sym_init_Z_uniform_uv}. 
Then, for some constant $C$, with probability at least $1-\delta$, \eqref{item_alg_gcd_l2_uv_noise} continues to hold and we further have
\begin{equation}
        \label{item_alg_gcd_linf_uv_noise}
        \distinf\{(\bU^t,\bV^t),(\bU^*,\bV^*)\}\; \le \; \big\{\rho^t\psi_{nq} + C\tfrac{\Delta_{\infty}(n,q,\delta)}{\alpha\sigma_{\min}}\big\}\, \omega_*. 
    \end{equation}
\end{theorem}
To the best of our knowledge, Theorem~\ref{thm_general_theory_UV_noise} provides the first convergence guarantee for regularization-free gradient descent in noisy asymmetric low-rank estimation, a setting where even model-specific results were previously unavailable. It completes our framework, which covers symmetric and asymmetric models, both with and without noise, and establishes contractions in both $\ell_2$ and $\ell_\infty$ errors. It reveals a common mechanism across these settings: the nonconvex procedure admits a disguised convexity. The regularity conditions, initialization requirements, and noise bounds can then be verified in a model-specific manner to certify this mechanism in each problem.

We close this section with two concrete examples. Due to space limit, we summarize the informal results below and defer the formal ones to Appendix~\ref{sec_example}.
\begin{example}\label{example_1}
    Matrix sensing is a central problem in information processing and machine learning, where the goal is to recover a low-rank matrix $\bX^*$ from a small number of linear measurements~\citep{jain2010guaranteed,recht2010guaranteed,tu2016low}. Consider the model \begin{equation*}y_i = \langle\bA_i,\bX^*\rangle + \xi_i\, \text{ for }i\in[m],\end{equation*}
where $\{\bA_i\}_{i=1}^m$ are sensing matrices known a priori, $\langle\cdot,\cdot\rangle$ denotes the Frobenius inner product, and $\{\xi_i\}_{i=1}^m$ are i.i.d. mean-zero sub-Gaussian random variables with sub-Gaussian norm $\sigma_\xi$. 
Apply the gradient descent \eqref{eq_one_step_update_uv} to the loss $\cL(\bX) = \sum_{i=1}^m(\langle\bA_i,\bX\rangle - y_i)^2/2m$.
Under suitable regularity conditions, we can apply Theorem~\ref{thm_general_theory_UV_noise} to obtain the $\ell_2$ error bound as in \eqref{item_alg_gcd_l2_uv_noise} with probability at least $1-\delta$, with $\alpha$, $\beta\asymp 1$ and
\begin{equation*}
    \Delta_2(n,q,\delta) \;\asymp\; \sigma_\xi \left( \tfrac{\log\{(n+q)/\delta\}}{m(n\wedge q)} \right)^{1/2}.
\end{equation*}
A related regularization-free convergence guarantee for gradient descent was established by \citet{ma2021beyond}, but only in the noiseless setting, i.e., $\xi_i = 0$. Our initialization requirement matches theirs; see Theorem~\ref{prop_ms_verify} in Appendix~\ref{sec_example} and the subsequent discussion.
\end{example}
\begin{example}\label{example_2}
Consider a Bernoulli low-rank response model with independent observations
\begin{equation*}
Y_{ij}\sim \mathrm{Bernoulli}(P_{ij}^*), \quad P_{ij}^* = \{1+\exp(-\alpha_0-X_{ij}^*)\}^{-1} 
\quad \bX^*=(X_{ij}^*)_{n\times q}=\bU^*(\bV^*)^\T,
\end{equation*}for $i\in[n]$ and $j\in[q]$,
where $\alpha_0$ is a known intercept and $\bX^*$ is a rank-$r$ signal matrix.
The intercept $\alpha_0$ accommodates sparse binary data, since a negative $\alpha_0$ can make all success probabilities small even when $\bX^*$ is low-rank. Such models arise in generalized low-rank models~\citep{Udell2016,cui2026beyond}, network representation learning~\citep{li2023statistical,ma2020universal}, and binary matrix completion~\citep{cai2013maxnorm,davenport20141}.
Assume that, for some $M_1,M_2$, $-M_1\le X_{ij}^*\le M_2,$ for $i\in[n],\ j\in[q]$.
Apply the gradient descent scheme \eqref{eq_one_step_update_uv} to the scaled loss $\cL(\bX) = \sum_{i=1}^n\sum_{j=1}^q \nu_{\star}\ell(X_{ij}+\alpha_0;Y_{ij})$ with $\ell(x,y) = \log(1+\exp(x)) - yx$ and $\nu_{\star} := e^{-(\alpha_0 + M_2)}$. Then under suitable regularity conditions, the $\ell_2$ and $\ell_{\infty}$ error bounds in Theorem~\ref{thm_general_theory_UV_noise} hold with probability at least $1-\delta$ with $\alpha=\tfrac14 e^{-(M_1+ M_2)}$, $\beta= 1$,
\begin{equation*}
    \Delta_{2}(n,q,\delta) \asymp \big(\tfrac{\nu_{\star}}{n\wedge q}+\tfrac{\nu_{\star}^2\log((n+q)/\delta)}{nq}\big)^{1/2}\quad\text{ and }\quad
\Delta_{\infty}(n,q,\delta)\asymp\big(\tfrac{r+\log((n\vee q)/\delta)}{n\wedge q}\ \nu_{\star}R_{\star}\big)^{1/2},
\end{equation*}for some properly selected $R_{\star}$.
This result demonstrates the applicability of our theory to nonlinear models, which are substantially more challenging than linear settings. 
The closest existing result studies a projected gradient descent algorithm and provides only $\ell_2$ guarantees~\citep{ma2020universal}, while our theory also establishes the $\ell_{\infty}$ result. Moreover, our theory covers the near-optimal sparsity regime where $e^{\alpha_0}\asymp (n\wedge q)^{-1+\varepsilon}$ for any fixed $\varepsilon$, where regularization-free trajectory guarantees of this type were previously unavailable for such models.

It is worth noting that for this model, we adopt a leave-one-out argument to avoid the condition on $\bar\Delta_{\infty}(n,q,\delta)$. This provides a precise characterization of the dependence between iterates $(\bU^t,\bV^t)$ and the statistical noise, while it also requires certain technical regularity conditions on the leave-one-out initializers; see Theorem~\ref{Thm_example2} in Appendix~\ref{sec_example} and the subsequent discussion. The formal guarantee is hence conditional on these leave-one-out initialization conditions and holds uniformly over $0\le t\le T_\star$. Here, $T_\star$ denotes the maximum number of iterations, which enters the scaling conditions through logarithmic factors, due to a union bound over the trajectory. In the considered regime, taking $T_\star\asymp \log(n\vee q)$ is sufficient for the linearly decaying term to match the statistical error. This example shows that, with additional model structure, the general condition in Theorems~\ref{thm_general_theory_noisy_Z} and~\ref{thm_general_theory_UV_noise} on $\bar\Delta_{\infty}(\cdot)$ can be relaxed. 
\end{example}

\section{Conclusion}\label{sec_conclude}

This paper develops a theoretical framework for nonconvex procedures in low-rank estimation. Our key device is an equivalent formulation of the original nonconvex procedure into a strongly convex one within a local region, achieved via a benign regularizer that does not alter the update. This benign regularizer makes explicit the implicit structural feature that governs the mechanism of the nonconvex algorithm. Our framework offers a new geometric explanation for the stable behavior of the regularization-free nonconvex procedure and lays a foundation for analyzing more complex models in which a low-rank matrix is only one component, such as low-rank-plus-sparse models~\citep{candes2011robust}, LoRA fine-tuning~\citep{hu2022lora}, and deep matrix factorization~\citep{arora2019implicit}.


\bibliographystyle{plainnat}
\bibliography{reference}

\newpage
\appendix
\numberwithin{equation}{section}

\renewcommand{\thetheorem}{\thesection.\arabic{theorem}}
\renewcommand{\theproposition}{\thesection.\arabic{proposition}}
\renewcommand{\theassumption}{\thesection.\arabic{assumption}}

\startcontents[appendix]

\section*{Contents}
\setcounter{tocdepth}{1}
\printcontents[appendix]{}{1}{}

\section{Example Details}\label{sec_example}

\subsection{Matrix Sensing}
We provide the details for Example~\ref{example_1}. Recall that the observations follow
\begin{equation*}
    y_i=\langle \bA_i,\bX^*\rangle+\xi_i,\qquad i\in[m],
\end{equation*}
where $\bX^*=\bU^*(\bV^*)^\T\in\RR^{n\times q}$ has rank $r$. Here, $\langle \bA,\bB\rangle = {\rm Tr}(\bA\bB^\T)$ denotes the Frobenius inner product where ${\rm Tr}(\cdot)$ is the trace operator. Throughout, we condition our analysis on the sensing matrices $\{\bA_i\}_{i=1}^m$, so that the probability statements are taken with respect to the measurement noise $\{\xi_i\}_{i=1}^m$.
Define the sensing operator $\cA(\bX) := \big(\langle\bA_1,\bX\rangle,\dots, \langle\bA_m,\bX\rangle\big)$. Assume that $\cA$ satisfies the rank-$2r$ restricted isometry property: for any rank-$2r$ matrix $\bX$, for some $0<\delta_0<1$ such that
\begin{equation}
    (1-\delta_0)\|\bX\|_{\Fn}^2\le m^{-1}\|\cA(\bX)\|^2\le (1+\delta_0)\|\bX\|_{\Fn}^2.\label{eq_rip_sensing}
\end{equation}
This condition is standard in the matrix sensing literature~\citep{jain2010guaranteed,ma2021beyond,recht2010guaranteed,tu2016low}, and holds for a broad class of measurement ensembles~\citep{recht2010guaranteed}. The loss is $\cL(\bX) = (2m)^{-1}\sum_{i=1}^m\big(\langle\bA_i,\bX\rangle - y_i\big)^2$. 
We then apply the first part of Theorem~\ref{thm_general_theory_UV_noise} to obtain
\begin{theorem}
    \label{prop_ms_verify}
Assume that $(\bU^*,\bV^*)$ are balanced as in Theorem~\ref{thm_general_theory_UV_noise} and the sample size satisfies
\begin{equation*}
m\gg \frac{(1+\delta_{0})\kappa^2r\sigma_\xi^2}{(1-\delta_{0})^4\sigma_{\min}^2} \cdot \frac{\log\{(n\vee q)/\delta\}}{n\wedge q},
\end{equation*}
If the initialization $(\bU^0,\bV^0)$ satisfies \eqref{eq_init_consis_main_R0} for $\epsilon = \infty$, for the iterates $\{(\bU^t,\bV^t)\}_{t\ge 0}$ generated by \eqref{eq_one_step_update_uv} with step size $\eta^t=\eta = (20\kappa\sigma_{\min})^{-1}$, we have with probability at least $1-\delta$,
\begin{equation*}
\disttwo\{(\bU^t,\bV^t),(\bU^*,\bV^*)\} \le \left(\; \rho^t\phi_{nq} + C\frac{\sigma_\xi}{(1-\delta_{0})\sigma_{\min}} \sqrt{\frac{(1+\delta_0)\log((n+q)/\delta)}{m(n\wedge q)}} \;\right)\tau_*,
\end{equation*} where $\rho = 1- (1-\delta_0)/(80\kappa)$.
\end{theorem}
\begin{proof}
    See Section~\ref{supp_sec_prove_example_1}. Here, we take $\alpha=1-\delta_0$, $\beta = 1+\delta_0$, and
\begin{equation*}
\Delta_{2}(n,q,\delta)\asymp \sigma_{\xi}  \sqrt{\tfrac{(1+\delta_0)\log((n+q)/\delta)}{m(n\wedge q)}} . 
\end{equation*}
The result in Example~\ref{example_1} is obtained by taking $\delta_0\in(0,1)$ to be some absolute constant.
\end{proof}

Our result is closely related to the balancing-free analysis of \citet{ma2021beyond} for asymmetric low-rank matrix sensing. They consider the noiseless model and show that, under a rank-$2r$ RIP condition, regularization-free gradient descent converges linearly once the initialization lies in a sufficiently small balanced basin of attraction. In our notation, their initialization condition can be written as
\begin{equation*}
    \phi_{nq}\lesssim \frac{c_0\sqrt{\sigma_{\min}}}{\tau_*\kappa^{3/2}} = \frac{c_0}{\kappa^{5/2}\sqrt{r}},
\end{equation*}
which matches ours up to constants when $\delta_0$ is bounded away from one. For the sample size assumption, it arises only in the noisy setting. When $\xi_i = 0$,  we can apply Theorem~\ref{thm_general_theory_uv} directly to obtain the convergence guarantee
\begin{equation*}
\disttwo\{(\bU^t,\bV^t),(\bU^*,\bV^*)\} \le  \rho^t\phi_{nq}\tau_*,
\end{equation*} 
provided \eqref{eq_rip_sensing} holds. In this case, no sample size condition is needed, which recovers the result of \citet{ma2021beyond} as a special case.
Moreover, we note that our sample size requirement is also close to the information-theoretical optimum where $m\asymp r(n\vee q)$ up to logarithmic factors~\citep{candes2011tight,recht2010guaranteed}.


Initialization can be obtained using a similar strategy presented in \citet{ma2021beyond}. One may first construct a rank-$r$ spectral or projected-gradient estimator of $\bX^*$, and then take its balanced factorization as $(\bU^0,\bV^0)$. It can be verified that these approaches can satisfy our initialization requirement as long as the perturbation from $m^{-1}\sum_{i=1}^m\xi_i\bA_i$ is sufficiently small.

\subsection{Low-rank Bernoulli response model}\label{section_example2}
We now describe the setup for the low-rank Bernoulli response model in detail. Recall that $Y_{ij}\sim{\rm Bernoulli}(P_{ij}^*)$ for $P_{ij}^* = 1/\{1+e^{-(\alpha_0+X^*_{ij})}\}$. Let
$$\sigma(x):=\sigma_0(\alpha_0+x) = \frac{\exp(\alpha_0+x)}{1+\exp(\alpha_0+x)}.$$
Without loss of generality, we assume $\alpha_0\le 0$, since one can always flip the labels $0$ and $1$ so that $0$ becomes the majority observation. We use $M_1$ and $M_2$ to quantify the range of the linear predictors, and assume without loss of generality that
\begin{equation*}
    -M_1\le (\bU\bV^\T)_{i,j}\le M_2,\qquad \forall\, i\in[n],\ j\in[q],
\end{equation*}for any $(\bU,\bV)\in\cDuvinf$. Here, subscript $\bA_{i,j}$ denotes the $(i,j)$th entry of matrix $\bA$. Since $\epsilon$ can be taken sufficiently small in the subsequent analysis, this condition can always be ensured by enlarging $M_1$ and $M_2$ slightly by constant factors. Let $M_{\star} = M_1 + M_2$. 

For the scaled loss function $\cL(\bX)= \sum_{i=1}^n\sum_{j=1}^q \nu_{\star}\ell(X_{ij};Y_{ij})$ for $\ell(x;y) = \log\{1+\exp(\alpha_0+x)\}-yx$ and $\nu_{\star} = e^{-(\alpha_0 + M_2)}$, under the notation of Assumption~\ref{assump_rect_rip_weighted}, one can verify that $\alpha$ and $\beta$ may be chosen as
\begin{equation*}
\alpha=\frac14 e^{-M_{\star}}, \qquad \beta=1.
\end{equation*}
Now we state our result.
\begin{theorem}
\label{Thm_example2}
Assume that $(\bU^*,\bV^*)$ are balanced as in Theorem~\ref{thm_general_theory_UV_noise}. Fix a maximal iteration count for~\eqref{eq_one_step_update_uv} as $T_{\star}$. Let $\zeta_r:=\big(\tfrac{\beta}{\alpha}+\sqrt r\big)\sqrt r\,\kappa$, $\ell_{\star} = \log\{(n+q)T_{\star}/\delta\}$, $L_{\star} = r + \ell_{\star}$. We assume that for any small constant $c_0$, the following holds
\begin{equation}\label{eq_bern_prop_scaling}
\begin{aligned}\frac{(\nu_\star+1)L_\star}{n\wedge q}&\le c_0 \\
\left[\frac{\left\{\left( \alpha^{-1}\kappa\sqrt{\kappa r}+ \kappa r \right)\zeta_r \right\}^2}{L_\star} + \frac{\nu_\star\kappa L_\star(\omega_*^4+\omega_*^{-2})} {(n\wedge q)\alpha^4\sigma_{\min}^4}
\right] &\le
c_0 \frac{\alpha^2\sigma_{\min}^2\epsilon^2(n\wedge q)}{\nu_\star L_\star\kappa^2r(\alpha^{-1}+\sqrt r)^2 }.
\end{aligned}\end{equation}
Then we can select $ R_\star$ such that for sufficiently large constant $C$, it holds that
\begin{equation}\label{eq_bern_prop_scaling_1}
    R_{\star} = C\left[\frac{\left\{\left( \alpha^{-1}\kappa\sqrt{\kappa r}+ \kappa r \right)\zeta_r \right\}^2}{L_\star} + \frac{\nu_\star\kappa L_\star(\omega_*^4+\omega_*^{-2})} {(n\wedge q)\alpha^4\sigma_{\min}^4}
\right] 
\end{equation}and we set
\begin{equation}
\Delta_{2}(n,q,\delta) \asymp \sqrt{\frac{\nu_{\star}}{n\wedge q}+\frac{\nu_{\star}^2\ell_{\star}}{nq}},\quad
\Delta_{\infty}(n,q,\delta) \asymp \sqrt{\frac{\nu_{\star}L_{\star}R_{\star}}{n\wedge q}}.
\end{equation}
Assume that the initialization $(\bU^0,\bV^0)$ satisfies
\eqref{eq_init_consis_main_R0} and \eqref{eq_sym_init_Z_uniform_uv} with $\epsilon\le c_0\alpha/(\kappa\sqrt{r})$ for sufficiently small constant $c_0$, and that the following leave-one-out initializers are available.
\begin{enumerate}[leftmargin = 0.6cm]
    \item For each $i\in[n]$, there is a row-wise LOO initializer $(\bU^{0,-i},\bV^{0,-i})$, measurable to $\sigma\big(\bY^{-i,}\big)$.
    \item For each $\ell\in[q]$, there is a column-wise LOO initializer $(\bU^{0,-\ell},\bV^{0,-\ell})$, measurable to $\sigma\big(\bY^{,-\ell}\big)$.
\end{enumerate}
Here, $(\bY^{-i})_{k\ell} := Y_{k\ell}1_{\{k\neq i\}}+P^*_{k\ell}1_{\{k=i\}}, \text{ for } i,k\in[n],\ \ell\in[q]$, and $(\bY^{-\ell})_{kj} := Y_{kj}1_{\{j\neq \ell\}}+P^*_{kj}1_{\{j=\ell\}}, \text{ for } k\in[n],\ \ell,j\in[q]$.
Let $\phi_{nq}^{\dagger}:= \phi_{nq} + C\tfrac{\Delta_2(n,q,\delta)}{\alpha\sigma_{\min}}$ and $\psi_{nq}^{\dagger} := \psi_{nq} + C\tfrac{\Delta_\infty(n,q,\delta)}{\alpha\sigma_{\min}}$.
Assume further that, for the same orthogonal matrix $\bR^0$ in \eqref{eq_init_consis_main_R0} and \eqref{eq_sym_init_Z_uniform_uv}, these LOO initializers satisfy
\begin{align*}
    &\max_{i\in[n]}\frac{\big\|\bU^{0,-i,}\bR^{0} - \bU^{0}\bR^0\big\|_{\Fn}}{\sqrt{n}} \vee \frac{\big\|\bV^{0,-i}\bR^{0} - \bV^{0}\bR^0\big\|_{\Fn}}{\sqrt{q}}\le c_0\frac{\alpha\sigma_{\min}}{\sqrt{\nu_{\star}\beta+\nu_{\star}^2L_{\star}/q}}\, \psi_{nq}^{\dagger}\omega_*;\\
    &\max_{\ell\in[q]}\frac{\big\|\bU^{0,-\ell}\bR^{0} - \bU^{0}\bR^0\big\|_{\Fn}}{\sqrt{n}} \vee\frac{\big\|\bV^{0,-\ell}\bR^{0} - \bV^{0}\bR^0\big\|_{\Fn}}{\sqrt{q}}\le \frac{\alpha\sigma_{\min}}{\sqrt{\nu_{\star}\beta+\nu_{\star}^2L_{\star}/n}}\, \psi_{nq}^{\dagger}\omega_*;\\
&\max_{i\in[n]}\left\{ \|\bU^{0,-i}\bR^0-\bU^*\|_{2\to\infty} \vee \|\bV^{0,-i}(\bR^0)^{-\T}-\bV^*\|_{2\to\infty} \right\}\le c_0\psi_{nq}^{\dagger}\omega_*;\\
&\max_{\ell\in[q]}\left\{ \|\bU^{0,-\ell}\bR^0-\bU^*\|_{2\to\infty} \vee \|\bV^{0,-\ell}(\bR^0)^{-\T}-\bV^*\|_{2\to\infty} \right\}\le c_0\psi_{nq}^{\dagger}\omega_*,\end{align*}
for sufficiently small constant $c_0$.
Let the iterates $\{(\bU^t,\bV^t)\}_{t\ge 0}$ be generated by \eqref{eq_one_step_update_uv}
with step size
\begin{equation}
\eta^t=\eta=\{10(\alpha+\beta)\kappa\sigma_{\min}\}^{-1},
\qquad
\rho:=1-\eta\alpha\sigma_{\min}/4.
\end{equation}
Then, with probability at least $1-\delta$, the $\ell_2$ and $\ell_{\infty}$ error contractions of
Theorem~\ref{thm_general_theory_UV_noise} hold for the Bernoulli model uniformly for all iterations $0\le t\le T_\star$:
\begin{align*}
    \disttwo\{(\bU^t,\bV^t),(\bU^*,\bV^*)\} \le \rho^t\phi_{nq}\tau_* + C\Big\{\frac{\nu_{\star}}{n\wedge q}+\frac{\nu_{\star}^2\log((n+q)/\delta)}{nq}\Big\}^{1/2}\,\frac{\tau_*}{\alpha\sigma_{\min}};\\
    \distinf\{(\bU^t,\bV^t),(\bU^*,\bV^*)\} \le  \rho^t\psi_{nq}\omega_* + C\Big\{\frac{r+\log((n\vee q)/\delta)}{n\wedge q}\ R_{\star}\Big\}^{1/2}\,\frac{\omega_*}{\alpha\sigma_{\min}},
\end{align*}
\end{theorem}
\begin{proof}
    See Section~\ref{supp_sec_prove_example2}.
\end{proof}

The scaling conditions in \eqref{eq_bern_prop_scaling} are mild and imply that the theory guarantees convergence in the near-optimal sparse regime. Suppose, for simplicity, that $r,\kappa,\sigma_{\min},\omega_*,M_1,M_2$ are constants and that $T_\star\asymp \log(n\vee q)$.  Then $L_\star = r+\log\{(n\vee q)T_\star/\delta\} \asymp \log\{(n\vee q)/\delta\}$ up to a logarithmic factor.  Then conditions in \eqref{eq_bern_prop_scaling} reduce to $\tfrac{\nu_{\star}L_\star}{n\wedge q}\ll 1$. Equivalently,
\[
    e^{\alpha_0}(n\wedge q) \gg \log\{(n\vee q)/\delta\}.
\]
This condition allows $e^{\alpha_0}\asymp (n\wedge q)^{-1+\varepsilon}$, for any fixed \(\varepsilon>0\), which matches the information-theoretically near-optimal scaling $nq e^{\alpha_0}\asymp (n\vee q)$ up to logarithmic factors.
To the best of our knowledge, this is the first convergence guarantee for such models in the optimal sparsity regime.

The initialization assumptions in Theorem~\ref{Thm_example2} should be viewed as technical conditions needed for the leave-one-out argument. They are imposed so that one can construct LOO gradient descent sequences that track the trajectory of $(\bU^{t},\bV^t)$ while enjoying favorable dependence structures (see details in Lemma~\ref{lem_tentative_loo_gradient_mc}). In particular, these auxiliary sequences remain close to $(\bU^{t},\bV^t)$ but are independent of certain rows or columns of $\bY$. This reflects an underlying geometric property in the optimization problem, where each row of $\bU^{t}$ and $\bV^t$ is not significantly affected by any single row or column of $\bY$. In other words, near the optimization trajectory, the statistical noise $\bY-\bP^*$ influences the landscape in an approximately uniform way. As for constructing initializers satisfying the stated rates, this is beyond the scope of this paper, and we leave this question for future work.

The quantity $T_{\star}$ appears only through logarithmic factors.  It is included because the LOO sequences allow us to apply a Bernstein inequality to terms of the form $\tilde\cG(\bU^{-i,t},\bV^{-i,t})\Delta$ at each step, where the relevant difference $\Delta$ is independent of the row or column noise being averaged.
These high-probability events are required to hold uniformly over all iterations $0\le t\le T_{\star}$, and the proof takes a union bound over these iterations.

\section{Additional Assumptions}\label{sec_ell_infty_requirement}
We introduce additional technical assumptions needed to establish $\ell_{\infty}$ convergence. Denote $\be_i^{(n)}$ the canonical basis in $\RR^{n}$. For any matrix $\bM\in\RR^{n\times q}$ and $i\in[n]$, $\bM_{i,\cdot}$ denotes its $i$th row. The following is a row-wise control on the Hessian matrix.
\begin{assumption}\label{assump_row_cross_curvature_sym}
There exists some $\epsilon$ and $\beta>0$ such that for either $\cD=\cDtwo$ or $\cD=\cDinf$, the following holds: for any $\bZ\in\cD$, $i\in[n]$, $\bh\in\RR^r$, and $\bW\in\RR^{n\times r}$ satisfying $\bW_{i,\cdot}=0$, by letting $\bW_i = \be_i^{(n)}\bh^\T\in\RR^{n\times r}$, it holds that
\begin{equation*} \big| \nabla_{\bX}^2\cL(\bZ\bZ^\T)\big[\cP_{\bZ}(\bW_i),\,\cP_{\bZ}(\bW)\big] \big| \le \beta\,\big(\|\bZ_{i,\cdot}\|\|\bh\|\big)\,\big(\|\bZ\|\|\bW\|_{\Fn}\big),\end{equation*}
where $\cP_{\bZ}(\bW) := \bZ\bW^\T + \bW\bZ^\T$.
\end{assumption}

Assumption~\ref{assump_row_cross_curvature_sym} strengthens the local RIP condition by adding a row-wise control on the Hessian. Here, $\beta$ may differ from the one in Assumption~\ref{assump_rsc}; for simplicity, we use the same notation, since one may take $\beta$ to be the maximum of the two constants.
While Assumption~\ref{assump_rsc} controls the overall size of the curvature on tangent directions, it does not rule out anisotropic behavior in which perturbation concentrating on one row interacts strongly with perturbations supported away from that row. A proper row-wise control is important for the $\ell_\infty$-analysis. The assumption is mild for loss function $\cL(\cdot)$ whose Hessian structure is diagonal, or approximately diagonal. For example, if
\begin{equation*}
\nabla_{\bX}^2\cL(\bX)[\bH_1,\bH_2] = \sum_{a,b}\omega_{ab}(X)\,\bH_{1,ab}\bH_{2,ab}, \qquad |\omega_{ab}(\bX)|\le \beta,
\end{equation*}
then $\cP_{\bZ}(\be_i \bh^\T)$ only overlaps with the $i$th row and column of $\cP_Z(\bW)$, which gives
\begin{equation*}
\big| \nabla_{\bX}^2\cL(X)[\cP_{\bZ}(\be_i^{(n)} \bh^\T),\cP_{\bZ}(\bW)] \big| \le C\beta\|\bZ_{i,\cdot}\|\|\bZ\|\|\bh\|\,\|\bW\|_{\Fn}, \qquad \bW_{i,\cdot}=0.
\end{equation*}
Hence Assumption~\ref{assump_row_cross_curvature_sym} holds naturally in entrywise models such as Example~\ref{example_2}. Geometrically, the RIP controls the strength of curvature, whereas Assumption~\ref{assump_row_cross_curvature_sym} controls how localized that curvature remains in the ambient coordinates.

Assumption~\ref{assump_gradient_lips} is the additional requirement for deriving $\ell_{\infty}$ error contraction.
\begin{assumption}\label{assump_gradient_lips}
There exist some $\epsilon$ and a constant $L_{\infty}$ such that for all $\bZ_1,\bZ_2\in\cD$, $\|\cG(\bZ_1)-\cG(\bZ_2)\|_{\twinf} \le L_{\infty}\|\bZ_1\bZ_1^\T-\bZ_2\bZ_2^\T\|_{\twinf}$, where $\cD=\cDtwo$ or $\cD=\cDinf$.
\end{assumption}
Similar to the above discussion, Assumption~\ref{assump_gradient_lips} is needed because a  spectral norm bound on 
$\cG(\bZ_1)-\cG(\bZ_2)$ does not suffice to control the maximum row-wise size of the gradient perturbation. For entrywise loss functions, Assumption~\ref{assump_gradient_lips} typically follows from a uniform bound on the second-order derivative of the link function over the local region $\cD$. 

Next, we present two assumptions required for establishing $\ell_{\infty}$ error contraction under the asymmetric model analogous to Assumptions~\ref{assump_row_cross_curvature_sym} and~\ref{assump_gradient_lips}.

\begin{assumption}\label{assump_row_cross_curvature_uv}
There exists some $\epsilon$ and $\beta>0$ such that, for either $\cD=\cDuvtwo$ or $\cD=\cDuvinf$, the following holds:  for $(\bU,\bV)\in\cD$, $i\in[n]$, $\bh\in\RR^r$, by letting $\bL_i = \be_i^{(n)}\bh^\T$ and $\bR_{j} = \be_j^{(q)}\bh^\T$, for every $(\bL,\bR)\in\RR^{n\times r}\times\RR^{q\times r}$, 
\begin{align*}\big|\nabla_{\bX}^2\cL(\bU\bV^\T)[\cP_{(\bU,\bV)}(\zero, \bL_i),\cP_{(\bU,\bV)}(\bR, \bL_{-i})\big|\le \beta\|\bV\bh\|\,\|\bU_{i,}\bR^\T\|,\\
\big|\nabla_{\bX}^2\cL(\bU\bV^\T)[\cP_{(\bU,\bV)}(\bR_j, \zero),\cP_{(\bU,\bV)}(\bR_{-j}, \bL)\big|\le \beta\|\bU\bh\|\,\|\bL(\bV_{j,})^\T\|,\end{align*}where $\cP_{(\bU,\bV)}(\bR,\bL) = \bU\bR^\T + \bL\bV^\T$.
Here, $\bL_{-i}$ denotes the matrix obtained from $\bL$ by replacing its $i$th row with $\zero$, and $\bR_{-j}$ is defined analogously by replacing the $j$th row of $\bR$ with $\zero$.
\end{assumption}

\begin{assumption}\label{assump_rect_lip}
    There exist $\epsilon$ and a constant $L_\infty>0$ such that for all $(\bU_1,\bV_1),(\bU_2,\bV_2)\in\cD$, $\big\|\cG(\bU_1,\bV_1)-\cG(\bU_2,\bV_2)\big\|_{\twinf} \le L_\infty\big\|\bU_1\bV_1^\T-\bU_2\bV_2^\T\big\|_{\twinf}$ and $\big\|\{\cG(\bU_1,\bV_1)-\cG(\bU_2,\bV_2)\}^\T\big\|_{\twinf} \le L_\infty\|\bV_1\bU_1^\T-\bV_2\bU_2^\T\|_{\twinf}$ where $\cD=\cDuvtwo$ or $\cD=\cDuvinf$.
\end{assumption}

The assumptions introduced in this section are technical regularity conditions for row-wise control, imposed to establish $\ell_{\infty}$ convergence, thereby enabling analysis under localizations $\cDinf$ and $\cDuvinf$. This analysis is important for two purposes. First, $\ell_{\infty}$ convergence provides a sharper characterization of gradient descent trajectory than $\ell_2$ control alone, and has played an important role in the analysis of nonconvex low-rank problems~\citep{chen2020nonconvex,chen2020noisy,ma2018implicit}. Second, the local RIP conditions in Assumptions~\ref{assump_rsc} and~\ref{assump_rect_rip_weighted} may fail under the larger localizations $\cDtwo$ and $\cDuvtwo$. For instance, in nonlinear models such as Example~\ref{example_2}, the curvature of the loss can vanish as the entries of $\bX$ diverge. Under $\cDtwo$ or $\cDuvtwo$, there is no guarantee that the largest entry of $\bX$ remains bounded. A row-wise localization of the factors is a natural way to control
$\|\bU\bV^\T-\bU^*(\bV^*)^\T\|_{\infty}$
and thereby ensure nonvanishing local curvature. A similar case appears in matrix completion, where restricted curvature is available only over incoherent or entrywise-controlled low-rank matrices, rather than over the entire low-rank variety
\citep{chen2015incoherence,keshavan2010matrix,sun2016guaranteed,zheng2016convergence}. Finally, we note that in some settings, these row-wise assumptions may not hold without additional structure. For instance, in matrix sensing presented in Example~\ref{example_1}, standard analyses typically establish $\ell_2$ convergence, rather than $\ell_{\infty}$ one. 
In such cases, since the RIP holds over the larger $\ell_2$ localization $\cDuvtwo$, the first part of Theorem~\ref{thm_general_theory_UV_noise} still yields $\ell_2$ convergence. 
This suggests that the additional row-wise conditions are not merely proof artifacts, but may reflect structural requirements needed for $\ell_{\infty}$ convergence.

Finally, we end this part with a summary of key notation used in the proof; see Table~\ref{tab:proof_notation}.

\begin{table}[h]
\centering
\caption{Key notation in the proofs.}
\label{tab:proof_notation}
\begin{tabular}{ll}
\toprule
Symbol & Meaning \\
\midrule
$n,q,r$ & Matrix dimensions and rank: $\bX^*\in\RR^{n\times q}$ has rank $r$\\
$\kappa $ & Condition number $\sigma_1(\bX^*)/\sigma_r(\bX^*)$ \\
$\sigma_{\min}$ & Normalized smallest nonzero singular value of $\bX^\ast$: $\sigma_{\min} = {\sigma_r(\bX^*)}/{\sqrt{nq}}$ \\
$\alpha$, $\beta$ & Local lower and upper curvature constants of the loss\\
$\phi_n,\psi_n$ & Symmetric $\ell_2$ and $\ell_\infty$ initialization radii \\
$\phi_{nq},\psi_{nq}$ & Asymmetric $\ell_2$ and $\ell_\infty$ initialization radii \\
$\eta ,\rho$ & Step size and contraction factor, typically $\rho=1-\eta\alpha\sigma_{\min}/4$\\
$\tau_\ast$ & Weighted Frobenius scale: $(n^{-1}\|\bU^\ast\|^2_{\Fn} + q^{-1}\|\bV^\ast\|^2_{\Fn})/2$ \\
$\omega_\ast$ & Row-wise scale $\|\bU^\ast\|_{2\to\infty}\vee\|\bV^\ast\|_{2\to\infty}$ \\
$\Delta_2$ & Operator-norm stochastic error \\
$\Delta_\infty$ & Row-wise stochastic error along true factor directions \\
$\bar\Delta_\infty$ & Worst-direction row-wise stochastic error \\
$\cG$ & First-order derivative: $\cG(\ ) = \nabla_{\bX}\cL(\ )$ with the argument being $\bZ$ under {\it symmetric} model \\& or $(\bU,\bV)$ under {\it asymmetric} model. \\
$\bR_t^*$ & Optimal alignment matrix under {\it symmetric} model, orthogonal \\
$\bG_t^\ast$ & Optimal alignment matrix under {\it asymmetric} model, invertible\\ 
\bottomrule
\end{tabular}
\end{table}

\section{Proofs of Results under Symmetric Model}\label{supp_sec_prove_sym}
This section proves the main results for the symmetric model in the deterministic setting (Theorem~\ref{thm_general_theory}) and the noisy setting (Theorem~\ref{thm_general_theory_noisy_Z}). Section~\ref{supp_sec_prelim_1} introduces notation and technical lemmas used throughout this section. Sections~\ref{supp_sec_prove_thm_general_theory} and~\ref{supp_prove_thm_general_theory_noisy_Z} present proofs for $\ell_2$ and $\ell_{\infty}$ error contractions under localization $\cDinf$. We begin with this localization because it is the more delicate case, where the proof must control the $\ell_2$ and $\ell_{\infty}$ errors simultaneously along the entire trajectory.  Section~\ref{supp_sec_prove_thm_general_theory_noisy_Z_l2} then gives the $\ell_{2}$ error contraction under the larger localization $\cDtwo$.

Before starting the formal proof, we provide a roadmap for the argument under localization $\cDinf$. The {\it noiseless} proof has three main components. First, the benign regularizer turns the aligned update into a gradient step on a locally strongly convex objective as in \eqref{alg_aligned_gcd}. This immediately yields the $\ell_2$ error contraction. Second, using the same aligned update, we derive the $\ell_{\infty}$ error for $\bZ^{t+1} \bR_t^* - \bZ^*$. Assumptions~\ref{assump_row_cross_curvature_sym} and~\ref{assump_gradient_lips} are used at this stage to control the row-wise effect of the second-order and first-order derivatives, respectively. Third, we quantify the rotation drift $\bR_{t+1}^*-\bR_{t}^*$ and transfer the error bound for $\|\bZ^{t+1}\bR_t^* - \bZ^*\|_{\twinf}$ to the optimally aligned error \begin{equation*}\distinf(\bZ^{t+1},\bZ^*) = \|\bZ^{t+1}\bR_{t+1}^* - \bZ^*\|_{\twinf}.\end{equation*}

In the {\it noisy} case, the iterates no longer contract exactly toward $\bZ^*$. We therefore first analyze the optimizer $\hat\bZ$, which is a statistically perturbed version of $\bZ^*$. Then with $\hat\bR = \argmin_{\bR\in\cO^r}\|\hat\bZ\bR - \bZ^*\|_{\Fn}$, we establish the first order condition for $\hat\bZ\hat\bR$. Specifically, we show that $\nabla_{\bZ}\cL(\hat\bZ\hat\bZ^\T) = \zero$ and $\nabla_{\bZ}p_{\alpha}^*(\hat\bZ\hat\bR)=\zero$. Thus $\hat\bZ\hat\bR$ is a valid contraction target for the noisy trajectory. We then re-center the analysis at $\hat\bZ\hat\bR$ and repeat the deterministic contraction argument, with the alignment now taken relative to this empirical target. The final error is the sum of the statistical error of $\hat\bZ\hat\bR$ and the algorithmic error towards $\hat\bZ\hat\bR$. 
We adopt this strategy because the non-decaying statistical error is absorbed into $\hat\bZ$, allowing the trajectory analysis to retain a purely contractive form up to the final statistical radius. For the same reason, although the noisy setting subsumes the noiseless one, we present our framework and results in the main text under the noiseless setting for clarity.

\subsection{Preliminaries}\label{supp_sec_prelim_1}

We use $\{\be_i^{(n)}\}$ to denote the canonical basis in $\RR^{n}$, and we omit the superscript $^{n}$ when the dimension is clear from the context. For a vector $\bx\in\RR^n$ and a subset $\cS\subseteq[n]$, write $\bx_{\cS}\in\RR^{|\cS|}$ for the sub-vector indexed by $\cS$. For a matrix $\bM\in\RR^{n\times m}$ and index sets $\cS_1\subseteq[n]$, $\cS_2\subseteq[m]$, write $\bM_{\cS_1,\cS_2}\in\RR^{|\cS_1|\times|\cS_2|}$ for the corresponding submatrix. In particular, $\bM_{i,\cdot}$ and $\bM_{\cdot,j}$ denote the $i$th row and $j$th column, while $\bM_{\cS_1,\cdot}$ and $\bM_{\cdot,\cS_2}$ denote row- and column-restricted submatrices. Let $\otimes$ be the Kronecker product.
Write $\bz=\mathrm{vec}(\bZ^\T)$, so that $\nabla_{\bz}^2$ denotes the Hessian with respect to the vectorized variable; later, $\nabla_{\bz_i\bz_j}^2$ denotes its $(i,j)$ block.

Let
\begin{equation*}
    \bM^*(\bZ) = n^{-1}\big\{(\bZ^*)^\T\bZ - \bZ^\T\bZ^*\big\}.
\end{equation*}
Then the benign regularizer for the symmetric model can be written as $p_{\alpha}^*(\bZ) = \tfrac{\alpha n^2}{4}\|\bM^*(\bZ)\|_{\Fn}^2$. Without loss of generality, we assume $\alpha<1$. For matrix-valued maps such as $\bM^*(\bZ)$, we write $D\bM^*(\bZ)[\bW]$ for the Fr\'echet derivative at $\bZ$ applied to the direction $\bW$, and $D^2\bM^*(\bZ)[\bW_1,\bW_2]$ for the second Fr\'echet derivative.

For a scalar-valued function $f$ and any direction $\bW$,
\begin{equation*}
\{\mathrm{vec}(\bW^\T)\}^\T \nabla_{\bz}^2 f(\bZ)\,\mathrm{vec}(\bW^\T) = \nabla_{\bZ}^2 f(\bZ)[\bW,\bW].
\end{equation*}
Throughout the proof, all constants $C,c>0$ are universal and may change from line to line. We typically use $c_0$ as constants that can be sufficiently small. Let $\rho=1-\eta\alpha\sigma_{\min}/4$.

Finally, we introduce the following technical lemmas used in the subsequent proofs. We start with a slightly more general version of Lemma~\ref{lemma_general_convex}:
\begin{lemma}\label{lemma_general_convex}
    Under Assumption~\ref{assump_rsc}, and with $\lambda = \alpha$, we have
    \begin{equation*}
        \min_{\bZ\in\cD}\lambda_{\min}\big\{n^{-1}\nabla_{\bz}^2 h_{\alpha}^*(\bZ) - 2n^{-1}\cG(\bZ) \otimes \bI_r\big\}\ge \alpha\sigma_r(n^{-1/2}\bZ^*)^2 -4n^{-1}\max_{\bZ\in\cD}\|\bZ - \bZ^*\|_{\Fn}^2.
    \end{equation*}
\end{lemma}
\begin{proof}
    See Section~\ref{supp_Sec_prove_lemma_general_convex}.
\end{proof}
We introduce the following result that bounds the eigenvalues for the principal submatrix of $\nabla_{\bz}^2\cL(\bZ\bZ^\T)$.
\begin{lemma}\label{coro_thm_general_convex}
    Under Assumption~\ref{assump_rsc},
    \begin{equation*}
        \min_{\bZ\in\cD}\min_{i\in[n]}\lambda_{\min}\big\{n^{-1}\nabla^2_{\bz_i}\cL(\bZ\bZ^\T) - 2 n^{-1}\cG_{i,i}(\bZ)\bI_r\big\}\ge \alpha\sigma_r(n^{-1/2}\bZ^*)^2 - 2n^{-1}\max_{\bZ\in\cD}\|\bZ-\bZ^*\|_{\Fn}^2,
    \end{equation*}
    and \begin{equation*}
        \max_{\bZ\in\cD}\max_{i\in[n]}\lambda_{\max}\big\{n^{-1}\nabla^2_{\bz_i}\cL(\bZ\bZ^\T) - 2 n^{-1}\cG_{i,i}(\bZ)\bI_r\big\}\le \beta\sigma_1(n^{-1/2}\bZ^*)^2 + 2n^{-1}\max_{\bZ\in\cD}\|\bZ-\bZ^*\|_{\Fn}^2.
    \end{equation*}
\end{lemma}\begin{proof} The proof follows from the same argument as in Lemma~\ref{lemma_general_convex} by restricting the perturbation to the form $\bW=\be_i\ba^\T$ for $\ba\in\RR^r$. We omit the routine details.\end{proof}

This lemma presents the properties of the optimal alignment in the symmetric model.
\begin{lemma}[Theorem 2 in \citet{ten1977orthogonal}]\label{lemma_ortho_align}
    For any $\bZ\in\RR^{n\times r}$, $\hat\bR^*$ is the solution to $\argmin_{\bR\in\cO^r}\|\hat\bZ\bR - \bZ^*\|_{\Fn}$ if and only if 
    \begin{equation*}
        \hat{\bR}^*{}^\T\hat\bZ^\T\bZ^* = (\bZ^*)^\T\hat\bZ\hat\bR^*\text{ is positive semidefinite}.
    \end{equation*}
\end{lemma}

This lemma relates perturbation bounds in the product space $ZZ^\T$ into aligned perturbation bounds for the factorization $\bZ$ itself.
\begin{lemma}[Lemma 5.3 in~\citet{tu2016low}]\label{lemma_zz_to_z}
For any $\bZ_1,\bZ_2 \in \mathbb{R}^{n \times r}$ such that
$\disttwo(\bZ_1,\bZ_2)\le c\|\bZ_1\|$ for some constant $c>0$, we have
\begin{equation*}
\|\bZ_1\bZ_1^\T-\bZ_2\bZ_2^\T\|_{\Fn}
\le
(2+c)\|\bZ_1\|\disttwo(\bZ_1,\bZ_2).
\end{equation*}
\end{lemma}

This lemma controls the change of the optimal Procrustes alignment between $\bA$ and $\bB$ when they are close in terms of Frobenius norm.
\begin{lemma}[Theorem 2.3 in \citet{mathias1993perturbation}]\label{lemma_procruste_problem_perturb}
Let $\bA,\bB\in\RR^{n\times r}$, and let $\bR^*\in\cO^r$ be the optimal rotation defined by
\begin{equation*}
\bR^*=\argmin_{\bR\in\cO^r}\|\bA\bR-\bB\|_{\Fn}.
\end{equation*}
If $\sigma_r(\bB)>0$ and $\|\bA-\bB\|_{\Fn}<\sigma_r(\bB)$, then the rotation drift satisfies
\begin{equation*}
\|\bR^*-\bI_r\|_{\Fn}
\le
\frac{2}{\sigma_r(\bA)+\sigma_r(\bB)}\|\bA-\bB\|_{\Fn}.
\end{equation*}
\end{lemma}

\subsection{Proof of Theorem~\ref{thm_general_theory}}\label{supp_sec_prove_thm_general_theory}
We prove in this subsection the stronger contraction statement under the localization $\cDinf$. Accordingly, throughout the proof we invoke Assumptions~\ref{assump_rsc},~\ref{assump_gradient_lips_1},~\ref{assump_row_cross_curvature_sym}, and~\ref{assump_gradient_lips} with $\cD=\cDinf$. The $\ell_2$-only part of Theorem~\ref{thm_general_theory}, where only Assumptions~\ref{assump_rsc}--\ref{assump_gradient_lips_1} are imposed with $\cD=\cDtwo$, is proved later in Section~\ref{supp_sec_prove_thm_general_theory_noisy_Z_l2} by the same argument after removing the row-wise estimates.
Throughout this proof, we use the specialization $\lambda=\alpha$ from Lemma~\ref{lemma_general_convex} and the constant step size $\eta^t\equiv \eta=\{10(\alpha+\beta)\kappa\sigma_{\min}\}^{-1}$ from Theorem~\ref{thm_general_theory}.
For each $t\ge 0$, let \begin{equation*}\bR_t^* = \argmin_{\bR\in\cO^r}\|\bZ^t\bR-\bZ^*\|_{\Fn},\quad\tilde\bZ^t:=\bZ^t\bR_t^*\quad\text{, and }\quad\bE_t = \bZ^t\bR_t^* - \bZ^*.\end{equation*} By the orthogonal Procrustes optimality condition (Lemma~\ref{lemma_ortho_align}), we know $(\bZ^*)^\T \tilde\bZ^t=\tilde\bZ^{t\T}\bZ^*$. Then the by construction, $\nabla_{\bZ}p_{\alpha}^*(\tilde\bZ^t)=\zero$ and $\nabla_{\bZ}p_{\alpha}^*(\bZ^*)=\zero$.

For the $\ell_2$ part of the theorem, the initialization bound \eqref{eq_sym_init_Z} gives the base case at $t=0$. For the $\ell_\infty$ part, the bounds \eqref{eq_sym_init_Z} and \eqref{eq_sym_init_Z_uniform} give the base case at $t=0$. Next, we show that provided \eqref{item_alg_gcd_l2} and \ref{item_alg_gcd_linf} hold for $0,1,2,\dots,t$, \eqref{item_alg_gcd_l2} and \ref{item_alg_gcd_linf} also hold for the iterate in the $(t+1)$th step. 

For $s\in[0,1]$, define $\bZ_t(s):=\bZ^*+s\bE_t$.
Because $\bZ_t(s)$ is already aligned with $\bZ^*$,
\begin{equation*}
\disttwo(\bZ_t(s),\bZ^*)\le s\|\bE_t\|_{\Fn}, \qquad \distinf(\bZ_t(s),\bZ^*)\le s\|\bE_t\|_{\twinf}.
\end{equation*}
With the induction hypothesis, we know
\begin{equation}
    n^{-1/2}\disttwo(\bZ_t(s),\bZ^*) \le n^{-1/2}\|\bE_t\|_{\Fn}\le \phi_n\frac{\|\bZ^*\|_{\Fn}}{\sqrt{ n}},\label{eq_l2_contract_to_Z_F}
\end{equation}
and 
\begin{equation}\begin{aligned}
    \distinf(\bZ_t(s),\bZ^*) \le \|\bE_t\|_{\twinf}\le\psi_n \|\bZ^*\|_{\twinf}
\end{aligned}.\label{eq_linft_contract_to_Z_infty}\end{equation}
Hence, by the theorem assumptions $\phi_n\le \epsilon/2$ and, for the $\ell_\infty$ part, $\psi_n\le \epsilon/2$, the bounds \eqref{eq_sym_init_Z}, \eqref{eq_sym_init_Z_uniform}, \eqref{eq_l2_contract_to_Z_F}, and \eqref{eq_linft_contract_to_Z_infty} imply that $\bZ_t(s)\in\cDtwo$ for all $s\in[0,1]$ in Step~1, and $\bZ_t(s)\in\cDinf$ for all $s\in[0,1]$ in Step~2.

Let $\bE_{t+1}^{(t)} = \bZ^{t+1}\bR_t^* - \bZ^*$.
Recall that in \eqref{alg_aligned_gcd}, we have established $\bZ^{t+1}\bR_t^* =\bZ^t\bR_t^* - \frac{\eta^t}{n}\nabla_{\bZ}h_{\lambda}^*(\bZ^t\bR_t^*)$. Since $\cG(\bZ^*)=\zero$, $\nabla_{\bZ}p_{\alpha}^*(\tilde\bZ^t)=\zero$ by the alignment property, and $\nabla_{\bZ}p_{\alpha}^*(\bZ^*)=\zero$ because $\bM^*(\bZ^*)=0$, we have 
\begin{equation*}
    \bE_{t+1}^{(t)} = \bE_t - \frac{\eta}{n}\big\{\nabla_{\bZ}h_{\alpha}^*(\tilde\bZ^t) - \nabla_{\bZ}h_{\alpha}^*(\bZ^*)\big\} .
\end{equation*}

We now split the proof into the $\ell_2$- and $\ell_\infty$-parts.

\medskip
\paragraph{\it \underline{Step 1: $\ell_2$-error contraction.}}
We prove \eqref{item_alg_gcd_l2} for the $(t+1)$th iterate. Since $\disttwo(\bZ^{t+1},\bZ^*)\le \|\bZ^{t+1}\bR_t^*-\bZ^*\|_{\Fn} = \|\bE_{t+1}^{(t)}\|_{\Fn}$,
it suffices to bound $\|\bE_{t+1}^{(t)}\|_{\Fn}$.

For $s\in[0,1]$, define \begin{equation*}
\cA_t(s):= n^{-1}\Big\{\nabla_{\bz}^2 h_{\alpha}^*(\bZ_t(s)) - 2\cG(\bZ_t(s))\otimes \bI_r\Big\}, \qquad \bar{\cA}_t:=\int_0^1 \cA_t(s)\,ds,
\end{equation*}
and
\begin{equation*}
\bar{\cG}_t:=\frac{2}{n}\int_0^1 \cG(\bZ_t(s))\,ds.
\end{equation*}
By the fundamental theorem of calculus (Theorem 4.2 in \citet{lang2012real}, Chapter XIII),
\begin{equation*}
\mathrm{vec}\!\Big(\big\{\nabla_{\bZ}h_{\alpha}^*(\tilde\bZ^t)-\nabla_{\bZ}h_{\alpha}^*(\bZ^*)\big\}^\T\Big) = \Big(\int_0^1 \nabla_{\bz}^2 h_{\alpha}^*(\bZ_t(s))\,ds\Big)\mathrm{vec}(\bE_t^\T).
\end{equation*}
Using $\nabla_{\bZ}\cL(\bZ\bZ^\T)=2\cG(\bZ)\bZ$, we therefore obtain
\begin{equation*}
\mathrm{vec}\!\big((\bE_{t+1}^{(t)})^\T\big) = (\bI_{nr}-\eta\bar{\cA}_t)\mathrm{vec}(\bE_t^\T) - \eta(\bar{\cG}_t\otimes \bI_r)\mathrm{vec}(\bE_t^\T).
\end{equation*}
Since $(\bB\otimes \bA^\T)\mathrm{vec}(\bX)=\mathrm{vec}(\bA\bX\bB)$, the second term equals $\mathrm{vec}(\bE_t^\T\bar{\cG}_t)$. Hence
\begin{equation*}
\|\bE_{t+1}^{(t)}\|_{\Fn} \le \big\|(\bI_{nr}-\eta\bar{\cA}_t)\mathrm{vec}(\bE_t^\T)\big\| + \eta\|\bE_t^\T\bar{\cG}_t\|_{\Fn} =: \gamma_{1,t}+\gamma_{2,t}.
\end{equation*}

We first bound $\gamma_{1,t}$. Since $\sigma_r(n^{-1/2}\bZ^*)^2={\sigma_r(\bX^*)}/{n}=\sigma_{\min}$, apply Lemma~\ref{lemma_general_convex} to get, for every $s\in[0,1]$,
\begin{equation*}
\lambda_{\min}\big(\cA_t(s)\big) \ge \alpha\sigma_{\min} - 4n^{-1}\|\bZ_t(s)-\bZ^*\|_{\Fn}^2 \ge \alpha\sigma_{\min} - 4n^{-1}\|\bE_t\|_{\Fn}^2.
\end{equation*}
Since $\bar{\cA}_t=\int_0^1 \cA_t(s)\,ds$, we can further obtain
\begin{equation*}
\lambda_{\min}(\bar{\cA}_t) \ge \int_0^1\lambda_{\min}\big(\cA_t(s)\big)\,ds \ge \alpha\sigma_{\min} - 4n^{-1}\|\bE_t\|_{\Fn}^2.
\end{equation*}
By the induction hypothesis $\|\bE_t\|_{\Fn}\le \phi_n\|\bZ^*\|_{\Fn}$ and $\|\bZ^*\|_{\Fn}^2\le  r\sigma_1(\bX^*)=nr\kappa\sigma_{\min}$,
\begin{equation*}
n^{-1}\|\bE_t\|_{\Fn}^2 \le \phi_n^2\,n^{-1}\|\bZ^*\|_{\Fn}^2 \le \phi_n^2 r\kappa\sigma_{\min} \le c_0^2\frac{\alpha^2}{\kappa}\sigma_{\min} \le c_0^2\alpha\sigma_{\min},
\end{equation*}
where we used $\phi_n\le c_0\alpha/(\kappa\sqrt r)$, $0<\alpha<1$, and $\kappa\ge 1$.
Hence, after shrinking $c_0$ if necessary,
\begin{equation}\label{eq_lambda_min_bar_cA}
\lambda_{\min}(\bar{\cA}_t)\ge \frac78\alpha\sigma_{\min}.
\end{equation}

Next, we bound $\lambda_{\max}(\bar{\cA}_t)$. This upper is only needed to ensure that $\bI_{nr}-\eta\bar{\cA}_t$ remains a contraction. For any $\bW\in\RR^{n\times r}$ with $\|\bW\|_{\Fn}=1$, Assumption~\ref{assump_rsc} gives
\begin{align*}
&n^{-1}\nabla_{\bZ}^2\cL(\bZ_t(s)\bZ_t(s)^\T)[\bW,\bW] - 2n^{-1}\langle \cG(\bZ_t(s)),\bW\bW^\T\rangle \\ &\qquad\le \beta n^{-1}\|\cP_{\bZ_t(s)}(\bW)\|_{\Fn}^2 \le 4\beta n^{-1}\|\bZ_t(s)\|^2,
\end{align*}
where the last step uses $\|\cP_{\bZ}(\bW)\|_{\Fn}\le 2\|\bZ\|\,\|\bW\|_{\Fn}$.
Also, we note that
\begin{equation*}
\frac1n\|\bZ_t(s)\|^2 \le \frac2n\|\bZ^*\|^2+\frac2n\|\bE_t\|_{\Fn}^2 \le 2\kappa\sigma_{\min}+2c_0^2\kappa\sigma_{\min} \le \frac{17}{8}\kappa\sigma_{\min},
\end{equation*}
after shrinking $c_0$ so that $2c_0^2\le 1/8$, where we used
$\|\bZ^*\|^2=\sigma_1(\bX^*)=n\kappa\sigma_{\min}$.

For the penalty term, recall
\begin{equation*}
p_\alpha^*(\bZ)=\frac{\alpha n^2}{4}\|\bM^*(\bZ)\|_{\Fn}^2, \qquad \bM^*(\bZ)=n^{-1}\big\{(\bZ^*)^\T\bZ-\bZ^\T\bZ^*\big\}. \end{equation*}
Since $\bM^*(\cdot)$ is linear, it follows that
\begin{equation*}
D\bM^*(\bZ)[\bW] = n^{-1}\big\{(\bZ^*)^\T\bW-\bW^\T\bZ^*\big\}, \qquad D^2\bM^*(\bZ)=0.
\end{equation*}
This hence gives
\begin{align*}
n^{-1}\nabla_{\bZ}^2p_\alpha^*(\bZ_t(s))[\bW,\bW] &= \frac{\alpha n}{2}\|D\bM^*(\bZ_t(s))[\bW]\|_{\Fn}^2 \\ &= \frac{\alpha}{2n}\|(\bZ^*)^\T\bW-\bW^\T\bZ^*\|_{\Fn}^2 \\ &\le 2\alpha n^{-1}\|\bZ^*\|^2\|\bW\|_{\Fn}^2 = 2\alpha\kappa\sigma_{\min}.
\end{align*}
Combining the bounds for the loss part and the penalty part yields
\begin{equation*}
\lambda_{\max}\big(\cA_t(s)\big) \le 4\beta\cdot\frac{17}{8}\kappa\sigma_{\min} + 2\alpha\kappa\sigma_{\min} \le 9(\alpha+\beta)\kappa\sigma_{\min},
\end{equation*}
for all $s\in[0,1]$. Hence the same bound holds for $\bar{\cA}_t$ that $\lambda_{\max}(\bar{\cA}_t)\le 9(\alpha+\beta)\kappa\sigma_{\min}$.
Since $\eta\le \{10(\alpha+\beta)\kappa\sigma_{\min}\}^{-1}$, we have $0\le \eta\lambda_{\max}(\bar{\cA}_t)\le \tfrac{9}{10}$.
Thus all eigenvalues of $\bI_{nr}-\eta\bar{\cA}_t$ lie in $[0,1)$, and therefore
$\big\|(\bI_{nr}-\eta\bar{\cA}_t)\mathrm{vec}(\bE_t^\T)\big\| \le \bigl(1-\eta\lambda_{\min}(\bar{\cA}_t)\bigr)\|\bE_t\|_{\Fn}$.
Using \eqref{eq_lambda_min_bar_cA}, we conclude
\begin{equation}
\gamma_{1,t}
\le
\left(1-\frac78\eta\alpha\sigma_{\min}\right)\|\bE_t\|_{\Fn}.
\label{eq:l2-linear}
\end{equation}

We next bound $\gamma_{2,t}$. By the induction hypothesis,
\begin{equation*}
\disttwo(\bZ^t,\bZ^*)=\|\bE_t\|_{\Fn} \le \phi_n\|\bZ^*\|_{\Fn} \le \phi_n\sqrt r\,\|\bZ^*\| \le c_0\|\bZ^*\|,
\end{equation*}
where the last step uses $\phi_n\le c_0\alpha/(\kappa\sqrt r)$ and $\alpha/\kappa\le 1$.
Since $\disttwo(\bZ_t(s),\bZ^*)\le \|\bE_t\|_{\Fn}$, Lemma~\ref{lemma_zz_to_z} implies
\begin{equation*}
\|\bZ_t(s)\bZ_t(s)^\T-\bZ^*(\bZ^*)^\T\|_{\Fn}
\le
(2+c_0)\|\bZ^*\|\,\disttwo(\bZ_t(s),\bZ^*)
\le
C\|\bZ^*\|\,\|\bE_t\|_{\Fn},
\end{equation*}
uniformly over $s\in[0,1]$. Therefore, using $\cG(\bZ^*)=\zero$ and Assumption~\ref{assump_gradient_lips_1} within $\cDinf$,
\begin{align*}
\|\bar{\cG}_t\| &\le \frac{2}{n}\int_0^1 \|\cG(\bZ_t(s))-\cG(\bZ^*)\|\,ds \\ &\le \frac{2L_2}{n}\max_{s\in[0,1]} \|\bZ_t(s)\bZ_t(s)^\T-\bZ^*(\bZ^*)^\T\|_{\Fn} \\
&\le CL_2\frac{\|\bZ^*\|}{n}\|\bE_t\|_{\Fn},
\end{align*}where we use $\|\cdot\|\le \|\cdot\|_{\Fn}$.
Consequently, we have $\gamma_{2,t}\le \eta\|\bar{\cG}_t\|\,\|\bE_t\|_{\Fn} \le CL_2\eta\frac{\|\bZ^*\|}{n}\|\bE_t\|_{\Fn}^2$.
We now check that
\begin{align}
\frac{\|\bZ^*\|}{n}\|\bE_t\|_{\Fn}\le \frac{\|\bZ^*\|}{\sqrt n}\cdot \phi_n\frac{\|\bZ^*\|_{\Fn}}{\sqrt n}  \le \phi_n\sqrt r\,\kappa\sigma_{\min} \le c_0\alpha\sigma_{\min},\label{eq_l2_contract_to_Z_F2}
\end{align}
where we used ${\|\bZ^*\|}/{\sqrt n}=\sqrt{\kappa\sigma_{\min}}$ and $
{\|\bZ^*\|_{\Fn}}/{\sqrt n}\le \sqrt{r\kappa\sigma_{\min}}$. Hence, this establishes that
\begin{equation*}
\gamma_{2,t} \le
CL_2c_0\,\eta\alpha\sigma_{\min}\|\bE_t\|_{\Fn}.
\end{equation*}
After shrinking $c_0$ further so that $CL_2c_0\le 1/2$, we obtain
\begin{equation}
\gamma_{2,t}\le \frac{\eta\alpha\sigma_{\min}}{2}\|\bE_t\|_{\Fn}. \label{eq_l2_score}
\end{equation}

Combining \eqref{eq:l2-linear} and \eqref{eq_l2_score}, we arrive at
\begin{equation*}
\disttwo(\bZ^{t+1},\bZ^*) \le \|\bE_{t+1}^{(t)}\|_{\Fn} \le \left(1-\frac38\eta\alpha\sigma_{\min}\right)\|\bE_t\|_{\Fn} \le \rho\,\|\bE_t\|_{\Fn},
\end{equation*}
where $\rho=1-\eta\alpha\sigma_{\min}/4$. Since $\|\bE_t\|_{\Fn}=\disttwo(\bZ^t,\bZ^*)$, iterating the recursion and using $\disttwo(\bZ^0,\bZ^*)\le \phi_n\|\bZ^*\|_{\Fn}$ yields
\begin{equation*}
\disttwo(\bZ^{t+1},\bZ^*) \le \rho^{t+1}\phi_n\|\bZ^*\|_{\Fn}.
\end{equation*}
This proves $\ell_2$ error contraction in \eqref{item_alg_gcd_l2} for the $(t+1)$th iterate under local region $\cDinf$.
\medskip
\paragraph{\it \underline{Step 2: $\ell_\infty$-error contraction.}}
We first control the alignment error $\bE_{t+1}^{(t)}=\bZ^{t+1}\bR_t^*-\bZ^*$, and then transfer the bound to the next optimal alignment $\bR_{t+1}^*$. As in Step 1, the whole interpolation segment $\{\bZ_t(s):0\le s\le 1\}$ lies in $\cDinf$.

We decompose the rowwise error into the term arising from alignment and the drift of the optimal alignment as follows
\begin{equation}\label{eq_linf_decomps}
\begin{aligned}
\distinf(\bZ^{t+1},\bZ^*) &= \|\bZ^{t+1}\bR_{t+1}^*-\bZ^*\|_{\twinf} \\ &\le \|\bZ^{t+1}\bR_t^*-\bZ^*\|_{\twinf} + \|\bZ^{t+1}\bR_{t+1}^*-\bZ^{t+1}\bR_t^*\|_{\twinf} \\
&\le \|\bE_{t+1}^{(t)}\|_{\twinf} + \|\bZ^{t+1}\bR_t^*\|_{\twinf} \big\|\bI_r-(\bR_t^*)^{-1}\bR_{t+1}^*\big\|. \end{aligned}
\end{equation}

We first bound $\|\bE_{t+1}^{(t)}\|_{\twinf}$.
Recall that $\tilde\bZ^t=\bZ^t\bR_t^*$, so by Lemma~\ref{lemma_ortho_align}, $(\bZ^*)^\T\tilde\bZ^t=\tilde\bZ^{t\T}\bZ^*$. Equivalently, we know 
$
(\bZ^*)^\T\bE_t=\bE_t^\T\bZ^*$.
Therefore, for every $s\in[0,1]$,
\begin{equation*}
\bM^*(\bZ_t(s)) = n^{-1}\big\{(\bZ^*)^\T\bZ_t(s)-\bZ_t(s)^\T\bZ^*\big\} = sn^{-1}\big\{(\bZ^*)^\T\bE_t-\bE_t^\T\bZ^*\big\} =0.
\end{equation*}
It follows that
\begin{equation*}
\nabla_{\bZ}p_\alpha^*(\bZ_t(s))=0, \qquad \frac{d}{ds}\nabla_{\bZ}p_\alpha^*(\bZ_t(s)) = \nabla_{\bZ}^2p_\alpha^*(\bZ_t(s))[\bE_t] =0,
\end{equation*}
for all $s\in[0,1]$. Thus, the penalty contributes nothing to the directional Hessian action along $\bE_t$.

For $i,j\in[n]$, we define
\begin{equation*}
\cH_{ij}(\bZ):= n^{-1}\big\{\nabla_{\bz_i\bz_j}^2\cL(\bZ\bZ^\T)-2G_{ij}(\bZ)\bI_r\big\}, \qquad
\bar{\cH}_{ij,t}:=\int_0^1 \cH_{ij}(\bZ_t(s))\,ds.
\end{equation*}
Using the same fundamental theorem of calculus expansion as in Step 1, together with
$\bar{\cG}_t=\frac{2}{n}\int_0^1\cG(\bZ_t(s))\,ds$, we obtain for each $i\in[n]$,
\begin{equation}\label{eq_linf_row_decomp}
(\bE_{t+1}^{(t)})_{i,\cdot}^\T = \big(\bI_r-\eta\bar{\cH}_{ii,t}\big)(\bE_t)_{i,\cdot}^\T - \eta\sum_{j\neq i}\bar{\cH}_{ij,t}(\bE_t)_{j,\cdot}^\T - \eta(\bar{\cG}_t\bE_t)_{i,\cdot}^\T.
\end{equation}
Consequently, we have the following decomposition
\begin{equation*}
\|\bE_{t+1}^{(t)}\|_{\twinf}\le \delta_{1,t}+\delta_{2,t}+\delta_{3,t},
\end{equation*}
where $\delta_{1,t}$, $\delta_{2,t}$, and $\delta_{3,t}$ denote the maxima over $i\in[n]$ of the three terms on the right-hand side of \eqref{eq_linf_row_decomp}.

We first bound $\delta_{1,t}$.
By Lemma~\ref{coro_thm_general_convex}, for every $s\in[0,1]$ and every $i\in[n]$,
\begin{equation*}
\lambda_{\min}\big(\cH_{ii}(\bZ_t(s))\big) \ge \alpha\sigma_{\min}-2n^{-1}\|\bE_t\|_{\Fn}^2,
\end{equation*} and
\begin{equation*}
\lambda_{\max}\big(\cH_{ii}(\bZ_t(s))\big) \le \beta\kappa\sigma_{\min}+2n^{-1}\|\bE_t\|_{\Fn}^2.
\end{equation*}
As in Step 1, we have by the induction hypothesis that
\begin{equation*}
n^{-1}\|\bE_t\|_{\Fn}^2 \le \phi_n^2 r\kappa\sigma_{\min} \le c_0^2\alpha\sigma_{\min},
\end{equation*}
where we used scaling condition $\phi_n\le c_0\alpha/(\kappa\sqrt r)$. Hence, after shrinking $c_0$ if necessary, we get
\begin{equation*}
\lambda_{\min}(\bar{\cH}_{ii,t})\ge \frac78\alpha\sigma_{\min}, \qquad\lambda_{\max}(\bar{\cH}_{ii,t})\le 1.01(\alpha+\beta)\kappa\sigma_{\min},
\end{equation*}
uniformly over $i\in[n]$. Since $\eta\le \{10(\alpha+\beta)\kappa\sigma_{\min}\}^{-1}$, we have
\begin{equation*}
0\le \eta\lambda_{\max}(\bar{\cH}_{ii,t})\le \frac15.
\end{equation*}
Therefore all eigenvalues of $\bI_r-\eta\bar{\cH}_{ii,t}$ lie in $[0,1)$, and subsequently
\begin{equation*}
\|\bI_r-\eta\bar{\cH}_{ii,t}\| = 1-\eta\lambda_{\min}(\bar{\cH}_{ii,t}) \le 1-\frac78\eta\alpha\sigma_{\min}.
\end{equation*}
Taking the maximum over $i$ yields
\begin{equation}
\delta_{1,t}\le \Big(1-\frac78\eta\alpha\sigma_{\min}\Big)\|\bE_t\|_{\twinf}. \label{eq_final_delta_1_t}
\end{equation}

Next, we bound $\delta_{2,t}$. Fix $i\in[n]$, vectors $\{\bu_j\}_{j\neq i}\subset\RR^r$, and $\ba\in\RR^r$ with $\|\ba\|=1$. Define $\bW_1:=\be_i\ba^\T$, $(\bW_2)_{j,\cdot}:=\bu_j^\T\ \ (j\neq i)$, and $(\bW_2)_{i,\cdot}:=0$.
Then we note that
\begin{equation*}
\nabla_{\bZ}^2\cL(\bZ\bZ^\T)[\bW_1,\bW_2] = \nabla_{\bX}^2\cL(\bZ\bZ^\T)\big[\cP_{\bZ}(\bW_1),\cP_{\bZ}(\bW_2)\big] + 2\langle \cG(\bZ),\bW_1\bW_2^\T\rangle.
\end{equation*}
Since $\bW_1$ is supported on row $i$ and $(\bW_2)_{i,\cdot}=0$, we have $\langle \cG(\bZ),\bW_1\bW_2^\T\rangle = \sum_{j\neq i}G_{ij}(\bZ)\,\ba^\T\bu_j$, and therefore
\begin{equation*}
\sum_{j\neq i}\ba^\T\cH_{ij}(\bZ)\bu_j = n^{-1}\nabla_{\bX}^2\cL(\bZ\bZ^\T)\big[\cP_{\bZ}(\bW_1),\cP_{\bZ}(\bW_2)\big].
\end{equation*}
Assumption~\ref{assump_row_cross_curvature_sym} now gives
\begin{equation*}
\Big| \sum_{j\neq i}\ba^\T\cH_{ij}(\bZ)\bu_j \Big| \le \beta n^{-1}\|\bZ_{i,\cdot}\|\,\|\bZ\|\,\|\bW_2\|_{\Fn} = \beta n^{-1}\|\bZ_{i,\cdot}\|\,\|\bZ\| \Big(\sum_{j\neq i}\|\bu_j\|^2\Big)^{1/2}.
\end{equation*}
Taking the supremum over $\|\ba\|=1$, we arrive at
\begin{equation*} \Big\| \sum_{j\neq i}\cH_{ij}(\bZ)\bu_j \Big\| \le \beta n^{-1}\|\bZ_{i,\cdot}\|\,\|\bZ\| \Big(\sum_{j\neq i}\|\bu_j\|^2\Big)^{1/2}.
\end{equation*}
Applying this with $\bu_j=(\bE_t)_{j,\cdot}^\T$, integrating over $s\in[0,1]$, and taking the maximum over $i$, we obtain
\begin{equation*}
\delta_{2,t} \le \eta\beta n^{-1}\max_{s\in[0,1]}\|\bZ_t(s)\|_{\twinf}\|\bZ_t(s)\|\,\|\bE_t\|_{\Fn}.
\end{equation*}
Now, by the induction on $\ell_2$ error contraction $ {\|\bE_t\|_{\Fn}}/{\|\bZ^*\|} \le \phi_n{\|\bZ^*\|_{\Fn}}/{\|\bZ^*\|}
\le \phi_n\sqrt r \le c_0{\alpha}/{\kappa} \le c_0$, it then holds that $\|\bZ_t(s)\| \le \|\bZ^*\|+\|\bE_t\|_{\Fn} \le (1+c_0)\|\bZ^*\|$, and, by the induction hypothesis $\|\bE_t\|_{\twinf}\le \psi_n\|\bZ^*\|_{\twinf}\le \epsilon/2\;\|\bZ^*\|_{\twinf}$,
\begin{equation*}
\|\bZ_t(s)\|_{\twinf} \le \|\bZ^*\|_{\twinf}+\|\bE_t\|_{\twinf} \le \big(1+\frac{\epsilon}2\big)\|\bZ^*\|_{\twinf}.
\end{equation*}
Using $\|\bZ^*\|^2=\sigma_1(\bX^*)=n\kappa\sigma_{\min}$, we conclude that
\begin{equation}
\delta_{2,t}\le C\eta\beta\sqrt{\frac{\kappa\sigma_{\min}}{n}}\|\bZ^*\|_{\twinf}\|\bE_t\|_{\Fn}.
\label{eq_final_delta_2_t}
\end{equation}

Finally, we bound $\delta_{3,t}$. Using $\cG(\bZ^*)=\zero$, Assumption~\ref{assump_gradient_lips}, and $\|\bA\bB\|_{\twinf}\le \|\bA\|_{\twinf}\|\bB\|$, we have
\begin{align*}
\delta_{3,t} &\le \eta\|\bar{\cG}_t\bE_t\|_{\twinf} \le \eta\|\bar{\cG}_t\|_{\twinf}\|\bE_t\|\\
&\le \frac{2\eta L_\infty}{n} \max_{s\in[0,1]} \|\bZ_t(s)\bZ_t(s)^\T-\bZ^*(\bZ^*)^\T\|_{\twinf}\,\|\bE_t\|_{\Fn}.
\end{align*}
Expanding the matrix difference by telescoping as
\begin{equation*}
\bZ_t(s)\bZ_t(s)^\T-\bZ^*(\bZ^*)^\T = s(\bZ^*\bE_t^\T+\bE_t\bZ^{*\T})+s^2\bE_t\bE_t^\T,
\end{equation*}
we obtain
\begin{align*}
\|\bZ_t(s)\bZ_t(s)^\T-\bZ^*(\bZ^*)^\T\|_{\twinf} &\le \|\bZ^*\bE_t^\T\|_{\twinf} + \|\bE_t\bZ^{*\T}\|_{\twinf} + \|\bE_t\bE_t^\T\|_{\twinf}
\\ &\le \|\bZ^*\|_{\twinf}\|\bE_t\|_{\Fn} + \|\bE_t\|_{\twinf}\|\bZ^*\| + \|\bE_t\|_{\twinf}\|\bE_t\|_{\Fn} \\ &\le \|\bZ^*\|_{\twinf}\|\bE_t\|_{\Fn} + (1+c_0)\|\bZ^*\|\|\bE_t\|_{\twinf},
\end{align*}
where in the last step we used $\|\bE_t\|_{\Fn}\le c_0\|\bZ^*\|$. Plugging this into the previous display gives
\begin{equation*}
\delta_{3,t} \le \frac{2\eta L_\infty}{n} \Big\{ \|\bZ^*\|_{\twinf}\|\bE_t\|_{\Fn}^2 + C\|\bZ^*\|\|\bE_t\|_{\twinf}\|\bE_t\|_{\Fn} \Big\}.
\end{equation*}
We next bound the two terms separately.
Since ${\|\bE_t\|_{\Fn}}/{\|\bZ^*\|} \le \phi_n{\|\bZ^*\|_{\Fn}}/{\|\bZ^*\|} \le \phi_n\sqrt r \le c_0{\alpha}/{\kappa}$, one can obtain
\begin{align*}
\frac{1}{n}\|\bZ^*\|_{\twinf}\|\bE_t\|_{\Fn}^2= \sqrt{\frac{\sigma_{\min}}{n\kappa}}\, \|\bZ^*\|_{\twinf}\|\bE_t\|_{\Fn} \times \frac{\kappa\|\bE_t\|_{\Fn}}{\|\bZ^*\|} \le c_0\alpha \sqrt{\frac{\sigma_{\min}}{n\kappa}}\, \|\bZ^*\|_{\twinf}\|\bE_t\|_{\Fn},
\end{align*} where we used $\|\bZ^*\|=\sqrt{n\kappa\sigma_{\min}}$. Similarly,
\begin{align*}
\frac{1}{n}\|\bZ^*\|\,\|\bE_t\|_{\twinf}\|\bE_t\|_{\Fn}= \alpha\sigma_{\min}\|\bE_t\|_{\twinf}\times
\frac{\kappa\|\bE_t\|_{\Fn}}{\alpha\|\bZ^*\|}\le c_0\alpha\sigma_{\min}\|\bE_t\|_{\twinf}.
\end{align*}
Hence, one has
\begin{equation*}
\delta_{3,t} \le CL_\infty c_0\,\eta\alpha\sigma_{\min}\|\bE_t\|_{\twinf} + C\eta\alpha \sqrt{\frac{\sigma_{\min}}{n\kappa}}\, \|\bZ^*\|_{\twinf}\|\bE_t\|_{\Fn}.
\end{equation*}
After shrinking $c_0$ further so that $CL_\infty c_0\le 1/16$, we arrive at
\begin{equation}
\delta_{3,t} \le \frac1{16}\eta\alpha\sigma_{\min}\|\bE_t\|_{\twinf} + C\eta\alpha \sqrt{\frac{\sigma_{\min}}{n\kappa}}\, \|\bZ^*\|_{\twinf}\|\bE_t\|_{\Fn}. \label{eq_final_delta_3_t}
\end{equation}

Substituting the bounds for $\delta_{1,t}$, $\delta_{2,t}$, and $\delta_{3,t}$ obtain in \eqref{eq_final_delta_1_t}, \eqref{eq_final_delta_2_t}, and \eqref{eq_final_delta_3_t} into \eqref{eq_linf_row_decomp}, and using $\alpha\le \beta$ and $\kappa\ge 1$, we get
\begin{equation*}
\|\bE_{t+1}^{(t)}\|_{\twinf} \le \Big(1-\frac{13}{16}\eta\alpha\sigma_{\min}\Big)\|\bE_t\|_{\twinf} + C\eta\beta\sqrt{\frac{\kappa\sigma_{\min}}{n}}\|\bZ^*\|_{\twinf}\|\bE_t\|_{\Fn}.
\end{equation*}
By the error contraction established in Step 1, we know $\|\bE_t\|_{\Fn}\le \rho^t\phi_n\|\bZ^*\|_{\Fn} \le \rho^t\phi_n\sqrt{nr\kappa\sigma_{\min}}$.
Therefore,
\begin{equation*}
\beta\sqrt{\frac{\kappa\sigma_{\min}}{n}}\|\bZ^*\|_{\twinf}\|\bE_t\|_{\Fn} \le \beta\kappa\sqrt r\,\sigma_{\min}\,\rho^t\phi_n\|\bZ^*\|_{\twinf}.
\end{equation*}
Using the theorem scaling condition $\tfrac{\beta}{\alpha}\kappa^{3/2}\sqrt r\,\tfrac{\phi_n}{\psi_n}\le c_0$,
and shrinking $c_0$ further so that $Cc_0\le 1/16$, we obtain $ C\beta\kappa\sqrt r\,\sigma_{\min}\,\phi_n \le \alpha\sigma_{\min}\psi_n/16$.
We therefore get
\begin{equation*}
\|\bE_{t+1}^{(t)}\|_{\twinf} \le \Big(1-\frac{13}{16}\eta\alpha\sigma_{\min}\Big)\|\bE_t\|_{\twinf} + \frac1{16}\eta\alpha\sigma_{\min}\rho^t\psi_n\|\bZ^*\|_{\twinf}.
\end{equation*}
Invoking the induction hypothesis $\|\bE_t\|_{\twinf}\le \rho^t\psi_n\|\bZ^*\|_{\twinf}$, we arrive at
\begin{equation}\label{eq_e_t_bar}
\|\bE_{t+1}^{(t)}\|_{\twinf} \le \Big(1-\frac34\eta\alpha\sigma_{\min}\Big)\rho^t\psi_n\|\bZ^*\|_{\twinf} \le (3\rho - 2)\rho^{t}\psi_n\|\bZ^*\|_{\twinf}.
\end{equation}
where the last inequality follows from $\rho = 1-\tfrac14\eta\alpha\sigma_{\min}$.

It remains to transfer this bound from the alignment $\bR_t^*$ to the next optimal alignment $\bR_{t+1}^*$. By orthogonal invariance, $(\bR_t^*)^{-1}\bR_{t+1}^*$ is the solution to
\begin{equation*}
\argmin_{\bR\in\cO^r}\|\bZ^{t+1}\bR_t^*\bR-\bZ^*\|_{\Fn}.
\end{equation*}

Now that
\begin{equation*}
\frac{\|\bE_{t+1}^{(t)}\|_{\Fn}}{\sigma_r(\bZ^*)} \le \rho^{t+1}\phi_n\frac{\|\bZ^*\|_{\Fn}}{\sqrt{n\sigma_{\min}}} \le \rho^{t+1}\phi_n\sqrt{r\kappa},
\end{equation*}
the theorem scaling condition $\tfrac{\beta}{\alpha}\kappa^{3/2}\sqrt r\,\tfrac{\phi_n}{\psi_n}\le c_0$ and $\kappa\ge 1$ imply $\sqrt {r\,\kappa}\,\tfrac{\phi_n}{\psi_n}\le c_0$, which further yields
\begin{equation*}\frac{\|\bE_{t+1}^{(t)}\|_{\Fn}}{\sigma_r(\bZ^*)} \le c_0\,\rho^{t+1}\psi_n. \end{equation*}
In particular, after shrinking $c_0$ if necessary, we have
$\|\bE_{t+1}^{(t)}\|_{\Fn}\le 0.01\,\sigma_r(\bZ^*)$, so Lemma~\ref{lemma_procruste_problem_perturb} applies and yields
\begin{align*}
\big\|(\bR_t^*)^{-1}\bR_{t+1}^*-\bI_r\big\|_{\Fn} &\le \frac{2}{\sigma_r(\bZ^{t+1}\bR_t^*)+\sigma_r(\bZ^*)}\|\bE_{t+1}^{(t)}\|_{\Fn} \\
&\le \frac{1.01}{\sigma_r(\bZ^*)}\|\bE_{t+1}^{(t)}\|_{\Fn},
\end{align*}
where the last step uses Weyl's inequality and
$\|\bE_{t+1}^{(t)}\|_{\Fn}\le 0.01\,\sigma_r(\bZ^*)$.

Returning to \eqref{eq_linf_decomps}, we obtain
\begin{align*}
\distinf(\bZ^{t+1},\bZ^*) &\le \|\bE_{t+1}^{(t)}\|_{\twinf} + \|\bZ^{t+1}\bR_t^*\|_{\twinf}\big\|(\bR_t^*)^{-1}\bR_{t+1}^*-\bI_r\big\|_{\Fn}\\
&\le \|\bE_{t+1}^{(t)}\|_{\twinf} + C\frac{\|\bE_{t+1}^{(t)}\|_{\Fn}}{\sigma_r(\bZ^*)} \Big( \|\bZ^*\|_{\twinf}+\|\bE_{t+1}^{(t)}\|_{\twinf} \Big).
\end{align*}
Using the bounds already proved, we get ${\|\bE_{t+1}^{(t)}\|_{\Fn}}/{\sigma_r(\bZ^*)} \le \sqrt{r\kappa}\phi_n\rho^{t+1}\le \sqrt{r\kappa}\phi_n$.
Therefore, one can obtain 
\begin{equation}\label{eq_e_t_bar_1}\begin{aligned}
    \distinf(\bZ^{t+1},\bZ^*) \le &\,\big(1+C\sqrt{r\kappa}\phi_n \big)\|\bE_{t+1}^{(t)}\|_{\twinf} + C\sqrt{r\kappa}\phi_n\rho^{t+1}\|\bZ^*\|_{\twinf}\\
    \le &\,\big(1+C\sqrt{r\kappa}\phi_n \big)(3\rho - 2)\rho^{t}\psi_n\|\bZ^*\|_{\twinf} + C\sqrt{r\kappa}\phi_n\rho^{t+1}\|\bZ^*\|_{\twinf}\\\le &\,\rho^{t+1}\psi_n\|\bZ^*\|_{\twinf}.
\end{aligned}\end{equation}
where the second inequality is due to $\|\bE_{t+1}^{(t)}\|_{\twinf}\le (3-2\rho)\rho^{t}\psi_n\|\bZ^*\|_{\twinf}$ from \eqref{eq_e_t_bar}. For the third inequality, we note that from $\tfrac{\beta}{\alpha}\kappa^{3/2}\sqrt{r}\tfrac{\phi_n}{\psi_n}\le c_0$, $\eta = \{10(\alpha + \beta)\kappa\sigma_{\min}\}^{-1}$, and $\psi_n\le \epsilon/2\le 1$, after shrinking $c_0$ if necessary, one has $C\sqrt{r\kappa}\phi_n\le \eta\alpha\sigma_{\min}/4 = 1-\rho$, which implies
\begin{equation*}
    \big(1+C\sqrt{r\kappa}\phi_n \big)(2 - \rho)\rho^{t}\psi_n\|\bZ^*\|_{\twinf}\le \big(1 + 2-2\rho)(3-2\rho)\rho^{t}\psi_n\|\bZ^*\|_{\twinf}\le( 2\rho -1)\rho^{t}\psi_n\|\bZ^*\|_{\twinf}.
\end{equation*}
This also allows us to bound the second term by $(1-\rho)\rho^t\|\bZ^*\|_{\twinf}$. Taking together we know \eqref{eq_e_t_bar_1} holds, which thus proves \ref{item_alg_gcd_linf} for the $(t+1)$th iterate and closes the induction.

\subsection{Proof of Theorem~\ref{thm_general_theory_noisy_Z}}\label{supp_prove_thm_general_theory_noisy_Z}

We prove in this subsection the stronger contraction statement under the localization $\cDinf$. Accordingly, throughout the proof we invoke Assumptions~\ref{assump_rsc},~\ref{assump_gradient_lips_1},~\ref{assump_row_cross_curvature_sym}, and~\ref{assump_gradient_lips} with $\cD=\cDinf$. The $\ell_2$-only part of Theorem~\ref{thm_general_theory_noisy_Z}, where only Assumptions~\ref{assump_rsc}--\ref{assump_gradient_lips_1} are imposed with $\cD=\cDtwo$, is proved later in Section~\ref{supp_sec_prove_thm_general_theory_noisy_Z_l2} by the same argument after removing the row-wise estimates.
Throughout the proof, we use the population first-order condition
\begin{equation*}
\bar\cG(\bZ^*)=\EE\,\cG(\bZ^*)=\zero.
\end{equation*}


In what follows, besides proving Theorem~\ref{thm_general_theory_noisy_Z}, we also establish error bounds for the local empirical minimizer that serves as the contraction target of the noisy gradient descent iterates. Specifically, define $\hat\bZ = \argmin_{\bZ\in\cDinf}\cL(\bZ\bZ^\T)$, and let $\hat\bR = \argmin_{\bR\in\cO^r}\|\hat\bZ\bR-\bZ^*\|_{\Fn}$. Under the assumptions and scaling conditions used to prove the \(\ell_\infty\) error contraction in Theorem~\ref{thm_general_theory_noisy_Z}, we show that the following holds for $\hat\bZ\hat\bR$ with probability at least $1-\delta$:
    \begin{equation}
        {\|\hat\bZ\hat\bR - \bZ^*\|_{\Fn}} \le C\frac{\Delta_2(n,\delta)}{\alpha\sigma_{\min}}\|\bZ^*\|_{\Fn}\quad\text{and}\quad{\|\hat\bZ\hat\bR - \bZ^*\|_{\twinf}}\le C\frac{\Delta_{\infty}(n,\delta)}{\alpha\sigma_{\min}}\|\bZ^*\|_{\twinf}.\label{eq_optimum_Z_error_bound}
    \end{equation}
    Moreover, Step~2 proves the first-order condition $\nabla_{\bZ}\cL(\hat\bZ\hat\bZ^\T) = \zero$. The results are then used in Step~3 in three ways: first, to define the effective initialization radii $\phi_n^\dagger,\psi_n^\dagger$; second, to transfer the conditioning and row-scale bounds from $\bZ^*$ to $\hat\bZ\hat\bR$; and third, to convert contraction toward $\tilde\bZ$ into the final error bounds relative to $\bZ^*$.

Fix the radius $\epsilon$ in the theorem, and then we can choose an auxiliary radius $\varepsilon>0$ such that 
\begin{equation}\label{eq_lower_bound_varepsilon}
\frac{1}{c_0}\Big( \sqrt{r\kappa}\frac{\Delta_2(n,\delta)}{\alpha\sigma_{\min}} \vee \frac{\Delta_\infty(n,\delta)}{\alpha\sigma_{\min}} \Big) \le \varepsilon \le \frac{c_0}{(\beta/\alpha+\sqrt r)\sqrt r\,\kappa}\,\epsilon,
\end{equation}where $c_{0}>0$ is sufficiently small.
Such a choice of $\varepsilon$ is possible because the first two noise
scaling conditions in Theorem~\ref{thm_general_theory_noisy_Z} imply that
the lower bound in \eqref{eq_lower_bound_varepsilon} is smaller than the
upper bound, after shrinking $c_0$ if necessary. With $\zeta_r := \big(\tfrac{\beta}{\alpha}+\sqrt{r}\big)\sqrt{r}\kappa$, we may also write
\begin{equation}\label{eq_thm_general_theory_noisy_Z_scaling}
\frac{\sqrt r\,\kappa}{\alpha}\varepsilon\le c_0,\quad \zeta_r\varepsilon\le c_0\epsilon, \quad \frac{\Delta_\infty(n,\delta)}{\alpha\sigma_{\min}} \le c_0\,\frac{\epsilon}{\zeta_r},\quad \frac{\Delta_2(n,\delta)}{\alpha\sigma_{\min}}\le c_0\frac{\epsilon\zeta_r}{\sqrt{r\kappa}},
\end{equation}for sufficiently small $c_0$.

We next introduce a smaller compact neighborhood $\cDde$ to construct a stationary point by minimizing the gradient norm on a compact set, and a slightly larger neighborhood $\cDd$ on which the local curvature and Lipschitz bounds will be invoked. In particular, define
\begin{equation}
\cDde = \Big\{ \bZ: \frac{\|\bZ-\bZ^*\|_{\Fn}}{\|\bZ^*\|_{\Fn}}\le \varepsilon, \ \  \frac{\|\bZ-\bZ^*\|_{\twinf}}{\|\bZ^*\|_{\twinf}}\le \epsilon \Big\}, \label{eq_define_bar_cD}
\end{equation} and 
\begin{equation}
\bar\cD_z^{(\infty)}(\epsilon) = \Big\{ \bZ: \|\bZ-\bZ^*\|_{\Fn}\le \epsilon\|\bZ^*\|_{\Fn}, \ \  \|\bZ-\bZ^*\|_{\twinf}\le \epsilon\|\bZ^*\|_{\twinf} \Big\}.
\label{eq_define_bar_cD_cD}
\end{equation}
Since $\varepsilon\le \epsilon$, we have
\begin{equation*}
\cDde\subseteq \cDd\subseteq \cDinf.
\end{equation*}

For $\hat\bZ\in\argmin_{\bZ\in\cDinf}\cL(\bZ\bZ^\T)$, by definition of $\hat\bR $, we know $\hat\bZ\hat\bR\in \cDd$. Our first goal is to construct a local reference point that can later be shown to satisfy the first-order condition and then equals $\hat\bZ\hat\bR$ with high probability. In particular, we define
\begin{equation*}
\tilde\bZ\in\argmin_{\bZ\in\cDde}\|\nabla_{\bZ}h_{\alpha}^*(\bZ)\|_{\twinf}^2.
\end{equation*}
Such a minimizer exists because $\cDde$ is compact and
$\bZ\mapsto \|\nabla_{\bZ}h_{\alpha}^*(\bZ)\|_{\twinf}^2$ is continuous.
Let
\begin{equation*}
\bE:=\tilde\bZ-\bZ^*.
\end{equation*}


The rest of the proof consists of three steps. The first two steps study the local reference point $\tilde\bZ$. Step~3 then proves that the gradient descent iterates contract toward this point. Once Step~2 establishes $\tilde\bZ=\hat\bZ\hat\bR$, the theorem follows by combining the contraction toward $\tilde\bZ$ with the statistical error bounds for $\tilde\bZ$. Specifically, the proof is organized as follows.
\begin{enumerate}[label=(\arabic*)]
    \item construct an interior stationary point $\tilde\bZ$ and establish the first-order optimality condition;
    \item establish the equivalence $\tilde\bZ=\hat\bZ\hat\bR$ and the bounds for $\tilde\bZ$, which proves \eqref{eq_optimum_Z_error_bound};
    \item prove contraction of $\bZ^t$ toward $\tilde\bZ$, or equivalently, towards $\hat\bZ\hat\bR$.
\end{enumerate}

\medskip
\paragraph{\it \underline{Step 1: Construct an interior stationary point $\tilde\bZ$ and establish first-order optimality.}}
Since $\bZ^*\in\cDde$ and $\nabla_{\bZ}p_\alpha^*(\bZ^*)=\zero$, the minimizing property of $\tilde\bZ$ implies
\begin{equation}\label{eq_st_score_unf}
\begin{aligned}
\|\nabla_{\bZ}h_\alpha^*(\tilde\bZ)\|_{\twinf} &\le \|\nabla_{\bZ}h_\alpha^*(\bZ^*)\|_{\twinf} = 2\|\cG(\bZ^*)\bZ^*\|_{\twinf} \\&= 2\|\tilde\cG(\bZ^*)\bZ^*\|_{\twinf} \le 2\sqrt n\,\Delta_\infty(n,\delta)\|\bZ^*\|, \end{aligned}
\end{equation}
where the last inequality follows from the first part of Equation~\eqref{assump_gradient_noise_2}.
Consequently,
\begin{equation}\label{eq_st_score_ave}
\|\nabla_{\bZ}h_\alpha^*(\tilde\bZ)\|_{\Fn} \le \sqrt n\,\|\nabla_{\bZ}h_\alpha^*(\tilde\bZ)\|_{\twinf} \le 2n\,\Delta_\infty(n,\delta)\|\bZ^*\|.
\end{equation}

For $s\in[0,1]$, let
\begin{equation*}
\bZ_s:=\bZ^*+s\bE.
\end{equation*}
Since $\tilde\bZ\in\cDde\subseteq \cDd$, the whole segment $\{\bZ_s:0\le s\le 1\}$ lies in $\cDd\subseteq \cDinf$.
Applying the integral mean value theorem to
$\mathrm{vec}\big(\{\nabla_{\bZ}h_\alpha^*(\cdot)\}^\T\big)$ along the segment from $\bZ^*$ to $\tilde\bZ$, we obtain
\begin{align}
\underbrace{\tilde\cH\;\mathrm{vec}(\bE^\T)}_{\gamma_1} +
\underbrace{\tilde\cG\;\mathrm{vec}(\bE^\T)}_{\gamma_2} +
\underbrace{\bar\cG\;\mathrm{vec}(\bE^\T)}_{\gamma_3} =
\underbrace{ n^{-1}\mathrm{vec}\Big( \big\{\nabla_{\bZ}h_\alpha^*(\tilde\bZ)-\nabla_{\bZ}h_\alpha^*(\bZ^*)\big\}^\T \Big)}_{\gamma_4}, \label{eq_st_expand_ave}
\end{align}
where $\tilde\cH$, $\tilde\cG$ and $\bar\cG$ are defined as \begin{equation*}
\tilde\cH = n^{-1}\int_0^1 \Big\{ \nabla_{\bz}^2 h_\alpha^*(\bZ_s)-2\cG(\bZ_s)\otimes \bI_r \Big\}\,ds, \end{equation*}
\begin{equation*}
\tilde\cG = \frac{2}{n}\int_0^1 \tilde\cG(\bZ_s)\otimes \bI_r\,ds, \qquad \bar\cG = \frac{2}{n}\int_0^1 \bar\cG(\bZ_s)\otimes \bI_r\,ds.\end{equation*}
We now bound $\gamma_1$–$\gamma_4$.

For $\gamma_1$, Lemma~\ref{lemma_general_convex} gives
\begin{equation*}
\lambda_{\min}\!\Big( n^{-1}\nabla_{\bz}^2 h_\alpha^*(\bZ_s)-2n^{-1}\cG(\bZ_s)\otimes \bI_r \Big) \ge \alpha\sigma_{\min}-4n^{-1}\|\bZ_s-\bZ^*\|_{\Fn}^2
\end{equation*}
for every $s\in[0,1]$. Since $\|\bZ_s-\bZ^*\|_{\Fn}=s\|\bE\|_{\Fn}\le \varepsilon\|\bZ^*\|_{\Fn}$ and $\|\bZ^*\|_{\Fn}^2\le nr\kappa\sigma_{\min}$,
\begin{equation*}
\lambda_{\min}(\tilde\cH) \ge \alpha\sigma_{\min}-4\varepsilon^2 r\kappa\sigma_{\min} \ge \frac{\alpha\sigma_{\min}}{2}, \end{equation*}
after shrinking $c_0$ if necessary, where we used $\varepsilon\sqrt r\,\kappa/\alpha\le c_0$.

For $\gamma_2$, Equation~\ref{assump_gradient_noise_1} yields
\begin{equation*}
\|\gamma_2\| \le \frac{2}{n}\max_{s\in[0,1]} \|\tilde\cG(\bZ_s)\bE\|_{\Fn} \le 2\Delta_2(n,\delta)\|\bE\|_{\Fn}.
\end{equation*}

For $\gamma_3$, using $\bar\cG(\bZ^*)=\zero$, Assumption~\ref{assump_gradient_lips_1} with $\bar\cG(\bZ)$ in place of $\cG(\bZ)$, and Lemma~\ref{lemma_zz_to_z}, one can check
\begin{align*}
\|\gamma_3\| \le &\,\frac{2}{n}\max_{s\in[0,1]} \|(\bar\cG(\bZ_s)-\bar\cG(\bZ^*))\bE\|_{\Fn} \\
\le &\,\frac{2L_2}{n} \max_{s\in[0,1]} \|\bZ_s\bZ_s^\T-\bZ^*(\bZ^*)^\T\|_{\Fn}\,\|\bE\|_{\Fn}
\\\le &\, \frac{CL_2\|\bZ^*\|}{n}\|\bE\|_{\Fn}^2,
\end{align*}
where the last step uses $\|\bE\|_{\Fn}\le \varepsilon\|\bZ^*\|_{\Fn}\le c_0\|\bZ^*\|$.

For $\gamma_4$, by \eqref{eq_st_score_ave},
\begin{align*}
\|\gamma_4\| &\le n^{-1}\|\nabla_{\bZ}h_\alpha^*(\tilde\bZ)\|_{\Fn} + n^{-1}\|\nabla_{\bZ}h_\alpha^*(\bZ^*)\|_{\Fn} \\ &\le 2\Delta_\infty(n,\delta)\|\bZ^*\| + 2n^{-1}\|\cG(\bZ^*)\bZ^*\|_{\Fn}
\\ &= 2\Delta_\infty(n,\delta)\|\bZ^*\| + 2n^{-1}\|\tilde\cG(\bZ^*)\bZ^*\|_{\Fn} \\ &\le 2\Delta_\infty(n,\delta)\|\bZ^*\|_{\Fn} + 2\Delta_2(n,\delta)\|\bZ^*\|_{\Fn}.
\end{align*}
Here, the last inequality follows from $n^{-1}\|\cG(\bZ^*)\bZ^*\|_{\Fn}\le n^{-1}\|\cG(\bZ^*)\|\|\bZ^*\|_{\Fn}\le \Delta_{2}(n,\delta)\|\bZ^*\|_{\Fn}$ by~\eqref{assump_gradient_noise_1}.

Combining the bounds for $\gamma_1$–$\gamma_4$, we obtain
\begin{equation}\label{eq_st_expand_ave1}
\begin{aligned}
\|\bE\|_{\Fn} \le &\, \lambda_{\min}(\tilde\cH)^{-1}\big(\|\gamma_2\| + \|\gamma_3\| + \|\gamma_4\|\big)\\ \le&\,
C\frac{\Delta_\infty(n,\delta)+\Delta_2(n,\delta)}{\alpha\sigma_{\min}}\|\bZ^*\|_{\Fn} + C\frac{\Delta_2(n,\delta)}{\alpha\sigma_{\min}}\|\bE\|_{\Fn} + C\frac{L_2\|\bZ^*\|}{\alpha\sigma_{\min}n}\|\bE\|_{\Fn}^2. \end{aligned}
\end{equation}
By the lower bound in \eqref{eq_lower_bound_varepsilon}, or equivalently
by the first two bounds in \eqref{eq_thm_general_theory_noisy_Z_scaling},
we have
\begin{equation*}
\frac{\Delta_\infty(n,\delta)}{\alpha\sigma_{\min}}\le c_0\varepsilon,
\qquad
\frac{\Delta_2(n,\delta)}{\alpha\sigma_{\min}}
\le
\frac{c_0}{\sqrt{r\kappa}}\varepsilon\le c_0\varepsilon,
\end{equation*}
and the localization of $\tilde\bZ$ such that
\begin{equation*}
\frac{\|\bZ^*\|\|\bZ^*\|_{\Fn}}{\alpha\sigma_{\min}n}\cdot \frac{\|\bE\|_{\Fn}}{\|\bZ^*\|_{\Fn}} \le \frac{\sqrt{\kappa n\sigma_{\min}}\sqrt{r\kappa n\sigma_{\min}}}{\alpha\sigma_{\min}n}\,\varepsilon = \varepsilon\frac{\kappa\sqrt r}{\alpha} \le c_0,
\end{equation*}
we can shrink $c_0$ if necessary and absorb the last two terms on the right-hand side of \eqref{eq_st_expand_ave1}. This in turn gives
\begin{equation}\label{eq_interer_z_ave}
\frac{\|\tilde\bZ-\bZ^*\|_{\Fn}}{\|\bZ^*\|_{\Fn}} = \frac{\|\bE\|_{\Fn}}{\|\bZ^*\|_{\Fn}} \le \frac{\varepsilon}{2}.
\end{equation}

Then it remains to establish the $\ell_\infty$ bound. For $i,j\in[n]$, define
\begin{equation*}
\tilde\cH_{ij} = n^{-1}\int_0^1 \Big\{ \nabla_{\bz_i\bz_j}^2\cL(\bZ_s\bZ_s^\T)-2G_{ij}(\bZ_s)\bI_r \Big\}\,ds, \end{equation*}
\begin{equation*}
\tilde\cG_{ij} = \frac{2}{n}\int_0^1 \tilde G_{ij}(\bZ_s)\bI_r\,ds, \qquad \bar\cG_{ij} = \frac{2}{n}\int_0^1 \bar G_{ij}(\bZ_s)\bI_r\,ds,
\end{equation*} and
\begin{equation*}
\tilde\cP = \int_0^1 \nabla_{\bz}^2 p_\alpha^*(\bZ_s)\,ds, \qquad \tilde\cP_{i,\cdot}:=(\tilde\cP)_{\cR_i,\cdot}, \qquad \cR_i=\{(i-1)r+1,\dots,ir\}.
\end{equation*}
Then \eqref{eq_st_expand_ave} is equivalent to the row-wise decomposition
\begin{equation}\label{eq_st_expand_uni} \begin{aligned}
&\underbrace{\tilde\cH_{ii}(\bE_{i,\cdot})^\T}_{\delta_{1,i}} + \underbrace{\sum_{j\neq i}\tilde\cH_{ij}(\bE_{j,\cdot})^\T}_{\delta_{2,i}} + \underbrace{\sum_{j=1}^n\bar\cG_{ij}(\bE_{j,\cdot})^\T}_{\delta_{3,i}} + \underbrace{\sum_{j=1}^n\tilde\cG_{ij}(\bE_{j,\cdot})^\T}_{\delta_{4,i}} + \underbrace{\tilde\cP_{i,\cdot}\mathrm{vec}(\bE^\T)}_{\delta_{5,i}} \\& = \underbrace{ n^{-1}\big\{\nabla_{\bZ}h_\alpha^*(\tilde\bZ)-\nabla_{\bZ}h_\alpha^*(\bZ^*)\big\}_{i,\cdot}^\T }_{\delta_{6,i}}.
\end{aligned}
\end{equation}
We arrange the decomposition as such to emphasize that $\delta_{1,i}$--$\delta_{3,i}$ can be treated similarly to $\delta_{1,t}$--$\delta_{3,t}$ in the proof of Theorem~\ref{thm_general_theory}, and $\delta_{4,i}$--$\delta_{6,i}$ arise from statistical noise.

Let
\begin{equation*}
\tilde\cH_{LD}:=\diag(\tilde\cH_{11},\dots,\tilde\cH_{nn}), \qquad \|\bA\|_{\infty,r}:=\max_{i\in[n]}\sum_{j=1}^n\|\bA_{\cR_i,\cR_j}\|,\text{ for }\bA\in\RR^{nr\times nr}
\end{equation*}
By Lemma~\ref{coro_thm_general_convex}, \eqref{eq_interer_z_ave}, and $\varepsilon\le \epsilon$, we have that
\begin{equation*}
(1-c_0)\alpha\sigma_{\min} \le \min_{i\in[n]}\lambda_{\min}(\tilde\cH_{ii}) \le \max_{i\in[n]}\lambda_{\max}(\tilde\cH_{ii}) \le C(\alpha+\beta)\kappa\sigma_{\min},
\end{equation*}
after shrinking $c_0$ if necessary. Hence, one has $\|\tilde\cH_{LD}^{-1}\|_{\infty,r} \le \{(1-c_0)\alpha\sigma_{\min}\}^{-1}$.
Therefore, \eqref{eq_st_expand_uni} yields
\begin{equation}\label{eq_st_expand_uni_infty_bound}
\|\bE\|_{\twinf} \le \frac{1}{(1-c_0)\alpha\sigma_{\min}} \max_{i\in[n]} \sum_{\ell=2}^6\|\delta_{\ell,i}\|.
\end{equation}

We now bound $\delta_{2,i}$–$\delta_{6,i}$.

For $\delta_{2,i}$, the same argument as in the $\delta_{2,t}$ bound in the proof of Theorem~\ref{thm_general_theory} gives
\begin{equation*}
\max_{i\in[n]}\|\delta_{2,i}\| \le C\beta n^{-1} \max_{s\in[0,1]}\|\bZ_s\|_{\twinf}\|\bZ_s\|\,\|\bE\|_{\Fn} \le C\beta\sqrt r\,\kappa\sigma_{\min}\,\varepsilon\,\|\bZ^*\|_{\twinf},
\end{equation*}
where we used
$\|\bZ_s\|_{\twinf}\le (1+c_0)\|\bZ^*\|_{\twinf}$, $\|\bZ_s\|\le (1+c_0)\|\bZ^*\|$, $\|\bE\|_{\Fn}\le \varepsilon\|\bZ^*\|_{\Fn}$, and $\|\bZ^*\|=\sqrt{n\kappa\sigma_{\min}}$.

For $\delta_{3,i}$, using $\bar\cG(\bZ^*)=\zero$, Assumption~\ref{assump_gradient_lips} with $\bar\cG(\bZ)$ in place of $\cG(\bZ)$, one can obtain
\begin{equation*}
\|\bZ_s\bZ_s^\T-\bZ^*(\bZ^*)^\T\|_{\twinf} \le \|\bZ^*\|_{\twinf}\|\bE\|_{\Fn} + (\|\bZ^*\|+\|\bE\|_{\Fn})\|\bE\|_{\twinf}.
\end{equation*}
It then follows that
\begin{align*}
\max_{i\in[n]}\|\delta_{3,i}\| &\le \frac{2}{n}\max_{s\in[0,1]} \|\bar\cG(\bZ_s)-\bar\cG(\bZ^*)\|_{\twinf}\,\|\bE\|_{\Fn} \\
&\le \frac{2L_\infty}{n} \max_{s\in[0,1]} \|\bZ_s\bZ_s^\T-\bZ^*(\bZ^*)^\T\|_{\twinf}\,\|\bE\|_{\Fn} \\
&\le CL_\infty\varepsilon^2 r\kappa\sigma_{\min}\|\bZ^*\|_{\twinf} + CL_\infty\varepsilon\sqrt r\,\kappa\sigma_{\min}\|\bE\|_{\twinf}.
\end{align*}

For $\delta_{4,i}$, \eqref{assump_gradient_noise_2} and the theorem's scaling conditions give $n^{-1}\|\tilde\cG(\bZ)\|_{\infty\to 1}\le \bar\Delta_{\infty}(n,\delta)\le \alpha\sigma_{\min}/4$ uniformly over $\bZ\in\cDinf$. Since $\bZ_s\in\cDinf$ for all $s\in[0,1]$,
\begin{equation*}
\max_{i\in[n]}\|\delta_{4,i}\| \le \frac{2}{n}\max_{s\in[0,1]}\|\tilde\cG(\bZ_s)\|_{\infty\to 1}\|\bE\|_{\twinf} \le 2\bar\Delta_{\infty}(n,\delta)\|\bE\|_{\twinf} \le \frac{\alpha\sigma_{\min}}{2}\|\bE\|_{\twinf}.
\end{equation*}

\begin{remark}\label{remark_relax_noise_for_consis}
    The term $\delta_{4,i}$ can be merged with $\delta_{1,i}$, which allows for a larger level of statistical noise. In particular, we assume $\bar\Delta_{\infty}(n,\delta)\le C\alpha\sigma_{\min}$ for some constant $C$, rather than imposing the more restrictive condition $\bar\Delta_{\infty}(n,\delta)\le \alpha\sigma_{\min}/4$. When $\bar\Delta_{\infty}(n,\delta)\gg \alpha\sigma_{\min}$, handling the noise typically requires exploiting additional problem structure and more delicate analysis; see, for example,  Section~\ref{section_example2}. Here, we present a simple argument that allows $\bar\Delta_{\infty}(n,\delta)\asymp \alpha\sigma_{\min}$

Note that
\begin{equation*}
    \delta_{1,i} + \delta_{4,i} = \{\underbrace{\tilde\cH_{LD} + \tilde\cG }_{:=\tilde\cH_Q}\}_{\cR_i,}\mathrm{vec}(\bE^\T).
\end{equation*}
Here, $\tilde\cH_{LD} = \diag(\tilde\cH_{11},\dots, \tilde\cH_{nn})$. Similarly, one can invert $\tilde\cH_Q$ in \eqref{eq_st_expand_uni} to yield 
\begin{equation}
    \|\bE\|_{\twinf}\le \big\|\tilde\cH_{Q}^{-1}\big\|_{\infty,r}\times \max_{i\in[n]}\big(\big\|\delta_{2,i}\big\| + \big\|\delta_{3,i}\big\| + \big\|\delta_{5,i}\big\| +\big\|\delta_{6,i}\big\| \big).\label{eq_st_expand_uni_infty_bound_2}
\end{equation}

With Lemma~\ref{coro_thm_general_convex}, $\| \tilde\bZ - \bZ^*\|_{\Fn}\le \epsilon\|\bZ^*\|_{\Fn}\le \epsilon\sqrt{nr\kappa\sigma_{\min}}$, scaling condition $\epsilon^2r\kappa\le c_0\alpha$, and the continuity of eigenvalues, one can check that
\begin{equation}
    (1-c_0)\alpha\sigma_{\min}\le \min_{i\in[n]}\lambda_{\min}\big(\tilde\cH_{ii}\big)\le \max_{i\in[n]}\lambda_{\max}\big(\tilde\cH_{ii}\big)\le \frac{\beta\kappa\sigma_{\min}}{2}.\label{eq_bar_delta0_bound_remark}
\end{equation}The block diagonal structure of $\tilde\cH_{ii}$ implies that \eqref{eq_bar_delta0_bound_remark} also holds when replacing $\tilde\cH_{ii}$ with $\tilde\cH_{LD}$. Note that with~\eqref{assump_gradient_noise_1}, similar to bounding $\gamma_2$ above, we know $\|\tilde\cG\|\le 2\Delta_2(n,\delta)$. With $\Delta_2(n,\delta)/(\alpha\sigma_{\min})\le c_0$ for sufficiently small $c_0$, we know
\begin{equation*}
    \frac{\|\tilde\cG\|}{\lambda_{\min}(\tilde\cH_{ii})}\le \frac{2\Delta_2(n,\delta)}{\alpha\sigma_{\min}} <1.
\end{equation*} Thus, the inverse of $\tilde\cH_{LD} + \tilde\cG$ can be expressed via the Neumann series 
\begin{align*}
    (\tilde\cH_{LD} + \tilde\cG)^{-1} = \tilde\cH_{LD}^{-1} + \tilde\cH_{LD}^{-1}\tilde\cG\tilde\cH_{LD}^{-1} + \tilde\cH_{LD}^{-1}\tilde\cG\tilde\cH_{LD}^{-1}\tilde\cG\tilde\cH_{LD}^{-1} + \cdots,
\end{align*}and 
\begin{equation*}
    \big\|(\tilde\cH_{LD} + \tilde\cG)^{-1}\big\|\le \sum_{l=0}^{\infty}\big\|\tilde\cH_{LD}^{-1}\big\|\big(\|\tilde\cG\|\|\tilde\cH_{LD}^{-1}\|\big)^l\le \frac{C}{\alpha\sigma_{\min}}.
\end{equation*}
Next, by the second part of~\eqref{assump_gradient_noise_2}, we know \begin{equation*}\|\tilde\cG\|_{\infty,r} \le \frac1n\max_{\bZ = \bZ^*+s\bE,s\in[0,1]}2\|\cG(\bZ) - \EE\cG(\bZ)\|_{\infty\to 1} \le 2 \bar\Delta_{\infty}(n,\delta)\le C\alpha\sigma_{\min}\end{equation*}
With $(\tilde\cH_{LD}^{-1})_{\cR_i,} = (\zero_{r\times (i-1)r}, \tilde\cH_{ii}^{-1},\zero_{r\times (n-i)r})$, $\|\tilde\cH_{LD}^{-1}\|_{\infty,r} \le \lambda_{\min}(\tilde\cH_{ii})^{-1}\le \{(1-c_0)\alpha\sigma_{\min}\}^{-1}$ Consequently, $\tilde\cH_{LD}^{-1}(\tilde\cG\tilde\cH_{LD}^{-1})^l$ can be bounded as follows
\begin{equation}\label{eq_bomega_upper1}
    \big\|\tilde\cH_{LD}^{-1}(\tilde\cG\tilde\cH_{LD}^{-1})^l\big\|_{\infty,r}\le\big\|\tilde\cH_{LD}^{-1}\big\|_{\infty,r}\big\|\Big(\big\|\tilde\cG\big\|_{\infty,r}\big\|\tilde\cH_{LD}^{-1}\big\|_{\infty,r}\Big)^l\le\frac{1}{(1-c_0)\alpha\sigma_{\min}}\Big(\frac{C}{1-c_0}\Big)^l.
\end{equation}
Moreover, since $\|\cdot\|_{\infty,r}\le \sqrt{n}\max_{i\in[n]}\|(\cdot)_{\cR_i,}\|$, we can bound it by 
\begin{equation}\label{eq_bomega_upper2}
    \big\|\tilde\cH_{LD}^{-1}(\tilde\cG\tilde\cH_{LD}^{-1})^l\big\|_{\infty,r}\le\sqrt{n}\max_{i\in[n]}\big\|\tilde\cH_{ii}^{-1}\big\|(\big\|\tilde\cG\big\|\big\|\tilde\cH_{LD}^{-1}\big\|\big)^l\le \frac{\sqrt{n}}{(1-c_0)\alpha\sigma_{\min}}\Big(\frac{2\Delta_2(n,\delta)}{\alpha\sigma_{\min}}\Big)^l.
\end{equation}

Let \begin{equation*}l_0 = \left\lceil \frac{\log n}{2\log \big\{4C\alpha\sigma_{\min}/\Delta_2(n,\delta)\big\}}\right\rceil = o(\log n)\end{equation*} 
as $\Delta_2(n,\delta)/(\alpha\sigma_{\min})\le c_0$ for any sufficiently small $c_0$. Use \eqref{eq_bomega_upper1} when $l\le l_0$ and \eqref{eq_bomega_upper2} when $l>l_0$ to bound $\|(\tilde\cH_{LD}+\tilde\cG)^{-1}\|_{\infty,r}$ by
\begin{equation}
    \|(\tilde\cH_{LD}+\tilde\cG)^{-1}\|_{\infty,r}\le \sum_{l=0}^{\infty}\big\|\tilde\cH_{LD}^{-1}(\tilde\cG\tilde\cH_{LD}^{-1})^l\big\|_{\infty,r}\le \frac{C^{l_0}}{\alpha \sigma_{\min}}\le \frac{n^{o(1)}}{\alpha\sigma_{\min}}.\label{eq_bar_delta1_bound}
\end{equation}
This bound depends on $\Delta_2(n,\delta)/(\alpha\sigma_{\min})$. When $\Delta_2(n,\delta)/(\alpha\sigma_{\min})\asymp n^{-c}$ for any constant $c$, we know \begin{equation*}\|(\tilde\cH_{LD}+\tilde\cG)^{-1}\|_{\infty,r}\le \frac{C}{\alpha\sigma_{\min}}.\end{equation*}
For simplicity, we present Theorem~\ref{thm_general_theory_noisy_Z} under small noise condition $\bar\Delta_{\infty}(n,\delta)\le \alpha\sigma_{\min}/4$.

\end{remark}

For $\delta_{5,i}$, we first compute the second-order derivative of $\nabla_{\bz}^2p^*_{\alpha}(\bZ)$ as
\begin{equation*}
\nabla_{\bz}^2 p_\alpha^*(\bZ) = \frac{\alpha n}{2}\sum_{k=1}^r\sum_{l=1}^r (\bvartheta_{lk}-\bvartheta_{kl})(\bvartheta_{lk}-\bvartheta_{kl})^\T,
\end{equation*}
where for each $lk,\in[r]$, $\bvartheta_{lk} = n^{-1/2}\big(z_{1l}^*\be_k^\T,\dots,z_{nl}^*\be_k^\T\big)^\T$. Hence $\tilde\cP=\nabla_{\bz}^2 p_\alpha^*(\bZ^*)$ is deterministic and subsequently
\begin{align*}
\big\|\tilde\cP_{i,\cdot}\mathrm{vec}(\bE^\T)\big\| &= \Big\| \frac{\alpha}{2n}\sum_{k=1}^r\sum_{l=1}^r \big(z_{il}^*\be_k^\T-z_{ik}^*\be_l^\T\big) \big((\bZ^*)^\T\tilde\bZ-\tilde\bZ^\T\bZ^*\big)_{l,k} \Big\| \\
&\le \frac{\alpha}{2n} \Big\{\sum_{k,l\in[r]}\|z_{il}^*\be_k^\T-z_{ik}^*\be_l^\T\|^2\Big\}^{1/2} \big\|(\bZ^*)^\T\tilde\bZ-\tilde\bZ^\T\bZ^*\big\|_{\Fn}\\
&\le C\alpha r\kappa\sigma_{\min}\,\varepsilon\,\|\bZ^*\|_{\twinf},
\end{align*}
uniformly over $i\in[n]$, where we used $\big\|(\bZ^*)^\T\tilde\bZ-\tilde\bZ^\T\bZ^*\big\|_{\Fn} \le 2\|\bZ^*\|\,\|\bE\|_{\Fn}$.

For $\delta_{6,i}$, by \eqref{eq_st_score_unf},
\begin{equation*}
\max_{i\in[n]}\|\delta_{6,i}\| \le n^{-1}\|\nabla_{\bZ}h_\alpha^*(\tilde\bZ)\|_{\twinf} + n^{-1}\|\nabla_{\bZ}h_\alpha^*(\bZ^*)\|_{\twinf} \le 4\Delta_\infty(n,\delta)\|\bZ^*\|_{\twinf}.
\end{equation*}

Substituting these bounds into \eqref{eq_st_expand_uni_infty_bound} and multiplying by $(1-c_0)\alpha\sigma_{\min}$ yields
\begin{align*}
    (1-c_0)\alpha\sigma_{\min}\|\bE\|_{\twinf}\le &\,\max_{i\in[n]}\big(\|\delta_{2,i}\|+ \|\delta_{3,i}\|+ \|\delta_{4,i}\| + \|\delta_{5,i}\|+ \|\delta_{6,i}\|\big)\\\le &\, C\alpha\sigma_{\min}\left\{\Big(\frac{\beta}{\alpha}+\sqrt r\Big)\sqrt r\,\kappa\,\varepsilon + \frac{\Delta_\infty(n,\delta)}{\alpha\sigma_{\min}}\right\}\|\bZ^*\|_{\twinf}\\
&\quad +\frac{\alpha\sigma_{\min}}{2}\|\bE\|_{\twinf} + CL_\infty\varepsilon\sqrt r\,\kappa\sigma_{\min}\|\bE\|_{\twinf}.
    \end{align*}
 Here the $\varepsilon^2 r\kappa$ contribution from $\delta_{3,i}$ is absorbed into the displayed $\zeta_r\varepsilon$ term by $\varepsilon\sqrt r\,\kappa/\alpha\le c_0$. By \eqref{eq_thm_general_theory_noisy_Z_scaling}, we have the following
\begin{equation*}
\frac{CL_\infty\varepsilon\sqrt r\,\kappa\sigma_{\min}}{(1-c_0)\alpha\sigma_{\min}} \le \frac{CL_{\infty}c_0}{1-c_0} \le \frac{1}{8},\quad\text{ and }\quad\frac{1}{2(1-c_0)}\le \frac{5}{8}.
\end{equation*} 
Therefore we obtain
\begin{equation*}
\Big(1-\frac{3}{4}\Big)\frac{\|\bE\|_{\twinf}}{\|\bZ^*\|_{\twinf}} \le C\Big(\frac{\beta}{\alpha}+\sqrt r\Big)\sqrt r\,\kappa\,\varepsilon + C\frac{\Delta_\infty(n,\delta)}{\alpha\sigma_{\min}}.
\end{equation*} By \eqref{eq_thm_general_theory_noisy_Z_scaling}, the right-hand side is at most $\epsilon/8$ after choosing $c_0$ small.
\begin{equation}\label{eq_interer_z_unf}
\frac{\|\tilde\bZ-\bZ^*\|_{\twinf}}{\|\bZ^*\|_{\twinf}} = \frac{\|\bE\|_{\twinf}}{\|\bZ^*\|_{\twinf}} \le \frac{\epsilon}{2},
\end{equation}after shrinking $c_0$ if necessary, where the last inequality uses $\zeta_r\varepsilon\le c_0\epsilon$ and $\Delta_2(n,\delta)/(\alpha\sigma_{\min})\le c_0\epsilon$, which follows from
\eqref{eq_thm_general_theory_noisy_Z_scaling}.

Combining \eqref{eq_interer_z_ave} and \eqref{eq_interer_z_unf}, we see that $\tilde\bZ$ is a strict interior point of $\cDde$. Therefore, $\tilde\bZ$ is a strict interior local minimizer of
\begin{equation*}
\bZ\mapsto \|\nabla_{\bZ}h_\alpha^*(\bZ)\|_{\twinf}^2 \qquad \text{over }\cDde. \end{equation*}

We next show that $h_\alpha^*$ is strongly convex on $\cDd$. By Lemma~\ref{lemma_general_convex},
\begin{align*}
\min_{\bZ\in\cDd} \lambda_{\min}\!\Big( n^{-1}\nabla_{\bz}^2 h_\alpha^*(\bZ)-2n^{-1}\cG(\bZ)\otimes \bI_r \Big) &\ge \alpha\sigma_{\min}-4n^{-1}\epsilon^2\|\bZ^*\|_{\Fn}^2 \\
&\ge \alpha\sigma_{\min}-4\epsilon^2 r\kappa\sigma_{\min}\\ &\ge (1-c_0)\alpha\sigma_{\min}.
\end{align*}Here, the last inequality follows from $\epsilon\le c_0\alpha/(\kappa\sqrt{r})$ and shrinking $c_0$ if necessary. Moreover, for $\bZ\in\cDd\subseteq\cDinf$,~\eqref{assump_gradient_noise_1}, Assumption~\ref{assump_gradient_lips_1} with $\bar\cG(\bZ)$ in place of $\cG(\bZ)$, and Lemma~\ref{lemma_zz_to_z} give
\begin{align*}
\frac1n\|\cG(\bZ)\| &\le \frac1n\|\tilde\cG(\bZ)\| + \frac1n\|\bar\cG(\bZ)-\bar\cG(\bZ^*)\| \\ &\le \Delta_2(n,\delta) + \frac{L_2}{n}\|\bZ\bZ^\T-\bZ^*(\bZ^*)^\T\|_{\Fn} \\ &\le \Delta_2(n,\delta)+CL_2\epsilon\sqrt r\,\kappa\sigma_{\min} \le c_0\alpha\sigma_{\min},
\end{align*}
after shrinking $c_0$ if necessary, where the last steps uses \eqref{eq_thm_general_theory_noisy_Z_scaling}. Hence, by Weyl's inequality, one can obtain
\begin{equation*}
\min_{\bZ\in\cDd} \lambda_{\min}\!\big(n^{-1}\nabla_{\bz}^2 h_\alpha^*(\bZ)\big) \ge \frac{\alpha\sigma_{\min}}{4} > 0.
\end{equation*}

At this point, we have shown that $\tilde\bZ$ lies strictly inside the auxiliary neighborhood. It remains to upgrade this minimizing property to the first-order condition. In particular, we introduce the following lemma.
\begin{lemma}\label{lemma_stationary_tildeZ}
Let $\cD\subseteq \RR^{n\times r}$ be compact, and let $\tilde\bZ\in\cD$ be a strict interior local minimizer of
\begin{equation*}
\bZ\mapsto \|\nabla_{\bZ}h_\alpha^*(\bZ)\|_{\twinf}^2
\end{equation*}
over $\cD$. If $\nabla_{\bz}^2 h_\alpha^*(\tilde\bZ)$ is positive definite, then \begin{equation*}
\nabla_{\bZ}h_\alpha^*(\tilde\bZ)=\zero.
\end{equation*} \end{lemma}
\begin{proof} See Section~\ref{supp_sec_prove_lemma_stationary_tildeZ}. \end{proof}

Since $\tilde\bZ\in\cDde$ is a strict interior local minimizer of $\|\nabla_{\bZ}h_\alpha^*(\bZ)\|_{\twinf}^2$ and $n^{-1}\nabla_{\bz}^2 h_\alpha^*(\tilde\bZ)\succeq \alpha\sigma_{\min}\bI/4$ within $\cDd$, Lemma~\ref{lemma_stationary_tildeZ} yields
\begin{equation}\label{eq_first_order_tilde_bZ} \nabla_{\bZ}h_\alpha^*(\tilde\bZ)=\zero.
\end{equation}
In addition, $h_\alpha^*(\cdot)$ is strongly convex on $\cDd$, so $\tilde\bZ$ is the unique minimizer of $h_\alpha^*(\cdot)$ on $\cDd$.


\medskip
\paragraph{\it \underline{ Step 2: Sharpen the statistical rates and identify $\tilde\bZ=\hat\bZ\hat\bR$.}} We now sharpen the bounds for $\|\bE\|_{\Fn}$ and $\|\bE\|_{\twinf}$, where recall $\bE=\tilde\bZ-\bZ^*$. We first sharpen the Frobenius error. Returning to \eqref{eq_st_expand_ave}, Step~1 gives $\nabla_{\bZ}h_{\alpha}^*(\tilde\bZ)=\zero$, so
\begin{equation*}
\gamma_4 = -\mathrm{vec}\!\Bigl(\bigl\{n^{-1}\nabla_{\bZ}h_{\alpha}^*(\bZ^*)\bigr\}^{\T}\Bigr).
\end{equation*}
Since $\bar\cG(\bZ^*)=\zero$ and $\nabla_{\bZ}p_{\alpha}^*(\bZ^*)=\zero$, one then has by \eqref{assump_gradient_noise_1} that
\begin{equation*}
n^{-1}\|\nabla_{\bZ}h_{\alpha}^*(\bZ^*)\|_{\Fn} = 2n^{-1}\|\tilde\cG(\bZ^*)\bZ^*\|_{\Fn} \le 2\Delta_2(n,\delta)\|\bZ^*\|_{\Fn}.
\end{equation*}
Hence, using the bounds for $\gamma_2$ and $\gamma_3$ from Step~1 together with $\lambda_{\min}(\tilde\cH)\ge \alpha\sigma_{\min}/2$, it holds that \begin{equation}\label{eq_step2_sharp_l2_pre}
\|\bE\|_{\Fn} \le C\frac{\Delta_2(n,\delta)}{\alpha\sigma_{\min}}\|\bZ^*\|_{\Fn} + C\frac{\Delta_2(n,\delta)}{\alpha\sigma_{\min}}\|\bE\|_{\Fn} + C\frac{\|\bZ^*\|}{\alpha\sigma_{\min}n}\|\bE\|_{\Fn}^2.
\end{equation}
Dividing both sides by $\|\bZ^*\|_{\Fn}$, and using the Step~1 bound ${\|\bE\|_{\Fn}}/{\|\bZ^*\|_{\Fn}}\le {\varepsilon}/{2}$,
we get
\begin{equation*}
\frac{\|\bZ^*\|}{\alpha\sigma_{\min}n}\|\bE\|_{\Fn} \le \frac{\|\bZ^*\|\|\bZ^*\|_{\Fn}}{\alpha\sigma_{\min}n}\cdot \frac{\varepsilon}{2} \le C\frac{\sqrt r\,\kappa}{\alpha}\,\varepsilon \le Cc_0.
\end{equation*}
Since $\Delta_2(n,\delta)/(\alpha\sigma_{\min})\le c_0$, after shrinking $c_0$ if necessary, the last two terms in \eqref{eq_step2_sharp_l2_pre} can be absorbed into the left-hand side. Therefore,
\begin{equation}
\frac{\|\bE\|_{\Fn}}{\|\bZ^*\|_{\Fn}} \le C\frac{\Delta_2(n,\delta)}{\alpha\sigma_{\min}}.
\label{eq_Z_tilde_consis}
\end{equation}
This sharper Frobenius bound also yields
\begin{equation}\begin{aligned}
\frac1n\|\cG(\tilde\bZ)\| \le&\, \frac1n\|\tilde\cG(\tilde\bZ)\| + \frac1n\|\bar\cG(\tilde\bZ)-\bar\cG(\bZ^*)\| \\\le&\, \Delta_2(n,\delta) + CL_2\,n^{-1}\|\tilde\bZ\tilde\bZ^\T-\bZ^*(\bZ^*)^\T\|_{\Fn} \\\le&\, \frac{\alpha\sigma_{\min}}{8},\end{aligned}\label{eq_bound_GZ_clean}
\end{equation}
after shrinking $c_0$ further, where we used Lemma~\ref{lemma_zz_to_z} and \eqref{eq_Z_tilde_consis}, and the last inequality again uses \(\sqrt r\,\kappa\Delta_2(n,\delta)/(\alpha\sigma_{\min})\le c_0\).

We next sharpen the row-wise $\ell_{\infty}$ bound. Returning to \eqref{eq_st_expand_uni_infty_bound}, the only difference from Step~1 is that we may now use \eqref{eq_Z_tilde_consis} in place of the crude bound $\|\bE\|_{\Fn}\le \varepsilon\|\bZ^*\|_{\Fn}$. The same calculations as in Step~1 give
\begin{align*}
\max_{i\in[n]}\|\delta_{2,i}\| &\le C\beta n^{-1}\|\bE\|_{\Fn}\max_{s\in[0,1]}\|\bZ^*+s\bE\|_{\twinf}\|\bZ^*+s\bE\| \\ &\le C\beta\sqrt r\,\kappa\sigma_{\min}\, \frac{\Delta_2(n,\delta)}{\alpha\sigma_{\min}}\, \|\bZ^*\|_{\twinf}, \end{align*}
and similarly
\begin{equation*}
\max_{i\in[n]}\|\delta_{5,i}\| \le C\alpha\sqrt r\,\kappa\sigma_{\min}\, \frac{\Delta_2(n,\delta)}{\alpha\sigma_{\min}}\, \|\bZ^*\|_{\twinf}.
\end{equation*}
For the deterministic term $\delta_{3,i}$, the bound in Step~1 combined with \eqref{eq_Z_tilde_consis} implies
\begin{equation*}
\max_{i\in[n]}\|\delta_{3,i}\| \le \frac{\alpha\sigma_{\min}}{8}\|\bE\|_{\twinf} +  C\sqrt r\,\kappa\sigma_{\min}\, \frac{\Delta_2(n,\delta)}{\alpha\sigma_{\min}}\, \|\bZ^*\|_{\twinf},
\end{equation*}after shrinking $c_0$ if necessary.
The noise term $\delta_{4,i}$ and score term $\delta_{6,i}$ remain unchanged:
\begin{equation*}
\max_{i\in[n]}\|\delta_{4,i}\| \le \frac{\alpha\sigma_{\min}}{2}\|\bE\|_{\twinf}, \qquad \max_{i\in[n]}\|\delta_{6,i}\|
\le 4\Delta_{\infty}(n,\delta)\|\bZ^*\|_{\twinf}.
\end{equation*}
Substituting these bounds into \eqref{eq_st_expand_uni_infty_bound}, and shrinking $c_0$ so that \begin{equation*}
\frac{1}{(1-c_0)\alpha\sigma_{\min}} \Big(\frac{\alpha\sigma_{\min}}{8}+\frac{\alpha\sigma_{\min}}{2}\Big) \le \frac{3}{4},
\end{equation*}
we arrive at
\begin{equation*}
\frac{\|\bE\|_{\twinf}}{\|\bZ^*\|_{\twinf}} \le \frac{3}{4}\frac{\|\bE\|_{\twinf}}{\|\bZ^*\|_{\twinf}} + C\Big(\frac{\beta}{\alpha}+\sqrt r\Big)\sqrt r\,\kappa\, \frac{\Delta_2(n,\delta)}{\alpha\sigma_{\min}} + C\frac{\Delta_{\infty}(n,\delta)}{\alpha\sigma_{\min}}.
\end{equation*}
Using the theorem scaling condition
\begin{equation*}
\Big(\frac{\beta}{\alpha}\sqrt{\kappa}+\sqrt r\Big)\sqrt r\,\kappa\, \frac{\Delta_2(n,\delta)}{\Delta_{\infty}(n,\delta)} \le c_0,
\end{equation*}
we thus conclude that
\begin{equation}
\frac{\|\bE\|_{\twinf}}{\|\bZ^*\|_{\twinf}} \le C\frac{\Delta_{\infty}(n,\delta)}{\alpha\sigma_{\min}}. \label{eq_Z_tilde_consis_inf}
\end{equation}

We now show that the penalty must vanish at $\tilde\bZ$:
\begin{equation} p_{\alpha}^*(\tilde\bZ)=0,\label{eq_penalty_zero_tilde_Z} \end{equation} so that $\tilde\bZ$ is also stationary for the original loss.
Suppose not. Let $\tilde\bR\in\argmin_{\bR\in\cO^r}\|\tilde\bZ\bR-\bZ^*\|_{\Fn}$. By \eqref{eq_Z_tilde_consis},
\begin{equation*}
\frac{\|\tilde\bZ-\bZ^*\|_{\Fn}}{\sigma_r(\bZ^*)} \le C\frac{\|\bZ^*\|_{\Fn}}{\sigma_r(\bZ^*)} \frac{\Delta_2(n,\delta)}{\alpha\sigma_{\min}} \le C\sqrt {r\,\kappa}\,\frac{\Delta_2(n,\delta)}{\alpha\sigma_{\min}} \le Cc_0<0.01
\end{equation*} for sufficiently small $c_0$. Hence Lemma~\ref{lemma_procruste_problem_perturb} applies and yields \begin{equation*}
\|\tilde\bR-\bI_r\| \le \|\tilde\bR-\bI_r\|_{\Fn} \le C\frac{\|\tilde\bZ-\bZ^*\|_{\Fn}}{\sigma_r(\bZ^*)} \le C\sqrt{r\,\kappa}\,\frac{\Delta_2(n,\delta)}{\alpha\sigma_{\min}} \le Cc_0\,\varepsilon,
\end{equation*}
where the last step uses the lower bound for $\varepsilon$ in \eqref{eq_lower_bound_varepsilon}. Now combine this with the strict interior bounds from Step~1:
\begin{equation*}
\|\tilde\bZ-\bZ^*\|_{\Fn}\le \frac{\varepsilon}{2}\|\bZ^*\|_{\Fn}, \qquad
\|\tilde\bZ-\bZ^*\|_{\twinf}\le \frac{\epsilon}{2}\|\bZ^*\|_{\twinf}. \end{equation*}
Then, one can check that
\begin{align*}
\|\tilde\bZ\tilde\bR-\bZ^*\|_{\Fn} &\le \|\tilde\bZ-\bZ^*\|_{\Fn} + \|\tilde\bZ\|\,\|\tilde\bR-\bI_r\| \\
&\le \frac{\varepsilon}{2}\|\bZ^*\|_{\Fn} + (\|\bZ^*\|+\|\tilde\bZ-\bZ^*\|_{\Fn})\,Cc_0\varepsilon \\
&\le \Big(\frac12+Cc_0(1+\varepsilon)\Big)\varepsilon\|\bZ^*\|_{\Fn} < \varepsilon\|\bZ^*\|_{\Fn},
\end{align*}
for $c_0$ small enough. Likewise, since $\varepsilon\le \epsilon$,
\begin{align*}
\|\tilde\bZ\tilde\bR-\bZ^*\|_{\twinf} &\le \|\tilde\bZ-\bZ^*\|_{\twinf} + \|\tilde\bZ\|_{\twinf}\|\tilde\bR-\bI_r\| \\ &\le \frac{\epsilon}{2}\|\bZ^*\|_{\twinf} + (\|\bZ^*\|_{\twinf}+\|\tilde\bZ-\bZ^*\|_{\twinf})\,Cc_0\varepsilon \\ &\le \frac{\epsilon}{2}\|\bZ^*\|_{\twinf} + Cc_0\varepsilon\|\bZ^*\|_{\twinf} < \epsilon\|\bZ^*\|_{\twinf}.
\end{align*}
Hence $\tilde\bZ\tilde\bR\in\cDde\subseteq \cDd$.
By Lemma~\ref{lemma_ortho_align}, we have $p_{\alpha}^*(\tilde\bZ\tilde\bR)=0$. Since $(\tilde\bZ\tilde\bR)(\tilde\bZ\tilde\bR)^\T=\tilde\bZ\tilde\bZ^\T$, it follows that \begin{equation*}
h_{\alpha}^*(\tilde\bZ\tilde\bR) = \cL(\tilde\bZ\tilde\bZ^\T) < \cL(\tilde\bZ\tilde\bZ^\T)+p_{\alpha}^*(\tilde\bZ) = h_{\alpha}^*(\tilde\bZ),
\end{equation*}
contradicting the fact that $\tilde\bZ$ is the unique minimizer of $h_{\alpha}^*(\cdot)$ within $\cDd$ as establish at the end of Step~1. This proves \eqref{eq_penalty_zero_tilde_Z}.
Next, since Step~1 established $\nabla_{\bZ}h_{\alpha}^*(\tilde\bZ)=\zero$, and the gradient of $p_{\alpha}^*(\cdot)$ vanishes whenever $p_{\alpha}^*(\cdot)=0$, we conclude that \begin{equation}
\nabla_{\bZ}\cL(\tilde\bZ\tilde\bZ^\T)=2\cG(\tilde\bZ)\tilde\bZ=\zero. \label{eq_tilde_first_order_L}
\end{equation}

We finally establish the equivalence between $\tilde\bZ$  and $\hat\bZ\hat\bR$. Recall that $\hat\bZ\in\argmin_{\bZ\in\cDinf}\cL(\bZ\bZ^\T)$, and let $\hat\bR\in\argmin_{\bR\in\cO^r}\|\hat\bZ\bR-\bZ^*\|_{\Fn}$. By definition, $\hat\bZ\hat\bR\in\cDd$, and Lemma~\ref{lemma_ortho_align} gives $p_{\alpha}^*(\hat\bZ\hat\bR)=0.$
Since $\tilde\bZ\in\cDd\subseteq \cDinf$ and $\hat\bZ$ minimizes $\cL(\bZ\bZ^\T)$ over $\cDinf$, we have $\cL(\hat\bZ\hat\bZ^\T)\le \cL(\tilde\bZ\tilde\bZ^\T)$. On the other hand, Step~1 showed that $\tilde\bZ$ is the unique minimizer of $h_{\alpha}^*(\cdot)$ on $\cDd$. It therefore follows that 
\begin{equation*}
h_{\alpha}^*(\tilde\bZ)\le h_{\alpha}^*(\hat\bZ\hat\bR)=\cL(\hat\bZ\hat\bZ^\T).
\end{equation*}
Invoke \eqref{eq_penalty_zero_tilde_Z} and we have $h_{\alpha}^*(\tilde\bZ)=\cL(\tilde\bZ\tilde\bZ^\T)$, which then implies $\cL(\tilde\bZ\tilde\bZ^\T) \le \cL(\hat\bZ\hat\bZ^\T) \le \cL(\tilde\bZ\tilde\bZ^\T)$.
Thus equality holds throughout, so
\begin{equation*}
h_{\alpha}^*(\tilde\bZ)=h_{\alpha}^*(\hat\bZ\hat\bR).
\end{equation*}
By uniqueness of the minimizer of $h_{\alpha}^*(\cdot)$ on $\cDd$, we conclude that $\tilde\bZ=\hat\bZ\hat\bR$. Consequently,
\begin{equation}\begin{aligned}
&\|\tilde\bZ-\bZ^*\|_{\Fn} = \|\hat\bZ\hat\bR-\bZ^*\|_{\Fn} \le C\frac{\Delta_2(n,\delta)}{\alpha\sigma_{\min}}\|\bZ^*\|_{\Fn}, \\ &\|\tilde\bZ-\bZ^*\|_{\twinf} = \|\hat\bZ\hat\bR-\bZ^*\|_{\twinf} \le C\frac{\Delta_{\infty}(n,\delta)}{\alpha\sigma_{\min}}\|\bZ^*\|_{\twinf}.
\end{aligned}\label{eq_error_zt_tilde_z_est}
\end{equation}This proves the \eqref{eq_optimum_Z_error_bound}, and also shows that the aligned solution $\hat\bZ\hat\bR$ is unique.

\medskip
\paragraph{\it \underline{ Step 3: $\ell_2$- and $\ell_{\infty}$-error contractions for $\bZ^t$.}} It remains to prove that the iterates contract toward the local stationary point $\tilde\bZ$ constructed in Steps~1--2. Unlike the noiseless case, in the noisy setting the iterates are not expected to contract exactly to $\bZ^*$; instead, Step~2 identified $\tilde\bZ=\hat\bZ\hat\bR$ and showed that $\tilde\bZ$ is already within statistical error of $\bZ^*$. Therefore, once we prove geometric contraction toward $\tilde\bZ$, the theorem follows by combining this contraction with \eqref{eq_error_zt_tilde_z_est} and the triangle inequality. Moreover, implied by the strong convexity of $h_{\alpha}^*(\cdot)$ within $\cDd$ and uniqueness of $\tilde\bZ$ as the minimizer, establishing convergence to $\tilde\bZ$ shall be a reasonable choice.

To absorb the statistical shift from $\bZ^*$ to $\tilde\bZ$, we slightly enlarge the initialization radii by the estimator error from Step~2 and define the effective radii
\begin{equation*}
\phi_n^\dagger:=\phi_n+C\frac{\Delta_2(n,\delta)}{\alpha\sigma_{\min}}, \qquad \psi_n^\dagger:=\psi_n+C\frac{\Delta_\infty(n,\delta)}{\alpha\sigma_{\min}},
\end{equation*}
where $C>0$ is a sufficiently large universal constant. 
We use the notation $\phi_n^\dagger$ and $\psi_n^\dagger$ to emphasize that the convergence to $\tilde\bZ$ admits an additional noise term. By the initialization bounds \eqref{eq_sym_init_Z} and \eqref{eq_sym_init_Z_uniform}, the scaling conditions in Theorem~\ref{thm_general_theory_noisy_Z}, and the estimator bounds \eqref{eq_error_zt_tilde_z_est}, after enlarging $C$ and shrinking $c_0$ if necessary, the effective radii satisfy the same scaling conditions:
\begin{equation}\label{eq_step3_dagger_scaling}
\phi_n^\dagger\le \frac{2}{3}\epsilon\wedge c_0\frac{\alpha}{\kappa\sqrt r}, \qquad \psi_n^\dagger\le \frac{2}{3}\epsilon, \qquad \frac{\beta}{\alpha}\kappa^{3/2}\, \sqrt r\, \frac{\phi_n^\dagger}{\psi_n^\dagger}\le c_0.
\end{equation}
For the last inequality, we use 
\begin{equation*}
    \frac{\phi_n^\dagger}{\psi_n^\dagger} \le C\left\{\frac{\phi_n}{\psi_n} +\frac{\Delta_2(n,\delta)}{\Delta_\infty(n,\delta)}\right\},
\end{equation*}
where the first term on the right side is bounded by the condition on the initialization condition in the theorem, and the second term can be bounded by the theorem's scaling condition.
Based on the above discussion, we know that to establish the theorem, it only suffices to show that the iterates contract geometrically toward $\tilde\bZ$ with these effective radii. In particular, we show in the following that for all $t\ge 0$,
\begin{equation}
\disttwo(\bZ^t,\tilde\bZ)\le \rho^t\phi_n^\dagger\|\bZ^*\|_{\Fn}, \qquad \distinf(\bZ^t,\tilde\bZ)\le \rho^t\psi_n^\dagger\|\bZ^*\|_{\twinf}, \label{eq_error_zt_tilde_z}
\end{equation}
where $\rho=1-\eta\alpha\sigma_{\min}/4$.
The base case $t=0$ follows from the triangle inequality and \eqref{eq_error_zt_tilde_z_est}:
\begin{equation*}
\disttwo(\bZ^0,\tilde\bZ) \le \disttwo(\bZ^0,\bZ^*)+\|\tilde\bZ-\bZ^*\|_{\Fn} \le \phi_n^\dagger\|\bZ^*\|_{\Fn}, \end{equation*}
and similarly,
\begin{equation*}
\distinf(\bZ^0,\tilde\bZ) \le \distinf(\bZ^0,\bZ^*)+\|\tilde\bZ-\bZ^*\|_{\twinf} \le \psi_n^\dagger\|\bZ^*\|_{\twinf}.
\end{equation*}

Now assume \eqref{eq_error_zt_tilde_z} holds at iteration $t$. To establish the error contraction in \eqref{eq_error_zt_tilde_z}, we align the current iterate with the new target $\tilde\bZ$ and work under the corresponding alignment error. In particular, for each $t\ge 0$, define
\begin{equation*}
\bR_t^\dagger:=\argmin_{\bR\in\cO^r}\|\bZ^t\bR-\tilde\bZ\|_{\Fn}, \qquad \widetilde{\bZ}^{t}:=\bZ^t\bR_t^\dagger, \qquad \tilde\bE_t:=\widetilde{\bZ}^{t}-\tilde\bZ.
\end{equation*}
By Lemma~\ref{lemma_ortho_align}, we have $\tilde\bZ^\T\widetilde{\bZ}^{t}=\widetilde{\bZ}^{t\T}\tilde\bZ$.
To reuse the contraction argument from Step~1 with $\tilde\bZ$ in place of $\bZ^*$, we recenter the alignment penalty at $\tilde\bZ$ by defining
\begin{equation*}
p_{\alpha,\dagger}^*(\bZ):= \frac{\alpha n^2}{4} \big\|n^{-1}\tilde\bZ^\T\bZ-n^{-1}\bZ^\T\tilde\bZ\big\|_{\Fn}^2, \qquad
h_{\alpha,\dagger}^*(\bZ):=\cL(\bZ\bZ^\T)+p_{\alpha,\dagger}^*(\bZ).
\end{equation*}
Then $\nabla_{\bZ}p_{\alpha,\dagger}^*(\widetilde{\bZ}^{t})=\zero \qquad\text{for all }t\ge 0$. Moreover, Step~1 gives $\nabla_{\bZ}h_{\alpha}^*(\tilde\bZ)=\zero$, and Step~2 gives $\nabla_{\bZ}\cL(\tilde\bZ\tilde\bZ^\T)=\zero$. Since also $\nabla_{\bZ}p_{\alpha,\dagger}^*(\tilde\bZ)=0$, we have
\begin{equation*}
\nabla_{\bZ}h_{\alpha,\dagger}^*(\tilde\bZ)=\zero.
\end{equation*}

We next record that $\tilde\bZ$ inherits the same basic conditioning and row scale as $\bZ^*$.
By \eqref{eq_error_zt_tilde_z_est}, Weyl's inequality, and $\|\tilde\bZ-\bZ^*\|_{\twinf}\le c_0\|\bZ^*\|_{\twinf}$, we also have
\begin{equation}\label{eq_step3_tildeZ_scale}
\sigma_r(\tilde\bZ)\ge (1-c_0)\sqrt{n\sigma_{\min}}, \qquad \|\tilde\bZ\|\le (1+c_0)\sqrt{n\kappa\sigma_{\min}}, \qquad \|\tilde\bZ\|_{\twinf}\le (1+c_0)\|\bZ^*\|_{\twinf}.
\end{equation}
For $s\in[0,1]$, similarly define
\begin{equation*}
\bZ_t^\dagger(s):=\tilde\bZ+s\tilde\bE_t. \end{equation*}
Provided that the induction hypothesis \eqref{eq_error_zt_tilde_z} holds for $t$, with \eqref{eq_error_zt_tilde_z_est}, we know
\begin{align*}
&\disttwo(\bZ_t^\dagger(s),\bZ^*) \le \|\tilde\bZ-\bZ^*\|_{\Fn}+s\|\tilde\bE_t\|_{\Fn} \le \Big( C\frac{\Delta_2(n,\delta)}{\alpha\sigma_{\min}}+\phi_n^\dagger \Big)\|\bZ^*\|_{\Fn},\\
&\distinf(\bZ_t^\dagger(s),\bZ^*) \le \|\tilde\bZ-\bZ^*\|_{\twinf}+s\|\tilde\bE_t\|_{\twinf} \le \Big( C\frac{\Delta_\infty(n,\delta)}{\alpha\sigma_{\min}}+\psi_n^\dagger \Big)\|\bZ^*\|_{\twinf}.
\end{align*}
By \eqref{eq_step3_dagger_scaling} and the scaling conditions for $\Delta_2(n,\delta)$ and $\Delta_{\infty}(n,\delta)$, we know 
\begin{equation*}
    C\frac{\Delta_2(n,\delta)}{\alpha\sigma_{\min}}+\phi_n^\dagger\le\epsilon, \qquad C\frac{\Delta_\infty(n,\delta)}{\alpha\sigma_{\min}}+\psi_n^\dagger\le\epsilon.
\end{equation*}Then both right-hand sides are bounded by $\epsilon\|\bZ^*\|_{\Fn}$ and $\epsilon\|\bZ^*\|_{\twinf}$, respectively, after shrinking $c_0$. Hence $\bZ_t^\dagger(s)\in \cDinf$ for all $s\in[0,1]$. Therefore, all assumptions stated with $\cD=\cDinf$ are available along this interpolation path. Now we let
\begin{equation*}
\tilde\bE_{t+1}^{(t)}:=\bZ^{t+1}\bR_t^\dagger-\tilde\bZ.
\end{equation*}
Using the gradient update, $\nabla_{\bZ}h_{\alpha,\dagger}^*(\tilde\bZ)=0$, and $\nabla_{\bZ}p_{\alpha,\dagger}^*(\widetilde{\bZ}^{t})=0$, we therefore obtain
\begin{equation*}
\tilde\bE_{t+1}^{(t)} = \widetilde{\bZ}^{t}-\tilde\bZ - \frac{\eta}{n} \Big\{ \nabla_{\bZ}h_{\alpha,\dagger}^*(\widetilde{\bZ}^{t}) - \nabla_{\bZ}h_{\alpha,\dagger}^*(\tilde\bZ) \Big\}.
\end{equation*}
In what follows, we repeat the proof of $\ell_2$ and $\ell_{\infty}$ error contractions similar to the proof in Section~\ref{supp_sec_prove_thm_general_theory}.

\medskip
\noindent{\bf $\ell_2$ error contraction.}
Applying the fundamental theorem of calculus (Theorem 4.2 in \citet{lang2012real}, Chapter XIII) along the segment $\{\bZ_t^\dagger(s):0\le s\le 1\}$, and splitting off the term $2\cG(\cdot)\otimes \bI_r$, we obtain
\begin{equation*}
\disttwo(\bZ^{t+1},\tilde\bZ) \le \tilde\gamma_{1,t}^\dagger+\tilde\gamma_{2,t}^\dagger+\tilde\gamma_{3,t}^\dagger, \end{equation*}
where
\begin{equation*}
\tilde\gamma_{1,t}^\dagger := \left\| \bigl(\bI_{nr}-\eta\bar{\cA}_t^\dagger\bigr)\mathrm{vec}(\tilde\bE_t^\T) \right\|, \qquad \bar{\cA}_t^\dagger := n^{-1}\int_0^1 \Big\{ \nabla_{\bz}^2 h_{\alpha,\dagger}^*(\bZ_t^\dagger(s)) - 2\cG(\bZ_t^\dagger(s))\otimes \bI_r \Big\}\,ds,
\end{equation*}
\begin{equation*}
\tilde\gamma_{2,t}^\dagger := \frac{2\eta}{n} \left\| \int_0^1\bigl\{\cG(\bZ_t^\dagger(s))-\cG(\tilde\bZ)\bigr\}\,ds\;\tilde\bE_t \right\|_{\Fn}, \qquad \tilde\gamma_{3,t}^\dagger := \frac{2\eta}{n}\|\cG(\tilde\bZ)\tilde\bE_t\|_{\Fn}.
\end{equation*}

The proof of Step~1 in Theorem~\ref{thm_general_theory} applies verbatim with $\bZ^*$ replaced by $\tilde\bZ$. Using \eqref{eq_step3_tildeZ_scale}, we obtain $ \lambda_{\min}(\bar{\cA}_t^\dagger)\ge \frac78\alpha\sigma_{\min}$, after shrinking $c_0$ if necessary. Therefore, one has
\begin{equation*}
\tilde\gamma_{1,t}^\dagger
\le
\Bigl(1-\frac78\eta\alpha\sigma_{\min}\Bigr)\|\tilde\bE_t\|_{\Fn}.
\end{equation*}

For $\tilde\gamma_{2,t}^\dagger$, we note that $\cG(\bZ) = \bar\cG(\bZ) + \tilde\cG(\bZ)$. Then, with Assumption~\ref{assump_gradient_lips_1} with $\bar\cG(\bZ)$ in place of $\cG(\bZ)$, Lemma~\ref{lemma_zz_to_z}, the localization $\|\tilde\bE_t\|_{\Fn}\le \phi_n^\dagger\|\bZ^*\|_{\Fn}\le \epsilon\|\bZ^*\|_{\Fn}$, and \eqref{assump_gradient_noise_1}, it holds that
\begin{align*}
\tilde\gamma_{2,t}^\dagger \le &\, C\eta\frac{\|\tilde\bZ\|}{n}\|\tilde\bE_t\|_{\Fn}^2 +C\eta\Delta_{2}(n,\delta)\|\tilde\bE_t\|_{\Fn} \\
\le &\, C\eta\phi_n^{\dagger}\sqrt r\,\kappa\,\sigma_{\min}\,\|\tilde\bE_t\|_{\Fn} +C\eta\Delta_{2}(n,\delta)\|\tilde\bE_t\|_{\Fn} \\\le &\,\frac18\eta\alpha\sigma_{\min}\|\tilde\bE_t\|_{\Fn},
\end{align*}
where the last inequality uses $\phi_n^\dagger\le c_0\alpha/(\kappa\sqrt r)$ from \eqref{eq_step3_dagger_scaling} and $\Delta_2(n,\delta)/(\alpha\sigma_{\min})\le c_0$, which follows from the theorem's noise scaling.

For $\tilde\gamma_{3,t}^\dagger$, \eqref{eq_bound_GZ_clean} implies
\begin{equation*}
\tilde\gamma_{3,t}^\dagger \le 2\eta\,\frac1n\|\cG(\tilde\bZ)\|\,\|\tilde\bE_t\|_{\Fn} \le \frac14\eta\alpha\sigma_{\min}\|\tilde\bE_t\|_{\Fn}.
\end{equation*}
Combining the above bounds, we conclude that
\begin{equation*}
\disttwo(\bZ^{t+1},\tilde\bZ) \le \Bigl(1-\frac12\eta\alpha\sigma_{\min}\Bigr)\|\tilde\bE_t\|_{\Fn} \le \rho\,\|\tilde\bE_t\|_{\Fn} \le \rho^{t+1}\phi_n^\dagger\|\bZ^*\|_{\Fn},
\end{equation*}
which therefore implies with the induction hypothesis that
\begin{equation}
\disttwo(\bZ^{t+1},\tilde\bZ)\le \rho^{t+1}\phi_n^\dagger\|\bZ^*\|_{\Fn}. \label{eq_error_2_zt_tilde_z}
\end{equation}

\medskip
\noindent{\bf $\ell_{\infty}$ error contraction.}
As in Step~2 of the proof of Theorem~\ref{thm_general_theory}, we have the decomposition
\begin{equation}\label{eq_step4_linf_decomp}
\distinf(\bZ^{t+1},\tilde\bZ) \le \|\tilde\bE_{t+1}^{(t)}\|_{\twinf} + \|\bZ^{t+1}\bR_t^\dagger\|_{\twinf} \big\|\bI_r-(\bR_t^\dagger)^{-1}\bR_{t+1}^\dagger\big\|,
\end{equation}
where $\bR_{t+1}^\dagger$ is the optimal alignment of $\bZ^{t+1}$ to $\tilde\bZ$.

We first bound $\|\tilde\bE_{t+1}^{(t)}\|_{\twinf}$. For $i,j\in[n]$, define
\begin{equation*}
\cH_{ij}^\dagger(\bZ) := n^{-1}\Bigl\{\nabla_{\bz_i\bz_j}^2\cL(\bZ\bZ^\T)-2G_{ij}(\bZ)\bI_r\Bigr\}.\end{equation*}
Since $\tilde\bZ^\T\widetilde{\bZ}^{t}=\widetilde{\bZ}^{t\T}\tilde\bZ$, we have $\tilde\bZ^\T(\tilde\bZ+s\tilde\bE_t) = (\tilde\bZ+s\tilde\bE_t)^\T\tilde\bZ$ for any $s\in[0,1]$.
Therefore, $p_{\alpha,\dagger}^*(\bZ_t^\dagger(s))=0$ and $\nabla_{\bZ}p_{\alpha,\dagger}^*(\bZ_t^\dagger(s))=0$ for all $s\in[0,1]$. Consequently, the penalty contributes nothing to the rowwise expansion along the interpolation segment. For each $i\in[n]$,
\begin{align*}
(\tilde\bE_{t+1}^{(t)})_{i,\cdot}^\T =\;& \Bigl\{ \bI_r-\eta\int_0^1\cH_{ii}^\dagger(\bZ_t^\dagger(s))\,ds \Bigr\} (\tilde\bE_t)_{i,\cdot}^\T - \eta\sum_{j\neq i}\int_0^1\cH_{ij}^\dagger(\bZ_t^\dagger(s))(\tilde\bE_t)_{j,\cdot}^\T\,ds \\
&\quad -\frac{2\eta}{n} \Bigl[ \int_0^1\tilde\cG(\bZ_t^\dagger(s))\,ds\;\tilde\bE_t \Bigr]_{i,\cdot}^\T -\frac{2\eta}{n} \Bigl[ \int_0^1\bigl\{\bar\cG(\bZ_t^\dagger(s))-\bar\cG(\bZ^*)\bigr\}\,ds\;\tilde\bE_t \Bigr]_{i,\cdot}^\T.
\end{align*}
As in Step~2 of Theorem~\ref{thm_general_theory}, we now bound the four contributions on the right-hand side separately. The first two terms are the diagonal and off-diagonal Hessian contributions, while the last two arise from splitting the gradient term into its stochastic part $\tilde\cG(\bZ)$ and deterministic mean part $\bar\cG(\bZ)$. Subsequently, we denote
\begin{equation*}
\|\tilde\bE_{t+1}^{(t)}\|_{\twinf} \le \tilde\delta_{1,t}^\dagger+\tilde\delta_{2,t}^\dagger+\tilde\delta_{3,t}^\dagger+\tilde\delta_{4,t}^\dagger.
\end{equation*}

We first bound the deterministic Hessian terms $\tilde\delta_{1,t}^\dagger$ and $\tilde\delta_{2,t}^\dagger$. The proofs are identical to those for $\delta_{1,t}$ and $\delta_{2,t}$ in Step~2 of Theorem~\ref{thm_general_theory}.
Specifically, by the same argument and using \eqref{eq_step3_tildeZ_scale}, we have
\begin{equation*}
\tilde\delta_{1,t}^\dagger \le \Bigl(1-\frac78\eta\alpha\sigma_{\min}\Bigr)\|\tilde\bE_t\|_{\twinf},
\end{equation*}
and
\begin{equation*}
\tilde\delta_{2,t}^\dagger \le C\eta\beta\sqrt{\frac{\kappa\sigma_{\min}}{n}}\, \|\tilde\bZ\|_{\twinf}\,\|\tilde\bE_t\|_{\Fn}.
\end{equation*}

For the stochastic term, ~\eqref{assump_gradient_noise_2} and the theorem's scaling condition give
\begin{equation*}
\tilde\delta_{3,t}^\dagger \le 2\eta \bar\Delta_{\infty}(n,\delta)\|\tilde\bE_t\|_{\twinf} \le
\frac12\eta\alpha\sigma_{\min}\|\tilde\bE_t\|_{\twinf}.
\end{equation*}This is the only place in the contraction step where the worst-direction row-wise noise level \(\bar\Delta_\infty(n,\delta)\) is used.

It remains to bound the deterministic mean term $\tilde\delta_{4,t}^\dagger$. The calculation is the same as for $\delta_{3,t}$ in Step~2 of Theorem~\ref{thm_general_theory}: apply the Lipschitz bound for $\bar\cG$ to $\bar\cG(\bZ_t^\dagger(s))-\bar\cG(\bZ^*)$, expand $\bZ_t^\dagger(s)\bZ_t^\dagger(s)^\T-\bZ^*(\bZ^*)^\T = (\tilde\bZ\tilde\bZ^\T-\bZ^*(\bZ^*)^\T) +s(\tilde\bZ\tilde\bE_t^\T+\tilde\bE_t\tilde\bZ^\T) +s^2\tilde\bE_t\tilde\bE_t^\T$,
and then use \eqref{eq_error_zt_tilde_z_est}, \eqref{eq_error_2_zt_tilde_z}, and \eqref{eq_step3_tildeZ_scale} to get
\begin{equation*}
\tilde\delta_{4,t}^\dagger \le \frac1{16}\eta\alpha\sigma_{\min}\|\tilde\bE_t\|_{\twinf} + C\eta\alpha\sqrt{\frac{\sigma_{\min}}{n\kappa}}\, \|\tilde\bZ\|_{\twinf}\,\|\tilde\bE_t\|_{\Fn},
\end{equation*}
after shrinking $c_0$ if necessary.

Combining the above bounds, and using $\alpha\le \beta$ and $\kappa\ge 1$, we obtain
\begin{equation*}
\|\tilde\bE_{t+1}^{(t)}\|_{\twinf} \le \Bigl(1-\frac{5}{16}\eta\alpha\sigma_{\min}\Bigr)\|\tilde\bE_t\|_{\twinf} + C\eta\beta\sqrt{\frac{\kappa\sigma_{\min}}{n}}\, \|\tilde\bZ\|_{\twinf}\,\|\tilde\bE_t\|_{\Fn}.
\end{equation*}
Using \eqref{eq_error_2_zt_tilde_z}, the induction hypothesis, and \eqref{eq_step3_tildeZ_scale}, we know $\|\tilde\bE_t\|_{\Fn} \le \rho^t\phi_n^\dagger\|\bZ^*\|_{\Fn} \le C\rho^t\phi_n^\dagger\sqrt{nr\kappa\sigma_{\min}},$, and consequently
\begin{equation*}
C\beta\sqrt{\frac{\kappa\sigma_{\min}}{n}}\,
\|\tilde\bZ\|_{\twinf}\,\|\tilde\bE_t\|_{\Fn}
\le
C\beta\kappa\sqrt r\,\sigma_{\min}\,\phi_n^\dagger\,\rho^t\|\tilde\bZ\|_{\twinf}.
\end{equation*}
Since \(\kappa\ge 1\), \eqref{eq_step3_dagger_scaling} implies
\[
C\beta\kappa\sqrt r\,\phi_n^\dagger \le C\alpha \left\{ \frac{\beta}{\alpha}\kappa^{3/2}\sqrt r
\frac{\phi_n^\dagger}{\psi_n^\dagger} \right\}\psi_n^\dagger \le \frac{\alpha}{32}\psi_n^\dagger
\]
after shrinking \(c_0\). Therefore, we arrive at
\begin{equation*}
\|\tilde\bE_{t+1}^{(t)}\|_{\twinf} \le \Bigl(1-\frac{5}{16}\eta\alpha\sigma_{\min}\Bigr)\|\tilde\bE_t\|_{\twinf}
+ \frac1{32}\eta\alpha\sigma_{\min}\rho^t\psi_n^\dagger\|\tilde\bZ\|_{\twinf}.
\end{equation*}
Invoking the induction hypothesis $ \|\tilde\bE_t\|_{\twinf}\le \rho^t\psi_n^\dagger\|\bZ^*\|_{\twinf} $ and $\|\tilde\bZ\|_{\twinf}\le (1+c_0)\|\bZ^*\|_{\twinf}$, one can obtain
\begin{equation*}
\|\tilde\bE_{t+1}^{(t)}\|_{\twinf} \le \Bigl(1-\frac{9}{32}\eta\alpha\sigma_{\min}\Bigr)\rho^t\psi_n^\dagger\|\tilde\bZ\|_{\twinf} .
\end{equation*}

It remains to bound the rotation drift in \eqref{eq_step4_linf_decomp}.
By \eqref{eq_error_2_zt_tilde_z} and \eqref{eq_step3_tildeZ_scale}, we note that
\begin{equation*}
\frac{\|\tilde\bE_{t+1}^{(t)}\|_{\Fn}}{\sigma_r(\tilde\bZ)} \le C\rho^{t+1}\phi_n^\dagger\frac{\|\bZ^*\|_{\Fn}}{\sqrt{n\sigma_{\min}}} \le C\rho^{t+1}\sqrt{r\kappa}\,\phi_n^\dagger \le Cc_0<0.01
\end{equation*}
for sufficiently small $c_0$. Hence Lemma~\ref{lemma_procruste_problem_perturb} applies and gives
\begin{equation*}
\big\|(\bR_t^\dagger)^{-1}\bR_{t+1}^\dagger-\bI_r\big\|_{\Fn} \le \frac{C}{\sigma_r(\tilde\bZ)}\|\tilde\bE_{t+1}^{(t)}\|_{\Fn}.
\end{equation*}
Therefore, using $\|\bZ^{t+1}\bR_t^\dagger\|_{\twinf}\le \|\tilde\bZ\|_{\twinf}+\|\tilde\bE_{t+1}^{(t)}\|_{\twinf}$, it follows that
\begin{align*}
\distinf(\bZ^{t+1},\tilde\bZ) &\le \|\tilde\bE_{t+1}^{(t)}\|_{\twinf} + C\frac{\|\tilde\bE_{t+1}^{(t)}\|_{\Fn}}{\sigma_r(\tilde\bZ)} \Bigl(\|\tilde\bZ\|_{\twinf}+\|\tilde\bE_{t+1}^{(t)}\|_{\twinf}\Bigr).
\end{align*}
Because $\phi_n^{\dagger}/\psi_n^{\dagger}$ scales similarly to $\phi_n/\psi_n$ in Theorem~\ref{thm_general_theory}, following a similar procedure in establishing \eqref{eq_e_t_bar_1} in Section~\ref{supp_sec_prove_thm_general_theory}, one can obtain
\begin{equation}
\distinf(\bZ^{t+1},\tilde\bZ) \le \rho^{t+1}\psi_n^\dagger\|\bZ^*\|_{\twinf}. \label{eq_error_inf_zt_tilde_z}
\end{equation}
Equations~\eqref{eq_error_2_zt_tilde_z} and~\eqref{eq_error_inf_zt_tilde_z} establish the induction \eqref{eq_error_zt_tilde_z} for all $t\ge 0$.

Finally, since $\tilde\bZ=\hat\bZ\hat\bR$ with probability at least $1-\delta$ by Step~2, we have $\disttwo(\bZ^t,\hat\bZ\hat\bR)=\disttwo(\bZ^t,\tilde\bZ)$ and $\distinf(\bZ^t,\hat\bZ\hat\bR)=\distinf(\bZ^t,\tilde\bZ)$.
Using \eqref{eq_error_zt_tilde_z}, \eqref{eq_error_zt_tilde_z_est}, $\rho^t\phi_n^\dagger\le \rho^t\phi_n+C\tfrac{\Delta_2(n,\delta)}{\alpha\sigma_{\min}}$, and $\rho^t\psi_n^\dagger \le \rho^t\psi_n+C\tfrac{\Delta_\infty(n,\delta)}{\alpha\sigma_{\min}}$,
we conclude that
\begin{equation*}
\disttwo(\bZ^t,\bZ^*) \le \rho^t\phi_n\|\bZ^*\|_{\Fn} + C\frac{\Delta_2(n,\delta)}{\alpha\sigma_{\min}}\|\bZ^*\|_{\Fn},
\end{equation*}
and
\begin{equation*}
\distinf(\bZ^t,\bZ^*) \le \rho^t\psi_n\|\bZ^*\|_{\twinf} + C\frac{\Delta_\infty(n,\delta)}{\alpha\sigma_{\min}}\|\bZ^*\|_{\twinf}.
\end{equation*}
This combines the Step~3 contraction with the intermediate optimizer bound \eqref{eq_optimum_Z_error_bound}, and completes the proof of Theorem~\ref{thm_general_theory_noisy_Z}.

\subsection{Proof sketch for the \texorpdfstring{$\ell_2$}{l2} part of Theorems~\ref{thm_general_theory} and~\ref{thm_general_theory_noisy_Z}}\label{supp_sec_prove_thm_general_theory_noisy_Z_l2}

We outline the common $\ell_2$ argument under the larger local region $\cD=\cDtwo$. Still, we take $\epsilon = c_0\alpha/(\kappa\sqrt{r})$ for a sufficiently small constant $c_0$ such that the initialization requirement \eqref{eq_sym_init_Z} still holds. In the noisy case, we work on the event where \eqref{assump_gradient_noise_1} holds with $\cD=\cDtwo$, and assume Assumptions~\ref{assump_rsc} and~\ref{assump_gradient_lips_1} hold with $\cD=\cDtwo$ and $\bar\cG(\cdot)$ in place of $\cG(\cdot)$. The noiseless case is recovered by setting $\tilde\cG\equiv 0$. Set
\begin{equation*}
\varepsilon:=M\frac{\Delta_2(n,\delta)}{\alpha\sigma_{\min}}, \qquad \bar\cD_{z}^{(2)}(\varepsilon) := \big\{ \bZ:\|\bZ-\bZ^*\|_{\Fn}\le \varepsilon\|\bZ^*\|_{\Fn} \big\},
\end{equation*}
where $M>0$ is a sufficiently large universal constant. By the smallness condition on $\Delta_2(n,\delta)/(\alpha\sigma_{\min})$, after shrinking $c_0$ if necessary we have $\bar\cD_{z}^{(2)}(\varepsilon)\subseteq \cDtwo$. As in Step~1 of the proof of Theorem~\ref{thm_general_theory_noisy_Z}, let
\begin{equation*}
\tilde\bZ\in\argmin_{\bZ\in\bar\cD_{z}^{(2)}(\varepsilon)}\|\nabla_{\bZ}h_\alpha^*(\bZ)\|_{\Fn}^2, \qquad h_\alpha^*(\bZ):=\cL(\bZ\bZ^\T)+p_\alpha^*(\bZ).
\end{equation*}
Since $\bar\cG(\bZ^*)=\zero$,
\begin{equation*}
\|\nabla_{\bZ}h_\alpha^*(\tilde\bZ)\|_{\Fn}\le \|\nabla_{\bZ}h_\alpha^*(\bZ^*)\|_{\Fn} = 2\|\tilde\cG(\bZ^*)\bZ^*\|_{\Fn} \le 2n\,\Delta_2(n,\delta)\|\bZ^*\|_{\Fn}.
\end{equation*}
Repeating the mean-value expansion in Step~1, but keeping only the Frobenius terms, gives
\begin{equation*}
\|\tilde\bZ-\bZ^*\|_{\Fn} \le C\frac{\Delta_2(n,\delta)}{\alpha\sigma_{\min}}\|\bZ^*\|_{\Fn} + C\frac{\Delta_2(n,\delta)}{\alpha\sigma_{\min}}\|\tilde\bZ-\bZ^*\|_{\Fn} + C\frac{L_2\|\bZ^*\|}{\alpha\sigma_{\min}n}\|\tilde\bZ-\bZ^*\|_{\Fn}^2.
\end{equation*}
The stated smallness conditions allow the last two terms to be absorbed into the left-hand side, so
\begin{equation*}
\|\tilde\bZ-\bZ^*\|_{\Fn} \le C\frac{\Delta_2(n,\delta)}{\alpha\sigma_{\min}}\|\bZ^*\|_{\Fn} \le \frac{\varepsilon}{2}\|\bZ^*\|_{\Fn}.
\end{equation*}
With a similar argument in Step~1 in the proof of Theorem~\ref{thm_general_theory_noisy_Z}, we can show that $h_\alpha^*(\cdot)$ is strongly convex on $\bar\cD_{z}^{(2)}(\varepsilon)$. Therefore the analogue of Lemma~\ref{lemma_stationary_tildeZ} with $\|\cdot\|_{\Fn}$ in place of $\|\cdot\|_{\twinf}$ yields $\nabla_{\bZ}h_\alpha^*(\tilde\bZ)=\zero$, and $\tilde\bZ$ is the unique minimizer of $h_\alpha^*(\cdot)$ on $\bar\cD_{z}^{(2)}(\varepsilon)$. The identification argument from Step~2 of Theorem~\ref{thm_general_theory_noisy_Z} then yields $p_\alpha^*(\tilde\bZ)=0$, $\nabla_{\bZ}\cL(\tilde\bZ\tilde\bZ^\T)=\zero$, and $\tilde\bZ=\hat\bZ\hat\bR$. This proves the optimizer bound.

For the iterates, only the $\ell_2$ part of Step~3 in the proof of Theorem~\ref{thm_general_theory_noisy_Z} is needed. Define
\begin{equation*}
\bR_t^\dagger:=\argmin_{\bR\in\cO^r}\|\bZ^t\bR-\tilde\bZ\|_{\Fn}, \qquad \tilde\bE_t:=\bZ^t\bR_t^\dagger-\tilde\bZ.
\end{equation*}
Repeating the $\ell_2$-contraction calculation in Step~3 gives
\begin{equation*}
\disttwo(\bZ^{t+1},\tilde\bZ) \le \tilde\gamma_{1,t}^\dagger+\tilde\gamma_{2,t}^\dagger+\tilde\gamma_{3,t}^\dagger,
\end{equation*}
where the same bounds as in Theorem~\ref{thm_general_theory_noisy_Z} yield
\begin{equation*}
\tilde\gamma_{1,t}^\dagger \le \Bigl(1-\frac78\eta\alpha\sigma_{\min}\Bigr)\|\tilde\bE_t\|_{\Fn}, \qquad \tilde\gamma_{2,t}^\dagger \le \frac18\eta\alpha\sigma_{\min}\|\tilde\bE_t\|_{\Fn}, \qquad \tilde\gamma_{3,t}^\dagger \le \frac14\eta\alpha\sigma_{\min}\|\tilde\bE_t\|_{\Fn}.
\end{equation*}
Therefore,
\begin{equation*}
\disttwo(\bZ^{t+1},\tilde\bZ) \le \rho\,\disttwo(\bZ^t,\tilde\bZ), \qquad \rho=1-\eta\alpha\sigma_{\min}/4.
\end{equation*}
Combining this with the initialization bound and the estimate for $\|\tilde\bZ-\bZ^*\|_{\Fn}$ gives the stated $\ell_2$ contraction toward $\bZ^*$.

\section{Proof of Results under Asymmetric Model}\label{supp_sec_prove_asym}


This section proves the main deterministic and noisy results for the asymmetric model, namely Theorems~\ref{thm_general_theory_uv} and~\ref{thm_general_theory_UV_noise}. The organization parallels that of Section~\ref{supp_sec_prove_sym}. Specifically, Section~\ref{Sec_prelmi} provides the notation and technical lemmas used throughout this section. Sections~\ref{supp_prove_thm_general_theory_uv} and~\ref{supp_sec_prove_thm_general_theory_UV_noise} establish both $\ell_2$ and $\ell_{\infty}$ error contractions under the stronger localization $\cDuvinf$ in the deterministic and noisy settings, respectively. Section~\ref{supp_sec_prove_thm_general_theory_noisy_uv_l2} then proves the $\ell_2$ error contraction under localization $\cDuvtwo$.

The proofs follow the same broad strategy as in the symmetric model, outlined at the beginning of Section~\ref{supp_sec_prove_sym}, but require additional control of the non-orthogonal alignment $\bG_t^*$. Specifically, for the {\it noiseless} case, we first use the asymmetric benign regularizer to construct the aligned update, which mirrors one on a locally strongly convex objective, up to the perturbation caused by $\bLambda_t^*-\bI_r$, as in \eqref{eq_aligned_gradient_update_ub}. We then use this property to establish the $\ell_2$ and $\ell_{\infty}$ error bounds for $\bU^{t+1}\bG_t^* - \bU^*$ and $\bV^{t+1}(\bG_t^*)^{-\T} - \bV^*$. The third step differs from the symmetric case. Instead of quantifying the one-step alignment drift $\bG_{t+1}^* - \bG_{t}^*$, we must control $\bG_{t+1}^* - \bR^0$, which is necessary to keep $\bLambda_{t+1}^* - \bI_r = (\bG_{t+1}^*)^\T \bG_{t+1}^* - \bI_r$ small along the trajectory. We obtain it through a telescoping argument that sums the increments $\bG_{s+1}^* - \bG_{s}^*$ over $s=0,1,\dots,t$ to obtain this.

For the {\it noisy} setting, again we study the empirical target $(\hat\bU,\hat\bV)$. Following the strategy of the noisy proof in Section~\ref{supp_sec_prove_sym}, we establish the first-order condition for $(\hat\bU\hat\bG,\hat\bV\hat\bG^{-\T})$ where \begin{equation*}\hat\bG \in \argmin_{\bG\in GL(r)}n^{-1}\|\hat\bU \hat\bG - \bU^*\|_{\Fn}^2 + q^{-1}\|\hat\bV\hat\bG^{-\T} - \bV^*\|_{\Fn}^2.\end{equation*} This makes $(\hat\bU\hat\bG,\hat\bV\hat\bG^{-\T})$ a valid contraction target for the algorithm. We then re-center the analysis at this empirical target and repeat the deterministic contraction argument, with the alignment now taken relative to $(\hat\bU\hat\bG,\hat\bV\hat\bG^{-\T})$. The use of $(\hat\bU\hat\bG,\hat\bV\hat\bG^{-\T})$ as the contraction target is again the key technical tool in the proof, as it ensures that, when telescoping the alignment drift, the non-decaying statistical error is absorbed into the empirical target rather than accumulated along the trajectory.


\subsection{Preliminaries}\label{Sec_prelmi}
We use $\{\be_i^{(n)}\}$ to denote the canonical basis in $\RR^{n}$, and we omit the superscript $^{n}$ when the dimension is clear from the context. For a vector $\bx\in\RR^n$ and a subset $\cS\subseteq[n]$, write $\bx_{\cS}\in\RR^{|\cS|}$ for the sub-vector indexed by $\cS$. For a matrix $\bM\in\RR^{n\times m}$ and index sets $\cS_1\subseteq[n]$, $\cS_2\subseteq[m]$, write $\bM_{\cS_1,\cS_2}\in\RR^{|\cS_1|\times|\cS_2|}$ for the corresponding submatrix. In particular, $\bM_{i,\cdot}$ and $\bM_{\cdot,j}$ denote the $i$th row and $j$th column, while $\bM_{\cS_1,}$ and $\bM_{,\cS_2}$ denote row- and column-restricted submatrices.
Write $\bz=\mathrm{vec}(\bZ^\T)$, so that $\nabla_{\bz}^2$ denotes the Hessian with respect to the vectorized variable; later, $\nabla_{\bz_i\bz_j}^2$ denotes its $(i,j)$ block.

Let $\bM^*(\bU,\bV) = n^{-1}(\bU - \bU^*)^\T\bU - q^{-1}\bV^\T(\bV - \bV^*)$. We use the same notation for the Fr\'echet derivative as in Section~\ref{supp_sec_prove_sym}. Also, throughout the proof, all constants $C,c>0$ are universal and may change from line to line. We typically use $c_0$ as a constant that can be sufficiently small. Let $\rho=1-\eta\alpha\sigma_{\min}/4$.

Define the weighted Frobenius and $2\to\infty$ norms by $\|(\bA,\bB)\|_{\Fnw} := \big(n^{-1}\|\bA\|_{\Fn}^2+q^{-1}\|\bB\|_{\Fn}^2\big)^{1/2}$ and $\|(\bA,\bB)\|_{\twinfw} := \max\big\{ \|\bA\|_{2\to\infty},\|\bB\|_{2\to\infty} \big\}$.
Throughout the asymmetric proofs, we slightly abuse the notation to denote
\begin{align*}
&\bZ:=\begin{pmatrix}\bU\\ \bV\end{pmatrix},\qquad
\bS_Z:=\begin{pmatrix}q^{-1/2}\bI_n&\zero\\ \zero&n^{-1/2}\bI_q\end{pmatrix},\qquad\\
&\bz:=\begin{pmatrix}\mathrm{vec}(\bU^\T)\\ \mathrm{vec}(\bV^\T)\end{pmatrix},\qquad
\bS_z:=\begin{pmatrix}q^{-1/2}\bI_{nr}&\zero\\ \zero&n^{-1/2}\bI_{qr}\end{pmatrix}.
\end{align*}
We also use the rectangular penalized objective $h_\alpha^\natural(\bU,\bV):=\cL(\bU\bV^\T)+p_\alpha^\natural(\bU,\bV)$. Recall that we have defined
\begin{equation*}
\tau_*^2:=\frac{n^{-1}\|\bU^*\|_{\Fn}^2+q^{-1}\|\bV^*\|_{\Fn}^2}{2}, \qquad \omega_*:=\|\bU^*\|_{2\to\infty}\vee\|\bV^*\|_{2\to\infty}.
\end{equation*}
For the block score, define the extended score matrix
\begin{equation*}
\cG_e(\bU,\bV):= \begin{pmatrix}
\zero & \cG(\bU,\bV)\otimes \bI_r\\ \cG(\bU,\bV)^\T\otimes \bI_r & \zero
\end{pmatrix}.
\end{equation*}

We will repeatedly use the simple consequences of the current theorem scaling, $\tau_*^2\le r\kappa\sigma_{\min}$, $\beta\ge \alpha$, and $\kappa\ge 1$:
\begin{equation*}
    \phi_{nq}\tau_*\le c_0\alpha\sqrt{\sigma_{\min}/\kappa}.
\end{equation*}

We introduce the following technical lemmas. The first lemma is the asymmetric analogue of Lemma~\ref{lemma_general_convex}. 
\begin{lemma} \label{lemma_general_convex_asym}
Under Assumption~\ref{assump_rect_rip_weighted}, 
we have
\begin{equation}\begin{aligned}
\min_{(\bU,\bV)\in\cD}\lambda_{\min}\!\Big[\bS_z\big\{ \nabla_{\bz}^2 h_{\alpha}^*(\bU,\bV)-  &\, \cG_e(\bU,\bV) \big\}\bS_z\Big] \ge {\alpha}\sigma_{\min} \\&\, - \max_{(\bU,\bV)\in\cD}\alpha\Big(\frac{\|\bU-\bU^*\|_{\Fn}^2}{n}+\frac{\|\bV-\bV^*\|_{\Fn}^2}{q}\Big),
\end{aligned}\label{eq_lambda_min_asym_lower}\end{equation}
where $\sigma_{\min} = \sigma_r(\bSigma^*)$.
 Moreover, for $h_{\alpha}^{\natural}(\bU,\bV) = \cL(\bU\bV^\T) + p_{\alpha}^{\natural}(\bU,\bV)$ with $p_{\alpha}^{\natural} = {\alpha}nq\|\bM^*(\bU,\bV)\|_{\Fn}^2/4$, we have
 \begin{equation*}
     \min_{(\bU,\bV)\in\cD}\lambda_{\min}\!\Big[ \bS_z\big\{\nabla_{\bz}^2 h_{\alpha}^{\natural}(\bU,\bV)-\cG_e(\bU,\bV)\big\}\bS_z \Big] \ge (1-c_0){\alpha\sigma_{\min}},
 \end{equation*}
 given that $\cD\subseteq \Big\{(\bU,\bV):{n^{-1}\|\bU-\bU^*\|_{\Fn}^2} + q^{-1}{\|\bV-\bV^*\|_{\Fn}^2}\le c_0\sigma_{\min}/\kappa\Big\}$ with sufficiently small $c_0>0$.
\end{lemma}
\begin{proof}
    See Section~\ref{supp_prove_lemma_general_convex_asym}.
\end{proof}
The next lemma is the asymmetric analogue of Lemma~\ref{coro_thm_general_convex}.
\begin{lemma}\label{lemma_rect_block_hessian}
Under Assumption~\ref{assump_rect_rip_weighted}, for every $(\bU,\bV)\in\cD$ and every $i\in[n]$,
\begin{equation*}
\alpha\,\lambda_{\min}(\bV^\T\bV) \le \lambda_{\min}\!\big\{\nabla^2_{\bu_i}\cL(\bU\bV^\T)\big\} \le \lambda_{\max}\!\big\{\nabla^2_{\bu_i}\cL(\bU\bV^\T)\big\} \le \beta\,\lambda_{\max}(\bV^\T\bV), \end{equation*}
and for every $\ell\in[q]$,
\begin{equation*}
\alpha\,\lambda_{\min}(\bU^\T\bU) \le \lambda_{\min}\!\big\{\nabla^2_{\bv_\ell}\cL(\bU\bV^\T)\big\} \le \lambda_{\max}\!\big\{\nabla^2_{\bv_\ell}\cL(\bU\bV^\T)\big\} \le \beta\,\lambda_{\max}(\bU^\T\bU).
\end{equation*}
Consequently,
\begin{align*}
\min_{(\bU,\bV)\in\cD}\min_{i\in[n]} \lambda_{\min}\!\big\{\nabla^2_{\bu_i}\cL(\bU\bV^\T)\big\} &\ge \alpha q\sigma_r(q^{-1/2}\bV^*) - \alpha\max_{(\bU,\bV)\in\cD} \big\|\bV^\T\bV-(\bV^*)^\T\bV^*\big\|,\\ 
\max_{(\bU,\bV)\in\cD}\max_{i\in[n]} \lambda_{\max}\!\big\{\nabla^2_{\bu_i}\cL(\bU\bV^\T)\big\} &\le \beta q\kappa\sigma_r(q^{-1/2}\bV^*) + \beta\max_{(\bU,\bV)\in\cD} \big\|\bV^\T\bV-(\bV^*)^\T\bV^*\big\|,\\
\min_{(\bU,\bV)\in\cD}\min_{\ell\in[q]} \lambda_{\min}\!\big\{\nabla^2_{\bv_\ell}\cL(\bU\bV^\T)\big\} &\ge \alpha n\sigma_r(n^{-1/2}\bU^*) - \alpha\max_{(\bU,\bV)\in\cD} \big\|\bU^\T\bU-(\bU^*)^\T\bU^*\big\|,\\
\max_{(\bU,\bV)\in\cD}\max_{\ell\in[q]} \lambda_{\max}\!\big\{\nabla^2_{\bv_\ell}\cL(\bU\bV^\T)\big\} &\le \beta n\kappa\sigma_r(n^{-1/2}\bU^*) + \beta\max_{(\bU,\bV)\in\cD} \big\|\bU^\T\bU-(\bU^*)^\T\bU^*\big\|.
\end{align*}
\end{lemma}
\begin{proof} The proof follows from the same argument as in Lemma~\ref{lemma_general_convex_asym} by restricting the perturbation to the form $\bW=\be_i\ba^\T$ for $\ba\in\RR^r$. We omit the routine details. Different from Lemma~\ref{coro_thm_general_convex}, $\cG$ is not involved here as it appears only in the off-diagonal parts in $\nabla_{\bz}^2h\prect(\bU,\bV)$.\end{proof}

Unlike the symmetric case, the alignment in the asymmetric model is taken over $GL(r)$ rather than $\cO^r$. The next lemma records the first-order optimality condition for the alignment.
\begin{lemma}\label{lemma_id_rec}
For any $\bU\in\RR^{n\times r}$ and $\bV\in\RR^{q\times r}$, define
\begin{equation*}
\Phi(\bG):=\|(\bU\bG-\bU^*,\,\bV\bG^{-\T}-\bV^*)\|_{\Fnw}^2, \qquad \bG\in GL(r).
\end{equation*}
If $\bI_r$ is a local minimizer of $\Phi(\cdot)$, then $\bM^*(\bU,\bV) = 0$. Equivalently,
\begin{equation}
p_\alpha^\natural(\bU,\bV) = \frac{\alpha nq}{4} \big\| n^{-1}(\bU-\bU^*)^\T\bU- q^{-1}\bV^\T(\bV-\bV^*)  \big\|_{\Fn}^2 =0. \label{st_id_rec}
\end{equation}
\end{lemma}
\begin{proof}See Section~\ref{supp_Sec_prove_lemma_id_rec}.\end{proof}

The following lemma shows that if an iterate is already close to the truth after a well-conditioned linear transform, then the optimal alignment exists and remains close to the reference rotation. This result will be repeatedly invoked to justify the existence of $\bG_t^*$ and to transfer estimates for the one-step aligned iterate to the optimally aligned iterate.
\begin{lemma}\em\label{lemma_alignment_near_rotation}
    Fix $\bm{Z}=\left[\begin{array}{c} \bm{U}\\ \bm{V} \end{array}\right]\in\mathbb{R}^{(n + q)\times r}$. Suppose that there exists a matrix $\bm{P}\in\mathbb{R}^{r\times r}$
    with $2/3\leq\sigma_{r}(\bm{P})\leq\sigma_{1}(\bm{P})\leq3/2$ such that 
    \begin{equation}
    \max\big\{ n^{-1/2}\| \bm{U}\bm{P}-\bm{U}^{\ast}\| _{\mathrm{F}},q^{-1/2}\| \bm{V}\bm{P}^{-\T}-\bm{V}^{\ast}\| _{\mathrm{F}}\big\} \leq\delta\leq \frac{1}{80}\big(\sigma_{r}(n^{-1/2}\bm{U}^{\ast})\wedge \sigma_r(q^{-1/2}\bV^{\ast})\big).\label{eq_assumption_alignment_near_rotation}
    \end{equation}
    Then the optimal alignment matrix $\bm{Q}^{\ast}\in\mathbb{R}^{r\times r}$ between $\bm{Z}$ and $\bm{Z}_{\ast}$ exists. In addition, the matrix $\bm{Q}^{\ast}$ satisfies 
    \begin{equation*}
    \| \bm{P}-\bm{Q}^{\ast}\| \le \| \bm{P}-\bm{Q}^{\ast}\| _{\mathrm{F}}\leq\frac{5\delta}{\sigma_r(\bU^*/\sqrt{n})},\;\| \bm{P}^{-\T}-{\bQ^{\ast}}^{-\T}\| \leq\| \bm{P}^{-\T}-{\bQ^{\ast}}^{-\T}\| _{\mathrm{F}}\leq\frac{5\delta}{\sigma_r(\bV^*/\sqrt{q})}.
    \end{equation*}
    \end{lemma}
\begin{proof}
    See Section~\ref{subsec_prove_lemma_alignment_near_rotation}.
\end{proof}

\subsection{Proof of Theorem~\ref{thm_general_theory_uv}}\label{supp_prove_thm_general_theory_uv}
We prove in this subsection the stronger contraction statement under the localization $\cDuvinf$. Accordingly, throughout the proof we invoke Assumptions~\ref{assump_rect_rip_weighted},~\ref{assump_rect_lip_1},~\ref{assump_row_cross_curvature_uv}, and~\ref{assump_rect_lip} with $\cD=\cDuvinf$. The $\ell_2$-only part of Theorem~\ref{thm_general_theory_uv}, where only Assumptions~\ref{assump_rect_rip_weighted}--\ref{assump_rect_lip_1} are imposed with $\cD=\cDuvtwo$, is proved later in Section~\ref{supp_sec_prove_thm_general_theory_noisy_uv_l2} by the same argument after removing the row-wise estimates.
Throughout this proof, $\eta^t\equiv \eta=\{10(\alpha+\beta)\kappa\sigma_{\min}\}^{-1}$. Besides the error bounds, we also track the existence of the optimal invertible alignment and its proximity to the reference orthogonal matrix $\bR^0$. In particular, for each $t\ge 0$, define the optimal invertible alignment, 
\begin{equation*}
\bG_t^*\in\argmin_{\bG\in GL(r)} \|(\bU^t\bG-\bU^*,\,\bV^t\bG^{-\T}-\bV^*)\|_{\Fnw}, \qquad \bLambda_t^*:=(\bG_t^*)^\T\bG_t^*,
\end{equation*}
whenever the minimum is attained. Let
\begin{equation*}
\tilde\bU^t:=\bU^t\bG_t^*, \qquad
\tilde\bV^t:=\bV^t(\bG_t^*)^{-\T}.
\end{equation*}
We also write
\begin{equation*}
\bE_U^t:=\tilde\bU^t-\bU^*, \qquad \bE_V^t:=\tilde\bV^t-\bV^*, \qquad
\be_t:= \begin{bmatrix} \mathrm{vec}\big((\bE_U^t)^\T\big)\\ \mathrm{vec}\big((\bE_V^t)^\T\big) \end{bmatrix}.
\end{equation*}
Then, one can easily see
\begin{equation*}
\|(\bE_U^t,\bE_V^t)\|_{\Fnw} = (nq)^{-1/2}\|\bS_z^{-1}\be_t\|, \qquad \|(\bE_U^t,\bE_V^t)\|_{\twinfw} = \|\bE_U^t\|_{2\to\infty}\vee \|\bE_V^t\|_{2\to\infty}.
\end{equation*}

Let $\iota_0\in(0,1/30)$ be a sufficiently small constant. We prove by induction on $t$ that, for all $0\le s\le t$, the following hold:
\begin{enumerate}[label=(\arabic*)]
    \item\label{item_uv_ind_exist_new} the optimal alignment matrix $\bG_s^*$ exists;
    \item\label{item_uv_ind_l2_new} $\|(\bE_U^s,\bE_V^s)\|_{\Fnw}\le \rho^s\phi_{nq}\tau_*$;
    \item\label{item_uv_ind_linf_new}$\|(\bE_U^s,\bE_V^s)\|_{\twinfw}\le \rho^s\psi_{nq}\omega_*$;
    \item\label{item_uv_ind_rot_new} $\|\bG_s^*-\bR^0\|\vee \|(\bG_s^*)^{-\T}-\bR^0\|\le \iota_0{\alpha}/(\beta{\kappa})$.
\end{enumerate}
Here $\bR^0\in\cO^r$ is the orthogonal matrix in \eqref{eq_init_consis_main_R0}. 
Since $\bR^0\in\cO^r$, we may equivalently replace $(\bU,\bV)$ by $(\bU\bR^0,\bV(\bR^0)^{-\T})=(\bU\bR^0,\bV\bR^0)$. Thus one may work with $\bR^0=\bI_r$ throughout. We keep $\bR^0$ only to make the alignment argument transparent.

Similar to the proof in the symmetric model, define the interpolation segment parameterized by $s\in[0,1]$ as
\begin{equation*}
\bU_t(s):=\bU^*+s\bE_U^t, \qquad \bV_t(s):=\bV^*+s\bE_V^t.
\end{equation*}
By the induction hypotheses,
\begin{equation*}
\|(\bE_U^t,\bE_V^t)\|_{\Fnw}\le \phi_{nq}\tau_*, \qquad \|(\bE_U^t,\bE_V^t)\|_{\twinfw}\le \psi_{nq}\omega_*.
\end{equation*}
Hence, one can check that, for every $s\in[0,1]$,
\begin{equation*}
\disttwo\{(\bU_t(s),\bV_t(s)),(\bU^*,\bV^*)\} \le s\|(\bE_U^t,\bE_V^t)\|_{\Fnw} \le \phi_{nq}\tau_*,
\end{equation*}
and similarly
\begin{equation*}
\distinf\{(\bU_t(s),\bV_t(s)),(\bU^*,\bV^*)\} \le s\|(\bE_U^t,\bE_V^t)\|_{\twinfw} \le \psi_{nq}\omega_*.
\end{equation*}
Since the current theorem assumes $\phi_{nq}\le \epsilon/2$ and $\psi_{nq}\le \epsilon/2$, the whole segment $\{(\bU_t(s),\bV_t(s)):0\le s\le 1\}$ lies in $\cDuvinf$. Therefore every assumption invoked below is available uniformly along this path.


Next, recall that in the asymmetric model, we can equivalently formulate the nonconvex procedure approximately into a strongly convex one within the local region, after imposing the optimal alignment $\bG_t^*$. Specifically, define the one-step iterate aligned by $\bG_t^*$:
\begin{equation*}
\bar\bU^{t+1}:=\bU^{t+1}\bG_t^*, \qquad \bar\bV^{t+1}:=\bV^{t+1}(\bG_t^*)^{-\T}.
\end{equation*}
Since $(\tilde\bU^t,\tilde\bV^t)$ optimally aligns with $(\bU^*,\bV^*)$ provided that $\bG_t^*$ exists according to the induction hypothesis~\ref{item_uv_ind_exist_new}, Lemma~\ref{lemma_id_rec} implies
\begin{equation*}
n^{-1}(\tilde\bU^t-\bU^*)^\T\tilde\bU^t = q^{-1}\tilde\bV^{t\T}(\tilde\bV^t-\bV^*),
\end{equation*}
and hence we know $\nabla_{\bU}p_\alpha^\natural(\tilde\bU^t,\tilde\bV^t)=\zero$ and $\nabla_{\bV}p_\alpha^\natural(\tilde\bU^t,\tilde\bV^t)=\zero$.
Therefore the gradient update can be rewritten as
\begin{align}
\bar\bU^{t+1} &= \tilde\bU^t -\frac{\eta}{q}\nabla_{\bU}h_\alpha^\natural(\tilde\bU^t,\tilde\bV^t) -\frac{\eta}{q}\nabla_{\bU}h_\alpha^\natural(\tilde\bU^t,\tilde\bV^t)(\bLambda_t^*-\bI_r), \label{eq_uv__U_new} \\
\bar\bV^{t+1} &= \tilde\bV^t -\frac{\eta}{n}\nabla_{\bV}h_\alpha^\natural(\tilde\bU^t,\tilde\bV^t) -\frac{\eta}{n}\nabla_{\bV}h_\alpha^\natural(\tilde\bU^t,\tilde\bV^t)\big\{(\bLambda_t^*)^{-1} -\bI_r\big\}. \label{eq_uv__V_new}
\end{align}

Now we proceed to prove the induction hypotheses~\ref{item_uv_ind_exist_new}--\ref{item_uv_ind_rot_new}, starting from $t=0$.
\medskip
\paragraph{\it\underline{Initialization.}}
We first verify the auxiliary existence and rotation statements at $t=0$, and then deduce the error bounds. By \eqref{eq_init_consis_main_R0},
\begin{equation*}
\max\Big\{ n^{-1/2}\|\bU^0\bR^0-\bU^*\|_{\Fn}, \, q^{-1/2}\|\bV^0(\bR^0)^{-\T}-\bV^*\|_{\Fn} \Big\} \le \phi_{nq}\tau_*.\end{equation*}
Since $\tau_*^2 = (n^{-1}\|\bU^*\|_{\Fn}^2+q^{-1}\|\bV^*\|_{\Fn}^2)/2 \le r\kappa\sigma_{\min}$, we have by the scaling condition $\phi_{nq}\le c_0\alpha/(\beta^2\kappa\sqrt{r\kappa})$ that
\begin{equation*}
\phi_{nq}\tau_* \le c_0\frac{\alpha}{\beta^2\kappa\sqrt{r\kappa}}\sqrt{r\kappa\sigma_{\min}} = c_0\frac{\alpha}{\beta^2\kappa}\sqrt{\sigma_{\min}}.
\end{equation*}
Applying Lemma~\ref{lemma_alignment_near_rotation} with $\bm P=\bR^0$, we see that $\bG_0^*$ exists and
\begin{align*}
\|\bG_0^*-\bR^0\|\vee \|(\bG_0^*)^{-\T}-\bR^0\| &\le 5\phi_{nq}\tau_* \Big( \frac{\sigma_r(\bU^*)}{\sqrt n}\wedge \frac{\sigma_r(\bV^*)}{\sqrt q} \Big)^{-1} \\
&\le 5c_0 \frac{\alpha\sqrt{\sigma_{\min}}}{ \beta{\kappa^2}\big(n^{-1/2}\sigma_r(\bU^*)\wedge q^{-1/2}\sigma_r(\bV^*)\big)} \\
&\le
\iota_0\frac{\alpha}{\beta{\kappa}}, \end{align*}
with $\iota_0<10^{-1}$ after shrinking $c_0$ if necessary. This immediately proves \ref{item_uv_ind_exist_new} and \ref{item_uv_ind_rot_new} at $t=0$. 
Since $\bG_0^*$ minimizes the weighted Frobenius distance, the initialization bound immediately yields the weighted $\ell_2$ statement \eqref{item_uv_ind_l2_new} at $t=0$. Under the additional row-wise initialization bound \eqref{eq_sym_init_Z_uniform_uv}, bound for $\|\bG_0^* - \bR^0\|\vee \|(\bG_0^*)^{-\T} - \bR^0\|$ above gives
\begin{equation*}
\|\bU^0\bG_0^*-\bU^*\|_{2\to\infty} \le \|\bU^0\bR^0-\bU^*\|_{2\to\infty} + \|\bU^0\|_{2\to\infty}\|\bG_0^*-\bR^0\|,
\end{equation*}
and similarly for the $V$-block. This yields \eqref{item_uv_ind_linf_new} at $t=0$. 
Now assume \ref{item_uv_ind_exist_new}--\ref{item_uv_ind_rot_new} hold for all $0\le s\le t$. We prove them for $t+1$.

\medskip
\paragraph{\it \underline{ Step 1: weighted $\ell_2$ error contraction.}}
Let
\begin{equation*}
\bar\be_{t+1}:= \begin{pmatrix} \mathrm{vec}\big((\bar\bU^{t+1}-\bU^*)^\T\big)\\ \mathrm{vec}\big((\bar\bV^{t+1}-\bV^*)^\T\big) \end{pmatrix}, \qquad 
\bw_t:=\bS_z^{-1}\be_t, \qquad \bar\bw_{t+1}:=\bS_z^{-1}\bar\be_{t+1},
\end{equation*}
and
\begin{equation*}
\vartheta_t:=\|\bLambda_t^*-\bI_r\|\vee\|(\bLambda_t^*)^{-1}-\bI_r\|.
\end{equation*}
Since $\cG(\bU^*,\bV^*)=\nabla_{\bX}\cL(\bX^*)=\zero$ in the noiseless case and
\begin{equation*}
\nabla_{\bU}p_\alpha^\natural(\tilde\bU^t,\tilde\bV^t) = \nabla_{\bV}p_\alpha^\natural(\tilde\bU^t,\tilde\bV^t) = \nabla_{\bU}p_\alpha^\natural(\bU^*,\bV^*) = \nabla_{\bV}p_\alpha^\natural(\bU^*,\bV^*) = \zero,
\end{equation*}
by Lemma~\ref{lemma_id_rec}, we have $\nabla_{\bU}h_\alpha^\natural(\bU^*,\bV^*)=\zero,$ and
$\nabla_{\bV}h_\alpha^\natural(\bU^*,\bV^*)=\zero.$
Therefore, \eqref{eq_uv__U_new}--\eqref{eq_uv__V_new} lead to 
\begin{align*}
\bar\bw_{t+1} =&\, \bw_t - \eta\bS_z\mathrm{vec}\Big\{ \big( \nabla_{\bZ}h_\alpha^\natural(\tilde\bU^t,\tilde\bV^t) - \nabla_{\bZ}h_\alpha^\natural(\bU^*,\bV^*) \big)^\T \Big\} \\&\,-
\eta\cR_t\bS_z\mathrm{vec}\Big\{ \big( \nabla_{\bZ}h_\alpha^\natural(\tilde\bU^t,\tilde\bV^t) - \nabla_{\bZ}h_\alpha^\natural(\bU^*,\bV^*)\big)^\T \Big\},
\end{align*}
where we let
\begin{equation*}
\cR_t:= \begin{pmatrix}(\bLambda_t^*-\bI_r)\otimes \bI_n & \zero\\ \zero & \big((\bLambda_t^*)^{-1}-\bI_r\big)\otimes \bI_q\end{pmatrix}
\end{equation*}
By the fundamental theorem of calculus (Theorem 4.2 in \citet{lang2012real}, Chapter XIII), we have the expansion
\begin{equation*}
\mathrm{vec}\Big\{ \big( \nabla_{\bZ}h_\alpha^\natural(\tilde\bU^t,\tilde\bV^t) - \nabla_{\bZ}h_\alpha^\natural(\bU^*,\bV^*) \big)^\T \Big\} = \Big( \int_0^1 \nabla_{\bz}^2 h_\alpha^\natural(\bU_t(s),\bV_t(s))\,ds \Big)\be_t.
\end{equation*}
Similar to the decomposition in the proof of Theorem~\ref{thm_general_theory}, define
\begin{equation*}
\bar\cA_t := \bS_z \int_0^1 \big\{ \nabla_{\bz}^2 h_\alpha^\natural(\bU_t(s),\bV_t(s)) - \cG_e(\bU_t(s),\bV_t(s)) \big\} \,ds\, \bS_z,
\end{equation*}
and
\begin{equation*}
\bar\cG_t := \bS_z \int_0^1 \big\{ \cG_e(\bU_t(s),\bV_t(s)) - \cG_e(\bU^*,\bV^*) \big\} \,ds\, \bS_z.
\end{equation*}
Since $\cG(\bU^*,\bV^*)=\zero$, we have $\cG_e(\bU^*,\bV^*)=\zero$, which therefore gives
\begin{equation*}
\bS_z\mathrm{vec}\Big\{\big( \nabla_{\bZ}h_\alpha^\natural(\tilde\bU^t,\tilde\bV^t) - \nabla_{\bZ}h_\alpha^\natural(\bU^*,\bV^*) \big)^\T \Big\} = (\bar\cA_t+\bar\cG_t)\bw_t.
\end{equation*}
Hence, the expansion becomes 
\begin{equation*}
\bar\bw_{t+1} = (\bI_{nr+qr}-\eta\bar\cA_t)\bw_t - \eta\bar\cG_t\bw_t - \eta\cR_t(\bar\cA_t+\bar\cG_t)\bw_t.
\end{equation*}
We next decompose the right-hand side into
\begin{equation*}
\gamma_{1,t}:=\|(\bI_{nr+qr}-\eta\bar\cA_t)\bw_t\|, \qquad
\gamma_{2,t}:=\eta\|\bar\cG_t\bw_t\|,
\end{equation*}
\begin{equation*}
\gamma_{3,t}:=\eta\|\cR_t(\bar\cA_t+\bar\cG_t)\bw_t\| \le \eta\vartheta_t\big(\|\bar\cA_t\bw_t\|+\|\bar\cG_t\bw_t\|\big),
\end{equation*}
so that $ \|\bar\bw_{t+1}\|\le \gamma_{1,t}+\gamma_{2,t}+\gamma_{3,t}$.
Here, $\gamma_{1,t}$--$\gamma_{3,t}$ play roles analogous to those of the quantities denoted by the same notation in Step 1 of the proof of Theorem~\ref{thm_general_theory} in Section~\ref{supp_sec_prove_thm_general_theory}.

We first bound $\gamma_{1,t}$. Note that we have shown that the interpolation segment lies in $\cDuvinf$, so Lemma~\ref{lemma_general_convex_asym} applies with $\cD=\cDuvinf$. In addition, note that
\begin{equation*}
n^{-1}\|\bE_U^t\|_{\Fn}^2+q^{-1}\|\bE_V^t\|_{\Fn}^2 = \|(\bE_U^t,\bE_V^t)\|_{\Fnw}^2 \le \phi_{nq}^2\tau_*^2 .
\end{equation*}
With the theorem's scaling condition $\phi_{nq}\tau_*\le c_0\alpha\sqrt{\sigma_{\min}/\kappa}$, we know this quantity is at most $c_0\sigma_{\min}/\kappa$ after shrinking $c_0$ if necessary. Substituting this localization bound into Lemma~\ref{lemma_general_convex_asym} yields,  after shrinking $c_0$, $\lambda_{\min}(\bar\cA_t)\ge 7{\alpha\sigma_{\min}}/{8}$.
To control the maximal eigenvalue, we note that, with Weyl's inequality, we know that
\begin{equation*}
\frac{\|\bU_t(s)\|^2}{n}\vee \frac{\|\bV_t(s)\|^2}{q}\le  \frac{\|\bU_t(s)\|^2 + \|\bE_U^t\|^2_{\Fn}}{n}\vee \frac{\|\bV_t(s)\|^2 + \|\bE_V^t\|^2_{\Fn}}{q}\le
1.01\kappa\sigma_{\min},
\end{equation*}
The same argument used in proving Lemma~\ref{lemma_general_convex_asym} yields $\lambda_{\max}(\bar\cA_t)\le 1.01(\alpha+\beta)\kappa\sigma_{\min}$.
Since $\eta\le \{10(\alpha+\beta)\kappa\sigma_{\min}\}^{-1}$, all eigenvalues of
$\bI_{nr+qr}-\eta\bar\cA_t$ lie in $[0,1)$, and we get
\begin{equation*}
\gamma_{1,t} \le \Big(1-\eta\lambda_{\min}(\bar\cA_t)\Big)\|\bw_t\| \le \Big(1-\frac78\eta\alpha\sigma_{\min}\Big)\|\bw_t\|.
\end{equation*}

Next, we bound $\gamma_{2,t}$.
By the definition of $\bar\cG_t$,
\begin{align*}
\gamma_{2,t} \le\;& \frac{\eta}{\sqrt{nq}} \Big\| \int_0^1 \big\{\cG(\bU_t(s),\bV_t(s))-\cG(\bU^*,\bV^*)\big\}\,ds \times \sqrt n\,\bE_V^t \Big\|_{\Fn} \\
&\;+ \frac{\eta}{\sqrt{nq}} \Big\| \int_0^1 \big\{\cG(\bU_t(s),\bV_t(s))-\cG(\bU^*,\bV^*)\big\}^\T\,ds \times \sqrt q\,\bE_U^t \Big\|_{\Fn}.
\end{align*}
By Assumption~\ref{assump_rect_lip_1}, we know
\begin{equation*}
\|\cG(\bU_t(s),\bV_t(s))-\cG(\bU^*,\bV^*)\| \le
L_2\|\bU_t(s)\bV_t(s)^\T-\bU^*(\bV^*)^\T\|_{\Fn}.
\end{equation*}
Also, by telescoping and $\|\cdot\|\le \|\cdot\|_{\Fn}$, one has
\begin{align*}
\|\bU_t(s)\bV_t(s)^\T-\bU^*(\bV^*)^\T\|_{\Fn}  &\le \|\bE_U^t\|_{\Fn}\|\bV^*\| + \|\bU^*\|\|\bE_V^t\|_{\Fn} + \|\bE_U^t\|_{\Fn}\|\bE_V^t\|_{\Fn} \\
&\le C\sqrt{nq}\sqrt{\kappa\sigma_{\min}}\, \|(\bE_U^t,\bE_V^t)\|_{\Fnw},
\end{align*}
where in the last step we used $q^{-1}\|\bV^*\|^2+n^{-1}\|\bU^*\|^2\le 2\kappa\sigma_{\min}$ and
\begin{equation*}
\frac{\sqrt n\,\|\bE_V^t\|_{\Fn}+\sqrt q\,\|\bE_U^t\|_{\Fn}}{\sqrt{nq}} \le C\|(\bE_U^t,\bE_V^t)\|_{\Fnw}.
\end{equation*}
We therefore conclude that
\begin{equation*}
\gamma_{2,t} \le C\eta\sqrt{nq}\sqrt{\kappa\sigma_{\min}}\, \|(\bE_U^t,\bE_V^t)\|_{\Fnw}^2.
\end{equation*}
From the scaling condition, we know $\sqrt{\kappa\sigma_{\min}}\phi_{nq}\tau_*\le c_0\alpha\sigma_{\min}/\beta$, Then by the induction hypothesis~\ref{item_uv_ind_l2_new}, we know $ \|(\bE_U^t,\bE_V^t)\|_{\Fnw} \le \phi_{nq}\tau_* \le c_0\alpha\sqrt{{\sigma_{\min}}/{\kappa}}$,
which in turn gives $ \sqrt{\kappa\sigma_{\min}}\, \|(\bE_U^t,\bE_V^t)\|_{\Fnw} \le c_0\alpha\sigma_{\min}$. After shrinking $c_0$ if necessary, we arrive at
\begin{equation*}
\gamma_{2,t} \le \frac{1}{16}\eta\alpha\sigma_{\min}\|\bw_t\|.
\end{equation*}

For $\gamma_{3,t}$, we now substitute the operator bound on $\bar\cA_t$ and the estimate for $\gamma_{2,t}$ into the alignment-imbalance term, which gives
\begin{equation*}
\gamma_{3,t} \le \eta\vartheta_t\|\bar\cA_t\bw_t\| + \eta\vartheta_t\|\bar\cG_t\bw_t\| \le \eta\vartheta_t\lambda_{\max}(\bar\cA_t)\|\bw_t\| + \vartheta_t\gamma_{2,t}.
\end{equation*}
One then has $\gamma_{3,t} \le 1.01\vartheta_t(\alpha+\beta)\kappa\sigma_{\min}\eta\|\bw_t\| + \vartheta_t\gamma_{2,t}$.
Next, by telescoping and $(\bR^0)^\T\bR^0 = \bR^0(\bR^0)^\T = \bI_r$, we have
    \begin{align*}
\|\bLambda_t^*-\bI_r\| &\le 2\|\bG_t^*-\bR^0\|+\|\bG_t^*-\bR^0\|^2,\\ \|(\bLambda_t^*)^{-1}-\bI_r\| &\le 2\|(\bG_t^*)^{-\T}-\bR^0\|+\|(\bG_t^*)^{-\T}-\bR^0\|^2.\end{align*}
Therefore, with the induction hypothesis~\ref{item_uv_ind_rot_new} at $t$, we arrive at
\begin{equation}
\vartheta_t = \|\bLambda_t^*-\bI_r\|\vee\|(\bLambda_t^*)^{-1}-\bI_r\| \le \frac{\alpha}{10\beta\kappa}, \label{eq_uv_balance_smallness_step1}
\end{equation}
after shrinking $\iota_0$ if necessary. Therefore,
\begin{equation*}
\gamma_{3,t} \le \frac{5}{8}\eta\alpha\sigma_{\min}\|\bw_t\|.
\end{equation*}

Combining the three estimates gives $\|\bar\bw_{t+1}\|\le \big(1-\frac14\eta\alpha\sigma_{\min}\big)\|\bw_t\| = \rho\|\bw_t\|$. Dividing by $\sqrt{nq}$, we obtain $\|(\bar\bU^{t+1}-\bU^*,\,\bar\bV^{t+1}-\bV^*)\|_{\Fnw}\le \rho\|(\bE_U^t,\bE_V^t)\|_{\Fnw}$.
By the induction hypothesis, we finally arrive at
\begin{equation}
\|(\bar\bU^{t+1}-\bU^*,\,\bar\bV^{t+1}-\bV^*)\|_{\Fnw} \le \rho^{t+1}\phi_{nq}\tau_*. \label{eq_uv__l2_new}
\end{equation}

\medskip
\paragraph{\it \underline{ Step 2: weighted $\ell_{\infty}$ error contraction.}}
We now bound $ \|(\bar\bU^{t+1}-\bU^*,\bar\bV^{t+1}-\bV^*)\|_{\twinfw}$. Recall that $ \bU_t(s):=\bU^*+s\bE_U^t$, and $\bV_t(s):=\bV^*+s\bE_V^t $ for $s\in[0,1]$. Since the penalty gradients vanish at both $(\tilde\bU^t,\tilde\bV^t)$ and $(\bU^*,\bV^*)$, the gradient difference in the row-wise mean-value expansion is the same as that for $\cL(\cdot)$. Hence, only the Hessian blocks of $\cL(\cdot)$ appear below.

To obtain a row-wise recursion, we now rewrite the mean-value expansion block by block. Specifically, for $i\in[n]$ and $\ell\in[q]$, define
\begin{align*}
\bar\cH_{ij,t}^{UU} &:= \frac{1}{q}\int_0^1 \nabla_{\bu_i\bu_j}^2 \cL\big(\bU_t(s),\bV_t(s)\big)\,ds,\\
\bar\cH_{i\ell,t}^{UV} &:= \frac{1}{q}\int_0^1 \Big\{\nabla_{\bu_i\bv_\ell}^2 \cL\big(\bU_t(s),\bV_t(s)\big)-G_{i\ell}\big(\bU_t(s),\bV_t(s)\big)\bI_r\Big\}\,ds,\\
\bar\cH_{\ell i,t}^{VU} &:= \frac{1}{n}\int_0^1 \Big\{\nabla_{\bv_\ell\bu_i}^2 \cL\big(\bU_t(s),\bV_t(s)\big)-G_{i\ell}\big(\bU_t(s),\bV_t(s)\big)\bI_r\Big\}\,ds,\\
\bar\cH_{\ell k,t}^{VV} &:= \frac{1}{n}\int_0^1 \nabla_{\bv_\ell\bv_k}^2 \cL\big(\bU_t(s),\bV_t(s)\big)\,ds,
\end{align*}
and
\begin{equation*}
\bar\cG_t:=\int_0^1 \cG\big(\bU_t(s),\bV_t(s)\big)\,ds =\int_0^1\Big\{\cG\big(\bU_t(s),\bV_t(s)\big)-\cG(\bU^*,\bV^*)\Big\}\,ds,
\end{equation*} where the second equality uses $\cG(\bU^*,\bV^*)=\zero$. Here, $\bar G_{i\ell,t}$ denotes the $(i,\ell)$ entry of $\bar\cG_t$.

By \eqref{eq_uv__U_new}--\eqref{eq_uv__V_new}, for each $i\in[n]$, we have
\begin{equation}\begin{aligned}
(\bar\bU^{t+1}-\bU^*)_{i,\cdot} =&\, \Big(\bI_r-\eta\bar\cH_{ii,t}^{UU}\Big)(\bE_U^t)_{i,\cdot} -\eta\sum_{j\neq i}\bar\cH_{ij,t}^{UU}(\bE_U^t)_{j,\cdot} -\eta\sum_{\ell=1}^q\bar\cH_{i\ell,t}^{UV}(\bE_V^t)_{\ell,\cdot}\\ &\, -\frac{\eta}{q}\sum_{\ell=1}^q \bar G_{i\ell,t}(\bE_V^t)_{\ell,\cdot} -\frac{\eta}{q}\Big[\nabla_{\bU}\cL(\tilde\bU^t,\tilde\bV^t)(\bLambda_t^*-\bI_r)\Big]_{i,\cdot},\end{aligned}\label{eq_uv_linf_row_U}
\end{equation}
and for each $\ell\in[q]$,
\begin{equation}\begin{aligned}
(\bar\bV^{t+1}-\bV^*)_{\ell,\cdot} =&\, \Big(\bI_r-\eta\bar\cH_{\ell\ell,t}^{VV}\Big)(\bE_V^t)_{\ell,\cdot} -\eta\sum_{k\neq \ell}\bar\cH_{\ell k,t}^{VV}(\bE_V^t)_{k,\cdot} -\eta\sum_{i=1}^n\bar\cH_{\ell i,t}^{VU}(\bE_U^t)_{i,\cdot}\\ &\,
-\frac{\eta}{n}\sum_{i=1}^n \bar G_{i\ell,t}(\bE_U^t)_{i,\cdot} -\frac{\eta}{n}\Big[\nabla_{\bV}\cL(\tilde\bU^t,\tilde\bV^t)\big\{(\bLambda_t^*)^{-1}-\bI_r\big\}\Big]_{\ell,\cdot}.
\end{aligned}\label{eq_uv_linf_row_V}
\end{equation}

Accordingly, define
\begin{equation*}
\delta_{m,t}:=\delta_{m,t}^{U}\vee \delta_{m,t}^{V}, \text{ for } m=1,2,3,4,
\end{equation*}
where each $\delta_{m,t}^{U}$ and $\delta_{m,t}^{V}$ are given as 
\begin{align*}
\delta_{1,t}^{U} &:= \max_{i\in[n]} \Big\| \big(\bI_r-\eta\bar\cH_{ii,t}^{UU}\big)(\bE_U^t)_{i,\cdot} \Big\|,\qquad\quad \delta_{1,t}^{V} := \max_{\ell\in[q]} \Big\| \big(\bI_r-\eta\bar\cH_{\ell\ell,t}^{VV}\big)(\bE_V^t)_{\ell,\cdot} \Big\|, \\ 
\delta_{2,t}^{U} &:= {\eta}\max_{i\in[n]} \Big\| \sum_{j\neq i}\bar\cH_{ij,t}^{UU}(\bE_U^t)_{j,\cdot} + \sum_{\ell=1}^q\bar\cH_{i\ell,t}^{UV}(\bE_V^t)_{\ell,\cdot} \Big\|,\\
\delta_{2,t}^{V} &:= {\eta}\max_{\ell\in[q]} \Big\| \sum_{k\neq \ell}\bar\cH_{\ell k,t}^{VV}(\bE_V^t)_{k,\cdot} + \sum_{i=1}^n\bar\cH_{\ell i,t}^{VU}(\bE_U^t)_{i,\cdot} \Big\|, \\ 
\delta_{3,t}^{U} &:= \frac{\eta}{q}\max_{i\in[n]} \Big\| \sum_{\ell=1}^q \bar G_{i\ell,t}(\bE_V^t)_{\ell,\cdot} \Big\|,\qquad\qquad \delta_{3,t}^{V} := \frac{\eta}{n}\max_{\ell\in[q]} \Big\| \sum_{i=1}^n \bar G_{i\ell,t}(\bE_U^t)_{i,\cdot} \Big\|, \\ 
\delta_{4,t}^{U} &:= \frac{\eta}{q}\Big\| \nabla_{\bU}\cL(\tilde\bU^t,\tilde\bV^t)(\bLambda_t^*-\bI_r)\Big\|_{2\to\infty},\quad \delta_{4,t}^{V} := \frac{\eta}{n}\Big\| \nabla_{\bV}\cL(\tilde\bU^t,\tilde\bV^t)\big\{(\bLambda_t^*)^{-1}-\bI_r\big\}\Big\|_{2\to\infty}.
\end{align*}
Equivalently, \eqref{eq_uv_linf_row_U} and \eqref{eq_uv_linf_row_V} are expressed as
\begin{equation}
\|(\bar\bU^{t+1}-\bU^*,\bar\bV^{t+1}-\bV^*)\|_{\twinfw} \le \delta_{1,t}+\delta_{2,t}+\delta_{3,t}+\delta_{4,t}. \label{eq_uv_linf__decomp}
\end{equation}
We next bound the four terms.

For $\delta_{1,t}$, similar to the proof of Lemma~\ref{lemma_rect_block_hessian}, one can check that, for every $s\in[0,1]$, we know
\begin{equation*}
\lambda_{\min}\big(\bar\cH_{ii,t}^{UU}\big)\ge \alpha q^{-1}\lambda_{\min}\big\{\bV_t(s)^\T\bV_t(s)\big\}, \qquad \lambda_{\max}\big(\bar\cH_{ii,t}^{UU}\big) \le \beta q^{-1}\lambda_{\max}\big\{\bV_t(s)^\T\bV_t(s)\big\},
\end{equation*}
and similarly
\begin{equation*}
\lambda_{\min}\big(\bar\cH_{\ell\ell,t}^{VV}\big) \ge \alpha n^{-1}\lambda_{\min}\big\{\bU_t(s)^\T\bU_t(s)\big\}, \qquad \lambda_{\max}\big(\bar\cH_{\ell\ell,t}^{VV}\big) \le \beta n^{-1}\lambda_{\max}\big\{\bU_t(s)^\T\bU_t(s)\big\}.
\end{equation*}
Next, because $q^{-1/2}\|\bV^*\|\le \sqrt{\kappa\sigma_{\min}}$, $\|(\bE_U^t,\bE_V^t)\|_{\Fnw}\le \phi_{nq}\tau_*$ by the induction hypothesis, and $\phi_{nq}\tau_*\le c_0\alpha\sqrt{\sigma_{\min}/\kappa}$ by the scaling condition
one has 
\begin{align*}
q^{-1}\big\|\bV_t(s)^\T\bV_t(s)-(\bV^*)^\T\bV^*\big\| &\le 2q^{-1/2}\|\bV^*\|\,q^{-1/2}\|\bE_V^t\|_{\Fn} + q^{-1}\|\bE_V^t\|_{\Fn}^2\\ &\le C\sqrt{\kappa\sigma_{\min}}\,\|(\bE_U^t,\bE_V^t)\|_{\Fnw} + \|(\bE_U^t,\bE_V^t)\|_{\Fnw}^2\\ &\le Cc_0\alpha\sigma_{\min}.
\end{align*}
Similarly, one can obtain $n^{-1}\big\|\bU_t(s)^\T\bU_t(s)-(\bU^*)^\T\bU^*\big\|
\le Cc_0\alpha\sigma_{\min}$.
Recall that we have assumed $q^{-1}\lambda_{\min}\big((\bV^*)^\T\bV^*\big) = n^{-1}\lambda_{\min}\big((\bU^*)^\T\bU^*\big) = \sigma_{\min}$. Then, after shrinking $c_0$ if necessary, we obtain
\begin{equation*}
\lambda_{\min}\big(\bar\cH_{ii,t}^{UU}\big)\wedge \lambda_{\min}\big(\bar\cH_{\ell\ell,t}^{VV}\big) \ge \frac78\alpha\sigma_{\min}.
\end{equation*}
Similarly, with $ q^{-1}\|\bV_t(s)\|^2 \le q^{-1}\|\bV^*\|^2+q^{-1}\|\bE_V^t\|_{\Fn}^2 \le 1.01\kappa\sigma_{\min} $,
and the same bound holds for $n^{-1}\|\bU_t(s)\|^2$. Again, following Step~2 in the proof of Theorem~\ref{thm_general_theory}, one can obtain
\begin{equation*}
\lambda_{\max}\big(\bar\cH_{ii,t}^{UU}\big)\vee \lambda_{\max}\big(\bar\cH_{\ell\ell,t}^{VV}\big) \le 1.01(\beta+\alpha)\kappa\sigma_{\min},
\end{equation*} after shrinking $c_0$ if necessary. 
Since $\eta\le \{10(\alpha+\beta)\kappa\sigma_{\min}\}^{-1}$, all eigenvalues of $\bI_r-\eta\bar\cH_{ii,t}^{UU}$ and $\bI_r-\eta\bar\cH_{\ell\ell,t}^{VV}$ lie in $[0,1)$, and therefore, we arrive at
\begin{equation}
\delta_{1,t} \le \Big(1-\frac78\eta\alpha\sigma_{\min}\Big)\|(\bE_U^t,\bE_V^t)\|_{\twinfw}. \label{eq_uv_delta1}
\end{equation}

For $\delta_{2,t}$, we again argue row by row. We first bound $\delta_{2,t}^{U}$.
Fix $i\in[n]$ and $\bh\in\RR^r$ with $\|\bh\|=1$, and define
\begin{equation*}
\bL_i:=\be_i^{(n)}\bh^\T,\qquad (\bL^{(i)})_{j,\cdot}:= \begin{cases} (\bE_U^t)_{j,\cdot}, & j\neq i,\\ \zero, & j=i, \end{cases} \qquad \bR:=\bE_V^t.
\end{equation*}
By the definitions of $\bar\cH_{ij,t}^{UU}$ and $\bar\cH_{i\ell,t}^{UV}$, we note that
\begin{align*}
&\,\bh^\T\Big\{ \sum_{j\neq i}\bar\cH_{ij,t}^{UU}(\bE_U^t)_{j,\cdot} + \sum_{\ell=1}^q\bar\cH_{i\ell,t}^{UV}(\bE_V^t)_{\ell,\cdot} \Big\} \\=&\,
\frac1q\int_0^1 \nabla_{\bX}^2\cL\big(\bU_t(s)\bV_t(s)^\T\big) \Big[ \cP_{(\bU_t(s),\bV_t(s))}(\zero,\bL_i),\, \cP_{(\bU_t(s),\bV_t(s))}(\bR,\bL^{(i)}) \Big]\,ds.
\end{align*}
Applying Assumption~\ref{assump_row_cross_curvature_uv}, we obtain
\begin{align*}
&\Big| \bh^\T\Big\{ \sum_{j\neq i}\bar\cH_{ij,t}^{UU}(\bE_U^t)_{j,\cdot} + \sum_{\ell=1}^q\bar\cH_{i\ell,t}^{UV}(\bE_V^t)_{\ell,\cdot} \Big\} \Big| \\
&\qquad\le \frac{\beta}{q}\int_0^1 \|\bV_t(s)\bh\|\, \|(\bU_t(s))_{i,\cdot}\bR^\T\|\,ds \\ &\qquad\le \frac{\beta}{q}\int_0^1 \|\bV_t(s)\|\,\|\bh\|\,\|(\bU_t(s))_{i,\cdot}\|\,\|\bR\|_{\Fn}\,ds \\ &\qquad\le \frac{\beta}{q}\, \sup_{s\in[0,1]}\|\bV_t(s)\|\, \sup_{s\in[0,1]}\|(\bU_t(s))_{i,\cdot}\|\, \|\bE_V^t\|_{\Fn}.
\end{align*}

By the induction hypothesis and the definition of the interpolation path, one can get
\begin{equation*}
\|\bU_t(s)\|_{2\to\infty}\vee \|\bV_t(s)\|_{2\to\infty} \le (1+\epsilon/2)\omega_* \le 2\omega_*.\end{equation*}
In addition, by Step~1, we already know $n^{-1/2}\|\bU_t(s)\|\vee q^{-1/2}\|\bV_t(s)\|\le C\sqrt{\kappa\sigma_{\min}}$. This further yields
\begin{align*}
&\Big|\bh^\T\Big\{ \sum_{j\neq i}\bar\cH_{ij,t}^{UU}(\bE_U^t)_{j,\cdot} + \sum_{\ell=1}^q\bar\cH_{i\ell,t}^{UV}(\bE_V^t)_{\ell,\cdot} \Big\} \Big| \\
&\qquad\le C\beta\, \frac{\sqrt q\,\sqrt{\kappa\sigma_{\min}}}{q}\, \omega_*\, \|\bE_V^t\|_{\Fn} = C\beta\omega_*\sqrt{\kappa\sigma_{\min}}\, \frac{\|\bE_V^t\|_{\Fn}}{\sqrt q}.
\end{align*}
Taking the supremum over all $\bh$ with $\|\bh\|=1$ yields
\begin{equation*}
\Big\| \sum_{j\neq i}\bar\cH_{ij,t}^{UU}(\bE_U^t)_{j,\cdot} + \sum_{\ell=1}^q\bar\cH_{i\ell,t}^{UV}(\bE_V^t)_{\ell,\cdot} \Big\| \le C\beta\omega_*\sqrt{\kappa\sigma_{\min}}\, \frac{\|\bE_V^t\|_{\Fn}}{\sqrt q}.
\end{equation*}
Since $q^{-1/2}\|\bE_V^t\|_{\Fn}\le \|(\bE_U^t,\bE_V^t)\|_{\Fnw}$ by definition, we conclude that
\begin{equation*}
\delta_{2,t}^{U} \le C\eta\beta\omega_*\sqrt{\kappa\sigma_{\min}}\, \|(\bE_U^t,\bE_V^t)\|_{\Fnw}.
\end{equation*}

Similarly, one can obtain $\delta_{2,t}^{V}\le C\eta\beta\omega_*\sqrt{\kappa\sigma_{\min}}\, \|(\bE_U^t,\bE_V^t)\|_{\Fnw}$. Combining the bounds for $\delta_{2,t}^{U}$ and $\delta_{2,t}^{V}$, we conclude that
\begin{equation}
\delta_{2,t} = \delta_{2,t}^{U}\vee \delta_{2,t}^{V} \le C\eta\beta\omega_*\sqrt{\kappa\sigma_{\min}}\, \|(\bE_U^t,\bE_V^t)\|_{\Fnw}.\label{eq_uv_delta2}
\end{equation}

For $\delta_{3,t}$, similar to bounding $\delta_{3,t}$ in the proof of Theorem~\ref{thm_general_theory} in Section~\ref{supp_sec_prove_thm_general_theory}, apply Assumption~\ref{assump_rect_lip} to get
\begin{align*}
\delta_{3,t}^{U} &\le \frac{\eta}{q}\max_{i\in[n]}\Big\|\sum_{\ell=1}^q \bar G_{i\ell,t}(\bE_V^t)_{\ell,\cdot}\Big\|\\ &\le \eta q^{-1/2}\|\bar\cG_t\|_{\twinf}\,q^{-1/2}\|\bE_V^t\|_{\Fn}\\ &\le \eta L_\infty q^{-1/2} \max_{s\in[0,1]} \|\bU_t(s)\bV_t(s)^\T-\bU^*(\bV^*)^\T\|_{\twinf} \times \|(\bE_U^t,\bE_V^t)\|_{\Fnw}.
\end{align*}
Now that
\begin{equation*}
\bU_t(s)\bV_t(s)^\T-\bU^*(\bV^*)^\T = s\,\bU^*(\bE_V^t)^\T + s\,\bE_U^t(\bV^*)^\T + s^2\,\bE_U^t(\bE_V^t)^\T.
\end{equation*}
Consequently, we have
\begin{align*}
\|\bU_t(s)\bV_t(s)^\T-\bU^*(\bV^*)^\T\|_{\twinf} &\le \|\bU^*\|_{2\to\infty}\|\bE_V^t\|_{\Fn} + \|\bV^*\|\|\bE_U^t\|_{\twinfw} + \|\bE_V^t\|_{\Fn}\|\bE_U^t\|_{\twinfw}\\ &\le \sqrt q\,\omega_*\|(\bE_U^t,\bE_V^t)\|_{\Fnw} + (\|\bV^*\|+\|\bE_V^t\|_{\Fn})\|(\bE_U^t,\bE_V^t)\|_{\twinfw}\\ &\le \sqrt q\,\omega_*\|(\bE_U^t,\bE_V^t)\|_{\Fnw} + C\sqrt{q\kappa\sigma_{\min}}\, \|(\bE_U^t,\bE_V^t)\|_{\twinfw},
\end{align*}
uniformly over $s\in[0,1]$, where we used
$\|\bE_V^t\|_{\Fn} \le \sqrt q\,\|(\bE_U^t,\bE_V^t)\|_{\Fnw} \le c_0\tfrac{\alpha(\alpha+\kappa)}{\beta^2\kappa^2}\sqrt{q\sigma_{\min}} \le c_0\sqrt{q\sigma_{\min}/\kappa} $ for sufficiently small $c_0$. Hence, we arrive at
\begin{align*}
\delta_{3,t}^{U} &\le C\eta\Big( \omega_*\|(\bE_U^t,\bE_V^t)\|_{\Fnw}^2 + \sqrt{\kappa\sigma_{\min}}\, \|(\bE_U^t,\bE_V^t)\|_{\Fnw}\|(\bE_U^t,\bE_V^t)\|_{\twinfw} \Big)\\
&\le \frac1{16}\eta\alpha\sigma_{\min}\|(\bE_U^t,\bE_V^t)\|_{\twinfw} + C\eta\omega_*\alpha\sqrt{\frac{\sigma_{\min}}{\kappa}}\, \|(\bE_U^t,\bE_V^t)\|_{\Fnw},
\end{align*}
where the last step uses $\|(\bE_U^t,\bE_V^t)\|_{\Fnw}\le \phi_{nq}\tau_*\le c_0\alpha\sqrt{\sigma_{\min}/\kappa}$ obtained by the induction hypothesis and the scaling condition. By symmetry, the same bound holds for $\delta_{3,t}^{V}$. Therefore, we have
\begin{equation}
\delta_{3,t} \le \frac1{16}\eta\alpha\sigma_{\min}\|(\bE_U^t,\bE_V^t)\|_{\twinfw} + C\eta\omega_*\alpha\sqrt{\frac{\sigma_{\min}}{\kappa}}\, \|(\bE_U^t,\bE_V^t)\|_{\Fnw}. \label{eq_uv_delta3}
\end{equation}
where the last step uses $\|(\bE_U^t,\bE_V^t)\|_{\Fnw}\le \phi_{nq}\tau_*$ from the induction hypothesis together with the scaling condition $\phi_{nq}\tau_*\le c_0\alpha\sqrt{\sigma_{\min}/\kappa}$.
Finally, we bound $\delta_{4,t}$.
Since $\nabla_{\bU}\cL(\bU^*,\bV^*)=\zero$ and $\nabla_{\bV}\cL(\bU^*,\bV^*)=\zero$, the same row-wise expansion as above gives
\begin{align*}
\frac{\eta}{q}\|\nabla_{\bU}\cL(\tilde\bU^t,\tilde\bV^t)\|_{2\to\infty} &\le \eta\max_{i\in[n]}\|\bar\cH_{ii,t}^{UU}(\bE_U^t)_{i,\cdot}\| +\delta_{2,t}^{U}+\delta_{3,t}^{U},\\ \frac{\eta}{n}\|\nabla_{\bV}\cL(\tilde\bU^t,\tilde\bV^t)\|_{2\to\infty}
&\le \eta\max_{\ell\in[q]}\|\bar\cH_{\ell\ell,t}^{VV}(\bE_V^t)_{\ell,\cdot}\| +\delta_{2,t}^{V}+\delta_{3,t}^{V}.
\end{align*}
Using the upper bounds on the diagonal Hessian blocks derived above, we obtain
\begin{align*}
\frac{\eta}{q}\|\nabla_{\bU}\cL(\tilde\bU^t,\tilde\bV^t)\|_{2\to\infty} \!\vee \frac{\eta}{n}\|\nabla_{\bV}\cL(\tilde\bU^t,\tilde\bV^t)\|_{2\to\infty} \le &\, 1.01\eta(\alpha+\beta)\kappa\sigma_{\min}\|(\bE_U^t,\bE_V^t)\|_{\twinfw} \\&\,+\delta_{2,t}+\delta_{3,t}\\\le &\, 2.02\eta\beta\kappa\sigma_{\min}\|(\bE_U^t,\bE_V^t)\|_{\twinfw} +\delta_{2,t}+\delta_{3,t}.
\end{align*}
Recall that from Step 1, we have defined $\vartheta_t:=\|\bLambda_t^*-\bI_r\|\vee\|(\bLambda_t^*)^{-1}-\bI_r\|$.
Therefore, by \eqref{eq_uv_balance_smallness_step1}, we know
\begin{equation*}
\vartheta_t
\le\frac{\alpha}{10\beta\kappa}.
\end{equation*}
Consequently, $\delta_{4,t}$ can be controlled by
\begin{align*}
\delta_{4,t} &\le \vartheta_t\Big( \frac{\eta}{q}\|\nabla_{\bU}\cL(\tilde\bU^t,\tilde\bV^t)\|_{2\to\infty} \;\vee\; \frac{\eta}{n}\|\nabla_{\bV}\cL(\tilde\bU^t,\tilde\bV^t)\|_{2\to\infty} \Big)\\ &\le \frac{1}{4}\eta\alpha\sigma_{\min}\|(\bE_U^t,\bE_V^t)\|_{\twinfw} + \vartheta_t\delta_{2,t} + \vartheta_t\delta_{3,t}.
\end{align*}
Next, by the estimate for $\delta_{2,t}$ obtained above, we know
\begin{equation*}
\vartheta_t\delta_{2,t} \le C\eta\alpha\sqrt{\frac{\sigma_{\min}}{\kappa}}\, \omega_*\|(\bE_U^t,\bE_V^t)\|_{\Fnw},
\end{equation*}
and a similar bound holds for $\vartheta_t\delta_{3,t}$, while its $\|(\bE_U^t,\bE_V^t)\|_{\twinfw}$-part is absorbed into the first term. Hence, we arrive at
\begin{equation}
\delta_{4,t} \le \frac1{4}\eta\alpha\sigma_{\min}\|(\bE_U^t,\bE_V^t)\|_{\twinfw} + C\eta\alpha\sqrt{\frac{\sigma_{\min}}{\kappa}}\, \omega_*\|(\bE_U^t,\bE_V^t)\|_{\Fnw}. \label{eq_uv_delta4}
\end{equation}

Combining \eqref{eq_uv_linf__decomp}, \eqref{eq_uv_delta1}, \eqref{eq_uv_delta2},
\eqref{eq_uv_delta3}, and \eqref{eq_uv_delta4}, and using $\beta\ge \alpha$ and $\kappa\ge 1$, we conclude that
\begin{align*}
\|(\bar\bU^{t+1}-\bU^*,\bar\bV^{t+1}-\bV^*)\|_{\twinfw}
\le &\,
\Big(1-\frac{9}{16}\eta\alpha\sigma_{\min}\Big)\|(\bE_U^t,\bE_V^t)\|_{\twinfw}
\\&\,+
C\eta\beta\sqrt{\kappa\sigma_{\min}}\,\omega_*\|(\bE_U^t,\bE_V^t)\|_{\Fnw}.
\end{align*}
By Step 1 and the induction hypotheses $\|(\bE_U^t,\bE_V^t)\|_{\Fnw}\le \rho^t\phi_{nq}\tau_*$ and $\|(\bE_U^t,\bE_V^t)\|_{\twinfw}\le \rho^t\psi_{nq}\omega_*$, we further have
\begin{align*}
\|(\bar\bU^{t+1}-\bU^*,\bar\bV^{t+1}-\bV^*)\|_{\twinfw} \le \Big(1-\frac{9}{16}\eta\alpha\sigma_{\min}\Big)\rho^t\psi_{nq}\omega_* + C\eta\beta\sqrt{\kappa\sigma_{\min}}\,\omega_*\,\rho^t\phi_{nq}\tau_*,
\end{align*}
where we used $\tau_*^2\le r\kappa\sigma_{\min}$.
By the scaling condition that $\tfrac{\beta}{\alpha}\kappa^{3/2}\sqrt r\,{\phi_{nq}}/{\psi_{nq}}\le c_0$, after shrinking $c_0$ if necessary, there is $C\beta\kappa\sqrt r\,\phi_{nq} \le \alpha\psi_{nq}/(16)$.
Consequently, with $\rho = 1-\eta\alpha\sigma_{\min}/4$,
\begin{equation}
\|(\bar\bU^{t+1}-\bU^*,\bar\bV^{t+1}-\bV^*)\|_{\twinfw} \le \Big(1-\frac12\eta\alpha\sigma_{\min}\Big)\rho^t\psi_{nq}\omega_* = (2\rho-1)\rho^{t}\psi_{nq}\omega_*.
\label{eq_uv_linf}
\end{equation}

\medskip
\paragraph{\it \underline{ Step 3: existence of $\bG_{t+1}^*$ and alignment of $(\bar\bU^{t+1},\bar\bV^{t+1})$ with $(\tilde\bU^{t+1},\tilde\bV^{t+1})$.}}
Now we prove the induction hypotheses~\ref{item_uv_ind_exist_new} and~\ref{item_uv_ind_rot_new}for step $t+1$. Then we transfer the one-step bounds for $(\bar\bU^{t+1},\bar\bV^{t+1})$ to iterate $(\tilde\bU^{t+1},\tilde\bV^{t+1})$, which further yields the induction hypotheses~\ref{item_uv_ind_l2_new} and~\ref{item_uv_ind_linf_new}. For each $t$, set
\begin{equation*}
\delta_{t+1}:= \max\Big\{ n^{-1/2}\|\bar\bU^{t+1}-\bU^*\|_{\Fn}, \, q^{-1/2}\|\bar\bV^{t+1}-\bV^*\|_{\Fn} \Big\}.
\end{equation*}
By \eqref{eq_uv__l2_new}, we know $\delta_{t+1}\le \|(\bar\bU^{t+1}-\bU^*,\bar\bV^{t+1}-\bV^*)\|_{\Fnw}\le \rho^{t+1}\phi_{nq}\tau_*$. Further with $\sigma_r(n^{-1/2}\bU^*)=\sigma_r(q^{-1/2}\bV^*)=\sqrt{\sigma_{\min}}$, $\tau_*\le \sqrt{r\kappa\sigma_{\min}}$, and
the scaling condition $\phi_{nq}\le c_0\frac{\alpha(\alpha+\kappa)}{\beta^2\kappa^2}\sqrt{\sigma_{\min}}$, it holds that, after shrinking $c_0$ if necessary,
\begin{equation*}
\delta_{t+1} \le \frac{1}{80}\sqrt{\sigma_{\min}} = \frac{1}{80}\Big(\sigma_r(n^{-1/2}\bU^*)\wedge \sigma_r(q^{-1/2}\bV^*)\Big).
\end{equation*}
Moreover, the induction hypothesis~\ref{item_uv_ind_rot_new} gives $\|\bG_t^*-\bR^0\|\vee \|(\bG_t^*)^{-\T}-\bR^0\|\le \iota_0\frac{\alpha}{\beta\kappa}<1/6$, after shrinking $\iota_0$ if necessary, so all singular values of $\bG_t^*$ lie in $[2/3,3/2]$.
Therefore Lemma~\ref{lemma_alignment_near_rotation}, applied to the pair $(\bU^{t+1},\bV^{t+1})$ with $\bm P=\bG_t^*$, yields that the optimal alignment matrix $\bG_{t+1}^*$ exists. In addition,
\begin{equation}
\begin{aligned}
\|\bG_{t+1}^*-\bG_t^*\| \vee \|(\bG_{t+1}^*)^{-\T}-(\bG_t^*)^{-\T}\| &\le \frac{5\delta_{t+1}}{\sqrt{\sigma_{\min}}} \\ &\le 5\rho^{t+1}\phi_{nq}\frac{\tau_*}{\sqrt{\sigma_{\min}}} \\ &\le C\rho^{t+1}\phi_{nq}\sqrt{r\kappa}, \end{aligned} \label{eq_uv_G_increment}
\end{equation}which can be sufficiently small due to the scaling condition.
This proves \ref{item_uv_ind_exist_new} at time $t+1$.

To verify \ref{item_uv_ind_rot_new}, we telescope the increments. Since the same argument gives
\eqref{eq_uv_G_increment} at every previous iteration, we have
\begin{equation}\label{eq_uv_G_increment_telesc}\begin{aligned}
\|\bG_{t+1}^*-\bR^0\| \vee \|(\bG_{t+1}^*)^{-\T}-\bR^0\| &\le \|\bG_0^*-\bR^0\| \vee \|(\bG_0^*)^{-\T}-\bR^0\| \\
&\quad +\sum_{s=0}^{t} \Big( \|\bG_{s+1}^*-\bG_s^*\| \vee \|(\bG_{s+1}^*)^{-\T}-(\bG_s^*)^{-\T}\| \Big) \\
&\le C\phi_{nq}\frac{\tau_*}{\sqrt{\sigma_{\min}}} \sum_{s=0}^{t+1}\rho^s \\
&\le C\frac{\phi_{nq}\sqrt{r\kappa}}{1-\rho}.
\end{aligned}\end{equation}
Hence, to guarantee
\begin{equation*}
\|\bG_{t+1}^*-\bR^0\| \vee \|(\bG_{t+1}^*)^{-\T}-\bR^0\| \le \iota_0\frac{\alpha}{\beta\kappa},
\end{equation*}
it suffices to show that
\begin{equation}
\phi_{nq} \le c_0(1-\rho)\frac{\alpha}{\beta\kappa\sqrt{r\kappa}} = c_0\,\eta\sigma_{\min}\frac{\alpha^2}{\beta\kappa\sqrt{r\kappa}}, \label{eq_uv_rot_scaling_needed}
\end{equation}
for $c_0>0$ sufficiently small. One can check that this holds under the scaling condition, $\eta = \{10(\alpha + \beta)\kappa\sigma_{\min}\}^{-1}$, and $\beta\ge \alpha$. This proves \ref{item_uv_ind_rot_new} at time $t+1$.

Now we show induction hypotheses~\ref{item_uv_ind_l2_new} and~\ref{item_uv_ind_linf_new}. First, \ref{item_uv_ind_l2_new} is easy to show. Note that $\bG_{t+1}^*$ minimizes the weighted Frobenius distance, which gives $ \|(\tilde\bU^{t+1}-\bU^*,\tilde\bV^{t+1}-\bV^*)\|_{\Fnw} \le \|(\bar\bU^{t+1}-\bU^*,\bar\bV^{t+1}-\bV^*)\|_{\Fnw}$.
Combining this with \eqref{eq_uv__l2_new} gives
\begin{equation*}
\|(\tilde\bU^{t+1}-\bU^*,\tilde\bV^{t+1}-\bV^*)\|_{\Fnw} \le \rho^{t+1}\phi_{nq}\tau_*,
\end{equation*}
which proves \ref{item_uv_ind_l2_new} at time $t+1$.

It remains to transfer the $2\to\infty$ bound. Define
\begin{equation*}
\bQ_{t+1}:=(\bG_t^*)^{-1}\bG_{t+1}^*.
\end{equation*}
Then $\tilde\bU^{t+1}=\bar\bU^{t+1}\bQ_{t+1}$ and $\tilde\bV^{t+1}=\bar\bV^{t+1}\bQ_{t+1}^{-\T}$. Consequently,
\begin{align*}
\|(\tilde\bU^{t+1}-\bU^*,&\,\tilde\bV^{t+1}-\bV^*)\|_{\twinfw} \le \|(\bar\bU^{t+1}-\bU^*,\bar\bV^{t+1}-\bV^*)\|_{\twinfw} \\
&\,+\Big( \|\bar\bU^{t+1}\|_{2\to\infty} \vee \|\bar\bV^{t+1}\|_{2\to\infty} \Big) \big( \|\bQ_{t+1}-\bI_r\| \vee \|\bQ_{t+1}^{-\T}-\bI_r\| \big).
\end{align*}
By \eqref{eq_uv_linf}, we have $\|(\bar\bU^{t+1}-\bU^*,\bar\bV^{t+1}-\bV^*)\|_{\twinfw}\le \rho^{t+1}\psi_{nq}\omega_*$, and thus
\begin{equation*}
\|\bar\bU^{t+1}\|_{2\to\infty}
\vee
\|\bar\bV^{t+1}\|_{2\to\infty}
\le
\omega_*+\rho^{t+1}\psi_{nq}\omega_*
\le
(1+c_0)\omega_*
\le
2\omega_*,
\end{equation*}
after shrinking $c_0$ if necessary.
Also, since all singular values of $\bG_t^*$ lie in $[2/3,3/2]$, we have $\|\bG_t^*\|\vee \|(\bG_t^*)^{-1}\| \le 3/2$. Hence
\begin{align*}
\|\bQ_{t+1}-\bI_r\| &= \|(\bG_t^*)^{-1}(\bG_{t+1}^*-\bG_t^*)\| \le \|(\bG_t^*)^{-1}\|\, \|\bG_{t+1}^*-\bG_t^*\|, \\
\|\bQ_{t+1}^{-\T}-\bI_r\| &= \|((\bG_{t+1}^*)^{-\T}-(\bG_t^*)^{-\T})\,(\bG_t^*)^\T\| \le \|\bG_t^*\|\, \|(\bG_{t+1}^*)^{-\T}-(\bG_t^*)^{-\T}\|.
\end{align*}
Combining these bounds with \eqref{eq_uv_G_increment}, we get
\begin{equation*}
\|\bQ_{t+1}-\bI_r\| \vee \|\bQ_{t+1}^{-\T}-\bI_r\| \le C\rho^{t+1}\phi_{nq}\frac{\tau_*}{\sqrt{\sigma_{\min}}} \le C\rho^{t+1}\phi_{nq}\sqrt{r\kappa}.
\end{equation*}
It then follows from \eqref{eq_uv_linf} that
\begin{equation*}
\|(\tilde\bU^{t+1}-\bU^*,\tilde\bV^{t+1}-\bV^*)\|_{\twinfw} \le (2\rho-1)\rho^{t}\psi_{nq}\omega_* + C\rho^{t+1}\phi_{nq}\sqrt{r\kappa}\,\omega_*.
\end{equation*}
With the scaling condition that
\begin{equation*}
\frac{\beta}{\alpha}\kappa^{3/2}\sqrt r\,\frac{\phi_{nq}}{\psi_{nq}}\le c_0
\qquad\text{and}\qquad \sqrt{r\kappa}\le \kappa^{3/2}\sqrt r,
\end{equation*}
after shrinking $c_0$ if necessary, we conclude that $C\phi_{nq}\sqrt{r\kappa}\le \psi_{nq}$. This finally leads to
\begin{equation*}
\|(\tilde\bU^{t+1}-\bU^*,\tilde\bV^{t+1}-\bV^*)\|_{\twinfw} \le \rho^{t+1}\psi_{nq}\omega_*.
\end{equation*}
This proves \ref{item_uv_ind_linf_new} at time $t+1$, and closes the induction.

\subsection{Proof of Theorem~\ref{thm_general_theory_UV_noise}}\label{supp_sec_prove_thm_general_theory_UV_noise}

We prove in this subsection the stronger contraction statement under the localization $\cDuvinf$. Accordingly, throughout the proof we invoke Assumptions~\ref{assump_rect_rip_weighted},~\ref{assump_rect_lip_1},~\ref{assump_row_cross_curvature_uv}, and~\ref{assump_rect_lip} with $\cD=\cDuvinf$, which has been defined as 
\begin{equation*}
    \cDuvinf = \Big\{(\bU,\bV):\frac{\disttwo\{(\bU,\bV),(\bU^*,\bV^*)\}}{\tau_*}\le \epsilon,\quad \frac{\distinf\{(\bU,\bV),(\bU^*,\bV^*)\}}{\omega_*}\le \epsilon\Big\}.
\end{equation*} The $\ell_2$-only part of Theorem~\ref{thm_general_theory_uv}, where only Assumptions~\ref{assump_rect_rip_weighted}--\ref{assump_rect_lip_1} are imposed with $\cD=\cDuvtwo$, is proved later in Section~\ref{supp_sec_prove_thm_general_theory_noisy_uv_l2} by the same argument after removing the row-wise estimates.


Throughout the proof, we use the population first-order condition $\bar\cG(\bU^*,\bV^*) = \EE\cG(\bU^*,\bV^*) = \zero$. We write $\Delta_2$ and $\Delta_{\infty}$ for $\Delta_2(n,q,\delta)$ and $\Delta_{\infty}(n,q,\delta)$ in \eqref{assump_gradient_noise_rect} for brevity.  We work throughout on the event in~\eqref{assump_gradient_noise_rect}, which has probability at least $1-\delta$. 


Similar to the proof of Theorem~\ref{thm_general_theory_UV_noise}, we introduce a local empirical optimizer
that will serve as the statistical contraction target. Let $(\hat\bU,\hat\bV) \in \argmin_{(\bU,\bV)\in\cDuvinf}\cL(\bU\bV^\T)$and $\hat\bG \in \argmin_{\bG\in GL(r)}\|(\hat\bU \hat\bG - \bU^*,\hat\bV\hat\bG^{-\T} - \bV^*)\|_{\Fnw}$. In what follows, we show that, with probability at least $1-\delta$, $(\hat\bU\hat\bG,\hat\bV\hat\bG^{-\T}) = (\tilde\bU,\tilde\bV)$ is unique and satisfies
\begin{equation}\label{eq_optimum_UV_error_bound}
    {\disttwo\{(\hat\bU,\hat\bV),(\bU^*,\bV^*)\}} \le C\frac{\Delta_2(n,q,\delta)}{\alpha\sigma_{\min}}\tau_*\,\text{;}\,{\distinf\{(\hat\bU,\hat\bV),(\bU^*,\bV^*)\}}\le C\frac{\Delta_{\infty}(n,q,\delta)}{\alpha\sigma_{\min}}{\omega_*},
\end{equation}for some universal constant $C$. Moreover, it satisfy the first-order condition $\nabla_{\bU}\cL(\hat\bU\hat\bV^\T) = \zero $ and $\nabla_{\bV}\cL(\hat\bU\hat\bV^\T) = \zero $, and thus serves as the contraction target of \eqref{eq_one_step_update_uv}.

Fix the radius $\epsilon$, we choose some $\varepsilon$ such that
\begin{equation}
\label{eq_thm_general_theory_noisy_uv_scaling}
\frac{1}{c_0}\Big(\sqrt{r\kappa}\frac{\Delta_2}{\alpha\sigma_{\min}} \vee \frac{\Delta_{\infty}}{\alpha\sigma_{\min}} \Big) \le \varepsilon \le c_0\frac{\epsilon}{\zeta_r},
\end{equation}
where $c_0>0$ is a sufficiently small universal constant. Moreover, after shrinking $c_0$ if necessary, we record the consequences
\begin{equation}
\label{eq_rect_step1_scaling_implications}
\frac{\Delta_2}{\alpha\sigma_{\min}}\le c_0\frac{\alpha}{\sqrt{\kappa r}}\varepsilon, \qquad \frac{\Delta_{\infty}}{\alpha\sigma_{\min}}\le c_0\big(\frac{\alpha}{\kappa\sqrt{r}}\wedge \varepsilon\big),\qquad \varepsilon\le c_0\big(\frac{\epsilon}{\zeta_r}\wedge \frac{\alpha}{\sqrt r\,\kappa}\big).
\end{equation}

We next introduce a smaller compact neighborhood $\cDuvde$ for the same $\epsilon$ and $\varepsilon$ to construct a stationary point by minimizing the gradient norm on a compact set, and a slightly larger neighborhood $\cDuvd$ on which the local curvature and Lipschitz bounds will be invoked. In particular, define
\begin{align}
&\cDuvde = \Big\{(\bU,\bV):\frac{\|(\bU - \bU^*,\bV-\bV^*)\|_{\Fnw}}{\tau_*}\le \varepsilon,\quad \frac{\|(\bU - \bU^*,\bV-\bV^*)\|_{\twinfw}}{\omega_*}\le \epsilon\Big\}.\label{eq_define_cdde_rec}\\
    &\cDuvd = \Big\{(\bU,\bV):\frac{\|(\bU - \bU^*,\bV-\bV^*)\|_{\Fnw}}{\tau_*}\le \epsilon,\quad \frac{\|(\bU - \bU^*,\bV-\bV^*)\|_{\twinfw}}{\omega_*}\le \epsilon\Big\},\label{eq_define_cdd_rec}
\end{align}
Since $\varepsilon\le \epsilon$, we have $\cDuvde\subseteq \cDuvd\subseteq \cDuvinf$. In particular, $(\bU^*,\bV^*)\in\cDuvde$.
Similar to the proof of Theorem~\ref{thm_general_theory_noisy_Z}, define
\begin{equation*}
(\tilde\bU,\tilde\bV) \in \argmin_{(\bU,\bV)\in\cDuvde} \big\|\bS_Z^2\nabla_{\bZ}h_{\alpha}^{\natural}(\bU,\bV)\big\|_{\twinfw}^2.
\end{equation*}
Such a minimizer exists because $\cDuvde$ is compact and $(\bU,\bV)\mapsto \|\bS_Z^2\nabla_{\bZ}h_{\alpha}^{\natural}(\bU,\bV)\|_{\twinfw}^2$ is continuous. Let
\begin{equation*}
\bE_U:=\tilde\bU-\bU^*, \qquad \bE_V:=\tilde\bV-\bV^*, \qquad \bE:= \begin{pmatrix} \bE_U\\ \bE_V \end{pmatrix}.
\end{equation*}
 Our proof proceeds as follows.
\begin{enumerate}[label=(\arabic*),leftmargin=0.6cm]
    \item constructs an interior stationary point $(\tilde\bU,\tilde\bV)$ of $h_\alpha^\natural$ in $\bar\cD_{\varepsilon,\epsilon}$ and establish the first-order optimality condition;
    \item identifies this point with the aligned constrained optimizer $(\hat\bU,\hat\bV)$ and sharpens the statistical rates to prove \eqref{eq_optimum_UV_error_bound}
    \item re-centers the recursion at $(\tilde\bU,\tilde\bV)$ and proves the contraction bounds for the gradient iterates.
\end{enumerate}

\paragraph{\it \underline{ Step 1: construct an interior stationary point of $h_{\alpha}^{\natural}$.}} Since $(\bU^*,\bV^*)\in\cDuvde$, $\nabla p_{\alpha}^{\natural}(\bU^*,\bV^*)=\zero$, and $(\tilde\bU,\tilde\bV)$ minimizes $\|\bS_Z^2\nabla_{\bZ}h_{\alpha}^{\natural}(\cdot)\|_{\twinfw}$ over $\cDuvde$, we have
\begin{equation}\label{eq_rect_score_inf_opt}
\begin{aligned}
\big\|\bS_Z^2\nabla_{\bZ} h_{\alpha}^{\natural}(\tilde\bU,\tilde\bV)\big\|_{\twinfw}
&\le \big\|\bS_Z^2\nabla_{\bZ} h_{\alpha}^{\natural}(\bU^*,\bV^*)\big\|_{\twinfw}\\
&= \max\Big\{ q^{-1}\|\cG(\bU^*,\bV^*)\bV^*\|_{\twinf}, \ n^{-1}\|\cG(\bU^*,\bV^*)^\T\bU^*\|_{\twinf} \Big\} \\
&= \max\Big\{ q^{-1}\|\tilde\cG(\bU^*,\bV^*)\bV^*\|_{\twinf}, \ n^{-1}\|\tilde\cG(\bU^*,\bV^*)^\T\bU^*\|_{\twinf} \Big\}\\
&\le C\Delta_{\infty} \max\Big\{\|\bV^*\|/\sqrt q,\ \|\bU^*\|/\sqrt n\Big\}\\
&\le C\Delta_{\infty}\sqrt{\kappa\sigma_{\min}},
\end{aligned}
\end{equation}
where we used $\bar\cG(\bU^*,\bV^*)=\zero$ and the bound in~\eqref{assump_gradient_noise_rect_2}.
Consequently,
\begin{equation}
\label{eq_rect_score_f}
\begin{aligned}
\big\|\bS_Z\nabla_{\bZ} h_{\alpha}^{\natural}(\tilde\bU,\tilde\bV)\big\|_{\Fn} &= \Big( q^{-1}\|\nabla_{\bU} h_{\alpha}^{\natural}(\tilde\bU,\tilde\bV)\|_{\Fn}^2 + n^{-1}\|\nabla_{\bV} h_{\alpha}^{\natural}(\tilde\bU,\tilde\bV)\|_{\Fn}^2 \Big)^{1/2} \\
&\le \sqrt{2nq}\,\big\|\bS_Z^2\nabla_{\bZ} h_{\alpha}^{\natural}(\tilde\bU,\tilde\bV)\big\|_{\twinfw} \\
&\le C\Delta_{\infty}\sqrt{nq\kappa\sigma_{\min}}.\end{aligned}
\end{equation}

For $s\in[0,1]$, define the interpolation segment
\begin{equation*}
(\bU(s),\bV(s))=(\bU^*,\bV^*)+s(\bE_U,\bE_V).
\end{equation*}
Since $(\tilde\bU,\tilde\bV)\in\cDuvde\subseteq \cDuvd$, the whole segment lies in $\cDuvd\subseteq\cDuvinf$. Applying the integral mean value theorem to $\mathrm{vec}\{\nabla_{\bZ}h_{\alpha}^{\natural}(\cdot)^\T\}$ along this segment yields
\begin{equation}\label{eq_rect_mvt_optimizer}\begin{aligned}
&\underbrace{\bS_z\bar\cH\bS_z\,\mathrm{vec}\big((\bS_Z^{-1}\bE)^\T\big)}_{\Gamma_{1}} + \underbrace{\bS_z\tilde\cG\bS_z\,\mathrm{vec}\big((\bS_Z^{-1}\bE)^\T\big)}_{\Gamma_{2}} + \underbrace{\bS_z\bar\cG\bS_z\,\mathrm{vec}\big((\bS_Z^{-1}\bE)^\T\big)}_{\Gamma_{3}} \\
&= \underbrace{\mathrm{vec}\Big(\bS_Z\{\nabla_{\bZ}h_{\alpha}^{\natural}(\tilde\bU,\tilde\bV)-\nabla_{\bZ}h_{\alpha}^{\natural}(\bU^*,\bV^*)\}^\T\Big)}_{\Gamma_{4}},
\end{aligned}
\end{equation}
where
\begin{align*}
\bar\cH := &\,  \int_0^1\Big\{\nabla_{\bz}^2 h_{\alpha}^{\natural}(\bU(s),\bV(s)) - \cG_e(\bU(s),\bV(s))\Big\}\,ds,\\
\tilde\cG :=&\,  \int_0^1\Big\{\cG_e(\bU(s),\bV(s)) - \EE\cG_e(\bU(s),\bV(s))\Big\}\,ds, \qquad
\bar\cG := \int_0^1\EE\cG_e(\bU(s),\bV(s))\,ds.
\end{align*}
Here $\cG_e(\bU,\bV)$ is the block operator defined in Lemma~\ref{lemma_general_convex_asym}.
We next bound $\Gamma_{1}$--$\Gamma_{4}$.

For $\Gamma_{1}$, Lemma~\ref{lemma_general_convex_asym} yields
\begin{equation}\label{eq_rect_barH_lb_optimizer}
\lambda_{\min}(\bS_z\bar\cH\bS_z) \ge \frac{\alpha\sigma_{\min}}{2}.
\end{equation}

For $\Gamma_{2}$, similar to the calculation of $\gamma_{2,t}$ as in the deterministic contraction proof (Theorem~\ref{thm_general_theory_uv}), apply the bound in~\eqref{assump_gradient_noise_rect_1} to obtain
\begin{equation}\label{eq_rect_barG_bound_optimizer1}
\begin{aligned}
\|\Gamma_{2}\| &\le \max_{s\in[0,1]} \big\|\bS_z\{\cG_e(\bU(s),\bV(s))- \EE\cG_e(\bU(s),\bV(s))\}\bS_z\big\| \cdot \|\bS_Z^{-1}\bE\|_{\Fn} \\
&\le C\Delta_2\|\bS_Z^{-1}\bE\|_{\Fn}\\&= C\Delta_2\sqrt{nq}\|(\bE_U,\bE_V)\|_{\Fnw}. \end{aligned}
\end{equation}
Here, the last step follows from $\|\bS_Z^{-1}\bE\|_{\Fn} = \sqrt{nq}\|(\bE_U,\bE_V)\|_{\Fnw}$.

For $\Gamma_{3}$, use $\bar\cG(\bU^*,\bV^*)=\zero$ and Assumption~\ref{assump_rect_lip_1} with $\bar\cG(\bU,\bV)$ in place of $\cG(\bU,\bV)$ to obtain
\begin{equation} \label{eq_rect_barG_bound_optimizer}
\begin{aligned}
\|\Gamma_{3}\| &\le \frac{1}{\sqrt{nq}} \max_{s\in[0,1]} \|\bar\cG(\bU(s),\bV(s))\| \,\|\bS_Z^{-1}\bE\|_{\Fn} \\
&\le \frac{L_2}{\sqrt{nq}} \max_{s\in[0,1]} \|\bU(s)\bV(s)^\T-\bU^*(\bV^*)^\T\|_{\Fn} \,\|\bS_Z^{-1}\bE\|_{\Fn}\\&\le C\sqrt{\kappa\sigma_{\min}nq}\,\|(\bE_U,\bE_V)\|_{\Fnw}^2.
\end{aligned}
\end{equation}
Here, we telescope
\begin{align*}
\|\bU(s)\bV(s)^\T-\bU^*(\bV^*)^\T\|_{\Fn} \le&\, \|\bU^*\|\,\|\bE_V\|_{\Fn} + \|\bV^*\|\,\|\bE_U\|_{\Fn} + \|\bE_U\|_{\Fn}\,\|\bE_V\|_{\Fn} \\\le&\, C\sqrt{nq\kappa\sigma_{\min}}\,\|(\bE_U,\bE_V)\|_{\Fnw},
\end{align*}
where the last step uses $\|(\bE_U,\bE_V)\|_{\Fnw}\le \varepsilon\tau_*\le c_0\sqrt{\kappa\sigma_{\min}}$ as $(\tilde\bU,\tilde\bV)\in\cDuvde$.

For $\Gamma_{4}$, note first that $\nabla p_\alpha^{\natural}(\bU^*,\bV^*)=\zero$, so
\begin{equation}\label{eq_bound_score_average_uv}
\begin{aligned}
\big\|\bS_Z\nabla_{\bZ}h_\alpha^{\natural}(\bU^*,\bV^*)\big\|_{\Fn} &= \Big( q^{-1}\|\cG(\bU^*,\bV^*)\bV^*\|_{\Fn}^2 + n^{-1}\|\cG(\bU^*,\bV^*)^\T\bU^*\|_{\Fn}^2 \Big)^{1/2}\\
&\le \|\cG(\bU^*,\bV^*)\| \Big(q^{-1}\|\bV^*\|_{\Fn}^2+n^{-1}\|\bU^*\|_{\Fn}^2 \Big)^{1/2}\\
&\le C\sqrt{nq}\,\Delta_2\,\tau_*.
\end{aligned}
\end{equation}
Combining this estimate with \eqref{eq_rect_score_f}, we arrive at
\begin{equation}\label{eq_rect_gamma4_bound}
\|\Gamma_{4}\| \le \big\|\bS_Z\nabla_{\bZ}h_\alpha^{\natural}(\tilde\bU,\tilde\bV)\big\|_{\Fn} + \big\|\bS_Z\nabla_{\bZ}h_\alpha^{\natural}(\bU^*,\bV^*)\big\|_{\Fn} \le C\Delta_{\infty}\sqrt{nq\kappa\sigma_{\min}} + C\sqrt{nq}\,\Delta_2\,\tau_*.
\end{equation}

Combining \eqref{eq_rect_barH_lb_optimizer}--\eqref{eq_rect_gamma4_bound} and using $\|\bS_Z^{-1}\bE\|_{\Fn}=\sqrt{nq}\,\|(\bE_U,\bE_V)\|_{\Fnw}$, we obtain
\begin{equation}\label{eq_rect_crude_l2}
\begin{aligned}
\|(\bE_U,\bE_V)\|_{\Fnw} \le &\,C\frac{\Delta_2}{\alpha\sigma_{\min}}\,\|(\bE_U,\bE_V)\|_{\Fnw} + C\frac{\sqrt{\kappa\sigma_{\min}}}{\alpha\sigma_{\min}}\,\|(\bE_U,\bE_V)\|_{\Fnw} ^2\\ &\,+ C\frac{\Delta_{\infty}}{\alpha\sigma_{\min}}\sqrt{\kappa\sigma_{\min}} +
C\frac{\Delta_2}{\alpha\sigma_{\min}}\tau_*.
\end{aligned}
\end{equation}
Since $(\tilde\bU,\tilde\bV)\in\cDuvde$, we have $\|(\bE_U,\bE_V)\|_{\Fnw}\le \varepsilon\tau_*$. Together with $\tau_*\le \sqrt{r\kappa\sigma_{\min}}$ and \eqref{eq_rect_step1_scaling_implications}, this implies
\begin{equation*}
\frac{\Delta_2}{\alpha\sigma_{\min}}\le c_0,\qquad\frac{\sqrt{\kappa\sigma_{\min}}}{\alpha\sigma_{\min}}\|(\bE_U,\bE_V)\|_{\Fnw} \le \varepsilon\frac{\sqrt r\,\kappa}{\alpha} \le c_0.
\end{equation*}
Hence, after shrinking $c_0$ if necessary, the first two terms on the right-hand side of \eqref{eq_rect_crude_l2} can be absorbed into the left-hand side, and we obtain
\begin{equation}
\label{eq_rect_crude_l2_absorb}
\|(\bE_U,\bE_V)\|_{\Fnw} \le C\frac{\Delta_{\infty}}{\alpha\sigma_{\min}}\sqrt{\kappa\sigma_{\min}} + C\frac{\Delta_2}{\alpha\sigma_{\min}}\tau_*.
\end{equation}
Note that $\tau_*^2=\sum_{j=1}^r \sigma_j(\bX^*)/\sqrt{nq}\ge (r\vee \kappa)\sigma_{\min}$, so $\tau_*\ge \sqrt{(r\vee \kappa)\sigma_{\min}}$. Combining this lower bound with \eqref{eq_thm_general_theory_noisy_uv_scaling} and the theorem scaling condition $\zeta_r\Delta_\infty/(\alpha\sigma_{\min})\le c_0\epsilon$ shows that the right-hand side of \eqref{eq_rect_crude_l2_absorb} is at most $(\varepsilon/2)\tau_*$ after shrinking $c_0$. That is,
\begin{equation}\label{eq_rect_step1_l2_half}
\|(\tilde\bU-\bU^*,\tilde\bV-\bV^*)\|_{\Fnw} \le \frac{1}{2}\varepsilon\tau_*.
\end{equation}

We next derive a crude weighted $2\to\infty$ bound. For $i,j\in[n]$ and $\ell,k\in[q]$, define
\begin{align*}
\tilde\bH_{ij}^{\cL,UU} &:= q^{-1}\int_0^1 \nabla_{\bu_i\bu_j}^2\cL\big(\bU(s)\bV(s)^\T\big)\,ds, \\
\tilde\bH_{i\ell}^{\cL,UV} &:= q^{-1}\int_0^1\Big\{\nabla_{\bu_i\bv_\ell}^2\cL\big(\bU(s)\bV(s)^\T\big)-G_{i\ell}\big(\bU(s),\bV(s)\big)\bI_r\Big\}\,ds,\\
\tilde\bH_{\ell i}^{\cL,VU} &:= n^{-1}\int_0^1\Big\{\nabla_{\bv_\ell\bu_i}^2\cL\big(\bU(s)\bV(s)^\T\big)-G_{i\ell}\big(\bU(s),\bV(s)\big)\bI_r\Big\}\,ds,\\
\tilde\bH_{\ell k}^{\cL,VV} &:= n^{-1}\int_0^1 \nabla_{\bv_\ell\bv_k}^2\cL\big(\bU(s)\bV(s)^\T\big)\,ds,
\end{align*}
and
\begin{equation*}
\tilde\cG_{i\ell} := \int_0^1 \tilde G_{i\ell}(\bU(s),\bV(s))\,ds, \qquad
\bar\cG_{i\ell} := \int_0^1 \bar G_{i\ell}(\bU(s),\bV(s))\,ds.
\end{equation*}
Then the weighted row-wise mean-value expansion gives, for each $i\in[n]$,
\begin{equation}
\label{eq_rect_row_expansion_U}
\Delta_{1,i}^{(U)}+\Delta_{2,i}^{(U)}+\Delta_{3,i}^{(U)}+\Delta_{4,i}^{(U)}+\Delta_{5,i}^{(U)} = \Delta_{6,i}^{(U)},
\end{equation}
where, by letting $\bE_{U,i} = (\bE_U)_{i,}$ and $\bE_{V,\ell} = (\bE_V)_{\ell,}$ for $i\in[n]$ and $\ell\in[q]$, we define
\begin{align*}
\Delta_{1,i}^{(U)}&:=\tilde\bH_{ii}^{\cL,UU}\bE_{U,i}^\T, \qquad\qquad\;
\Delta_{2,i}^{(U)}:=\sum_{j\neq i}\tilde\bH_{ij}^{\cL,UU}\bE_{U,j}^\T + \sum_{\ell=1}^q \tilde\bH_{i\ell}^{\cL,UV}\bE_{V,\ell}^\T, \\
\Delta_{3,i}^{(U)}&:=q^{-1}\sum_{\ell=1}^q \bar\cG_{i\ell}\bE_{V,\ell}^\T, \qquad\quad
\Delta_{4,i}^{(U)}:=q^{-1}\sum_{\ell=1}^q \tilde\cG_{i\ell}\bE_{V,\ell}^\T, \\
\Delta_{5,i}^{(U)}&:=q^{-1}\Big\{\nabla_{\bu_i}p_\alpha^{\natural}(\tilde\bU,\tilde\bV)-\nabla_{\bu_i}p_\alpha^{\natural}(\bU^*,\bV^*)\Big\},\\
\Delta_{6,i}^{(U)}&:=q^{-1}\Big\{\nabla_{\bu_i}h_\alpha^{\natural}(\tilde\bU,\tilde\bV)-\nabla_{\bu_i}h_\alpha^{\natural}(\bU^*,\bV^*)\Big\}.
\end{align*}
Likewise, for each $\ell\in[q]$,
\begin{equation}\label{eq_rect_row_expansion_V}
\Delta_{1,\ell}^{(V)}+\Delta_{2,\ell}^{(V)}+\Delta_{3,\ell}^{(V)}+\Delta_{4,\ell}^{(V)}+\Delta_{5,\ell}^{(V)} = \Delta_{6,\ell}^{(V)},
\end{equation}
where
\begin{align*}
\Delta_{1,\ell}^{(V)}&:=\tilde\bH_{\ell\ell}^{\cL,VV}\bE_{V,\ell}^\T, \qquad\qquad\;
\Delta_{2,\ell}^{(V)}:=\sum_{k\neq \ell}\tilde\bH_{\ell k}^{\cL,VV}\bE_{V,k}^\T + \sum_{i=1}^n \tilde\bH_{\ell i}^{\cL,VU}\bE_{U,i}^\T, \\
\Delta_{3,\ell}^{(V)}&:=n^{-1}\sum_{i=1}^n \bar\cG_{i\ell}\bE_{U,i}^\T, \qquad\quad
\Delta_{4,\ell}^{(V)}:=n^{-1}\sum_{i=1}^n \tilde\cG_{i\ell}\bE_{U,i}^\T, \\
\Delta_{5,\ell}^{(V)}&:=n^{-1}\Big\{\nabla_{\bv_\ell}p_\alpha^{\natural}(\tilde\bU,\tilde\bV)-\nabla_{\bv_\ell}p_\alpha^{\natural}(\bU^*,\bV^*)\Big\}, \\
\Delta_{6,\ell}^{(V)}&:=n^{-1}\Big\{\nabla_{\bv_\ell}h_\alpha^{\natural}(\tilde\bU,\tilde\bV)-\nabla_{\bv_\ell}h_\alpha^{\natural}(\bU^*,\bV^*)\Big\}.
\end{align*}

Let
\begin{equation*}
\tilde\cH_{LD} := \mathrm{diag}\Big( \tilde\bH_{11}^{\cL,UU},\ldots,\tilde\bH_{nn}^{\cL,UU}, \tilde\bH_{11}^{\cL,VV},\ldots,\tilde\bH_{qq}^{\cL,VV} \Big).
\end{equation*}
By Lemma~\ref{lemma_rect_block_hessian} and the fact that $(\bU(s),\bV(s))\in \cDuvd$ for all $s\in[0,1]$, we have
\begin{equation} \label{eq_rect_diag_block_lb}
\min\Big\{ \min_{i\in[n]}\lambda_{\min}\big(\tilde\bH_{ii}^{\cL,UU}\big), \  \min_{\ell\in[q]}\lambda_{\min}\big(\tilde\bH_{\ell\ell}^{\cL,VV}\big) \Big\} \ge (1-c_0)\alpha\sigma_{\min}.
\end{equation}
Similar to the treatment of $\tilde\cH_{LD}$ in the proof of Theorem~\ref{thm_general_theory_noisy_Z}, by the block diagonal structure of $\tilde\cH_{LD}$, one has
\begin{equation} \label{eq_rect_solve_diag_U}
\begin{aligned}
\|(\bE_U,\bE_V)\|_{\twinfw} &= \|\tilde\cH_{LD}^{-1}\tilde\cH_{LD}\,\mathrm{vec}(\bE^\T)\|_{\infty,r} \\
&\le \|\tilde\cH_{LD}^{-1}\|_{\infty,r} \sum_{p=2}^6 \Big( \max_{i\in[n]}\|\Delta_{p,i}^{(U)}\| \vee \max_{\ell\in[q]}\|\Delta_{p,\ell}^{(V)}\| \Big) \\
&\le \frac{1}{(1-c_0)\alpha\sigma_{\min}} \sum_{p=2}^6 \Big( \max_{i\in[n]}\|\Delta_{p,i}^{(U)}\| \vee \max_{\ell\in[q]}\|\Delta_{p,\ell}^{(V)}\| \Big).
\end{aligned}
\end{equation}
Here, we abuse the notation to denote 
\begin{equation*}\|\bA\|_{\infty,r}:=\max_{i\in[n+q]}\sum_{j=1}^{n+q}\|\bA_{\cR_i,\cR_j}\|,\text{ for }\bA\in\RR^{(n+q)r\times (n+q)r}\end{equation*}
We now bound the five terms on the right-hand side.

\underline{\emph{Bound for $\Delta_{2,i}^{(U)}$ and $\Delta_{2,\ell}^{(V)}$.}} Using Assumption~\ref{assump_row_cross_curvature_uv} exactly as in the Step 2 in the proof of Theorem~\ref{thm_general_theory_uv}, we obtain
\begin{equation}
\label{eq_rect_off_term_bound}
\max_{i\in[n]}\|\Delta_{2,i}^{(U)}\| \vee \max_{\ell\in[q]}\|\Delta_{2,\ell}^{(V)}\| \le C\beta\sqrt{\kappa\sigma_{\min}}\,\omega_*\,\|(\bE_U,\bE_V)\|_{\Fnw}.
\end{equation}

\underline{\emph{Bound for $\Delta_{3,i}^{(U)}$ and $\Delta_{3,\ell}^{(V)}$.}}
For the $U$-part, by Cauchy--Schwarz,
\begin{equation*}
\max_{i\in[n]}\|\Delta_{3,i}^{(U)}\| \le q^{-1}\|\bar\cG_{\mathrm{ave}}\|_{2\to\infty}\|\bE_V\|_{\Fn} \le q^{-1}\|\bar\cG_{\mathrm{ave}}\|_{\twinf}\|\bE_V\|_{\Fn},
\end{equation*}
where $\bar\cG_{\mathrm{ave}} := \int_0^1 \bar\cG(\bU(s),\bV(s))\,ds$.
Moreover, by Assumption~\ref{assump_rect_lip}, we have $ \|\bar\cG_{\mathrm{ave}}\|_{\twinf} \le L_{\infty}$ $\times\max_{s\in[0,1]}\|\bU(s)\bV(s)^\T-\bU^*(\bV^*)^\T\|_{\twinf}$.
For each $s\in[0,1]$,
\begin{equation*}
\|\bU(s)\bV(s)^\T-\bU^*(\bV^*)^\T\|_{\twinf} \le \|\bU^*\|_{2\to\infty}\|\bE_V\|_{\Fn} + \|\bE_U\|_{2\to\infty}\|\bV^*\| + \|\bE_U\|_{2\to\infty}\|\bE_V\|_{\Fn},
\end{equation*}
and therefore
\begin{equation*}
\|\bar\cG_{\mathrm{ave}}\|_{\twinf} \le C\sqrt q\,\omega_*\,\|(\bE_U,\bE_V)\|_{\Fnw} + C\sqrt{q\kappa\sigma_{\min}}\,\|\bE\|_{\twinfw}.
\end{equation*}
Using $\|\bE_V\|_{\Fn}\le \sqrt q\,\|(\bE_U,\bE_V)\|_{\Fnw}$, we obtain
\begin{equation*}
\max_{i\in[n]}\|\Delta_{3,i}^{(U)}\| \le C\omega_*\,\|(\bE_U,\bE_V)\|_{\Fnw}^2 + C\sqrt{\kappa\sigma_{\min}}\,\|(\bE_U,\bE_V)\|_{\Fnw}\|\bE\|_{\twinfw}.
\end{equation*}
The same argument for the $V$-part gives
\begin{equation} \label{eq_rect_score_term_bound_1}
\begin{aligned}
\max_{i\in[n]}\|\Delta_{3,i}^{(U)}\| \vee \max_{\ell\in[q]}\|\Delta_{3,\ell}^{(V)}\| &\le C\omega_*\,\|(\bE_U,\bE_V)\|_{\Fnw}^2 + C\sqrt{\kappa\sigma_{\min}}\,\|(\bE_U,\bE_V)\|_{\Fnw}\|\bE\|_{\twinfw} \\
&\le c_0\sqrt{\kappa\sigma_{\min}}\,\omega_*\,\|(\bE_U,\bE_V)\|_{\Fnw} + c_0\alpha\sigma_{\min}\,\|\bE\|_{\twinfw},
\end{aligned}
\end{equation}
where the second line uses \eqref{eq_rect_step1_scaling_implications} and $\|(\bE_U,\bE_V)\|_{\Fnw}\le \varepsilon\tau_*\le c_0\alpha(\sigma_{\min}/\kappa)^{1/2}$.

\underline{\emph{Bound for $\Delta_{4,i}^{(U)}$ and $\Delta_{4,\ell}^{(V)}$.}}
By Assumption~\ref{assump_gradient_noise_rect}, \begin{align*}
\max_{i\in[n]}\|\Delta_{4,i}^{(U)}\| &\le q^{-1}\|\tilde\cG_{\mathrm{ave}}\|_{\infty\to 1}\|\bE_V\|_{2\to\infty} \le \bar\Delta_{\infty}(n,q,\delta)\,\|\bE_V\|_{2\to\infty},\\
\max_{\ell\in[q]}\|\Delta_{4,\ell}^{(V)}\| &\le n^{-1}\|\tilde\cG_{\mathrm{ave}}^\T\|_{\infty\to 1}\|\bE_U\|_{2\to\infty} \le \bar\Delta_{\infty}(n,q,\delta)\,\|\bE_U\|_{2\to\infty},
\end{align*}
where $\tilde\cG_{\mathrm{ave}} := \int_0^1 \tilde\cG(\bU(s),\bV(s))\,ds$.
Hence, by the theorem assumption $\bar\Delta_{\infty}(n,q,\delta)\le \alpha\sigma_{\min}/4$,
\begin{equation}\label{eq_rect_score_term_bound_2}
\max_{i\in[n]}\|\Delta_{4,i}^{(U)}\| \vee \max_{\ell\in[q]}\|\Delta_{4,\ell}^{(V)}\| \le \frac{\alpha\sigma_{\min}}{4}\,\|\bE\|_{\twinfw}.
\end{equation}

\underline{\emph{Bound for $\Delta_{5,i}^{(U)}$ and $\Delta_{5,\ell}^{(V)}$.}}
Recall that $M(\bU,\bV)$ has been defined as $n^{-1}(\bU-\bU^*)^\T\bU - q^{-1}\bV^\T(\bV-\bV^*)$ and thus we write $p_\alpha^\natural(\bU,\bV) = \frac{\alpha nq}{4}\|M(\bU,\bV)\|_{\Fn}^2$. For a perturbation $(\bH_U,\bH_V)\in\RR^{n\times r}\times \RR^{q\times r}$,
\begin{equation} \label{eq_rect_penalty_dM}
DM(\bU,\bV)[\bH_U,\bH_V] = n^{-1}\Big\{\bH_U^\T\bU+(\bU-\bU^*)^\T\bH_U\Big\} - q^{-1}\Big\{\bH_V^\T(\bV-\bV^*)+\bV^\T\bH_V\Big\}.
\end{equation}
Hence the gradients are
\begin{equation} \label{eq_rect_penalty_grad_formula}
\begin{aligned}
\nabla_{\bU}p_\alpha^\natural(\bU,\bV) &= \frac{\alpha q}{2}\Big\{\bU M(\bU,\bV)^\T+(\bU-\bU^*)M(\bU,\bV)\Big\}, \\ \nabla_{\bV}p_\alpha^\natural(\bU,\bV) &= -\frac{\alpha n}{2}\Big\{(\bV-\bV^*)M(\bU,\bV)^\T+\bV M(\bU,\bV)\Big\},
\end{aligned}
\end{equation}
and the directional Hessian action is
\begin{equation} \label{eq_rect_penalty_hess_formula}
\begin{aligned}
D\!\left(\nabla_{\bU}p_\alpha^\natural\right)(\bU,\bV)[\bH_U,\bH_V] = &\, \frac{\alpha q}{2}\Big\{ \bH_U\big(M^\T+M\big) +\bU\big(DM[\bH_U,\bH_V]\big)^\T
\\&\qquad+(\bU-\bU^*)DM[\bH_U,\bH_V] \Big\},\\
D\!\left(\nabla_{\bV}p_\alpha^\natural\right)(\bU,\bV)[\bH_U,\bH_V] =&\,
-\frac{\alpha n}{2}\Big\{ \bH_V\big(M^\T+M\big) +(\bV-\bV^*)\big(DM[\bH_U,\bH_V]\big)^\T \\&\qquad\qquad+(\bV - \bV^*)\,DM[\bH_U,\bH_V] \Big\},
\end{aligned}
\end{equation}
where $M=M(\bU,\bV)$ and $DM[\bH_U,\bH_V]=DM(\bU,\bV)[\bH_U,\bH_V]$.
Since $\nabla p_\alpha^\natural(\bU^*,\bV^*)=\zero$, we may use the simpler identity $\Delta_{5,i}^{(U)} = q^{-1}\nabla_{\bu_i}p_\alpha^\natural(\tilde\bU,\tilde\bV)$ and $\Delta_{5,\ell}^{(V)} = n^{-1}\nabla_{\bv_\ell}p_\alpha^\natural(\tilde\bU,\tilde\bV)$. The decomposition $M(\tilde\bU,\tilde\bV) = n^{-1}\bE_U^\T\bU^* - q^{-1}(\bV^*)^\T\bE_V + n^{-1}\bE_U^\T\bE_U - q^{-1}\bE_V^\T\bE_V$ implies that
\begin{equation*}
\|M(\tilde\bU,\tilde\bV)\| \le C\sqrt{\kappa\sigma_{\min}}\,\|(\bE_U,\bE_V)\|_{\Fnw} + \|(\bE_U,\bE_V)\|_{\Fnw}^2 \le C\sqrt{\kappa\sigma_{\min}}\,\|(\bE_U,\bE_V)\|_{\Fnw},
\end{equation*}
where the last step again uses $\|(\bE_U,\bE_V)\|_{\Fnw}\le \varepsilon\tau_*\le c_0\sqrt{\kappa\sigma_{\min}}$.
Applying \eqref{eq_rect_penalty_grad_formula}, we find
\begin{equation*}
q^{-1}\|\nabla_{\bu_i}p_\alpha^\natural(\tilde\bU,\tilde\bV)\| \le \frac{\alpha}{2}\Big(\|\bu_i^*\|+2\|\bE_{U,i}\|\Big)\|M(\tilde\bU,\tilde\bV)\|, \end{equation*}
and similarly
\begin{equation*}
n^{-1}\|\nabla_{\bv_\ell}p_\alpha^\natural(\tilde\bU,\tilde\bV)\| \le \frac{\alpha}{2}\Big(\|\bv_\ell^*\|+2\|\bE_{V,\ell}\|\Big)\|M(\tilde\bU,\tilde\bV)\|.
\end{equation*}
Therefore,
\begin{equation} \label{eq_rect_pen_term_bound}
\begin{aligned}
\max_{i\in[n]}\|\Delta_{5,i}^{(U)}\| \vee \max_{\ell\in[q]}\|\Delta_{5,\ell}^{(V)}\| &\le C\alpha\|M(\tilde\bU,\tilde\bV)\| \big(\omega_*+\|\bE\|_{\twinfw}\big) \\
&\le C\alpha\sqrt{\kappa\sigma_{\min}}\,\omega_*\,\|(\bE_U,\bE_V)\|_{\Fnw} + C\alpha\sqrt{\kappa\sigma_{\min}}\,\|(\bE_U,\bE_V)\|_{\Fnw}\|\bE\|_{\twinfw}\\
&\le C\alpha\sqrt{\kappa\sigma_{\min}}\,\omega_*\,\|(\bE_U,\bE_V)\|_{\Fnw} + c_0\alpha\sigma_{\min}\,\|\bE\|_{\twinfw}.
\end{aligned}
\end{equation}

\underline{\emph{Bound for $\Delta_{6,i}^{(U)}$ and $\Delta_{6,\ell}^{(V)}$.}}
Using the minimizing property of $(\tilde\bU,\tilde\bV)$ together with \eqref{eq_rect_score_inf_opt}, we have
\begin{equation}\label{eq_rect_noise_term_bound}
\begin{aligned}
\max_{i\in[n]}\|\Delta_{6,i}^{(U)}\| \vee \max_{\ell\in[q]}\|\Delta_{6,\ell}^{(V)}\| &\le \big\|\bS_Z^2\nabla_{\bZ}h_\alpha^\natural(\tilde\bU,\tilde\bV)\big\|_{\twinfw} + \big\|\bS_Z^2\nabla_{\bZ}h_\alpha^\natural(\bU^*,\bV^*)\big\|_{\twinfw} \\
&\le C\Delta_{\infty} \max\Big\{ \|\bV^*\|/\sqrt q, \ \|\bU^*\|/\sqrt n \Big\} \\
&\le C\Delta_{\infty}\omega_*, \end{aligned}
\end{equation}
where the last step uses $\|\bV^*\|/\sqrt q\le \|\bV^*\|_{2\to\infty}\le \omega_*$ and $\|\bU^*\|/\sqrt n\le \|\bU^*\|_{2\to\infty}\le \omega_*$.

Substituting \eqref{eq_rect_off_term_bound}, \eqref{eq_rect_score_term_bound_1}, \eqref{eq_rect_score_term_bound_2}, \eqref{eq_rect_pen_term_bound}, and \eqref{eq_rect_noise_term_bound} into \eqref{eq_rect_solve_diag_U}, we obtain
\begin{equation*}
\|\bE\|_{\twinfw} \le \frac{1}{3}\|\bE\|_{\twinfw} + C\sqrt{\frac{\kappa}{\sigma_{\min}}}\Big(\frac{\beta}{\alpha}+1\Big)\omega_*\,\|(\bE_U,\bE_V)\|_{\Fnw} + C\frac{\Delta_{\infty}}{\alpha\sigma_{\min}}\omega_*.
\end{equation*}
Using \eqref{eq_rect_step1_l2_half} and $\|(\bE_U,\bE_V)\|_{\Fnw}\le \frac{1}{2}\varepsilon\tau_* \le C\varepsilon\sqrt{r\kappa\sigma_{\min}}$, we conclude that
\begin{equation*}
\|\bE\|_{\twinfw} \le \frac{1}{3}\|\bE\|_{\twinfw} + C\Big(\frac{\beta}{\alpha}+1\Big)\sqrt r\,\kappa\,\varepsilon\,\omega_* + C\frac{\Delta_{\infty}}{\alpha\sigma_{\min}}\omega_*.
\end{equation*}
By \eqref{eq_thm_general_theory_noisy_uv_scaling} and $\varepsilon\le \epsilon$, the last two terms on the right-hand side are bounded by $C\zeta_r\varepsilon\,\omega_*+C \Delta_{\infty}\omega_*/(\alpha\sigma_{\min}) \le Cc\epsilon\,\omega_*+Cc_0\varepsilon\,\omega_* \le c_0\epsilon\omega_*,$. After shrinking $c_0$ if necessary, this yields
\begin{equation} \label{eq_rect_step1_linf_half}
\|(\tilde\bU-\bU^*,\tilde\bV-\bV^*)\|_{\twinfw} = \|\bE\|_{\twinfw} \le \frac{1}{2}\epsilon\omega_*.
\end{equation}
Therefore $(\tilde\bU,\tilde\bV)$ lies in the interior of $\cDuvde$.

It remains to show that $(\tilde\bU,\tilde\bV)$ is stationary for $h_\alpha^\natural$, and that $h_\alpha^\natural$ is uniformly strongly convex on $\cDuvd$. For any $(\bU,\bV)\in\cDuvd$, Lemma~\ref{lemma_general_convex_asym} gives $\lambda_{\min}\!\big( \bS_z\big\{\nabla_{\bz}^2 h_\alpha^\natural(\bU,\bV)-\cG_e(\bU,\bV)\big\}\bS_z \big) \ge {\alpha\sigma_{\min}}/{2}$.
Moreover, using the population first-order condition $\bar\cG(\bU^*,\bV^*)=\zero$, Equation~\ref{assump_gradient_noise_rect_1}, and Assumption~\ref{assump_rect_lip}, we obtain
\begin{align*}
\sup_{(\bU,\bV)\in\cDuvd} \big\|\bS_z\cG_e(\bU,\bV)\bS_z\big\| &= \sup_{(\bU,\bV)\in\cDuvd} \frac{\|\cG(\bU,\bV)\|}{\sqrt{nq}} \\
&\le \Delta_2 + \sup_{(\bU,\bV)\in\cDd} \frac{\|\bar\cG(\bU,\bV)-\bar\cG(\bU^*,\bV^*)\|}{\sqrt{nq}} \\
&\le \Delta_2 + \frac{L_2}{\sqrt{nq}} \sup_{(\bU,\bV)\in\cDuvd} \|\bU\bV^\T-\bU^*(\bV^*)^\T\|_{\Fn}.
\end{align*}
For $(\bU,\bV)\in\cDuvd$, we note that we can bound
\begin{align*}
\frac{\|\bU\bV^\T-\bU^*(\bV^*)^\T\|_{\Fn}}{\sqrt{nq}} &\le \frac{\|\bU - \bU^*\|_{\Fn}}{\sqrt n}\frac{\|\bV^*\|}{\sqrt q} + \frac{\|\bV - \bV^*\|_{\Fn}}{\sqrt q}\frac{\|\bU^*\|}{\sqrt n} + \frac{\|\bU - \bU^*\|_{\Fn}}{\sqrt n}\frac{\|\bV - \bV^*\|_{\Fn}}{\sqrt q} \\ &\le C\epsilon\tau_*\sqrt{\kappa\sigma_{\min}} + \epsilon^2\tau_*^2 \\ &\le C\epsilon\sqrt r\,\kappa\,\sigma_{\min},
\end{align*}
where we used $\|(\bU - \bU^*,\bV - \bV^*)\|_{\Fnw}\le \epsilon\tau_*$ and $\tau_*\le\sqrt{r\kappa\sigma_{\min}}$.
Hence we know
\begin{equation*}
\sup_{(\bU,\bV)\in\cDuvd} \big\|\bS_z\cG_e(\bU,\bV)\bS_z\big\| \le \Delta_2 + C L_2\epsilon\sqrt r\,\kappa\,\sigma_{\min} \le c_0\alpha\sigma_{\min},
\end{equation*}
after shrinking $c_0$ if necessary, where we used \eqref{eq_rect_step1_scaling_implications}. It therefore follows that
\begin{equation} \label{eq_rect_uniform_strong_convexity}
\inf_{(\bU,\bV)\in\cDuvd} \lambda_{\min}\!\Big( \bS_z\nabla_{\bz}^2 h_\alpha^\natural(\bU,\bV)\bS_z \Big) \ge
\frac{\alpha\sigma_{\min}}{4}.
\end{equation}
Since $\bS_z$ is invertible, $h_\alpha^\natural$ is strongly convex on $\cDuvd$; in particular, it has a unique minimizer there.
Because $(\tilde\bU,\tilde\bV)$ is an interior minimizer of $(\bU,\bV)\mapsto \|\bS_Z^2\nabla_{\bZ}h_\alpha^\natural(\bU,\bV)\|_{\twinfw}^2$ over $\cDuvde$, and \eqref{eq_rect_uniform_strong_convexity} implies that $\nabla_{\bz}^2 h_\alpha^\natural(\tilde\bU,\tilde\bV)$ is positive definite, we must have
\begin{equation} \label{eq_rect_stationary_tilde}
\nabla_{\bZ} h_\alpha^\natural(\tilde\bU,\tilde\bV)=\zero.
\end{equation}
Indeed, if $\nabla_{\bZ} h_\alpha^\natural(\tilde\bU,\tilde\bV)\neq \zero$, then with the following direction
\begin{equation*}
\bd := -\Big\{\nabla_{\bz}^2 h_\alpha^\natural(\tilde\bU,\tilde\bV)\Big\}^{-1} \mathrm{vec}\!\left(\nabla_{\bZ} h_\alpha^\natural(\tilde\bU,\tilde\bV)^\T\right), \end{equation*}
and $(\bD_U,\bD_V)$ denoting the matrix version of $\bd$, we would have
\begin{align*}
\mathrm{vec}\!\left(\bS_Z^2\nabla_{\bZ} h_\alpha^\natural(\tilde\bU+t\bD_U,\tilde\bV+t\bD_V)^\T\right) =&\,  \mathrm{vec}\!\left(\bS_Z^2\nabla_{\bZ} h_\alpha^\natural(\tilde\bU,\tilde\bV)^\T\right) \\ &\, - t\,\mathrm{vec}\!\left(\bS_Z^2\nabla_{\bZ} h_\alpha^\natural(\tilde\bU,\tilde\bV)^\T\right) + o(t)
\end{align*}
as $t\to 0$. Therefore, it holds that $\|\bS_Z^2\nabla_{\bZ} h_\alpha^\natural(\tilde\bU+t\bD_U,\tilde\bV+t\bD_V)\|_{\twinfw} < \|\bS_Z^2\nabla_{\bZ} h_\alpha^\natural(\tilde\bU,\tilde\bV)\|_{\twinfw}$ for all sufficiently small $t>0$, contradicting the interior optimality of $(\tilde\bU,\tilde\bV)$.
Hence, we establish that $(\tilde\bU,\tilde\bV)$ is stationary for $h_\alpha^\natural$. Furthermore, since $h_\alpha^\natural(\cdot)$ is strongly convex on $\cDuvd$, \eqref{eq_rect_stationary_tilde} further implies that $(\tilde\bU,\tilde\bV)$ is the unique minimizer of $h_\alpha^\natural(\cdot)$ on $\cDuvd$.
This completes Step~1.

\paragraph{\it \underline{Step 2: sharpen the statistical rates.}}
We now sharpen the weighted Frobenius and weighted $\ell_\infty$ bounds for $(\bE_U,\bE_V) = (\tilde\bU - \bU^*, \tilde\bV - \bV^*)$. We first sharpen the weighted Frobenius error.
Returning to \eqref{eq_rect_mvt_optimizer}, Step~1 gives \eqref{eq_rect_stationary_tilde}, so $\Gamma_{4,i}$ on the right-hand side becomes
\begin{equation*}
\Gamma_4 = -\mathrm{vec}\Big( \bS_Z\nabla_{\bZ}h_\alpha^\natural(\bU^*,\bV^*)^\T \Big).
\end{equation*}
Since \eqref{eq_bound_score_average_uv} gives that $\big\|\bS_Z\nabla_{\bZ}h_\alpha^\natural(\bU^*,\bV^*)\big\|_{\Fn} \le C\sqrt{nq}\,\Delta_2\,\tau_*$, together with the bounds on $\Gamma_2$ and $\Gamma_3$ from Step~1, we follow the derivation of \eqref{eq_rect_crude_l2} to obtain
\begin{equation} \label{eq_rect_step2_l2_pre}
\|(\bE_U,\bE_V)\|_{\Fnw} \le C\frac{\Delta_2}{\alpha\sigma_{\min}}\|(\bE_U,\bE_V)\|_{\Fnw} + C\frac{\sqrt{\kappa\sigma_{\min}}}{\alpha\sigma_{\min}}\|(\bE_U,\bE_V)\|_{\Fnw}^2 + C\frac{\Delta_2}{\alpha\sigma_{\min}}\tau_*.
\end{equation}
Using \eqref{eq_rect_step1_l2_half} and $\tau_*\le \sqrt{r\kappa\sigma_{\min}}$, we further have
\begin{equation*}
\frac{\sqrt{\kappa\sigma_{\min}}}{\alpha\sigma_{\min}}\|(\bE_U,\bE_V)\|_{\Fnw} \le C\frac{\sqrt{\kappa\sigma_{\min}}}{\alpha\sigma_{\min}}\varepsilon\tau_* \le C\varepsilon\frac{\sqrt r\,\kappa}{\alpha} \le Cc_0.
\end{equation*}
Together with $\Delta_2/(\alpha\sigma_{\min})\le c_0$, after shrinking $c_0$ if necessary, the first two terms on the right-hand side of \eqref{eq_rect_step2_l2_pre} can be absorbed into the left-hand side. Therefore, we arrive at
\begin{equation} \label{eq_rect_step2_l2_final}
\|(\tilde\bU-\bU^*,\tilde\bV-\bV^*)\|_{\Fnw} \le C\frac{\Delta_2}{\alpha\sigma_{\min}}\tau_*.
\end{equation}

We next sharpen the weighted $\ell_\infty$ bound.
Returning to the last pre-absorption inequality in Step~1, one has
\begin{equation} \label{eq_rect_step2_linf_pre}
\|(\bE_U,\bE_V)\|_{\twinfw} \le \frac{1}{3}\|(\bE_U,\bE_V)\|_{\twinfw} + C\Big(\frac{\beta}{\alpha}+1\Big)\sqrt r\,\kappa\, \frac{\Delta_2}{\alpha\sigma_{\min}}\omega_* + C\frac{\Delta_\infty}{\alpha\sigma_{\min}}\omega_*,
\end{equation}
where the second term on the right side follows from \eqref{eq_rect_step2_l2_final}. Therefore, we have from the scaling condition \eqref{eq_rect_step1_scaling_implications} that $(\alpha^{-1}\beta + 1)\sqrt r\,\kappa \le (\alpha^{-1}\beta + \sqrt{r}) \sqrt{r}\kappa = \zeta_r\le c_0\Delta_{\infty}/\Delta_2$. We therefore obtain
\begin{equation} \label{eq_rect_step2_linf_final}
\|(\tilde\bU-\bU^*,\tilde\bV-\bV^*)\|_{\twinfw} \le C\frac{\Delta_\infty}{\alpha\sigma_{\min}}\omega_*.
\end{equation}

Next, we show that $(\tilde\bU,\tilde\bV) = (\hat\bU\hat\bQ,\hat\bV\hat\bQ^{-\T})$. We first claim
\begin{equation} \label{eq_rect_penalty_zero}
p_\alpha^\natural(\tilde\bU,\tilde\bV)=0.
\end{equation}
Suppose not.
Apply Lemma~\ref{lemma_alignment_near_rotation} with $\bm P=\bI_r$ to $(\tilde\bU,\tilde\bV)$. By \eqref{eq_rect_step2_l2_final}, \eqref{eq_rect_step1_scaling_implications}, and
$\tau_*\le \sqrt{r\kappa\sigma_{\min}}$,
\begin{equation*}
\max\Big\{n^{-1/2}\|\tilde\bU-\bU^*\|_{\Fn}, \ q^{-1/2}\|\tilde\bV-\bV^*\|_{\Fn} \Big\}
\le
C\frac{\Delta_2}{\alpha\sigma_{\min}}\tau_* \le Cc_0\sqrt{\sigma_{\min}} < \frac{1}{80}\sqrt{\sigma_{\min}}
\end{equation*}
for sufficiently small $c_0$.
Hence the optimal alignment matrix $\bQ_e\in\argmin_{\bG\in GL(r)}\|(\tilde\bU\bG-\bU^*,\tilde\bV\bG^{-\T}-\bV^*)\|_{\Fnw}$ exists. Moreover, Lemma~\ref{lemma_alignment_near_rotation} implies
\begin{equation*}
\|\bQ_e-\bI_r\|\vee \|\bQ_e^{-\T}-\bI_r\| \le C\frac{\|(\tilde\bU-\bU^*,\tilde\bV-\bV^*)\|_{\Fnw}}{\sqrt{\sigma_{\min}}} \le C\sqrt{r\kappa}\frac{\Delta_2}{\alpha\sigma_{\min}} \le
Cc_0\varepsilon,
\end{equation*}
where the last step uses the lower bound in \eqref{eq_thm_general_theory_noisy_uv_scaling}.

Together with the strict interior bounds from Step~1, one has
\begin{equation*}
\|(\tilde\bU-\bU^*,\tilde\bV-\bV^*)\|_{\Fnw}\le \frac{1}{2}\varepsilon\tau_*, \qquad \|(\tilde\bU-\bU^*,\tilde\bV-\bV^*)\|_{\twinfw}\le \frac{1}{2}\epsilon\omega_*,
\end{equation*}
which subsequently gives
\begin{align*}
\|(\tilde\bU\bQ_e-\bU^*,\tilde\bV\bQ_e^{-\T}-\bV^*)\|_{\Fnw} &\le \|(\tilde\bU-\bU^*,\tilde\bV-\bV^*)\|_{\Fnw} + \|(\tilde\bU(\bQ_e-\bI_r),\tilde\bV(\bQ_e^{-\T}-\bI_r))\|_{\Fnw} \\
&\le \frac{1}{2}\varepsilon\tau_* + C\tau_*\big(\|\bQ_e-\bI_r\|\vee \|\bQ_e^{-\T}-\bI_r\|\big) \\
&\le \frac{1}{2}\varepsilon\tau_*+Cc_0\varepsilon\tau_* < \varepsilon\tau_*.
\end{align*}
Similarly, one can get 
\begin{align*}
\|(\tilde\bU\bQ_e-\bU^*,\tilde\bV\bQ_e^{-\T}-\bV^*)\|_{\twinfw} &\le \|(\tilde\bU-\bU^*,\tilde\bV-\bV^*)\|_{\twinfw} + \|(\tilde\bU(\bQ_e-\bI_r),\tilde\bV(\bQ_e^{-\T}-\bI_r))\|_{\twinfw} \\
&\le \frac{1}{2}\epsilon\omega_* + C\omega_*\big(\|\bQ_e-\bI_r\|\vee \|\bQ_e^{-\T}-\bI_r\|\big) \\
&\le \frac{1}{2}\epsilon\omega_*+Cc_0\varepsilon\omega_* < \epsilon\omega_*,
\end{align*}
for sufficiently small $c_0$.
Therefore, we arrive at $(\tilde\bU\bQ_e,\tilde\bV\bQ_e^{-\T})\in \cDuvde\subseteq \cDuvd$. Applying Lemma~\ref{lemma_id_rec} to the aligned pair
$(\tilde\bU\bQ_e,\tilde\bV\bQ_e^{-\T})$, we obtain $p_\alpha^\natural(\tilde\bU\bQ_e,\tilde\bV\bQ_e^{-\T})=0$.
The rotational invariance $(\tilde\bU\bQ_e)(\tilde\bV\bQ_e^{-\T})^\T = \tilde\bU\tilde\bV^\T$ therefore yields
\begin{equation*}
h_\alpha^\natural(\tilde\bU\bQ_e,\tilde\bV\bQ_e^{-\T}) = \cL(\tilde\bU\tilde\bV^\T) < \cL(\tilde\bU\tilde\bV^\T)+p_\alpha^\natural(\tilde\bU,\tilde\bV) = h_\alpha^\natural(\tilde\bU,\tilde\bV),
\end{equation*}
contradicting the fact that $(\tilde\bU,\tilde\bV)$ is the unique minimizer of $h_\alpha^\natural(\cdot)$ on $\cDuvd$. This proves \eqref{eq_rect_penalty_zero}.

Since \eqref{eq_rect_stationary_tilde} and \eqref{eq_rect_penalty_zero} hold, and the gradient of $p_\alpha^\natural$ vanishes whenever $p_\alpha^\natural=0$, we conclude that
\begin{equation} \label{eq_rect_tilde_first_order_L}
\nabla_{\bU}\cL(\tilde\bU\tilde\bV^\T)=\zero, \qquad 
\nabla_{\bV}\cL(\tilde\bU\tilde\bV^\T)=\zero.
\end{equation}
Equivalently, there is 
\begin{equation} \label{eq_rect_tilde_G_first_order}
\cG(\tilde\bU,\tilde\bV)\tilde\bV=\zero, \qquad \cG(\tilde\bU,\tilde\bV)^\T\tilde\bU=\zero.
\end{equation}

We finally identify $(\tilde\bU,\tilde\bV)$ with the aligned constrained optimizer $(\hat\bU,\hat\bV)$. 
By definition of $\hat\bQ\in\argmin_{\bG\in GL(r)} = \|(\hat\bU\bG-\bU^*,\hat\bV\bG^{-\T}-\bV^*)\|_{\Fnw}$, it holds that $(\hat\bU\hat\bQ,\hat\bV\hat\bQ^{-\T})\in\cDuvd$, and therefore Lemma~\ref{lemma_id_rec} yields $p_\alpha^\natural(\hat\bU\hat\bQ,\hat\bV\hat\bQ^{-\T})=0$.
Since $(\tilde\bU,\tilde\bV)\in\cDuvd\subseteq\cDuvinf$ and $(\hat\bU,\hat\bV)$ minimizes $\cL(\bU\bV^\T)$ over $\cDuvinf$, we have
\begin{equation*}
h_\alpha^\natural(\hat\bU\hat\bQ,\hat\bV\hat\bQ^{-\T}) = \cL(\hat\bU\hat\bV^\T) \le \cL(\tilde\bU\tilde\bV^\T) = h_\alpha^\natural(\tilde\bU,\tilde\bV),
\end{equation*}
where the last equality uses \eqref{eq_rect_penalty_zero}.
On the other hand, $(\tilde\bU,\tilde\bV)$ is the unique minimizer of $h_\alpha^\natural(\cdot)$ on $\cDuvd$, and consequently, $h_\alpha^\natural(\tilde\bU,\tilde\bV) \le h_\alpha^\natural(\hat\bU\hat\bQ,\hat\bV\hat\bQ^{-\T})$.
Hence, equality holds, and uniqueness implies
\begin{equation}
\label{eq_rect_balanced_identification}
(\tilde\bU,\tilde\bV)=(\hat\bU\hat\bQ,\hat\bV\hat\bQ^{-\T}).
\end{equation}
Consequently, with the rates for $(\tilde\bU,\tilde\bV)$ established above and the equivalence \eqref{eq_rect_balanced_identification}, we conclude that the optimizer bounds in~\eqref{eq_optimum_UV_error_bound} hold, and that the aligned solution $(\hat\bU\hat\bQ,\hat\bV\hat\bQ^{-\T})$ is unique and satisfies the first-order condition.
Step~2 is completed.

\paragraph{\it \underline{ Step 3: $\ell_2$- and $\ell_{\infty}$ error contraction for $(\bU^t,\bV^t)$.}}
By Step~2, we have $(\hat\bU\hat\bQ,\hat\bV\hat\bQ^{-\T})=(\tilde\bU,\tilde\bV)$.
It suffices to control the gradient iterates relative to the re-centered pair
$(\tilde\bU,\tilde\bV)$.
Specifically, we define the re-centered balancing penalty
\begin{equation*}
p_{\alpha,\dagger}^{\natural}(\bU,\bV) := \frac{\alpha nq}{4} \Big\|n^{-1}(\bU-\tilde\bU)^\T\bU-q^{-1}\bV^\T(\bV-\tilde\bV)\Big\|_{\Fn}^2,
\end{equation*}
and the associated objective $h_{\alpha,\dagger}^{\natural}(\bU,\bV):=\cL(\bU\bV^\T)+p_{\alpha,\dagger}^{\natural}(\bU,\bV)$.
By \eqref{eq_rect_tilde_first_order_L}, there is
\begin{equation}
\label{eq_rect_recentered_stationary}
\nabla_{\bU}h_{\alpha,\dagger}^{\natural}(\tilde\bU,\tilde\bV)=\zero, \qquad
\nabla_{\bV}h_{\alpha,\dagger}^{\natural}(\tilde\bU,\tilde\bV)=\zero.
\end{equation}

For later use, similar to the purpose of defining $\phi^{\dagger}$ and $\psi_n^{\dagger}$, we let
\begin{equation}
\label{eq_rect_step4_radii}
\phi_{nq}^{\dagger}:=\phi_{nq}+C\frac{\Delta_2}{\alpha\sigma_{\min}}, \qquad
\psi_{nq}^{\dagger}:=\psi_{nq}+C\frac{\Delta_{\infty}}{\alpha\sigma_{\min}},
\end{equation}
where $C>0$ is a sufficiently large universal constant. By the initialization assumptions in Theorem~\ref{thm_general_theory_uv}, the noise assumptions in Theorem~\ref{thm_general_theory_UV_noise}, and the estimator bounds proved in Steps~2,
after enlarging $C$ and shrinking $c_0$ if necessary, we have
\begin{equation}
\label{eq_rect_step4_dagger_scaling}
\phi_{nq}^{\dagger}\le \frac{2}{3}\epsilon \wedge \frac{\alpha}{\beta^2\kappa^2\sqrt{r\kappa}}, \qquad
\psi_{nq}^{\dagger}\le \frac{2}{3}\epsilon, \qquad \frac{\beta}{\alpha}\kappa\sqrt{\kappa r}\frac{\phi_{nq}^{\dagger}}{\psi_{nq}^{\dagger}}\le c_0.
\end{equation}
Here, the bound for the ratio $\phi_{nq}^{\dagger}/\psi_{nq}^{\dagger}$ can be verified via $\phi_{nq}^{\dagger}/\psi_{nq}^{\dagger}\le \phi_{nq}/\psi_{nq} + \Delta_2/\Delta_{\infty}$ together with the initialization requirement~\eqref{eq_init_consis_main_R0} and~\eqref{eq_sym_init_Z_uniform_uv} and the theorem's scaling condition \eqref{eq_rect_step1_scaling_implications}.

Moreover, \eqref{eq_rect_step2_l2_final}, \eqref{eq_rect_step2_linf_final}, and Weyl's inequality give
\begin{equation} \label{eq_rect_step4_tilde_scale}
\sigma_r(\tilde\bU/\sqrt n)\wedge \sigma_r(\tilde\bV/\sqrt q) \ge (1-c_0)\sqrt{\sigma_{\min}}, \qquad
\frac{\|\tilde\bU\|}{\sqrt n}\vee \frac{\|\tilde\bV\|}{\sqrt q} \le C\sqrt{\kappa\sigma_{\min}}, \end{equation}
\begin{equation} \label{eq_rect_step4_tilde_infty_scale}
\|\tilde\bU\|_{2\to\infty}\vee \|\tilde\bV\|_{2\to\infty} \le 2\omega_*, \qquad \tilde\tau_*^2:=\frac{n^{-1}\|\tilde\bU\|_{\Fn}^2+q^{-1}\|\tilde\bV\|_{\Fn}^2}{2} \le 2\tau_*^2.
\end{equation}
The explicit Hessian formulas for $p_{\alpha}^{\natural}$ established in Step~1 are reference invariant. Therefore the proof of Theorem~\ref{thm_general_theory_uv} applies verbatim after replacing $(\bU^*,\bV^*)$ by $(\tilde\bU,\tilde\bV)$. In particular, the same lower and upper bounds for the re-centered Hessian blocks hold along any segment contained in $\cDuvinf$, up to changing absolute constants.

For each $t\ge 0$, let $\bG_t^{\dagger} \in \argmin_{\bG\in GL(r)} \| (\bU^t\bG-\tilde\bU,\,\bV^t\bG^{-\T}-\tilde\bV)\|_{\Fnw}$,
and define the aligned iterates and errors
\begin{equation}\label{eq_define_loo_et}
\widetilde{\bU}^{t}:=\bU^t\bG_t^{\dagger}, \qquad
\widetilde{\bV}^{t}:=\bV^t(\bG_t^{\dagger})^{-\T}, \qquad
\tilde\bE_U^t:=\widetilde{\bU}^{t}-\tilde\bU, \qquad
\tilde\bE_V^t:=\widetilde{\bV}^{t}-\tilde\bV.
\end{equation}
Since $\bG_t^{\dagger}$ is the optimal alignment of $(\bU^t,\bV^t)$ to $(\tilde\bU,\tilde\bV)$, the proof of Lemma~\ref{lemma_id_rec}, with $(\bU^*,\bV^*)$ replaced by $(\tilde\bU,\tilde\bV)$, leads to $n^{-1}(\widetilde{\bU}^{t}-\tilde\bU)^\T\widetilde{\bU}^{t} = q^{-1}\widetilde{\bV}^{t\T}(\widetilde{\bV}^{t}-\tilde\bV)$. This identity is the re-centered analogue of the balancing relation used in the deterministic proof, and it again forces the penalty gradient to vanish at the aligned iterate. Subsequently, we know that
\begin{equation*}
\nabla_{\bU}p_{\alpha,\dagger}^{\natural}(\widetilde{\bU}^{t},\widetilde{\bV}^{t})=\zero, \qquad \nabla_{\bV}p_{\alpha,\dagger}^{\natural}(\widetilde{\bU}^{t},\widetilde{\bV}^{t})=\zero.
\end{equation*}
Consequently, writing $\bLambda_t^{\dagger}:=(\bG_t^{\dagger})^\T\bG_t^{\dagger}$, the gradient update
can be rewritten as the re-centered version
\begin{align}
\bar\bU^{t+1} &:=\bU^{t+1}\bG_t^{\dagger} =\widetilde{\bU}^{t} -\frac{\eta}{q}\nabla_{\bU}h_{\alpha,\dagger}^{\natural}(\widetilde{\bU}^{t},\widetilde{\bV}^{t}) -\frac{\eta}{q}\nabla_{\bU}h_{\alpha,\dagger}^{\natural}(\widetilde{\bU}^{t},\widetilde{\bV}^{t})(\bLambda_t^{\dagger}-\bI_r), \label{eq_rect__U_dagger} \\
\bar\bV^{t+1} &:=\bV^{t+1}(\bG_t^{\dagger})^{-\T} =\widetilde{\bV}^{t} -\frac{\eta}{n}\nabla_{\bV}h_{\alpha,\dagger}^{\natural}(\widetilde{\bU}^{t},\widetilde{\bV}^{t}) -\frac{\eta}{n}\nabla_{\bV}h_{\alpha,\dagger}^{\natural}(\widetilde{\bU}^{t},\widetilde{\bV}^{t})\big\{(\bLambda_t^{\dagger})^{-1}-\bI_r\big\}.\label{eq_rect__V_dagger}
\end{align}

In what follows, we follow the proof of Theorem~\ref{thm_general_theory_uv} to prove by induction on $t$ that, for all $0\le s\le t$,
\begin{enumerate}[label=(\arabic*)]
    \item\label{item_uv_ind_exist_dagger} the optimal alignment matrix $\bG_s^{\dagger}$ exists;
    \item\label{item_uv_ind_l2_dagger} $\|(\tilde\bE_U^s,\tilde\bE_V^s)\|_{\Fnw}\le \rho^s\phi_{nq}^{\dagger}\tau_*$;
    \item\label{item_uv_ind_linf_dagger} $\|(\tilde\bE_U^s,\tilde\bE_V^s)\|_{\twinfw}\le \rho^s\psi_{nq}^{\dagger}\omega_*$;
    \item\label{item_uv_ind_rot_dagger} $\|\bG_s^{\dagger}-\bR^0\|\vee \|(\bG_s^{\dagger})^{-\T}-\bR^0\|
    \le \iota_0\alpha/(\beta\kappa)$.
\end{enumerate}
for some sufficiently small $\iota_0\in(0,1/30)$.

\medskip
\noindent{\bf Initialization.}
By the theorem's condition on the initialization $(\bU^0,\bV^0)$, we know that 
\begin{equation*}
\max\Big\{
 n^{-1/2}\|\bU^0\bR^0-\tilde\bU\|_{\Fn},
 q^{-1/2}\|\bV^0(\bR^0)^{-\T}-\tilde\bV\|_{\Fn}
\Big\}
\le \phi_{nq}^{\dagger}\tau_*.
\end{equation*}
With the scaling condition $\phi_{nq}^{\dagger}\tau_*\le c_0{\alpha}/(\beta^2\kappa^2\sqrt{r\kappa})\times \sqrt{r\kappa\sigma_{\min}}\le c_0{\alpha}(\beta\kappa)^{-1}\sqrt{\sigma_{\min}}$ by \eqref{eq_rect_step4_dagger_scaling} and $\tau_*\le \sqrt{r\kappa\sigma_{\min}}$, Lemma~\ref{lemma_alignment_near_rotation}, applied with $\bP=\bR^0$ and reference $(\tilde\bU,\tilde\bV)$, yields that $\bG_0^{\dagger}$ exists and that
\begin{equation}
\|\bG_0^{\dagger}-\bR^0\|\vee \|(\bG_0^{\dagger})^{-\T}-\bR^0\| \le \frac{5\phi_{nq}^{\dagger}\tau_*}{\sigma_r(\tilde\bU/\sqrt n)\wedge \sigma_r(\tilde\bV/\sqrt q)} \le C\phi_{nq}^{\dagger}\frac{\tau_*}{\sqrt{\sigma_{\min}}} \le \iota_0\frac{\alpha}{\beta\kappa},\label{item_uv_ind_exist_dagger_t0}
\end{equation}
after shrinking $c_0$ if necessary.
This proves induction hypotheses~\ref{item_uv_ind_exist_dagger} and \ref{item_uv_ind_rot_dagger} at $t=0$.

The theorem assumptions together with \eqref{eq_rect_step2_l2_final} and \eqref{eq_rect_step2_linf_final} imply
\begin{equation*}
\|(\bU^0\bR^0-\tilde\bU,\bV^0(\bR^0)^{-\T}-\tilde\bV)\|_{\Fnw} \le \phi_{nq}^{\dagger}\tau_*, \qquad \|(\bU^0\bR^0-\tilde\bU,\bV^0(\bR^0)^{-\T}-\tilde\bV)\|_{\twinfw}
\le \psi_{nq}^{\dagger}\omega_*.
\end{equation*}
Hence with \eqref{item_uv_ind_exist_dagger_t0}, one can verify that the induction hypotheses~\ref{item_uv_ind_l2_dagger} and \ref{item_uv_ind_linf_dagger} hold at $t=0$.
Now assume \ref{item_uv_ind_exist_dagger}--\ref{item_uv_ind_rot_dagger} hold for all $0\le s\le t$. We verify them at time $t+1$.

\medskip
\noindent{\bf Weighted $\ell_2$ contraction for iterate $(\bU^t,\bV^t)$.}
For $s\in[0,1]$ and $t\ge 0$, define
\begin{equation*}
\bU_t^{\dagger}(s):=\tilde\bU+s\tilde\bE_U^t, \qquad \bV_t^{\dagger}(s):=\tilde\bV+s\tilde\bE_V^t.
\end{equation*}
By \eqref{eq_rect_step2_l2_final}, \eqref{eq_rect_step2_linf_final}, and the induction hypotheses, we have
\begin{align*}
\disttwo\{(\bU_t^{\dagger}(s),\bV_t^{\dagger}(s)),(\bU^*,\bV^*)\} &\le \|(\tilde\bU-\bU^*,\tilde\bV-\bV^*)\|_{\Fnw}  + s\|(\tilde\bE_U^t,\tilde\bE_V^t)\|_{\Fnw} \\
&\le \Big(C\frac{\Delta_2}{\alpha\sigma_{\min}}+\phi_{nq}^{\dagger}\Big)\tau_* \le \epsilon\tau_*, \\
\distinf\{(\bU_t^{\dagger}(s),\bV_t^{\dagger}(s)),(\bU^*,\bV^*)\} &\le \|(\tilde\bU-\bU^*,\tilde\bV-\bV^*)\|_{\twinfw} +s\|(\tilde\bE_U^t,\tilde\bE_V^t)\|_{\twinfw} \\ &\le \Big(C\frac{\Delta_{\infty}}{\alpha\sigma_{\min}}+\psi_{nq}^{\dagger}\Big)\omega_*
\le \epsilon\omega_*,
\end{align*}
where the last two inequalities follow from \eqref{eq_rect_step1_scaling_implications} and
\eqref{eq_rect_step4_dagger_scaling}. Hence, we know that for all $s\in[0,1]$,
\begin{equation}
\label{eq_rect_step4_segment_localization}
(\bU_t^{\dagger}(s),\bV_t^{\dagger}(s))\in \cDuvinf .
\end{equation}
Therefore, all assumptions stated with $\cD=\cDuvinf$ are available along this interpolation path. Further let
\begin{equation*}
\tilde\be_t:= \begin{pmatrix} \mathrm{vec}\big((\tilde\bE_U^t)^\T\big)\\ \mathrm{vec}\big((\tilde\bE_V^t)^\T\big) \end{pmatrix}, \;
\tilde\bw_t:=\bS_z^{-1}\tilde\be_t, \;
\bar\be_{t+1}:= \begin{pmatrix}
\mathrm{vec}\big((\bar\bU^{t+1}-\tilde\bU)^\T\big)\\ \mathrm{vec}\big((\bar\bV^{t+1}-\tilde\bV)^\T\big)
\end{pmatrix}, \;
\bar\bw_{t+1}:=\bS_z^{-1}\bar\be_{t+1},
\end{equation*}
and
\begin{equation*}
\vartheta_t^{\dagger}:=\|\bLambda_t^{\dagger}-\bI_r\|\vee \|(\bLambda_t^{\dagger})^{-1}-\bI_r\|, \qquad
\cR_t^{\dagger}:= \begin{pmatrix} (\bLambda_t^{\dagger}-\bI_r)\otimes \bI_n & \zero\\ \zero & \big((\bLambda_t^{\dagger})^{-1}-\bI_r\big)\otimes \bI_q
\end{pmatrix}.
\end{equation*}
Using \eqref{eq_rect__U_dagger}--\eqref{eq_rect__V_dagger},
\eqref{eq_rect_recentered_stationary}, and the fundamental theorem of calculus (Theorem 4.2 in \citet{lang2012real}, Chapter XIII) along the segment
$\{(\bU_t^{\dagger}(s),\bV_t^{\dagger}(s)):0\le s\le 1\}$, we know
\begin{equation*} \bar\bw_{t+1} = (\bI_{nr+qr}-\eta\bar\cA_t^{\dagger})\tilde\bw_t -\eta\tilde\cG_t^{\dagger}\tilde\bw_t -\eta\bar\cG_t^{\dagger}\tilde\bw_t -\eta\cR_t^{\dagger}(\bar\cA_t^{\dagger}+\tilde\cG_t^{\dagger}+\bar\cG_t^{\dagger})\tilde\bw_t,
\end{equation*}
where we define
\begin{equation*}
\bar\cA_t^{\dagger} := \bS_z\int_0^1 \Big\{\nabla_{\bz}^2 h_{\alpha,\dagger}^{\natural}(\bU_t^{\dagger}(s),\bV_t^{\dagger}(s)) -\cG_e(\bU_t^{\dagger}(s),\bV_t^{\dagger}(s))\Big\}ds\,\bS_z,
\end{equation*}
\begin{equation*}
\tilde\cG_t^{\dagger} := \bS_z\int_0^1 \Big\{\cG_e(\bU_t^{\dagger}(s),\bV_t^{\dagger}(s))- \EE\cG_e(\bU_t^{\dagger}(s),\bV_t^{\dagger}(s))\Big\}ds\,\bS_z, \end{equation*}
\begin{equation*}
\bar\cG_t^{\dagger} := \bS_z\int_0^1\EE\cG_e(\bU_t^{\dagger}(s),\bV_t^{\dagger}(s))ds\,\bS_z.
\end{equation*}
Therefore, we bound $\|\bar\bomega_{t+1}\|$ by
\begin{equation*}
\|\bar\bw_{t+1}\| \le \gamma_{1,t}^{\dagger}+\gamma_{2,t}^{\dagger}+\gamma_{3,t}^{\dagger}+\gamma_{4,t}^{\dagger},
\end{equation*}
with
\begin{equation*}
\gamma_{1,t}^{\dagger}:=\|(\bI_{nr+qr}-\eta\bar\cA_t^{\dagger})\tilde\bw_t\|, \qquad
\gamma_{2,t}^{\dagger}:=\eta\|\tilde\cG_t^{\dagger}\tilde\bw_t\|, \qquad
\gamma_{3,t}^{\dagger}:=\eta\|\bar\cG_t^{\dagger}\tilde\bw_t\|,
\end{equation*}
\begin{equation*}
\gamma_{4,t}^{\dagger} := \eta\|\cR_t^{\dagger}(\bar\cA_t^{\dagger}+\tilde\cG_t^{\dagger}+\bar\cG_t^{\dagger})\tilde\bw_t\|.
\end{equation*}

The same proof as in Step~1 of Theorem~\ref{thm_general_theory_uv}, using the reference pair $(\tilde\bU,\tilde\bV)$ and \eqref{eq_rect_step4_segment_localization}, gives
\begin{equation} \label{eq_rect_step4_A_dagger_bounds}
\lambda_{\min}(\bar\cA_t^{\dagger})\ge \frac78\alpha\sigma_{\min},
\qquad
\lambda_{\max}(\bar\cA_t^{\dagger})\le 2(\alpha+\beta)\kappa\sigma_{\min}.
\end{equation}
It follows that
\begin{equation} \label{eq_rect_step4_gamma1}
\gamma_{1,t}^{\dagger} \le \Big(1-\frac78\eta\alpha\sigma_{\min}\Big)\|\tilde\bw_t\|.
\end{equation}

Next, similar to bounding $\gamma_{2,t}$ in the proof of Theorem~\ref{thm_general_theory_uv}, Equation~\ref{assump_gradient_noise_rect_1} and the theorem's scaling condition on $\Delta_2$ yield
\begin{equation} \label{eq_rect_step4_gamma2}
\gamma_{2,t}^{\dagger} \le C\eta\Delta_2\|\tilde\bw_t\| \le \frac1{16}\eta\alpha\sigma_{\min}\|\tilde\bw_t\|,
\end{equation} after shrinking $c_0$ if necessary.

For $\gamma_{3,t}^{\dagger}$, use $\bar\cG(\bU^*,\bV^*)=\zero$,
Assumption~\ref{assump_rect_lip} with $\bar\cG(\bU,\bV)$ in place of $\cG(\bU,\bV)$, and \eqref{eq_rect_step4_segment_localization} to get
\begin{equation*}
\|\bar\cG_t^{\dagger}\| \le \sup_{s\in[0,1]} \frac{\|\bar\cG(\bU_t^{\dagger}(s),\bV_t^{\dagger}(s))\|}{\sqrt{nq}} \le \frac{L_2}{\sqrt{nq}} \sup_{s\in[0,1]} \|\bU_t^{\dagger}(s)\bV_t^{\dagger}(s)^\T-\bU^*(\bV^*)^\T\|_{\Fn} \le C\epsilon\sqrt r\,\kappa\sigma_{\min}.
\end{equation*}
The theorem scaling condition implies that $\sqrt r\,\kappa\epsilon/\alpha\le c_0$, and therefore, we have
\begin{equation} \label{eq_rect_step4_gamma3}
\gamma_{3,t}^{\dagger} \le C\eta\epsilon\sqrt r\,\kappa\sigma_{\min}\|\tilde\bw_t\| \le \frac1{16}\eta\alpha\sigma_{\min}\|\tilde\bw_t\|.
\end{equation}

Finally, exactly as in the proof of Theorem~\ref{thm_general_theory_uv}, the induction hypothesis~\ref{item_uv_ind_rot_dagger} implies
\begin{equation} \label{eq_rect_step4_balance_smallness}
\vartheta_t^{\dagger} \le \frac{\alpha}{10\beta\kappa},
\end{equation} after shrinking $c_0$ if necessary. Combining this with
\eqref{eq_rect_step4_A_dagger_bounds}--\eqref{eq_rect_step4_gamma3}, we obtain
\begin{equation} \label{eq_rect_step4_gamma4}
\gamma_{4,t}^{\dagger} \le \eta\vartheta_t^{\dagger} \Big(\lambda_{\max}(\bar\cA_t^{\dagger})\|\tilde\bw_t\|+\|\tilde\cG_t^{\dagger}\tilde\bw_t\|+\|\bar\cG_t^{\dagger}\tilde\bw_t\|\Big) \le \frac58\eta\alpha\sigma_{\min}\|\tilde\bw_t\|.
\end{equation}
Combining \eqref{eq_rect_step4_gamma1}, \eqref{eq_rect_step4_gamma2},
\eqref{eq_rect_step4_gamma3}, and \eqref{eq_rect_step4_gamma4}, we conclude that $\|\bar\bw_{t+1}\|\le (1-\frac14\eta\alpha\sigma_{\min})\|\tilde\bw_t\|$.
With the induction hypothesis, we know 
\begin{equation} \label{eq_rect_step4__l2}
\|(\bar\bU^{t+1}-\tilde\bU,\bar\bV^{t+1}-\tilde\bV)\|_{\Fnw} \le \rho\|(\tilde\bE_U^t,\tilde\bE_V^t)\|_{\Fnw} \le \rho^{t+1}\phi_{nq}^{\dagger}\tau_*.
\end{equation}

\medskip
\noindent{\bf Weighted $\ell_{\infty}$ contraction for iterate $(\bU^t,\bV^t)$.} 
Write $\bar\cH_{ii,t}^{UU,\dagger}$, $\bar\cH_{i\ell,t}^{UV,\dagger}$, and $\bar\cH_{\ell\ell,t}^{VV,\dagger}$ for the averaged Hessian blocks defined exactly as in Step~2 of the proof of Theorem~\ref{thm_general_theory_uv}, but now for the re-centered objective $h_{\alpha,\dagger}^{\natural}$ along the segment $\{(\bU_t^{\dagger}(s),\bV_t^{\dagger}(s)):0\le s\le 1\}$.
The same proof as in Theorem~\ref{thm_general_theory_uv}, using
\eqref{eq_rect_step4_segment_localization}, \eqref{eq_rect_step4_tilde_scale}, and \eqref{eq_rect_step4_tilde_infty_scale}, yields
\begin{equation} \label{eq_rect_step4_delta12}
\delta_{1,t}^{\dagger} \le \Big(1-\frac78\eta\alpha\sigma_{\min}\Big) \|(\tilde\bE_U^t,\tilde\bE_V^t)\|_{\twinfw}, \qquad
\delta_{2,t}^{\dagger} \le C\eta\beta\sqrt{\kappa\sigma_{\min}}\,\omega_* \|(\tilde\bE_U^t,\tilde\bE_V^t)\|_{\Fnw},
\end{equation}
where $\delta_{1,t}^{\dagger}$ and $\delta_{2,t}^{\dagger}$ denote the diagonal and off-diagonal Hessian contributions similar to $\delta_{1,t}$ and $\delta_{2,t}$ in Section~\ref{supp_prove_thm_general_theory_uv}, respectively.

The new terms come from the stochastic and deterministic parts of $\cG$.
Define
\begin{align*}
\delta_{3,t}^{\dagger,U} &:= \frac{\eta}{q} \max_{i\in[n]} \Big\| \Big[ \int_0^1 \tilde\cG(\bU_t^{\dagger}(s),\bV_t^{\dagger}(s))ds\;\tilde\bE_V^t \Big]_{i,\cdot} \Big\|, \\
\delta_{3,t}^{\dagger,V} &:= \frac{\eta}{n} \max_{\ell\in[q]} \Big\| \Big[ \int_0^1 \tilde\cG(\bU_t^{\dagger}(s),\bV_t^{\dagger}(s))^\T ds\;\tilde\bE_U^t \Big]_{\ell,\cdot} \Big\|,
\end{align*}
and let $\delta_{3,t}^{\dagger}:=\delta_{3,t}^{\dagger,U}\vee \delta_{3,t}^{\dagger,V}$.
By~\eqref{assump_gradient_noise_rect_3} and the theorem's scaling condition for $\bar\Delta_{\infty}(n,q,\delta)$, one has
\begin{equation} \label{eq_rect_step4_delta3}
\delta_{3,t}^{\dagger} \le \eta \bar\Delta_{\infty}(n,q,\delta) \|(\tilde\bE_U^t,\tilde\bE_V^t)\|_{\twinfw} \le \frac14\eta\alpha\sigma_{\min} \|(\tilde\bE_U^t,\tilde\bE_V^t)\|_{\twinfw}.
\end{equation}

Next, we further define
\begin{align*}
\delta_{4,t}^{\dagger,U} &:= \frac{\eta}{q} \max_{i\in[n]} \Big\| \Big[ \int_0^1 \bar\cG(\bU_t^{\dagger}(s),\bV_t^{\dagger}(s))ds\;\tilde\bE_V^t \Big]_{i,\cdot} \Big\|, \\
\delta_{4,t}^{\dagger,V} &:= \frac{\eta}{n} \max_{\ell\in[q]} \Big\| \Big[ \int_0^1 \bar\cG(\bU_t^{\dagger}(s),\bV_t^{\dagger}(s))^\T ds\;\tilde\bE_U^t \Big]_{\ell,\cdot} \Big\|,
\end{align*}
and let $\delta_{4,t}^{\dagger}:=\delta_{4,t}^{\dagger,U}\vee \delta_{4,t}^{\dagger,V}$.
By Assumption~\ref{assump_rect_lip} with $\bar\cG(\bU,\bV)$ in place of $\cG(\bU,\bV)$, $\bar\cG(\bU^*,\bV^*)=\zero$, and
\eqref{eq_rect_step4_segment_localization}, we have for every $s\in[0,1]$,
\begin{align*}
q^{-1/2}\|\bU_t^{\dagger}(s)\bV_t^{\dagger}(s)^\T-\bU^*(\bV^*)^\T\|_{\twinf} \le &\, \|\bU^*\|_{2\to\infty}\frac{\|\bV_t^{\dagger}(s)-\bV^*\|_{\Fn}}{\sqrt q} \\&\,+
\Big(\frac{\|\bV^*\|}{\sqrt q}+\frac{\|\bV_t^{\dagger}(s)-\bV^*\|_{\Fn}}{\sqrt q}\Big) \|\bU_t^{\dagger}(s)-\bU^*\|_{2\to\infty} \\
\le &\, \omega_*\,\epsilon\tau_* + C\sqrt{\kappa\sigma_{\min}}\,\epsilon\omega_* \\\le &\,
C\alpha\sqrt{\frac{\sigma_{\min}}{\kappa}}\,\omega_*,
\end{align*}
where the last step uses $\tau_*\le \sqrt{r\kappa\sigma_{\min}}$ and the theorem scaling condition $\sqrt r\,\kappa\epsilon/\alpha\le c_0$.
Therefore, one has
\begin{equation} \label{eq_rect_step4_delta4}
\begin{aligned}\delta_{4,t}^{\dagger} \le &\, C\eta L_{\infty} \Big( q^{-1/2}\sup_{s\in[0,1]} \|\bU_t^{\dagger}(s)\bV_t^{\dagger}(s)^\T-\bU^*(\bV^*)^\T\|_{\twinf} \Big) \|(\tilde\bE_U^t,\tilde\bE_V^t)\|_{\Fnw}
\\\le &\,  C\eta\alpha\sqrt{\frac{\sigma_{\min}}{\kappa}}\,\omega_* \|(\tilde\bE_U^t,\tilde\bE_V^t)\|_{\Fnw}.
\end{aligned}\end{equation}

Finally, let $\delta_{5,t}^{\dagger}$ denote the balancing term coming from $\bLambda_t^{\dagger}-\bI_r$ and $(\bLambda_t^{\dagger})^{-1}-\bI_r$.
Exactly as in the proof of Theorem~\ref{thm_general_theory_uv}, \eqref{eq_rect_step4_A_dagger_bounds}, \eqref{eq_rect_step4_delta12}, \eqref{eq_rect_step4_delta3}, \eqref{eq_rect_step4_delta4}, and \eqref{eq_rect_step4_balance_smallness} imply
\begin{equation} \label{eq_rect_step4_delta5}
\delta_{5,t}^{\dagger} \le \frac14\eta\alpha\sigma_{\min} \|(\tilde\bE_U^t,\tilde\bE_V^t)\|_{\twinfw} + C\eta\alpha\sqrt{\frac{\sigma_{\min}}{\kappa}}\,\omega_* \|(\tilde\bE_U^t,\tilde\bE_V^t)\|_{\Fnw}.
\end{equation}
Combining \eqref{eq_rect_step4_delta12}--\eqref{eq_rect_step4_delta5}, and using $\beta\ge \alpha$ and $\kappa\ge 1$, we obtain
\begin{equation*}
\|(\bar\bU^{t+1}-\tilde\bU,\bar\bV^{t+1}-\tilde\bV)\|_{\twinfw} \le \Big(1-\frac{9}{16}\eta\alpha\sigma_{\min}\Big) \|(\tilde\bE_U^t,\tilde\bE_V^t)\|_{\twinfw} + C\eta\beta\sqrt{\kappa\sigma_{\min}}\,\omega_* \|(\tilde\bE_U^t,\tilde\bE_V^t)\|_{\Fnw}.
\end{equation*}
Since $\tfrac{\beta}{\alpha}\kappa^{3/2}\sqrt{r}\phi_{nq}^{\dagger}/\psi_{nq}^{\dagger}\le c_0$ by \eqref{eq_rect_step4_dagger_scaling}, we know $C\beta\kappa\sqrt r\,\phi_{nq}^{\dagger}\le\tfrac1{16}\alpha\psi_{nq}^{\dagger}$ after shrinking $c_0$. Therefore with induction hypotheses \ref{item_uv_ind_l2_dagger}--\ref{item_uv_ind_linf_dagger}, we arrive at
\begin{equation} \label{eq_rect_step4__linf}
\|(\bar\bU^{t+1}-\tilde\bU,\bar\bV^{t+1}-\tilde\bV)\|_{\twinfw} \le (2\rho-1)\rho^{t}\psi_{nq}^{\dagger}\omega_*.
\end{equation}

\medskip
\noindent{\bf Update the alignment and transfer the bounds.}
Set
\begin{equation*}
\delta_{t+1}^{\dagger} := \max\Big\{ n^{-1/2}\|\bar\bU^{t+1}-\tilde\bU\|_{\Fn}, q^{-1/2}\|\bar\bV^{t+1}-\tilde\bV\|_{\Fn} \Big\}.
\end{equation*}
Equation~\eqref{eq_rect_step4__l2} already shows that $\delta_{t+1}^{\dagger} \le \|(\bar\bU^{t+1}-\tilde\bU,\bar\bV^{t+1}-\tilde\bV)\|_{\Fnw} \le \rho^{t+1}\phi_{nq}^{\dagger}\tau_*$.
Then using \eqref{eq_rect_step4_dagger_scaling}, \eqref{eq_rect_step4_tilde_scale}, and $\tau_*\le \sqrt{r\kappa\sigma_{\min}}$, we obtain $\delta_{t+1}^{\dagger}\le Cc_0\sqrt{\sigma_{\min}} <\big\{\sigma_r(\tilde\bU/\sqrt n)\wedge \sigma_r(\tilde\bV/\sqrt q)\big\}/80$, after shrinking $c_0$ if necessary. 
Moreover, the induction hypothesis~\ref{item_uv_ind_rot_dagger} implies
\begin{equation*}
\|\bG_t^{\dagger}-\bR^0\|\vee \|(\bG_t^{\dagger})^{-\T}-\bR^0\| \le \iota_0\frac{\alpha}{\beta\kappa}<\frac16,
\end{equation*}
so all singular values of $\bG_t^{\dagger}$ lie in $[2/3,3/2]$.
Applying Lemma~\ref{lemma_alignment_near_rotation} to the pair $(\bU^{t+1},\bV^{t+1})$ with $\bm P=\bG_t^{\dagger}$ and reference pair $(\tilde\bU,\tilde\bV)$, we conclude that $\bG_{t+1}^{\dagger}$ exists and
\begin{equation} \label{eq_rect_step4_G_increment}\begin{aligned}
\|\bG_{t+1}^{\dagger}-\bG_t^{\dagger}\| \vee \|(\bG_{t+1}^{\dagger})^{-\T}-(\bG_t^{\dagger})^{-\T}\| \le &\, \frac{5\delta_{t+1}^{\dagger}}{\sigma_r(\tilde\bU/\sqrt n)\wedge \sigma_r(\tilde\bV/\sqrt q)}\\ \le&\, C\rho^{t+1}\phi_{nq}^{\dagger}\frac{\tau_*}{\sqrt{\sigma_{\min}}}\\
\le&\, C\rho^{t+1}\phi_{nq}^{\dagger}\sqrt{r\kappa}.\end{aligned}
\end{equation} 
This proves induction hypothesis~\ref{item_uv_ind_exist_dagger} at time $t+1$.
Similar to \eqref{eq_uv_G_increment_telesc}, telescoping the above increments yields
\begin{equation*}
\|\bG_{t+1}^{\dagger}-\bR^0\|\vee \|(\bG_{t+1}^{\dagger})^{-\T}-\bR^0\| \le C\phi_{nq}^{\dagger}\frac{\sqrt{r\kappa}}{1-\rho}.
\end{equation*}
By the first scaling condition in \eqref{eq_rect_step4_dagger_scaling}, the right-hand side is at most $\iota_0\alpha/(\beta\kappa)$ after shrinking $c_0$ if necessary. Hence induction hypothesis~\ref{item_uv_ind_rot_dagger} also holds at time $t+1$.

Since $\bG_{t+1}^{\dagger}$ minimizes the weighted Frobenius distance to $(\tilde\bU,\tilde\bV)$, we have
\begin{equation*}
\|(\widetilde{\bU}^{t+1}-\tilde\bU,\widetilde{\bV}^{t+1}-\tilde\bV)\|_{\Fnw} \le \|(\bar\bU^{t+1}-\tilde\bU,\bar\bV^{t+1}-\tilde\bV)\|_{\Fnw}.
\end{equation*}
Combining this with \eqref{eq_rect_step4__l2} proves the induction hypothesis~\ref{item_uv_ind_l2_dagger} at time $t+1$.

To transfer the weighted $\ell_{\infty}$ bound, define $\bQ_{t+1}^{\dagger}:=(\bG_t^{\dagger})^{-1}\bG_{t+1}^{\dagger}$. Then we know $\widetilde{\bU}^{t+1}=\bar\bU^{t+1}\bQ_{t+1}^{\dagger}$ and $\widetilde{\bV}^{t+1}=\bar\bV^{t+1}(\bQ_{t+1}^{\dagger})^{-\T}$.
Therefore, one has
\begin{align*}
\|(\widetilde{\bU}^{t+1}-\tilde\bU,&\widetilde{\bV}^{t+1}-\tilde\bV)\|_{\twinfw} \le \|(\bar\bU^{t+1}-\tilde\bU,\bar\bV^{t+1}-\tilde\bV)\|_{\twinfw} \\ &\quad+ \Big(\|\bar\bU^{t+1}\|_{2\to\infty}\vee \|\bar\bV^{t+1}\|_{2\to\infty}\Big) \Big(\|\bQ_{t+1}^{\dagger}-\bI_r\|\vee \|(\bQ_{t+1}^{\dagger})^{-\T}-\bI_r\|\Big).
\end{align*}
Further with \eqref{eq_rect_step4__linf} and \eqref{eq_rect_step4_tilde_infty_scale}, we have
\begin{equation*}
\|\bar\bU^{t+1}\|_{2\to\infty}\vee \|\bar\bV^{t+1}\|_{2\to\infty} \le \|\tilde\bU\|_{2\to\infty}\vee \|\tilde\bV\|_{2\to\infty} + \|(\bar\bU^{t+1}-\tilde\bU,\bar\bV^{t+1}-\tilde\bV)\|_{\twinfw} \le 3\omega_*,
\end{equation*}after shrinking $c_0$ if necessary.
Also, all singular values of $\bG_t^{\dagger}$ lie in $[2/3,3/2]$, which yields with Lemma~\ref{lemma_alignment_near_rotation} that
\begin{equation*}
\|\bQ_{t+1}^{\dagger}-\bI_r\| \vee \|(\bQ_{t+1}^{\dagger})^{-\T}-\bI_r\| \le C\Big( \|\bG_{t+1}^{\dagger}-\bG_t^{\dagger}\| \vee \|(\bG_{t+1}^{\dagger})^{-\T}-(\bG_t^{\dagger})^{-\T}\| \Big)
\le C\rho^{t+1}\phi_{nq}^{\dagger}\sqrt{r\kappa},
\end{equation*} where the last step uses \eqref{eq_rect_step4_G_increment}.
Substituting the above bound and \eqref{eq_rect_step4__linf} yields
\begin{equation*}
\|(\widetilde{\bU}^{t+1}-\tilde\bU,\widetilde{\bV}^{t+1}-\tilde\bV)\|_{\twinfw} \le (2\rho-1)\rho^{t}\psi_{nq}^{\dagger}\omega_* + C\rho^{t+1}\phi_{nq}^{\dagger}\sqrt{r\kappa}\,\omega_*.
\end{equation*}
Since $\tfrac{\beta}{\alpha}\kappa^{3/2}\sqrt{r}\phi_{nq}^{\dagger}/\psi_{nq}^{\dagger} \le c_0$ and $\eta =\{(10(\alpha+\beta)\kappa\sigma_{\min}\}^{-1}$, we know  
\begin{equation*}
    \rho^{t+1}\phi_{nq}^{\dagger}\sqrt{r\kappa}\le \frac{1}{4}\eta\alpha\sigma_{\min}\rho^t\psi_{nq}^{\dagger} = (1-\rho)\rho^t\psi_{nq}^{\dagger},
\end{equation*}after shrinking $c_0$ if necessary. This eventually leads to
\begin{equation*}
\|(\widetilde{\bU}^{t+1}-\tilde\bU,\widetilde{\bV}^{t+1}-\tilde\bV)\|_{\twinfw}\le \rho^{t+1}\psi_{nq}^{\dagger}\omega_*,
\end{equation*}
which proves the induction hypothesis~\ref{item_uv_ind_linf_dagger} at time $t+1$ and closes the induction.

To conclude, we have shown that, for all $t\ge 0$,
\begin{equation}
\label{eq_rect_recentered_contraction}
\disttwo\{(\bU^t,\bV^t),(\tilde\bU,\tilde\bV)\} \le \rho^t\phi_{nq}^{\dagger}\tau_*, \qquad \distinf\{(\bU^t,\bV^t),(\tilde\bU,\tilde\bV)\} \le \rho^t\psi_{nq}^{\dagger}\omega_*.
\end{equation}
Hence, combining \eqref{eq_rect_recentered_contraction} with the optimizer bounds \eqref{eq_rect_step2_l2_final}, \eqref{eq_rect_step2_linf_final}, and the definition \eqref{eq_rect_step4_radii}, we prove the iterate bounds in Theorem~\ref{thm_general_theory_UV_noise}.

\subsection{Proof sketch for the \texorpdfstring{$\ell_2$}{l2} part of Theorems~\ref{thm_general_theory_uv} and~\ref{thm_general_theory_UV_noise}}\label{supp_sec_prove_thm_general_theory_noisy_uv_l2}
We outline the common $\ell_2$ argument under the larger local region $\cD=\cDuvtwo$. Here, similar to the proof of Theorem~\ref{thm_general_theory_UV_noise}, we take $\epsilon = c_0\alpha/(\kappa\sqrt{r})$ without loss of generality.
In the noisy case, we work on the event where \eqref{assump_gradient_noise_rect_1} holds with $\cD=\cDuvtwo$, and assume Assumptions~\ref{assump_rect_rip_weighted} and~\ref{assump_rect_lip_1} hold with $\cD=\cDuvtwo$ and $\bar\cG(\bU,\bV)$ in place of $\cG(\bU,\bV)$. The noiseless case is recovered by setting $\tilde\cG\equiv 0$. Set
\begin{equation*}
\varepsilon:=M\frac{\Delta_2(n,q,\delta)}{\alpha\sigma_{\min}}, \qquad \bar\cD^{(2)}_{uv}(\varepsilon) := \Big\{ (\bU,\bV): \|(\bU-\bU^*,\bV-\bV^*)\|_{\Fnw}\le \varepsilon\tau_* \Big\},
\end{equation*}
where $M>0$ is a sufficiently large universal constant. After shrinking $c_0$ if necessary, the smallness assumption on $\Delta_2(n,q,\delta)/(\alpha\sigma_{\min})$ ensures $\bar\cD^{(2)}_{uv}(\varepsilon)\subseteq \cDuvtwo$.

As in Step~1 of the proof of Theorem~\ref{thm_general_theory_UV_noise}, define
\begin{equation*}
(\tilde\bU,\tilde\bV)\in \argmin_{(\bU,\bV)\in\bar\cD^{(2)}_{uv}(\varepsilon)}\Big\{ q^{-1}\|\nabla_{\bU}h_{\alpha}^{\natural}(\bU,\bV)\|_{\Fn}^2 + n^{-1}\|\nabla_{\bV}h_{\alpha}^{\natural}(\bU,\bV)\|_{\Fn}^2 \Big\},
\end{equation*}
where $h_{\alpha}^{\natural}(\bU,\bV):=\cL(\bU\bV^\T)+p_{\alpha}^{\natural}(\bU,\bV)$. Since $\bar\cG(\bU^*,\bV^*)=\zero$ and $\nabla p_{\alpha}^{\natural}(\bU^*,\bV^*)=\zero$, we know by definition that
\begin{align*}
&\,\Big\{q^{-1}\|\nabla_{\bU}h_{\alpha}^{\natural}(\tilde\bU,\tilde\bV)\|_{\Fn}^2 + n^{-1}\|\nabla_{\bV}h_{\alpha}^{\natural}(\tilde\bU,\tilde\bV)\|_{\Fn}^2\Big\}^{1/2}\\
\le &\,\Big\{q^{-1}\|\nabla_{\bU}h_{\alpha}^{\natural}(\bU^*,\bV^*)\|_{\Fn}^2 + n^{-1}\|\nabla_{\bV}h_{\alpha}^{\natural}(\bU^*,\bV^*)\|_{\Fn}^2\Big\}^{1/2}\le C\sqrt{nq}\,\Delta_2(n,q,\delta)\tau_*.
\end{align*}
Repeating the weighted Frobenius-norm mean-value expansion from the optimizer analysis of Theorem~\ref{thm_general_theory_UV_noise}, but dropping all row-wise terms, yields
\begin{equation*}
\begin{aligned}
\|(\tilde\bU-\bU^*,\tilde\bV-\bV^*)\|_{\Fnw} \le&\; C\frac{\Delta_2(n,q,\delta)}{\alpha\sigma_{\min}}\tau_* + C\frac{\Delta_2(n,q,\delta)}{\alpha\sigma_{\min}}\|(\tilde\bU-\bU^*,\tilde\bV-\bV^*)\|_{\Fnw} \\
&\;+ C\frac{\sqrt{\kappa\sigma_{\min}}}{\alpha\sigma_{\min}} \|(\tilde\bU-\bU^*,\tilde\bV-\bV^*)\|_{\Fnw}^2.
\end{aligned}
\end{equation*}
The theorem scaling allows us absorb the last two terms into the left-hand side, so
\begin{equation*}
\|(\tilde\bU-\bU^*,\tilde\bV-\bV^*)\|_{\Fnw} \le
C\frac{\Delta_2(n,q,\delta)}{\alpha\sigma_{\min}}\tau_* \le \frac{\varepsilon}{2}\tau_*.
\end{equation*}
The same argument from Lemma~\ref{lemma_general_convex_asym} then shows that $h_{\alpha}^{\natural}$ is strongly convex on the aligned $\epsilon$-neighborhood, so $(\tilde\bU,\tilde\bV)$ is an interior stationary point and hence the unique minimizer of $h_{\alpha}^{\natural}$ on that set. The same identification argument as in step~2 in the proof of Theorem~\ref{thm_general_theory_UV_noise} in Section~\ref{supp_sec_prove_thm_general_theory_UV_noise}, when restricted to the weighted Frobenius neighborhood, gives
\begin{equation*}
p_{\alpha}^{\natural}(\tilde\bU,\tilde\bV)=0, \qquad \nabla_{\bU}\cL(\tilde\bU\tilde\bV^\T)=\zero, \qquad \nabla_{\bV}\cL(\tilde\bU\tilde\bV^\T)=\zero,
\end{equation*}
and $(\tilde\bU,\tilde\bV)=(\hat\bU\hat\bG,\hat\bV\hat\bG^{-\T})$. This proves the bound for the estimator.

Now we prove the result for the iterates. Let
\begin{equation*}
\phi_{nq}^{\dagger} := \phi_{nq} + C\frac{\Delta_2(n,q,\delta)}{\alpha\sigma_{\min}},
\end{equation*}
and fix a sufficiently small constant $\iota_0\in(0,1/30)$. Similar to the weighted $\ell_2$ contraction argument in Step~3 of the proof of Theorem~\ref{thm_general_theory_UV_noise}, one can prove by induction on $t$ that, for all $0\le s\le t$,
\begin{enumerate}[label=$(\roman*)$]
    \item \label{item_l2_uv_induc1}the optimal alignment matrix $\bG_s^\dagger\in\argmin_{\bG\in GL(r)} \|(\bU^s\bG-\tilde\bU,\bV^s\bG^{-\T}-\tilde\bV)\|_{\Fnw}$ exists;
    \item \label{item_l2_uv_induc2}for $\tilde\bE_s:= (\bU^s\bG_s^\dagger-\tilde\bU,\, \bV^s(\bG_s^\dagger)^{-\T}-\tilde\bV)$, then $\|\tilde\bE_s\|_{\Fnw}\le \rho^s\phi_{nq}^{\dagger}\tau_*$;
    \item \label{item_l2_uv_induc3}$\|\bG_s^\dagger-\bR^0\|\vee \|(\bG_s^\dagger)^{-\T}-\bR^0\|\le \iota_0\alpha/(\beta\kappa)$.
\end{enumerate}
The base case follows from the initialization bound together with the estimate for $\|(\tilde\bU-\bU^*,\tilde\bV-\bV^*)\|_{\Fnw}$, exactly as in Step~3 of Theorem~\ref{thm_general_theory_UV_noise}. Now assume \ref{item_l2_uv_induc1}--\ref{item_l2_uv_induc3} hold up to time $t$, and define the  iterate $ \bar\bU^{t+1}:=\bU^{t+1}\bG_t^\dagger$, $\bar\bV^{t+1}:=\bV^{t+1}(\bG_t^\dagger)^{-\T}$.
The induction hypotheses imply that the line segment between $(\tilde\bU,\tilde\bV)$ and $(\bU^t\bG_t^\dagger,\bV^t(\bG_t^\dagger)^{-\T})$ stays inside $\cDuvtwo$. Therefore, the same one-step decomposition as in Step~3 gives
\begin{equation*}
\|(\bar\bU^{t+1}-\tilde\bU,\bar\bV^{t+1}-\tilde\bV)\|_{\Fnw} \le \gamma_{1,t}^\dagger+\gamma_{2,t}^\dagger+\gamma_{3,t}^\dagger+\gamma_{4,t}^\dagger,
\end{equation*}
with
\begin{equation*}
\gamma_{1,t}^\dagger \le
\Bigl(1-\frac78\eta\alpha\sigma_{\min}\Bigr)\|\tilde\bE_t\|_{\Fnw},
\qquad
\gamma_{2,t}^\dagger+\gamma_{3,t}^\dagger+\gamma_{4,t}^\dagger
\le
\frac58\eta\alpha\sigma_{\min}\|\tilde\bE_t\|_{\Fnw}.
\end{equation*}
Hence
\begin{equation*}
\|(\bar\bU^{t+1}-\tilde\bU,\bar\bV^{t+1}-\tilde\bV)\|_{\Fnw} \le \rho\|\tilde\bE_t\|_{\Fnw} \le \rho^{t+1}\phi_{nq}^{\dagger}\tau_*.
\end{equation*}
Since the right-hand side is again of order $o(\sqrt{\sigma_{\min}})$ under the same scaling assumptions, the same alignment perturbation lemma as in Step~3 shows that $\bG_{t+1}^\dagger$ exists and the same strategy of telescoping the increment yields
\begin{equation*}
\|\bG_{t+1}^\dagger-\bR^0\|\vee \|(\bG_{t+1}^\dagger)^{-\T}-\bR^0\| \le \iota_0\frac{\alpha}{\beta\kappa}.
\end{equation*}
This proves \ref{item_l2_uv_induc1} and \ref{item_l2_uv_induc3}. Finally, by the optimality of $\bG_{t+1}^\dagger$,
\begin{equation*}
\|\tilde\bE_{t+1}\|_{\Fnw} \le \|(\bar\bU^{t+1}-\tilde\bU,\bar\bV^{t+1}-\tilde\bV)\|_{\Fnw} \le \rho^{t+1}\phi_{nq}^{\dagger}\tau_*.
\end{equation*}
This closes the induction. Combining the bound on $\|\tilde\bE_t\|_{\Fnw}$ with the estimate for $\|(\tilde\bU-\bU^*,\tilde\bV-\bV^*)\|_{\Fnw}$ yields the stated contraction toward $(\bU^*,\bV^*)$.

\section{Proof of Examples}\label{sec_prove_examples}
Throughout this section, we use the notation defined in Section~\ref{Sec_prelmi}. We present a proof of \eqref{supp_sec_prove_example0} in Section~\ref{supp_sec_prove_example0}. The proof of Theorems~\ref{prop_ms_verify} and~\ref{Thm_example2} are then provided in Sections~\ref{supp_sec_prove_example_1} and~\ref{supp_sec_prove_example2}, respectively.

\subsection{Proof of Bounds for Linear  Model in~\eqref{example0}}\label{supp_sec_prove_example0}
As a warm-up, we verify \eqref{example0}.  For the quadratic loss, we have $\nabla_{\bX}\cL(\bX) = n^{-1}(\bX-\bX^*+\bE)$ and $\tilde\cG(\bZ) = \nabla_{\bX}\cL(\bZ\bZ^\T) - \EE\nabla_{\bX}\cL(\bZ\bZ^\T) = n^{-1}\bE$.
Note that $\tilde\cG(\bZ)$ is independent of $\bZ$, so all suprema over $\bZ\in\cD$ are automatic.
For operator norm of $\bE$, by the standard spectral norm bound for an $n\times n$ matrix with independent mean-zero sub-Gaussian entries~\citep{bandeira2016sharp}, there is
\begin{equation*}
    \|\bE\|\lesssim \sigma\{\sqrt n+\sqrt{\log(1/\delta)}\}
\end{equation*}
with probability at least \(1-\delta\). Therefore, for \eqref{assump_gradient_noise_1} to hold, we let \begin{equation*}
    \Delta_2(n,\delta)\asymp\sigma\tfrac{\sqrt{n + \log(1/\delta)}}{n}.
\end{equation*}
Next, we control $\Delta_{\infty}(n,\delta)$. For each row $i\in[n]$, let $\be_i^\T$ denote the $i$th row of $\bE$. Then $\|\bE\bZ^*\|_{\twinf} = \max_{1\le i\le n}\|\be_i^\T\bZ^*\|_2$. For a fixed $i$, any unit vector $\bu\in\mathbb S^{r-1}$, we know $\be_i^\T\bZ^*\bu$ is a mean-zero sub-Gaussian random variable with sub-Gaussian norm bounded by $\|\be_i^\T\bZ^*\bu\|_{\psi_2}\lesssim \sigma\|\bZ^*\bu\|_2 \le \sigma\|\bZ^*\|$. A standard $\varepsilon$-net argument on $\mathbb S^{r-1}$ then gives, for every $u>0$,
\begin{equation*}
    \PP\left( \|\be_i^\T\bZ^*\|_2\ge C\sigma\|\bZ^*\|\sqrt{r+u}\right)\le e^{-u}.\end{equation*}
Taking \(t=\log(n/\delta)\) and applying a union bound over $i=1,\ldots,n$, we obtain
\begin{equation*}
    \|\bE\bZ^*\|_{\twinf} \lesssim \sigma\|\bZ^*\|\sqrt{r+\log(n/\delta)},\end{equation*}
with probability at least $1-\delta$. Hence, for \eqref{assump_gradient_noise_2} to hold, we let 
\begin{equation*}
    \Delta_{\infty}(n,\delta)\asymp\sigma\sqrt{\frac{r+\log(n/\delta)}{n}}.\end{equation*}
Finally, we bound the maximal row-wise $\ell_1$ norm. By definition,  $\|\bE\|_{\infty\to 1} = \max_{1\le i\le n}\sum_{j=1}^{n}|E_{ij}|$.
For each $i,j\in[n]$, since $E_{ij}$ is sub-Gaussian with $\|E_{ij}\|_{\psi_2}\le\sigma$, the folded variable $|E_{ij}|-\EE|E_{ij}|$ is sub-exponential with sub-exponential norm bounded by $C\sigma$, and satisfied $\EE|E_{ij}|\le C\sigma$.
Therefore, Bernstein's inequality gives, for each fixed row \(i\) and every $t>0$, we know $ \sum_{j=1}^{n}|E_{ij}| \le C\sigma n + C\sigma\{\sqrt{nt}+t\}$
with probability at least $1-e^{-t}$. Taking again $t=\log(n/\delta)$ and applying a union bound over $i=1,\dots,n$, we get
\begin{equation*}
    \|\bE\|_{\infty\to 1}\lesssim \sigma n + \sigma\sqrt{n\log(n/\delta)} + \sigma\log(n/\delta)
\end{equation*}
with probability at least \(1-\delta\). Finally, for \eqref{assump_gradient_noise_2} to hold, we let 
\[
    \bar\Delta_{\infty}(n,\delta)\asymp   \sigma+\sigma\sqrt{\frac{\log(n/\delta)}{n}}
        +\frac{\log(n/\delta)}{n}.
\]
When \(\log(n/\delta)\lesssim n\), the final term is dominated by the square-root term, and the right-hand side is of order \(\sigma\). This proves the claimed scale of \(\bar\Delta_\infty(n,\delta)\).

\subsection{Proof of Theorem~\ref{prop_ms_verify}}\label{supp_sec_prove_example_1}

For the sensing model $y_i = \langle \bA_i,\bX^*\rangle+\xi_i$, we treat the measurement matrix $\bA_i$ as fixed. When they are random in nature, our analysis can be viewed as conditioning on them. Then for the empirical loss $\cL(\bM)=\frac{1}{2m}\|\by-\cA(\bX)\|^2$, we define the population version as $\bar \cL(\bX)=\frac1{2m}\|\cA(\bX) - \cA(\bX^*)\|^2$ where $\bX^* = \bU^*(\bV^*)^\T$. Therefore, for any $(\bU,\bV)$, $\bar\cG(\bU, \bV) = \cA^*\cA(\bU\bV^\T - \bX^*)/m$ where $\cA^*\cA(\bH) = \sum_{i=1}^m\langle\cA_i,\bH\rangle\bA_i\in\RR^{n\times q}$ and $\tilde\cG(\bU,\bV) = -m^{-1}\sum_{i=1}^m\xi_i\bA_i$ with $\cA^*(\cdot)$ being the self-adjoint operator of $\cA(\cdot)$.

\underline{\it Verification of Assumption~\ref{assump_rect_rip_weighted}.} Because $\nabla_{\bX}^2\cL(\bX)[\bH_1,\bH_2] = m^{-1}\langle \cA(\bH_1),\cA(\bH_2)\rangle$, Assumption~\ref{assump_rect_rip_weighted} is readily verified by \eqref{eq_rip_sensing} and $\|\bH\|_{\Fn}=1$, with $\alpha = 1-\delta_0$ $\beta = 1+ \delta_0$, and $\epsilon = \infty$ ($\cD = \RR^{n\times r}\times \RR^{q\times r}$).

\underline{\it Verification of Assumption~\ref{assump_rect_lip_1}.} For any $\bH:=\bU_1\bV_1^\T-\bU_2\bV_2^\T$ with $\bU_1,\bU_2\in\RR^{n\times r}$ and $\bV_1,\bV_2\in\RR^{q\times r}$, because $\mathrm{rank}(\bH)\le 2r$, we know 
\begin{align*}
\|\bar{\cG}(\bU_1,\bV_1)-\bar{\cG}(\bU_2,\bV_2)\| &=\frac{1}{m}\|\cA^*\cA(\bH)\| \\ &=\sup_{\|\ba\|_2=\|\bb\|_2=1}\frac{1}{m}\big|\ba^\T \cA^*\cA(\bH)\bb\big| \\ &=\sup_{\|\ba\|_2=\|\bb\|_2=1}\frac{1}{m}\big|\langle \ba\bb^\T,\cA^*\cA(\bH)\rangle\big| \\ &=\sup_{\|\ba\|_2=\|\bb\|_2=1}\frac{1}{m}\big|\langle \cA(\ba\bb^\T),\cA(\bH)\rangle\big| \\ 
&\le \sup_{\|\ba\|_2=\|\bb\|_2=1}\frac{1}{m}\|\cA(\ba\bb^\T)\|_2\,\|\cA(\bH)\|_2.
\end{align*}Here, the first equality follows from that $\cA^*\cA(\cdot)$ is linear operator. Now $\mathrm{rank}(\ba\bb^\T)=1\le 2r$ and $\|\ba\bb^\T\|_{\Fn}=1$, so \eqref{eq_rip_sensing} yields $\|\cA(\ba\bb^\T)\|\le \sqrt{1+\delta_0}$. Similarly, because $\mathrm{rank}(\bH)\le 2r$, $\|\cA(\bH)\|\le \sqrt{1+\delta_0}\,\|\bH\|_{\Fn}$. Combining these, we conclude that
Recalling the definition of $\bH$, we obtain
\[
\|\bar{\cG}(\bU_1,\bV_1)-\bar{\cG}(\bU_2,\bV_2)\| \le (1+\delta_0)\|\bU_1\bV_1^\T-\bU_2\bV_2^\T\|_{\Fn}\le 2\|\bU_1\bV_1^\T-\bU_2\bV_2^\T\|_{\Fn}.
\]
That is, Assumption~\ref{assump_rect_lip_1} holds for $L_2 = 2$ and $\epsilon = \infty$ ($\cD = \RR^{n\times r}\times \RR^{q\times r}$.

\underline{\it Verification of noise condition for $\Delta_2(n,q,\delta)$ in~\eqref{assump_gradient_noise_rect_1}.} Matrix Bernstein inequality (Theorem 4.1.1 in \citet{tropp2015introduction}) for sums of independent mean-zero sub-exponential rectangular matrices gives an event $\mathcal E$ with $\PP(\mathcal E)\ge 1-\delta$ such that on $\mathcal E$,
\begin{equation*}
(nq)^{-1/2}\|\tilde\cG(\bU,\bV)\| = \frac{1}{m}(nq)^{-1/2}\|\sum_{i=1}^m\xi_i\bA_i\| \le C\sigma_{\xi}  \sqrt{\frac{\nu_A\log((n+q)/\delta)}{m^2nq}} 
\end{equation*}where $\nu_A := \max\big\{\|\sum_{i=1}^m\bA_i\bA_i^\T\|,\|\sum_{i=1}^m\bA_i^\T\bA_i\|\big\}$. With \eqref{eq_rip_sensing}, one can check that $\nu_A \le m(1+\delta_0)(n\vee q)$. Consequently, for \eqref{assump_gradient_noise_rect_1}, one may take
\begin{equation*}
\Delta_{2}(n,q,\delta)\asymp \sigma_{\xi}  \sqrt{\frac{(1+\delta_0)\log((n+q)/\delta)}{m(n\wedge q)}} . 
\end{equation*}
The theorem's scaling condition implies that $\Delta_2(n,q,\delta)\le c_0(1-\delta_0)^2\sigma_{\min}/(\kappa\sqrt{r})$ as required by Theorem~\ref{thm_general_theory_UV_noise}.
Thus, we have verified every condition required for the $\ell_2$ error contraction part of Theorem~\ref{thm_general_theory_UV_noise}. This proves Theorem~\ref{prop_ms_verify}.

\subsection{Proof of Theorem~\ref{Thm_example2}}\label{supp_sec_prove_example2}

In this proof, we first verify the conditions required by Theorem~\ref{thm_general_theory_UV_noise} except that on $\bar\Delta_{\infty}(n,q,\delta)$. Next, we show how to establish the result in Theorem~\ref{thm_general_theory_UV_noise} under the Bernoulli low-rank model with a leave-one-out argument that remedies that $\bar\Delta_{\infty}(n,q,\delta)$ cannot be properly bounded. Define the scaled logistic loss $\ell_{\alpha_0}(x;y):=\nu_{\star}\{\log(1+\exp(\alpha_0+x))-y(\alpha_0+x)\}$. Recall $\sigma(x) = \exp(\alpha_0+x)/\{\exp(\alpha_0+x)+1\}$ and then we note that the derivatives of $\ell_{\alpha_0}$ can be written as
\begin{equation*}
    s(x) = \partial_x \ell_{\alpha_0}(x;y) = \nu_{\star}\{\sigma(x) - y\},\qquad w(x) = \partial_x^2\ell_{\alpha_0}(x;y)=\nu_{\star}\sigma(x)\{1-\sigma(x)\}
\end{equation*}
Recall that $\nu_{\star} = e^{-(\alpha_0+M_2)}$. Obviously, $P_{ij}^*\le e^{\alpha_0+M_2}\le \nu_{\star}^{-1}$ uniformly across $i\in[n]$ and $j\in[q]$.

\paragraph{\it \underline{ Verifying of conditions of Theorem~\ref{thm_general_theory_UV_noise}.}} We first prove the following assertions.
\begin{enumerate}[label = (\alph*), leftmargin=0.6cm]
\item\label{item_bern_1}
Recall that $\alpha=\frac14 e^{-M_{\star}}$, $\beta=1$. For any $(\bU,\bV)\in\cDuvinf$ and any $\bH$ of the form $\cP_{(\bU,\bV)}(\bR,\bL)=\bU\bR^\T+\bL\bV^\T$, we have
\begin{equation*}
\alpha\|\bH\|_{\Fn}^2 \le \nabla_{\bX}^2\cL(\bU\bV^\T)[\bH,\bH] \le \beta\|\bH\|_{\Fn}^2,
\end{equation*} which verifies Assumption~\ref{assump_rect_rip_weighted} under the Bernoulli model.

\item\label{item_bern_2}
Since the Hessian is diagonal entrywise, Assumption~\ref{assump_row_cross_curvature_uv} holds with the same $\beta$.

\item\label{item_bern_3} We have $\cG(\bU,\bV) = \nabla_{\bX}\cL(\bU\bV^\T) = \nu_{\star}\{\sigma\{(\bU\bV^\T)_{ij}\} - Y_{ij}\}_{n\times q}$ For $\bar\cG(\bU,\bV) = \EE\cG(\bU,\bV) = \nu_{\star}\{\sigma(\bU\bV^\T) - \sigma(\bU^*\bV^*{}^\T)\}$, we have
\begin{align*}
&\|\bar\cG(\bU_1,\bV_1)-\bar\cG(\bU_2,\bV_2)\| \le \|\bU_1\bV_1^\T-\bU_2\bV_2^\T\|_{\Fn},\\
&\|\bar\cG(\bU_1,\bV_1)- \bar\cG(\bU_2,\bV_2)\|_{\twinf} \le \|\bU_1\bV_1^\T-\bU_2\bV_2^\T\|_{\twinf},\\ &\| \bar\cG(\bU_1,\bV_1)^\T - \bar\cG(\bU_2,\bV_2)^\T\|_{\twinf} \le \|\bV_1\bU_1^\T-\bV_2\bU_2^\T\|_{\twinf},
\end{align*}
which verifies Assumption~\ref{assump_rect_lip} with $\bar\cG(\bU,\bV)$ under the place of $\cG(\bU,\bV)$ under the Bernoulli model.

\item\label{item_bern_4}
Writing $\bP^*= \sigma(\bX^*)$, we have
\begin{equation*}
\tilde\cG(\bU,\bV) = \cG(\bU,\bV)-\bar\cG(\bU,\bV) = \nu_{\star}(\bP^* - \bY),
\end{equation*}
which is independent of $(\bU,\bV)$. Therefore, $\Delta_{2}(n,q,\delta)$, $\Delta_{\infty}(n,q,\delta)$, and $\bar\Delta_{\infty}(n,q,\delta)$ in \eqref{assump_gradient_noise_rect} does not depend on the region $\cD$. For \eqref{assump_gradient_noise_rect_1}, we can take
\begin{align*}
\Delta_2(n,q,\delta)\asymp \nu_{\star}\sqrt{\frac{e^{ \alpha_0 + M_2}}{n\wedge q}+\frac{\log((n+q)/\delta)}{nq}}.
\end{align*}

\item\label{item_bern_5}
With $R_{\star}>1$, for \eqref{assump_gradient_noise_rect_2}, we can take $\Delta_\infty(n,q,\delta)\asymp \sqrt{\nu_{\star}L_{\star}R_{\star}/(n\wedge q)}$.

\item\label{item_bern_6}
Let $\bar v_{\infty} = \max\left\{ \max_{i\in[n]}q^{-1}\sum_{j=1}^q 2P_{ij}^*(1-P_{ij}^*), \max_{j\in[q]}n^{-1}\sum_{i=1}^n 2P_{ij}^*(1-P_{ij}^*) \right\} \le \frac12\wedge 2e^{M_2}$. For \eqref{assump_gradient_noise_rect_3}, we can only take
\begin{align*}
\bar\Delta_\infty(n,q,\delta)\asymp  \nu_{\star}\bar v_\infty + C\nu_{\star}\left\{ \sqrt{\frac{\log(n/\delta)}{q}} \vee \sqrt{\frac{\log(q/\delta)}{n}} \right\}.
\end{align*}
Subsequently, we have 
\begin{equation*}
    \frac{\bar\Delta_{\infty}(n,q,\delta)}{\alpha\sigma_{\min}} \le \frac{4e^{M_1 + M_2+\alpha_0}}{\sigma_{\min}}
\end{equation*}
\end{enumerate}

For \ref{item_bern_1}, we note that for any $(\bR,\bL)\in\RR^{q\times r}\times\RR^{n\times r}$, we have $$\nabla_{\bX}^2\cL(\bU\bV^\T)[\bH,\bH] = e^{-\alpha_0-M_2}\sum_{i=1}^n\sum_{j=1}^q \sigma\big((\bU\bV^\T)_{ij}\big)\big\{1-\sigma\big((\bU\bV^\T)_{ij}\big)\big\}\,H_{ij}^2.$$
Note that $\inf_{x\in[-M_1,M_2]}\sigma(x)\{1-\sigma(x)\} = \frac{e^{\alpha_0-M_1}}{(1+e^{\alpha_0-M_1})^2}\ge \frac12 e^{\alpha_0-M_1}$ and $\sup_{x\in[-M_1,M_2]}\sigma(x)\{1-\sigma(x)\} \le  \nu_{\star}^{-1}$, which proves~\ref{item_bern_1}.

For \ref{item_bern_2}, fix $(\bU,\bV)\in\cDuvinf$, $i\in[n]$, $\bh\in\RR^r$, and let $\bL_i=\be_i\bh^\T$. Then $\cP_{(\bU,\bV)}(\zero,\bL_i)=\bL_i\bV^\T=\be_i(\bV\bh)^\T$, which is supported only on row $i$.
Now let $(\bL,\bR)\in\RR^{n\times r}\times\RR^{q\times r}$ with $\bL_{-i}$ having zero $i$th row, and define $\bH_1:=\cP_{(\bU,\bV)}(\zero,\bL_i)=\be_i(\bV\bh)^\T$, $\bH_2:=\cP_{(\bU,\bV)}(\bR,\bL_{-i})=\bU\bR^\T+\bL_{-i}\bV^\T$.
Since $\bL_{-i}$ has zero $i$th row, the $i$th row of $\bH_2$ is $(\bH_2)_{i,\cdot}=\bU_{i,\cdot}\bR^\T$. One can then verify
\begin{equation*}
\nabla_{\bX}^2\cL(\bU\bV^\T)[\bH_1,\bH_2] = \sum_{j=1}^q w\big((\bU\bV^\T)_{ij}\big)\,(\bV\bh)_j\,(\bU_{i,\cdot}\bR^\T)_j.
\end{equation*}
By~\ref{item_bern_1}, we know $w(\cdot)\le 1$. Therefore, by Cauchy--Schwarz, $\big| \nabla_{\bX}^2\cL(\bU\bV^\T)[\bH_1,\bH_2] \big| \le \|\bV\bh\|\,\|\bU_{i,\cdot}\bR^\T\|$.
This proves the first inequality in Assumption~\ref{assump_row_cross_curvature_uv}.
The second part is identical by symmetry.

For~\ref{item_bern_3}, let $\bX_1:=\bU_1\bV_1^\T$ and $\bX_2:=\bU_2\bV_2^\T$.
Since $w(x)\le 1$ in the considered local region, the mean value theorem gives
\begin{equation*}
\sigma\big((\bX_1)_{ij}\big)-\sigma\big((\bX_2)_{ij}\big) = \sigma'(\xi_{ij})\big\{(\bX_1)_{ij}-(\bX_2)_{ij}\big\}\le e^{\alpha_0+M_2}\big\{(\bX_1)_{ij}-(\bX_2)_{ij}\big\}
\end{equation*}
Then
\begin{equation*}\|\sigma(\bX_1) - \sigma(\bX_2)\|_{\Fn}^2 \le e^{2\alpha_0+2M_2}\sum_{i=1}^n\sum_{j=1}^q \big\{(\bX_1-\bX_2)_{ij}\big\}^2 = e^{2\alpha_0+2M_2}\|\bX_1-\bX_2\|_{\Fn}^2.
\end{equation*}
Therefore, $\|\bar\cG(\bU_1,\bV_1)-\bar\cG(\bU_2,\bV_2)\|\le \|\bar\cG(\bU_1,\bV_1)-\bar\cG(\bU_2,\bV_2)\|_{\Fn}\le \|\bX_1-\bX_2\|_{\Fn}$. Similarly, one can verify the row-wise Lipschitz conditions.

For~\ref{item_bern_4}, recall that $\bP^*=\sigma(\bX^*)$. Then $\tilde\cG(\bU,\bV) = \cG(\bU,\bV) - \bar\cG(\bU,\bV) = e^{-\alpha_0-M_2}(\bP-\bY)$, which is independent of $(\bU,\bV)$. To bound $ \bP-\bY$, define the symmetric dilation
\begin{equation*}
\cA= \begin{pmatrix}
\zero & \bY-\bP^*\\ (\bY-\bP^*)^\T & \zero
\end{pmatrix}.
\end{equation*}
We introduce the following lemma
\begin{lemma}[Theorem 5.2 in \citet{LeiRinaldo2015}]
    Let $\bY=(Y_{ij})\in\{0,1\}^{n\times q}$ have independent entries with $Y_{ij}\sim \mathrm{Bernoulli}(P_{ij}^*),$ for $i\in[n]$ and $j\in[q]$. Assume that $(n+q)\max_{i\in[n],\,j\in[q]} P_{ij}^*\le d$. Then for any $C_0>0$, there exists a constant $C=C(C_0)>0$ such that
\begin{equation*}
\|\bY-\bP^*\| \le C\sqrt{d+\log(n+q)}
\end{equation*}
with probability at least $1-(n+q)^{-C_0}$.
\end{lemma}
Applying the lemma with $\|\cA\| = \|\bY - \bP^*\|$ and $\max\Big\{\max_i\sum_{j=1}^q P_{ij}^*,\ \max_j\sum_{i=1}^n P_{ij}^*\Big\} \le (n\vee q)e^{M_2+\alpha_0}$, we obtain \ref{item_bern_4}.

 For~\ref{item_bern_5}, fix $i\in[n]$ and $\bu\in\SSS^{r-1}$. Write $a_j=\langle \bV^*_{j,\cdot},\bu\rangle$ and $S_{i,\bu}:=\sum_{j=1}^q (Y_{ij}-P_{ij}^*)a_j $. The variables $(Y_{ij}-P_{ij}^*)a_j$ are independent and centered. Moreover,
\begin{equation*}
    \sum_{j=1}^q {\rm Var}\{(Y_{ij}-P_{ij}^*)a_j\} = \sum_{j=1}^q P_{ij}^*(1-P_{ij}^*)a_j^2 \le (\nu_{\star})^{-1}\sum_{j=1}^q a_j^2 \le \|\bV^*\|^2/\nu_{\star}.
\end{equation*}
 Also, we know $\max_{j\in[q]}|a_j|\le \|\bV^*\|_{2\to\infty}$. Therefore, Bernstein's inequality gives, for any $t>0$,
\begin{equation*}
    \PP\left(|S_{i,\bu}| \ge C\{\|\bV^*\|\sqrt{t/\nu_{\star}} +\|\bV^*\|_{2\to\infty}t\} \right) \le 2e^{-t}.
\end{equation*}
Let $\cN$ be a $1/2$-net of $\SSS^{r-1}$ with $|\cN|\le 5^r$.
Taking $t=C\{r+\log(n/\delta)\}$ and applying a union bound over $\bu\in\cN$ and $i\in[n]$, together with the standard net inequality,
yields
\begin{equation*}
    \|(\bY-\bP^*)\bV^*\|_{2\to\infty} \le C\left\{ \|\bV^*\|\sqrt{\nu_{\star}^{-1}\{r+\log(n/\delta)\}} + \|\bV^*\|_{2\to\infty}\{r+\log(n/\delta)\} \right\}.
\end{equation*}
Multiplying by the loss scaling $\nu_{\star}$ and dividing by
$\sqrt q\,\|\bV^*\|$ gives
\begin{equation*}
\frac{ \nu_{\star}\|(\bY-\bP^*)\bV^*\|_{2\to\infty} }{ \sqrt q\,\|\bV^*\| }
\le C\left\{ \sqrt{ \frac{\nu_{\star}\{r+\log(n/\delta)\}}{q} } + \nu_{\star} \frac{\|\bV^*\|_{2\to\infty}}{\|\bV^*\|} \frac{r+\log(n/\delta)}{\sqrt q} \right\}.
\end{equation*}The bound for $(\bY-\bP^*)^\T\bU^*$ follows in the same way. With $R_{\star}>1$, we finish the proof.

For~\ref{item_bern_6}, because $Z_{ij} := |Y_{ij} - P_{ij}^*|$ are independent and $\EE Z_{ij} = 2P_{ij}^*(1-P_{ij}^*)$, we know by Hoeffding' inequality that 
\begin{equation*}
\PP\Big(\frac1q\sum_{j=1}^q Z_{ij} \ge \frac1q\sum_{j=1}^q 2P_{ij}^*(1-P_{ij}^*) + t \Big) \le e^{-2qt^2}.
\end{equation*}Union bound over $i\in[n]$ yields that with probability at least $1-\delta$,
\begin{equation*}
q^{-1}\|\bY-\bP^*\|_{\infty\to 1} \le \max_{i\in[n]}\frac1q\sum_{j=1}^q 2P_{ij}^*(1-P_{ij}^*) + C\sqrt{\frac{\log(n/\delta)}{q}}.
\end{equation*}
For $n^{-1}\|(\bY-\bP^*)^\T\|_{\infty\to 1}$, we can obtain a similar result. Combine these two, and we prove~\ref{item_bern_6}.


\paragraph{\it \underline{Leave-one-out analysis.}} The above analysis shows that $\bar\Delta_{\infty}(n,q,\delta)$ is comparably large in the considered setup. In this part, we show how the leave-one-out analysis can circumvent this issue. In the proof of Theorem~\ref{thm_general_theory_UV_noise}, the  upper bound for $\bar\Delta_{\infty}(n,q,\delta)$ is used in Step~1 to bound $\max_{i\in[n]}\|\Delta_{4,i}^{(U)}\| \vee \max_{\ell\in[q]}\|\Delta_{4,\ell}^{(V)}\| $ in \eqref{eq_rect_score_term_bound_2}, and in Step~3 to bound $\delta_{3,t}^{\dagger}$ in \eqref{eq_rect_step4_delta3}. In what follows, we show that both steps can instead be established without invoking $\bar\Delta_{\infty}(n,q,\delta)$. 

For $\max_{i\in[n]}\|\Delta_{4,i}^{(U)}\| \vee \max_{\ell\in[q]}\|\Delta_{4,\ell}^{(V)}\|$, note that for the Bernoulli response model, $\tilde\cG(\bU,\bV) = \nu_{\star}(\bP^* - \bY)$ is independent of $(\bU,\bV)$, and thus $\tilde\cG_{i\ell} = \int_0^1\tilde G_{i\ell}ds =\nu_{\star}( P^*_{i\ell} - Y_{i\ell})$. Hence $\max_{i\in[n]}\|\Delta_{4,i}^{(U)}\|= \max_{i\in[n]}\big\|\nu_{\star}\sum_{\ell=1}^q (P^*_{i\ell} - Y_{i\ell})\bE_{V,\ell}^\T\big\|= \nu_{\star}\big\|(\bP^* - \bY)\bE_{V}\big\|_{\twinf}$ and $\max_{\ell\in[q]}\|\Delta_{4,\ell}^{(V)}\|=  \big\|\nu_{\star}(\bP^* - \bY)^\T\bE_{U}\big\|_{\twinf}$. To control these quantities, we introduce an auxiliary point $(\bar\bU,\bar\bV)$, which plays the role of a leave-one-out proxy. As shown in the next lemma, $(\bar\bU,\bar\bV)$ is very close to $(\tilde\bU,\tilde\bV)$ in weighted Frobenius norm, while each row $\bar\bU_{i,}$ depends only weakly on $\{Y_{i\ell}\}_{\ell\in[q]}$, and each row $\bar\bV_{j,}$ depends only weakly on $\{Y_{ij}\}_{i\in[n]}$, yielding a sharp control on $q^{-1}\|(\bP^*-\bY)(\bar\bV-\bV^*)\|_{2\to\infty}\vee n^{-1}\|(\bP^*-\bY)^\T(\bar\bU-\bU^*)\|_{2\to\infty}$. Write $\ell_{\alpha_0}'(x) = \partial_{x}\ell_{\alpha_0}(x;Y_{ij})$ and $\ell_{\alpha_0}''(x) = \partial_{x}^2\ell_{\alpha_0}(x;Y_{ij})$ for short when $Y_{ij}$ is clear from the context.

\begin{lemma}\label{lemma_rect_auxiliary_bar}
Assume the conditions of the Bernoulli response model in Example~2, and suppose the event in
Theorem~\ref{thm_general_theory_UV_noise} holds. Construct $(\bar\bU,\bar\bV)$ as follows:
\begin{align*}    \bar\bU_{i,} &=\bU_{i,}^*-\Big\{\sum_{j\in[q]}\ell_{\alpha_0}''(X_{ij}^*)(\bV_{j,}^*)^\T\bV_{j,}^*\Big\}^{-1} \sum_{j\in[q]}\ell_{\alpha_0}'(X_{ij}^*)(\bV_{j,}^*)^\T,\qquad i\in[n],\\
    \bar\bV_{j,}&=\bV_{j,}^*-\Big\{\sum_{i\in[n]}\ell_{\alpha_0}''(X_{ij}^*)(\bU_{i,}^*)^\T\bU_{i,}^*\Big\}^{-1}\sum_{i\in[n]}\ell_{\alpha_0}'(X_{ij}^*)(\bU_{i,}^*)^\T,\qquad j\in[q].
\end{align*}
Let $R_{\bar Z} := \big\|\bS_Z^2\nabla_{\bZ}h_{\alpha}^{\natural}(\bar\bU,\bar\bV)\big\|_{\twinfw}$.
Then, with probability at least $1-\delta$, the following hold:
\begin{enumerate}
    \item The score at $(\bar\bU,\bar\bV)$ is bounded by
    \begin{equation}\label{eq_lemma_rect_auxiliary_bar_score_simple}
    R_{\bar Z}\le C\Delta_{\infty}(n,q,\delta)\omega_* .
    \end{equation}
    \item The auxiliary point is close to $(\tilde\bU,\tilde\bV)$ in weighted Frobenius norm:
    \begin{equation}\label{eq_lemma_rect_auxiliary_bar_close}
    \|(\bar\bU-\tilde\bU,\bar\bV-\tilde\bV)\|_{\Fnw}\le C\,\frac{R_{\bar Z}}{\alpha\sigma_{\min}}.
    \end{equation}
    \item The quadratic noise terms satisfy
    \begin{equation}\label{eq_lemma_rect_auxiliary_bar_noise}\begin{aligned}
    &\,\nu_{\star}\Big(\frac{\|(\bP^*-\bY)(\bar\bV-\bV^*)\|_{2\to\infty}}{q}\vee \frac{\|(\bP^*-\bY)^\T(\bar\bU-\bU^*)\|_{2\to\infty}}{n}\Big)\\\le &\,C\Big\{
    \nu_\star L_\star \sqrt{\frac{\beta}{\alpha nq\sigma_{\min}}} + \frac{\nu_\star\beta\,\omega_*}{\alpha (n\wedge q)\sigma_{\min}} + \frac{\nu_\star^2L_\star\,\omega_*}{\alpha nq\sigma_{\min}}
\Big\}\le C\Delta_{\infty}(n,q,\delta)\omega_*.
    \end{aligned}
    \end{equation}
\end{enumerate}
\end{lemma}
\begin{proof}
    See Section~\ref{supp_sec_prove_lemma_rect_auxiliary_bar}.
\end{proof}

With this lemma, we write $\max_{i\in[n]}\|\Delta_{4,i}^{(U)}\|$ as 
\begin{align*}
    \max_{i\in[n]}\|\Delta_{4,i}^{(U)}\| = &\,\nu_{\star}\big\|(\bP^* - \bY)\bE_{V}\big\|_{\twinf}\\
    \le &\,\nu_{\star}\big\|(\bP^* - \bY)(\bar\bV - \bV^*)\big\|_{\twinf} + \nu_{\star}\big\|\bP^* - \bY\big\|\big\|\bar\bV - \tilde\bV\big\|_{\Fn}\\ \le &\,C\frac{\Delta_{\infty}}{\alpha\sigma_{\min}}\omega_*,
\end{align*}where the second inequality follows from parts 2 and 3 of
Lemma~\ref{lemma_rect_auxiliary_bar}, and the last inequality follows from the scaling conditions.
By symmetry, the same bound holds for $\max_{\ell\in[q]}\|\Delta_{4,\ell}^{(V)}\|$.
Therefore, \eqref{eq_rect_score_term_bound_2} can be similarly established. 

Consequently, the minimizer $(\tilde\bU,\tilde\bV)$ lies in the interior of $\cDuvde$. The remainder of Step~1, namely the proof that $(\tilde\bU,\tilde\bV)$ is an interior stationary point of $h_\alpha^\natural(\cdot)$, now proceeds exactly as in the proof of Theorem~\ref{thm_general_theory_UV_noise}, since it only uses \eqref{eq_rect_step1_l2_half}, \eqref{eq_rect_step1_linf_half}, and the already verified curvature and Lipschitz conditions.
Likewise, Step~2 can be repeated, because once the conclusions of Step~1 are available, the subsequent argument only invokes the same local curvature, Lipschitz, and noise bounds, all of which have already been verified for Example~2 above.

It remains to sharpen the row-wise analysis in Step~3 by deriving a bound for $\delta_{3,t}^{\dagger}$ as in \eqref{eq_rect_step4_delta3}. Our idea is to invoke the LOO initialization and construct a LOO sequence along with $(\bU^t,\bV^t)$ such that the sequence is close to $(\bU^t,\bV^t)$ and has tractable dependence with respect to the noise matrix $\bY - \bP^*$. We summarize the result in Lemma~\ref{lem_tentative_loo_gradient_mc}.

\begin{lemma}
\label{lem_tentative_loo_gradient_mc}
Under the setting of Theorem~\ref{thm_general_theory_UV_noise}, let $\tilde\bE^t_V $ and $\tilde\bE_U^t$ be defined as in \eqref{eq_define_loo_et}.  On an event with probability at least $1-\delta$, uniformly for all integers $0\le t\le T_\star$, we have
\begin{equation}\label{eq_loo_mc_delta3_target}
\begin{aligned}\delta_{3,t}^{\dagger,U}\vee \delta_{3,t}^{\dagger,V} = &\,
\frac{\eta\nu_{\star}}{q}\|(\bY-\bP^*)\tilde\bE_V^t\|_{2\to\infty}\vee \frac{\eta\nu_{\star}}{n}\|(\bY-\bP^*)^\T\tilde\bE_U^t\|_{2\to\infty}\\ \le &\,\eta\big\{\frac{1}{8}\alpha\sigma_{\min}\rho^t\psi_{nq}^{\dagger} + C\Delta_{\infty}(n,q,\delta)\big\}\omega_*.
\end{aligned}\end{equation}
Consequently, the induction hypotheses in Step~3 of the proof of
Theorem~\ref{thm_general_theory_UV_noise} is verified at time $t+1$.
\end{lemma}
\begin{proof}
    See Section~\ref{supp_sec_prove_lem_tentative_loo_gradient_mc}.
\end{proof}
The bound \eqref{eq_loo_mc_delta3_target} replaces the bound for $\bar\Delta_\infty(n,q,\delta)$ in \eqref{eq_rect_step4_delta3}. The contraction part is absorbed exactly as in Step~3, while the non-contracting part contributes the final statistical error $C\Delta_\infty(n,q,\delta)\omega_*/(\alpha\sigma_{\min})$. Hence, the result of Theorem~\ref{thm_general_theory_UV_noise} applies to the Bernoulli low-rank model without assuming $\bar\Delta_{\infty}(n,q,\delta)\le \alpha\sigma_{\min}/4$.

\section{Proof of Technical Lemmas}

\subsection{Proof of Lemma~\ref{lemma_general_convex_1} and Lemma~\ref{lemma_general_convex}}\label{supp_Sec_prove_lemma_general_convex}

    Throughout the proof, we write the objective as $ h_{\alpha}^*(\cdot)$. In what follows, we prove Lemma~\ref{lemma_general_convex} only, and Lemma~\ref{lemma_general_convex_1} can be immediately obtained with $\cD=\cDtwo$, Weyl's inequality, and 
    \begin{equation*}
        \|\cG(\bZ)\|\le L_2\|\bZ\bZ^\T - \bZ^*(\bZ^*)^\T\|_{\Fn}\le L_2(2\epsilon+\epsilon^2)\|\bZ^*\|_{\Fn}^2\le 4L_2\epsilon\|\bZ^*\|_{\Fn}^2,
    \end{equation*}by Assumption~\ref{assump_gradient_lips_1}.
    For any $\bZ\in\cD$ and any direction $\bW \in \mathbb{R}^{n \times r}$, the Hessian quadratic form for $ h_{\alpha}^*(\bZ)$ is:
\begin{equation*}
    \nabla^2_{\bZ} h_{\alpha}^*(\bZ)[\bW, \bW] = \nabla^2_{\bZ} \cL(\bZ\bZ^\T)[\bW, \bW] + \nabla^2_{\bZ} p^*_{\alpha}(\bZ)[\bW, \bW].
\end{equation*}
For $\nabla_{\bz}^2\cL(\bZ\bZ^\T)$, one can check $\big\{\mathrm{vec}(\bW^\T)\big\}^\T \nabla_{\bz}^2\cL(\bZ\bZ^\T)\mathrm{vec}(\bW^\T) =\nabla^2_{\bZ} \cL(\bZ\bZ^\T)[\bW, \bW]$. Using the matrix chain rule, there is 
\begin{align*}
    \nabla^2_{\bZ} \cL(\bZ\bZ^\T)[\bW, \bW] = \nabla^2_{\bX} \cL(\bZ\bZ^\T)[\cP_{\bZ}(\bW), \cP_{\bZ}(\bW)] + 2 \langle \nabla_{\bX} \cL(\bZ\bZ^\T), \bW\bW^\T \rangle,
\end{align*}where $\cP_{\bZ}(\bW) = \bZ\bW^\T + \bW\bZ^\T$.
With $\cG(\bZ) = \nabla_{\bX}\cL(\bZ\bZ^\T)$,  Assumption~\ref{assump_rsc} then yields:
\begin{equation*}
    \nabla^2_{\bZ} \cL(\bZ\bZ^\T)[\bW, \bW] - 2\langle \cG(\bZ),\bW\bW^\T\rangle\ge \alpha\|\cP_{\bZ}(\bW)\|_{\Fn}^2.
\end{equation*}
With the decomposition $\cP_{\bZ}(\bW)=\cP_{\bZ^*}(\bW)+\cP_{\bZ-\bZ^*}(\bW)$ and the inequality $\|\bA+\bB\|_{\Fn}^2\ge \frac12\|\bA\|_{\Fn}^2-\|\bB\|_{\Fn}^2$, we arrive at 
\begin{equation*}
    \|\cP_{\bZ}(\bW)\|_{\Fn}^2\ge \frac12\|\cP_{\bZ^*}(\bW)\|_{\Fn}^2 -\|\cP_{\bZ-\bZ^*}(\bW)\|_{\Fn}^2.
\end{equation*}
By noting that $\|\cP_{\bZ^*}(\bW)\|_{\Fn}^2 = 2|\bW\bZ^{*\T}\|_{\Fn}^2
+2\tr\big\{(\bW^\T\bZ^*)^2$ and $\|\cP_{\bZ-\bZ^*}(\bW)\|_{\Fn} = \|(\bZ-\bZ^*)\bW^\T+\bW(\bZ-\bZ^*)^\T\|_{\Fn} \le 2\|\bZ-\bZ^*\|_{\Fn}\|\bW\|_{\Fn}$, we obtain
\begin{align}\label{eq:loss-lower}
    \nabla^2_{\bZ} \cL(\bZ\bZ^\T)[\bW, \bW] \ge \alpha\left[ \|\bZ^*\bW^\T\|_\Fn^2 + \tr\{(\bW^\T\bZ^*)^2\} \right] -4\alpha\|\bW\|_{\Fn}^2\|\bZ - \bZ^*\|_{\Fn}^2.
\end{align}

Next, for $p_{\alpha}^*(\bZ)$, one can compute that the second-order derivative of the regularization term $p^*_{\alpha}(\bZ)$ in direction $\bW$ is:
\begin{align*}
    \nabla^2_{\bZ}p^*(\bZ)[\bW, \bW] =&\, \frac{\alpha}{2} \| (\bZ^*)^\T \bW - \bW^\T \bZ^* \|_\Fn^2\\=&\, \alpha \left\{ \|\bW^\T \bZ^*\|_\Fn^2 - \tr(\bW^\T \bZ^*)^2) \right\}.
\end{align*}
Adding this to \eqref{eq:loss-lower}, the trace terms cancel, and we get
\begin{align*}
    \nabla^2_{\bZ}  h^*_{\alpha}(\bZ)[\bW, \bW] -2\langle \cG(\bZ),\bW\bW^\T\rangle \ge&\, 
    \alpha\|\bW(\bZ^*)^\T\|_\Fn^2 + \alpha \|\bW^\T \bZ^*\|_\Fn^2\\&\,-4\|\bW\|_{\Fn}^2\|\bZ - \bZ^*\|_{\Fn}^2 \\
    \ge &\,\left\{\alpha\sigma_{r}(\bZ^*)^2-4\|\bZ - \bZ^*\|_{\Fn}^2\right\}\|\bW\|_{\Fn}^2.
\end{align*}
The last inequality follows from  $\|\bW(\bZ^*)^\T\|_\Fn^2 = \tr\{\bW^\T \bW \bZ^*(\bZ^*)^\T\} \ge \sigma_r(\bZ^*)^2\|\bW\|_{\Fn}^2$.

Now let $\bz=\mathrm{vec}(\bZ^\T)$. Using $\langle \cG(\bZ),\bW\bW^\T\rangle = \big\{\mathrm{vec}(\bW^\T)\big\}^\T \big(\cG(\bZ)\otimes \bI_r\big) \mathrm{vec}(\bW^\T)$, we obtain 
\begin{align*}
\lambda_{\min}\big\{n^{-1}\nabla_{\bz}^2 h_{\alpha}^*(\bZ)
\!-\!2n^{-1}\cG(\bZ)\otimes \bI_r\big\}
\!=\!
\min_{\|\bW\|_{\Fn}=1}
\big(
\nabla_{\bZ}^2 h_{\alpha}^*(\bZ)[\bW,\bW]
\!-\!2\langle \cG(\bZ),\bW\bW^\T\rangle
\big)/n.
\end{align*}
Taking the minimum over $\bZ\in\cD$ completes the proof.

\subsection{Proof of Lemma~\ref{lemma_stationary_tildeZ}}\label{supp_sec_prove_lemma_stationary_tildeZ}
First, we note that with Lemma~\ref{lemma_general_convex} and $\|\tilde\bZ - \bZ^*\|_{\Fn}\le \epsilon\|\bZ^*\|$,
\begin{equation*}
\lambda_{\min}\Big\{ n^{-1}\nabla_{\bz}^2 h_{\alpha}^*(\bZ)-2n^{-1}\cG(\bZ)\otimes \bI_r \Big\} \ge \alpha\sigma_{\min}-4n^{-1}\epsilon^2\|\bZ^*\|_{\Fn}^2.
\end{equation*}
Then with $\|\bZ^*\|_{\Fn}^2\le nr\kappa\sigma_{\min}$, $\epsilon^2 r\kappa\le c_0\alpha$ for sufficiently small $c_0$, we know 
\begin{equation*}
\lambda_{\min}\Big\{ n^{-1}\nabla_{\bz}^2 h_{\alpha}^*(\bZ)-2n^{-1}\cG(\bZ)\otimes \bI_r \Big\} \ge 3\alpha\sigma_{\min}/4.
\end{equation*}
In bounding $\gamma_2$ in the proof of Theorem~\ref{thm_general_theory_noisy_Z}, we know
\begin{align*}
n^{-1}\|\cG(\bZ)\| \le 
\Delta_2(n,\delta) + (2+\epsilon)\epsilon C \sqrt{r}\kappa\sigma_{\min}.
\end{align*}
With $\epsilon\sqrt{r}\kappa\le c_0\alpha$ and $\Delta_2(n,\delta)/(\alpha\sigma_{\min})\le c_0$ for sufficiently small, we concluded that, for any $\bZ\in\cD_{\epsilon}$,
\begin{align*}
\lambda_{\min}\bigl(n^{-1}\nabla_{\bz}^2 h_{\alpha}^*(\bZ)\bigr)\ge&\,\lambda_{\min}\Big\{ n^{-1}\nabla_{\bz}^2 h_{\alpha}^*(\bZ)-2n^{-1}\cG(\bZ)\otimes \bI_r \Big\} - 2n^{-1}\|\cG(\bZ)\| \\ \ge&\,
\frac{1}{2}\alpha\sigma_{\min}>0.
\end{align*}

Since $\tilde\bZ$ is a strict interior point of $\cD_{\epsilon}$, it is a local minimizer to $\|\nabla_{\bZ}h_{\alpha}^*(\bZ)\|_{\twinf}^2$.
Suppose, for contradiction, that $\nabla_{\bZ}h_{\alpha}^*(\tilde\bZ)\neq \zero$. Let
\begin{equation*}
\tilde\bz:=\mathrm{vec}(\tilde\bZ^\T), \qquad 
\bv:=\mathrm{vec}\bigl(\{\nabla_{\bZ}h_{\alpha}^*(\tilde\bZ)\}^\T\bigr),
\qquad
\bH:=\nabla_{\bz}^2 h_{\alpha}^*(\tilde\bZ).\end{equation*}
Then $\bv\neq \zero$ and $\bH\succ 0$. Define $\bd:=-\bH^{-1}\bv$ and let $\bD\in\RR^{n\times r}$ be the matrix satisfying $\mathrm{vec}(\bD^\T)=\bd$.
By differentiability of $\nabla_{\bZ}h_{\alpha}^*(\cdot)$,
\begin{equation*}
\mathrm{vec}\Bigl(\bigl\{\nabla_{\bZ}h_{\alpha}^*(\tilde\bZ+s\bD)\bigr\}^\T\Bigr) = \bv+s\bH\bd+o(s) = (1-s)\bv+o(s), \text{ as } s\to 0.
\end{equation*}
Equivalently,
\begin{equation*}
\nabla_{\bZ}h_{\alpha}^*(\tilde\bZ+t\bD) = (1-t)\nabla_{\bZ}h_{\alpha}^*(\tilde\bZ)+\bR_t, \qquad \|\bR_t\|_{\twinf}=o(t).
\end{equation*}
Consequently, for sufficiently small $t>0$, there is
\begin{align*}
\|\nabla_{\bZ}h_{\alpha}^*(\tilde\bZ+t\bD)\|_{\twinf} < &\, (1-t)\|\nabla_{\bZ}h_{\alpha}^*(\tilde\bZ)\|_{\twinf} + \|\bR_t\|_{\twinf}\\\le &\, (1-t)\|\nabla_{\bZ}h_{\alpha}^*(\tilde\bZ)\|_{\twinf} + o(t) <
\|\nabla_{\bZ}h_{\alpha}^*(\tilde\bZ)\|_{\twinf},
\end{align*}
This contradicts the fact that $\tilde\bZ$ is a local minimizer of $\|\nabla_{\bZ}h_{\alpha}^*(\bZ)\|_{\twinf}^2$. This completes the proof.

\subsection{Proof of Lemma~\ref{lemma_general_convex_asym}}\label{supp_prove_lemma_general_convex_asym}

Fix any $(\bU,\bV)\in\cD$, and let $(\bL,\bR)\in\RR^{n\times r}\times\RR^{q\times r}$ be arbitrary. Let $\bw:=\mathrm{vec}\big((\bL^\T,\bR^\T)\big)$, and define
\begin{equation*}
\bar\bL:=q^{-1/2}\bL, \qquad \bar\bR:=n^{-1/2}\bR.
\end{equation*}
Then, by construction, $\bS_Z\bw=\mathrm{vec}\big((\bar\bL^\T,\bar\bR^\T)\big)$, where by abuse of notation
$\bS_Z$ stands for $\bS_Z\otimes \bI_r$ when acting on $\bz$.

By the chain rule,
\begin{equation}\label{eq_rect_quad_form}\begin{aligned}
\bw^\T\bS_z\big\{\nabla_{\bz}^2 h_\alpha^*(\bU,\bV)-\cG_e(\bU,\bV)\big\}\bS_z\bw &=
\nabla_{\bX}^2\cL(\bU\bV^\T)\big[\bU\bar\bR^\T+\bar\bL\bV^\T,\, \bU\bar\bR^\T+\bar\bL\bV^\T\big] \\
&\qquad + \nabla_{(\bU,\bV)}^2p_\alpha^*(\bU,\bV)\big[(\bar\bL,\bar\bR),(\bar\bL,\bar\bR)\big].\end{aligned}
\end{equation}
Here $\cG_e(\bU,\bV)$ is the block matrix collecting the first-order terms produced by the nonlinear map $(\bU,\bV)\mapsto \bU\bV^\T$.

Since the map inside $p_\alpha^*(\bU,\bV)$ is linear in $(\bU,\bV)$, one can check
\begin{equation}\begin{aligned}
\nabla_{(\bU,\bV)}^2p_\alpha^*(\bU,\bV)\big[(\bar\bL,\bar\bR),(\bar\bL,\bar\bR)\big] &= \frac{\alpha nq}{2} \left\| n^{-1}(\bU^*)^\T\bar\bL - q^{-1}\bar\bR^\T\bV^* \right\|_{\Fn}^2\\ &= 
\frac{\alpha}{2} \left\| n^{-1/2} \bL^\T\bU^* - q^{-1/2}(\bV^*)^\T\bR \right\|_{\Fn}^2.\end{aligned}
\label{eq_rect_penalty_hessian_correct}
\end{equation}

Hence, plugging \eqref{eq_rect_penalty_hessian_correct} into \eqref{eq_rect_quad_form}, and with Assumption~\ref{assump_rect_rip_weighted}, we have
\begin{equation}\begin{aligned}
\bw^\T\bS_z\big\{\nabla_{\bz}^2 h_\alpha^*(\bU,\bV) -\cG_e(\bU,\bV)\big\}\bS_z\bw \ge &\,\alpha\|\bU\bar\bR^\T+\bar\bL\bV^\T\|_{\Fn}^2
\\&\,+ \frac{\alpha}{2} \left\| n^{-1/2}\bL^\T\bU^*-q^{-1/2}(\bV^*)^\T\bR \right\|_{\Fn}^2.\end{aligned}
\label{eq_rect_step1}
\end{equation}

Now decompose
\begin{equation*}
\bU\bar\bR^\T+\bar\bL\bV^\T = \underbrace{n^{-1/2}\bU^*\bR^\T+q^{-1/2}\bL(\bV^*)^\T}_{=: \bA} + \underbrace{n^{-1/2}(\bU-\bU^*)\bR^\T+q^{-1/2}\bL(\bV-\bV^*)^\T}_{=: \bB}.
\end{equation*}
Using $\|\bA+\bB\|_{\Fn}^2\ge \frac12\|\bA\|_{\Fn}^2-\|\bB\|_{\Fn}^2$, we obtain from \eqref{eq_rect_step1} that
\begin{equation}\begin{aligned}
\bw^\T\bS_z\big\{\nabla_{\bz}^2 h_\alpha^*(\bU,\bV)-\cG_e(\bU,\bV)\big\}\bS_z\bw \ge &\,\frac{\alpha}{2}\|\bA\|_{\Fn}^2 -\alpha\|\bB\|_{\Fn}^2
\\&\,+ \frac{\alpha}{2} \left\| n^{-1/2}\bL^\T\bU^*-q^{-1/2}(\bV^*)^\T\bR \right\|_{\Fn}^2.\end{aligned}
\label{eq_rect_step2}\end{equation}

We now combine the first and third terms. By direct expansion, one can obtain\begin{align*}
\|\bA\|_{\Fn}^2 =&\, \frac12\left\| n^{-1/2}\bU^*\bR^\T+q^{-1/2}\bL(\bV^*)^\T \right\|_{\Fn}^2 \\ =&\,
n^{-1}\|\bU^*\bR^\T\|_{\Fn}^2 + q^{-1}\|\bL(\bV^*)^\T\|_{\Fn}^2 + \frac{2}{\sqrt{nq}}\tr\!\big(\bL^\T\bU^*\bR^\T\bV^*\big),
\end{align*} and \begin{align*}
\big\| n^{-1/2}\bL^\T\bU^*-q^{-1/2}(\bV^*)^\T\bR \big\|_{\Fn}^2 = n^{-1}\|\bL^\T\bU^*\|_{\Fn}^2 + q^{-1}\|(\bV^*)^\T\bR\|_{\Fn}^2 \!-\! \frac{2}{\sqrt{nq}}\tr\!\big(\bL^\T\bU^*\bR^\T\bV^*\big).
\end{align*}
The cross terms cancel, so
\begin{equation}\begin{aligned}
\frac12\|\bA\|_{\Fn}^2 + \frac12 \left\| n^{-1/2}\bL^\T\bU^*-q^{-1/2}(\bV^*)^\T\bR \right\|_{\Fn}^2 =&\, \frac12 n^{-1}\|\bU^*\bR^\T\|_{\Fn}^2 + \frac12 q^{-1}\|\bL(\bV^*)^\T\|_{\Fn}^2\\
&\, + \frac12 n^{-1}\|\bL^\T\bU^*\|_{\Fn}^2 + \frac12 q^{-1}\|(\bV^*)^\T\bR\|_{\Fn}^2.\end{aligned} \label{eq_rect_identity1}
\end{equation}
Using the balancing condition $n^{-1}(\bU^*)^\T\bU^* = q^{-1}(\bV^*)^\T\bV^* = \bSigma^*$, we have
\begin{align*}
n^{-1}\|\bU^*\bR^\T\|_{\Fn}^2 = \tr(\bR\bSigma^*\bR^\T) \quad\text{ and }\quad q^{-1}\|\bL(\bV^*)^\T\|_{\Fn}^2 = \tr(\bL^\T\bL\bSigma^*).
\end{align*}
Therefore \eqref{eq_rect_identity1} further simplifies to
\begin{equation} \label{eq_rect_identity2}
\eqref{eq_rect_identity1} = \tr(\bL^\T\bL\bSigma^*)+\tr(\bR\bSigma^*\bR^\T)
\ge \sigma_{\min}\big(\|\bL\|_{\Fn}^2+\|\bR\|_{\Fn}^2\big).
\end{equation}

It remains to bound $\|\bB\|_{\Fn}^2$. By Cauchy Schwarz,
\begin{equation}\begin{aligned} \|\bB\|_{\Fn}
\le &\, n^{-1/2}\|\bU-\bU^*\|_{\Fn}\|\bR\|_{\Fn} + q^{-1/2}\|\bV-\bV^*\|_{\Fn}\|\bL\|_{\Fn}\\
\le&\,\Big( \frac{\|\bU-\bU^*\|_{\Fn}^2}{n} + \frac{\|\bV-\bV^*\|_{\Fn}^2}{q} \Big)^{1/2} \big(\|\bL\|_{\Fn}^2+\|\bR\|_{\Fn}^2\big)^{1/2},\end{aligned}
\label{eq_rect_B_bound1}
\end{equation}

Substituting \eqref{eq_rect_identity2} and \eqref{eq_rect_B_bound1} into \eqref{eq_rect_step2}, we obtain
\begin{align*}
\bw^\T\bS_z\big\{\nabla_{\bz}^2 h_\alpha^*(\bU,\bV)-\cG_e(\bU,\bV)\big\}\bS_z\bw \ge &\,\alpha \sigma_{\min}\big(\|\bL\|_{\Fn}^2+\|\bR\|_{\Fn}^2\big) \\&\,-
\Big( \frac{\|\bU-\bU^*\|_{\Fn}^2}{n} + \frac{\|\bV-\bV^*\|_{\Fn}^2}{q} \Big)\big(\|\bL\|_{\Fn}^2+\|\bR\|_{\Fn}^2\big).
\end{align*}
Since $\|\bw\|^2=\|\bL\|_{\Fn}^2+\|\bR\|_{\Fn}^2$, taking the infimum over all $\|\bw\|=1$ and then
the minimum over $(\bU,\bV)\in\cD$ yields the results.

Now we study $h_{\alpha}^{\natural}(\cdot,\cdot)$. Let $\cA_*(\bU,\bV):= n^{-1}(\bU^*)^\T \bU-q^{-1}\bV^\T \bV^* $, and $\cA_{\natural}(\bU,\bV):= n^{-1}(\bU-\bU^*)^\T \bU-q^{-1}\bV^\T(\bV-\bV^*)$.
Then one has
\begin{equation*}
p_\alpha^*(\bU,\bV)=\frac{\alpha nq}{4}\|\cA_*(\bU,\bV)\|_{\Fn}^2,
\qquad p_\alpha^\natural(\bU,\bV)=\frac{\alpha nq}{4}\|\cA_\natural(\bU,\bV)\|_{\Fn}^2.
\end{equation*}

Since $\cA_*$ is linear, $D\cA_*(\bU,\bV)[(\bar\bL,\bar\bR)] = n^{-1}(\bU^*)^\T \bar\bL-q^{-1}\bar\bR^\T \bV^*$ and $D^2\cA_{*}(\bU,\bV) = \zero$. For $\cA_{\natural}$, we note that
\begin{align*}
D\cA_\natural(\bU,\bV)[(\bar\bL,\bar\bR)] =&\, n^{-1}(\bU^*)^\T \bar\bL-q^{-1}\bar\bR^\T \bV^*\quad(:=\bDelta_1) \\&\,+ n^{-1}\bar\bL^\T \bE_U+n^{-1}\bE_U^\T \bar\bL - q^{-1}\bar\bR^\T \bE_V-q^{-1}\bE_V^\T \bar\bR\quad(:=\bDelta_2),
\end{align*}
where $\bE_U = \bU - \bU^*$ and $\bE_V = \bV - \bV^*$, and
\begin{equation*}
D^2\cA_\natural(\bU,\bV)[(\bar\bL,\bar\bR),(\bar\bL,\bar\bR)]
=
2n^{-1}\bar\bL^\T\bar\bL-2q^{-1}\bar\bR^\T\bar\bR.
\end{equation*}

Therefore with identity $D^2\big(\|\cA(\btheta)\|_{\Fn}^2\big)[\bH,\bH] = 2\|D\cA(\btheta)[\bH]\|_{\Fn}^2 + 2\langle\cA(\btheta),D^2\cA(\btheta)[\bH,\bH]\rangle$ for any mapping $\cA(\btheta)$, we have
\begin{equation}\begin{aligned}
&\nabla_{(\bU,\bV)}^2 p_\alpha^\natural(\bU,\bV)\big[(\bar\bL,\bar\bR),(\bar\bL,\bar\bR)\big] - \nabla_{(\bU,\bV)}^2 p_\alpha^*(\bU,\bV)\big[(\bar\bL,\bar\bR),(\bar\bL,\bar\bR)\big]\\
&\qquad = \frac{\alpha nq}{2} \Big( \|D\cA_\natural\|_{\Fn}^2-\|D\cA_*\|_{\Fn}^2 \Big) +
\frac{\alpha nq}{2} \Big\langle \cA_\natural,\, D^2\cA_\natural[(\bar\bL,\bar\bR),(\bar\bL,\bar\bR)] \Big\rangle .\end{aligned}
\label{eq_natural_perturb_1}
\end{equation}

We now bound the right hand side of \eqref{eq_natural_perturb_1}. First, as $\|\bU^*\|^2/n\le \kappa\sigma_{\min}$ and $\|\bV^*\|^2/q\le \kappa\sigma_{\min}$, one has
\begin{equation*}
\sqrt{nq}\,\|\bDelta_1\|_{\Fn} = \sqrt{nq}\,\|D\cA_*\|_{\Fn} \le C\sqrt{\kappa\sigma_{\min}} \big(\|\bL\|_{\Fn}^2+\|\bR\|_{\Fn}^2\big)^{1/2}.
\end{equation*}
Similarly, with
\begin{equation*}
    d:= \frac{\|\bU-\bU^*\|_{\Fn}^2}{n} + \frac{\|\bV-\bV^*\|_{\Fn}^2}{q}\le c_0\sigma_{\min}/\kappa
\end{equation*}
one can verify with Cauchy inequality that
\begin{equation*}
\sqrt{nq}\,\|\bDelta_2\|_{\Fn} \le C\sqrt d\, \big(\|\bL\|_{\Fn}^2+\|\bR\|_{\Fn}^2\big)^{1/2}\le C\sqrt{\frac{c_0\sigma_{\min}}{\kappa} }\big(\|\bL\|_{\Fn}^2+\|\bR\|_{\Fn}^2\big)^{1/2}.
\end{equation*}
Hence with Cauchy inequality,
\begin{equation}\begin{aligned}
    \frac{\alpha nq}2\Big| \|D\cA_\natural\|_{\Fn}^2-\|D\cA_*\|_{\Fn}^2 \Big| = &\, \frac{\alpha nq}2\Big| \|\bDelta_1 + \bDelta_2\|_{\Fn}^2-\|\bDelta_1\|_{\Fn}^2 \Big| \\= &\,\frac{\alpha nq}2\Big|2 \Tr\big(\bDelta_1 \bDelta_2^\T \big) + \|\bDelta_2\|_{\Fn}^2\Big|\\
    \le &\,C\sqrt{c_0}\alpha\sigma_{\min}\big(\|\bL\|_{\Fn}^2+\|\bR\|_{\Fn}^2\big)
\end{aligned}\label{eq_natural_perturb_2}\end{equation}

Moreover, with $\cA_\natural = n^{-1}\bE_U^\T \bU^*-q^{-1}(\bV^*)^\T \bE_V + n^{-1}\bE_U^\T \bE_U-q^{-1}\bE_V^\T \bE_V$, we have
\begin{equation*}
\|\cA_\natural\|_{\Fn} \le C\big(\sqrt{\kappa\sigma_{\min}d}+d\big)\le C\sqrt{c_0}\sigma_{\min}.
\end{equation*}
Together with
\begin{equation*}
nq\left\| D^2\cA_\natural[(\bar\bL,\bar\bR),(\bar\bL,\bar\bR)] \right\|_{\Fn} \le C\big(\|\bL\|_{\Fn}^2+\|\bR\|_{\Fn}^2\big),
\end{equation*}and Cauchy inequality,
we arrive at
\begin{align}
\frac{\alpha nq}{2}
\left| \Big\langle \cA_\natural,\, D^2\cA_\natural[(\bar\bL,\bar\bR),(\bar\bL,\bar\bR)] \Big\rangle \right| &\le C\sqrt{c_0}\alpha\sigma_{\min}\big(\|\bL\|_{\Fn}^2+\|\bR\|_{\Fn}^2\big).
\label{eq_natural_perturb_3}
\end{align}

Combining \eqref{eq_natural_perturb_1}, \eqref{eq_natural_perturb_2}, and
\eqref{eq_natural_perturb_3}, we obtain
\begin{align*}
\Big| \nabla_{(\bU,\bV)}^2 p_\alpha^\natural(\bU,\bV)\big[(\bar\bL,\bar\bR),(\bar\bL,\bar\bR)\big] - &\,\nabla_{(\bU,\bV)}^2 p_\alpha^*(\bU,\bV)\big[(\bar\bL,\bar\bR),(\bar\bL,\bar\bR)\big] \Big| \\ \le&\,C\sqrt{c_0}\alpha\sigma_{\min}\big(\|\bL\|_{\Fn}^2+\|\bR\|_{\Fn}^2\big).
\end{align*}

With \eqref{eq_lambda_min_asym_lower} and Weyl's inequality, we arrive at 
\begin{align*}
\min_{(\bU,\bV)\in\cD}\lambda_{\min}\!\Big[\bS_z\big\{ \nabla_{\bz}^2 h_{\alpha}^{\natural}(\bU,\bV)-   \cG_e(\bU,\bV) \big\}\bS_z\Big] \ge &\,{\alpha}\sigma_{\min}  - \alpha d - C\sqrt{c_0}\alpha\sigma_{\min}\\\ge &\,{\alpha}\sigma_{\min}  - \alpha \frac{c_0\sigma_{\min}}{\kappa} - C\sqrt{c_0}\alpha\sigma_{\min},
\end{align*}
The proof is then finished after shrinking $c_0$ sufficiently small.

\subsection{Proof of Lemma~\ref{lemma_id_rec}}\label{supp_Sec_prove_lemma_id_rec}

For any $\bH\in\RR^{r\times r}$, let $\bG (s)=\bI_r+s\bH$. Since $\bG (s)^{-\T}=\bI_r-s\bH^\T+o(s)$,
we have
\begin{align*}
\frac{d}{ds}\Phi(\bG (s))\Big|_{s=0} &= \frac{2}{n}\langle \bU-\bU^*,\,\bU\bH\rangle -\frac{2}{q}\langle \bV-\bV^*,\,\bV\bH^\T\rangle \\ &= 2\Big\langle n^{-1}(\bU-\bU^*)^\T\bU - q^{-1}\bV^\T(\bV-\bV^*), \,\bH \Big\rangle.
\end{align*}
If $\bI_r$ is a local minimizer, the derivative above vanishes for every $\bH$, which proves that $\bM^*(\bU,\bV) = \zero$. The equivalence with \eqref{st_id_rec} is immediate from the
definition of $p_\alpha^\natural$.

\subsection{Proof of Lemma~\ref{lemma_alignment_near_rotation}}\label{subsec_prove_lemma_alignment_near_rotation}
\begin{proof}
    The proof is similar to that of Lemma 1 in \cite{ma2021beyond}. For completeness, we state it here. Define 
    \begin{equation*}
        f_{alg}(\bQ) = n^{-1}\big\|\bU\bQ - \bU^{*}\big\|_{\Fn}^2 + q^{-1}\big\|\bV\bQ^{-\T} - \bV^{*}\big\|_{\Fn}.
    \end{equation*}
    First the first part, consider the minimizer of $f_{alg}$ inside a small region around $\bP$
    \begin{align}
        \min_{\bQ\in\RR^{r\times r}\text{ is invertible}}&\,f_{alg}(\bQ)\label{eq_opti_aux_for_x_star}\\&\,\|\bQ-\bP\|_{\Fn}\le 5\delta\frac{\sqrt{n}}{\sigma_{r}\left(\bm{U}^{*}\right)}.\nonumber
    \end{align}
    By Weyl's inequality, we obtain that for any feasible $\bQ$, $\sigma_{\min}(\bQ)\ge \sigma_{\min}(\bP) - 5\delta\sqrt{n}/\sigma_r(\bU^{*})$ $\ge 1/2$ given that $\delta\le\sigma_r(\bU^{*})/(80\sqrt{n})$. Therefore $f_{alg}(\bQ)$ is continuous over the feasible region \eqref{eq_opti_aux_for_x_star}. Notably this region is compact inside the manifold of invertible matrices, which implies the existence of a solution to \eqref{eq_opti_aux_for_x_star}. Then it suffices to show that this solution corresponds to the optimal alignment $\bQ^{*}$. Note that by definition
    \begin{equation*}
        n^{-1}\big\|\bU\bQ^*-\bU^{*}\big\|_{\Fn}^2\le f_{alg}(\bQ^{*})\le f_{alg}(\bP) \le 2\delta^2.
    \end{equation*}
    With $n^{-1}\|\bU\bQ^*-\bU^{*}\|_{\Fn}^2\ge n^{-1}\|\bU\bQ^*-\bU\bP\|_{\Fn}^2 - n^{-1}\|\bU\bP-\bU^{*}\|_{\Fn}^2$ and $n^{-1/2}\|\bU\bP-\bU^{*}\|_{\Fn}\le \delta$, there is
    \begin{align*}
        \|\bQ^*-\bP\|_{\Fn}\le &\, \sigma_r(\bU)^{-1}\big\|\bU\bQ^*-\bU\bP\big\|_{\Fn}\\\le &\,\sqrt{n}\sigma_r(\bU)^{-1}\sqrt{n^{-1}\|\bU\bQ^*-\bU^{*}\|_{\Fn}^2 + n^{-1}\|\bU\bP-\bU^{*}\|_{\Fn}^2}\\ \le &\,\sqrt{3n}\sigma_r(\bU)^{-1}\delta.
    \end{align*}
    Then it suffices to show that $\sigma_{r}(\bU^*)^{-1}\ge \sqrt{3}\sigma_r(\bU)^{-1}/5$. Again, by Weyl's inequality,
    \begin{equation*}
        \big|\sigma_{\min}(\bU\bP) - \sigma_{\min}(\bU)\big|\le \big\|\bU\bP - \bU\big\|_{\Fn}\le \delta\le \frac{1}{80}\sigma_r(\bU^{*}).
    \end{equation*} Therefore, we know $\tfrac{79}{80}\sigma_r(\bU)\le\sigma_{\min}(\bU\bP)\le \sigma_{r}(\bU)\sigma_{\max}(\bP)\le \tfrac{3}{2}\sigma_r(\bU)$,
    which gives $\sigma_{r}(\bU)\ge \tfrac{79}{120}\sigma_r(\bU) \ge 0.66\sigma_r(\bU)\ge \sqrt{3}\sigma_r(\bU)/5$.
    We then conclude that
    \begin{equation*}\|\bQ^* - \bP\|_{\Fn}\le \sqrt{3}\sigma_r(\bU)^{-1}\delta\le 5 \delta \{\sigma_r(\bU^*/\sqrt{n})\}^{-1},\end{equation*}
    which proves the first part. The second part follows similarly by noting that $2/3\le \sigma_{r}(\bP^{-\T})\le \sigma_1(\bP^{-\T})\le 3/2$.
\end{proof}

\subsection{Proof of Lemma~\ref{lemma_rect_auxiliary_bar}}\label{supp_sec_prove_lemma_rect_auxiliary_bar}
Throughout this proof, all concentration bounds are stated on finite events whose failure probabilities are summed by a union bound. Namely, one may replace $\delta$ by a sufficiently small constant multiple of $\delta$ in each Bernstein bound below.
For each $i\in[n]$ and $j\in[q]$, write $\xi_{ij}:=\ell_{\alpha_0}'(X_{ij}^*)=\nu_{\star}\{\sigma(X_{ij}^*)-Y_{ij}\}$ and $\omega_{ij}:=\ell_{\alpha_0}''(X_{ij}^*)=\nu_{\star}\sigma(X_{ij}^*)\{1-\sigma(X_{ij}^*)\}$. We then have $\EE\xi_{ij}=0$, $\EE\xi_{ij}^2 = \nu_\star\omega_{ij} \le \nu_\star\beta$, and $|\xi_{ij}|\le \nu_\star$.
By construction,
\begin{align}\label{eq_rect_bar_newton_u}
\sum_{j=1}^q \xi_{ij}\bV_{j,}^* + \sum_{j=1}^q \omega_{ij}(\bar\bE_{U,i}\bV_{j,}^{*\,\T})\bV_{j,}^* = \zero, \qquad i\in[n],\\\label{eq_rect_bar_newton_v}
\sum_{i=1}^n \xi_{ij}\bU_{i,}^* + \sum_{i=1}^n \omega_{ij}(\bar\bE_{V,j}\bU_{i,}^{*\,\T})\bU_{i,}^* = \zero, \qquad j\in[q].
\end{align}

We first bound $R_{\bar Z} = \big\|\bS_Z^2\nabla_{\bZ}h_{\alpha}^{\natural}(\bar\bU,\bar\bV)\big\|_{\twinfw}$. For the Bernoulli model, for each $(i,j)$, $\bar X_{ij}-X_{ij}^*
= \bar\bE_{U,i}\bV_{j,}^{*\,\T} + \bU_{i,}^*\bar\bE_{V,j}^{\T} + \bar\bE_{U,i}\bar\bE_{V,j}^{\T}$,
and a second-order Taylor expansion gives
\begin{equation*}
\ell'_{\alpha_0}(\bar X_{ij}) = \ell'_{\alpha_0}(X_{ij}^*) + \ell''_{\alpha_0}(X_{ij}^*)(\bar X_{ij}-X_{ij}^*) + \bar R_{ij}\; \text{, where}\;
|\bar R_{ij}| \le C|\bar X_{ij}-X_{ij}^*|^2.
\end{equation*}
Substituting this into the row gradients and using
\eqref{eq_rect_bar_newton_u}--\eqref{eq_rect_bar_newton_v}, we obtain
\begin{align}
q^{-1}\nabla_{\bU_{i,}}\cL(\bar\bU,\bar\bV) ={}&\, q^{-1}\sum_{j=1}^q \xi_{ij}\bar\bE_{V,j} + q^{-1}\sum_{j=1}^q \omega_{ij}(\bU_{i,}^*\bar\bE_{V,j}^{\T})\bV_{j,}^* +\bar\bR_{U,i}, \label{eq_rect_bar_grad_u} \\
n^{-1}\nabla_{\bV_{j,}}\cL(\bar\bU,\bar\bV) ={}&\, n^{-1}\sum_{i=1}^n \xi_{ij}\bar\bE_{U,i} + n^{-1}\sum_{i=1}^n \omega_{ij}(\bV_{j,}^*\bar\bE_{U,i}^{\T})\bU_{i,}^* +\bar\bR_{V,j}, \label{eq_rect_bar_grad_v}
\end{align}
where, by telescoping, the remainders satisfy
\begin{equation}\label{eq_rect_bar_grad_remainder} \begin{aligned}
\max_{i\in[n]}\|\bar\bR_{U,i}\| \vee \max_{j\in[q]}\|\bar\bR_{V,j}\| \le C\beta\sqrt{\kappa\sigma_{\min}}\, \|(\bar\bE_U,\bar\bE_V)\|_{\Fnw}\, \|(\bar\bE_U,\bar\bE_V)\|_{\twinfw}+ C\beta\omega_*\, \|(\bar\bE_U,\bar\bE_V)\|_{\Fnw}^2.
\end{aligned}
\end{equation}
We next invoke the explicit form of $(\bar\bE_U,\bar\bE_V)$ to control \eqref{eq_rect_bar_grad_u} and \eqref{eq_rect_bar_grad_v}. Let $\bH_{V,j}:=\sum_{i=1}^n \omega_{ij}(\bU_{i,}^*)^\T\bU_{i,}^*$ and $\bH_{U,i}:=\sum_{j=1}^q \omega_{ij}(\bV_{j,}^*)^\T\bV_{j,}^*$ for $j\in[q]$ and $i\in[n]$.
Then the error can be written as $\bar\bE_{V,j} = - \bH_{V,j}^{-1}\sum_{i=1}^n \xi_{ij}(\bU_{i,}^*)^\T$ and $\bar\bE_{U,i} = -\bH_{U,i}^{-1}\sum_{j=1}^q \xi_{ij}(\bV_{j,}^*)^\T$.
By Lemma~\ref{lemma_rect_block_hessian}, we know that, uniformly over $i\in[n]$ and $j\in[q]$, it holds that
\begin{equation*}
\|\bH_{V,j}^{-1}\|\le \frac{1}{\alpha n\sigma_{\min}},\quad\text{ and }\quad
\|\bH_{U,i}^{-1}\|\le \frac{1}{\alpha q\sigma_{\min}}.
\end{equation*}
Consider the first term on the right side of \eqref{eq_rect_bar_grad_u}. For each $i\in[n]$, there is
\begin{equation}\begin{aligned}
q^{-1}\big\|(\xi\bar\bE_V)_{i,\cdot}\big\| =&\,q^{-1}\Big\|\sum_{j=1}^q \xi_{ij}\bH_{V,j}^{-1}\sum_{i'=1}^n \xi_{i'j}(\bU_{i'}^*)^\T\Big\|
\\ \le &\, \Big\|\sum_{j=1}^q \xi_{ij}^2\bH_{V,j}^{-1}(\bU_i^*)^\T\Big\| + \Big||\sum_{i'\neq i}\sum_{j=1}^q \xi_{ij}\xi_{i'j}\bH_{V,j}^{-1}(\bU_{i'}^*)^\T\Big\|\\ \le &\,\frac{\|\bU_i^*\|}{\alpha n\sigma_{\min}}\cdot \frac{1}{q}\sum_{j=1}^q \xi_{ij}^2 + \frac{1}{q}\Big\|\sum_{i'\neq i}\sum_{j=1}^q \xi_{ij}\xi_{i'j}\bH_{V,j}^{-1}(\bU_{i'}^*)^\T\Big\|.
\end{aligned}\label{eq_rect_bar_XiEV_bound1}
\end{equation}
Hence, by Bernstein's inequality applied to the bounded variables
$\xi_{ij}^2-\EE\xi_{ij}^2$, uniformly over $i\in[n]$, the diagonal part in \eqref{eq_rect_bar_XiEV_bound1} is bounded by
\begin{equation}\label{eq_rect_bar_XiEV_diag_new}
    \frac{\|\bU_i^*\|}{\alpha n\sigma_{\min}}\cdot \frac{1}{q}\sum_{j=1}^q \xi_{ij}^2 \le C\frac{\nu_\star\beta\,\omega_*}{\alpha n\sigma_{\min}} + C\frac{\nu_\star^2L_\star\,\omega_*}{\alpha nq\sigma_{\min}} .
\end{equation}
Here, we use that, if $\bA_1,\ldots,\bA_m\in\RR^r$ are fixed, or measurable with respect to a $\sigma$-field independent of $\{\xi_{ij}\}_{j=1}^m$, then with probability at least $1-Ce^{-cL}$,
\begin{equation}\label{eq_aux_cond_bern_template_new}
\big\| \sum_{j=1}^m \xi_{ij}\bA_j \big\| \le C\Big\{ \sqrt{\nu_\star\beta L}\, \big(\sum_{j=1}^m\|\bA_j\|^2\big)^{1/2} + \nu_\star L\,\max_{j\in[m]}\|\bA_j\| \Big\}.
\end{equation}
For the off-diagonal part, let $T_{i,2} := q^{-1}\sum_{j=1}^q \xi_{ij}\bA_{ij}^{(-i)}$ for $\bA_{ij}^{(-i)} := \bH_{V,j}^{-1}\sum_{i'\neq i}\xi_{i'j}(\bU_{i'}^*)^\T $.
Applying \eqref{eq_aux_cond_bern_template_new} conditionally on $\cF_{-i}$, and then applying Bernstein again to the $\cF_{-i}$-measurable quantities $\{\bA_{ij}^{(-i)}\}_{j=1}^q$, yields, uniformly over $i\in[n]$,
\begin{equation*}
    \|T_{i,2}\| \le C\Big\{ \nu_\star L_\star\sqrt{\frac{\beta}{\alpha nq\sigma_{\min}}} + \frac{\nu_\star^2L_\star\,\omega_*}{\alpha nq\sigma_{\min}}\Big\}.
\end{equation*}
Together with the bound for the residual, we conclude that
\begin{equation}
\label{eq_rect_bar_XiEV_bound2}
q^{-1}\|\xi\bar\bE_V\|_{2\to\infty} \le  C\Big\{\nu_\star L_\star
    \sqrt{\frac{\beta}{\alpha nq\sigma_{\min}}} + \frac{\nu_\star\beta\,\omega_*}{\alpha n\sigma_{\min}} + \frac{\nu_\star^2L_\star\,\omega_*}{\alpha nq\sigma_{\min}}\Big\}.
\end{equation}
By symmetry, one can obtain as well
\begin{equation}
\label{eq_rect_bar_XtEU_bound2} n^{-1}\|\xi^\T\bar\bE_U\|_{2\to\infty} \le C\Big\{\nu_\star L_\star
    \sqrt{\frac{\beta}{\alpha nq\sigma_{\min}}} + \frac{\nu_\star\beta\,\omega_*}{\alpha q\sigma_{\min}} + \frac{\nu_\star^2L_\star\,\omega_*}{\alpha nq\sigma_{\min}}\Big\}.
\end{equation}

Next, consider the second term on the right side of \eqref{eq_rect_bar_grad_u}, which can be bounded following the same strategy. Specifically, let $\bZ_{ijk}:=q^{-1}\omega_{ij}\xi_{kj}\{\bU_{i,}^*\bH_{V,j}^{-1}(\bU_{k,}^*)^{\T}\}\bV_{j,}^*$, which are independent and centered across $k,i\in[n]$ and $j\in[q]$, as the only randomness is from $\xi_{kj}$. Then the second term of \eqref{eq_rect_bar_grad_u} can be written as $-\sum_{j=1}^q\sum_{k=1}^n\bZ_{ijk}$. For any unit vector $\ba\in\mathbb S^{r-1}$, with $\omega_{ij}\le \beta$ and $\|\bH_{V,j}^{-1}\|\le \{\alpha n\sigma_{\min}\}^{-1}$, we know 
\begin{align*}
\sum_{j\in[q],k\in[n]}\EE(\ba^\T\bZ_{ijk})^2 &=\nu_{\star}q^{-2}\sum_{j=1}^q\omega_{ij}^2(\ba^\T\bV_{j,}^*)^2 \bU_{i,}^*\bH_{V,j}^{-1}\Big\{\sum_{k=1}^n\omega_{kj}(\bU_{k,}^*)^\T\bU_{k,}^*\Big\}\bH_{V,j}^{-1}(\bU_{i,}^*)^\T\\
&\le \frac{\nu_{\star}\beta^2\|\bU_{i,}^*\|^2}{\alpha n q^2\sigma_{\min}}\sum_{j=1}^q(\ba^\T\bV_{j,}^*)^2 \le C\frac{\beta^2\omega_*^2\kappa}{\alpha nq}.
\end{align*}
Moreover, by definition, one can check that $\max_{j\in[q],i,k\in[n]}\|\bZ_{ijk}\|\le C\nu_{\star}{\beta\omega_*^3}/{(\alpha nq\sigma_{\min})}$. Thus vector Bernstein, a $1/2$-net argument, and a union bound over $i\in[n]$ give
\begin{equation}\label{eq_rect_bar_cross_bound_u}
q^{-1}\max_{i\in[n]} \Big\|\sum_{j=1}^q \omega_{ij}(\bU_{i,}^*\bar\bE_{V,j}^{\T})\bV_{j,}^*\Big\| = \max_{i\in[n]}\big\|\sum_{j=1}^q\sum_{k=1}^n\bZ_{ijk}\big\|\le C\beta\omega_* \sqrt{\frac{\nu_\star\kappa L_\star}{\alpha nq}} + C\frac{\nu_\star\beta\omega_*^3L_\star}{\alpha nq\sigma_{\min}}.\end{equation}
By symmetry, $n^{-1}\max_{j\in[q]} \big\|\sum_{i=1}^n \omega_{ij}(\bV_{j,}^*\bar\bE_{U,i}^{\T})\bU_{i,}^*\big\|$ can be bounded similarly.

For the penalty term in $R_{\bar Z} = \big\|\bS_Z^2\nabla_{\bZ}h_{\alpha}^{\natural}(\bar\bU,\bar\bV)\big\|_{\twinfw}$, recall that $\bM(\bU,\bV) = n^{-1}(\bU-\bU^*)^\T\bU-q^{-1}\bV^\T(\bV-\bV^*)$. Hence, at $(\bar\bU,\bar\bV)$, we telescope
\begin{equation}\label{eq_rect_bar_M_expand}
\bM(\bar\bU,\bar\bV) = n^{-1}\bar\bE_U^\T\bU^* - q^{-1}(\bV^*)^\T\bar\bE_V + n^{-1}\bar\bE_U^\T\bar\bE_U - q^{-1}\bar\bE_V^\T\bar\bE_V.
\end{equation}
We now invoke the expansion of $(\bar\bE_U,\bar\bE_V)$. Let $\bM_U = -n^{-1}\sum_{i=1}^n \bH_{U,i}^{-1}\big\{\sum_{j=1}^q \xi_{ij}(\bV_{j,}^*)^\T\big\}\bU_{i,}^*$, $\bM_V = -q^{-1}\sum_{j=1}^q
(\bV_{j,}^*)^\T \bH_{V,j}^{-1} \big\{\sum_{i=1}^n \xi_{ij}(\bU_{i,}^*)^\T\big\}$, and $\bM_{\rm quad} = n^{-1}\bar\bE_U^\T\bar\bE_U-q^{-1}\bar\bE_V^\T\bar\bE_V$. Then  \eqref{eq_rect_bar_M_expand} becomes
\begin{equation}\label{eq_rect_bar_M_decomp}
\bM(\bar\bU,\bar\bV) = \bM_U-\bM_V+\bM_{\rm quad},
\end{equation}
Note that $\bM_U$ is a sum of independent mean zero random matrices over $i\in[n]$, while
$\bM_V$ is a sum of independent mean-zero random matrices over $j\in[q]$.
For each $i$, letting $\bM_{U,i} =  - n^{-1}\bH_{U,i}^{-1} \big\{\sum_{j=1}^q \xi_{ij}(\bV_{j,}^*)^\T\big\}\bU_{i,}^*$, we have $\EE(\bM_{U,i})=\zero$ and
\begin{align*}
\Big\|\EE(\bM_{U,i}\bM_{U,i}^\T)\Big\| &\le \frac{\|\bU_{i,}^*\|^2}{n^2} \Big\| \bH_{U,i}^{-1} \EE\Big[ \Big\{\sum_{j=1}^q \xi_{ij}(\bV_{j,}^*)^\T\Big\} \times \Big\{\sum_{j=1}^q \xi_{ij}\bV_{j,}^*\Big\} \Big] \bH_{U,i}^{-1} \Big\| \\
&= \frac{\nu_{\star}\|\bU_{i,}^*\|^2}{n^2} \big\| \bH_{U,i}^{-1} \Big\{\sum_{j=1}^q \omega_{ij}(\bV_{j,}^*)^\T\bV_{j,}^*\Big\} \bH_{U,i}^{-1} \big\| \\
&= \frac{\nu_{\star}\|\bU_{i,}^*\|^2}{n^2}\|\bH_{U,i}^{-1}\| \le \frac{\nu_{\star}\|\bU_{i,}^*\|^2}{\alpha n^2 q\sigma_{\min}}.
\end{align*}
Here, the expectation is taken with respect to $\bY$, and note that $w_{ij}$ is independent of $\bY$, and so is $\bH_{U,i}$. Thus, we arrive at
\begin{equation*}
\Big\|\sum_{i=1}^n \EE(\bM_{U,i}\bM_{U,i}^\T)\Big\| \le \frac{\nu_{\star}\|\bU^*\|_{\Fn}^2}{\alpha n^2 q\sigma_{\min}}\le C \frac{\nu_{\star}\kappa r}{\alpha nq}.
\end{equation*}
Moreover, since $|\xi_{ij}|\le 1$ for all $(i,j)$, we know $\|\sum_{j=1}^q \xi_{ij}(\bV_{j,}^*)^\T\| \le \sum_{j=1}^q \|\bV_{j,}^*\| \le q\|\bV^*\|_{2\to\infty}$, which yields
\begin{equation*}
\|\bM_{U,i}\| \le \frac{1}{n}\|\bH_{U,i}^{-1}\| \Big\|\sum_{j=1}^q \xi_{ij}(\bV_{j,}^*)^\T\Big\| \|\bU_{i,}^*\| \le \frac{\nu_{\star}\|\bU_{i,}^*\|\,\|\bV^*\|_{2\to\infty}}{\alpha n\sigma_{\min}}.
\end{equation*}
Thus matrix Bernstein yields
\begin{equation}\label{eq_rect_bar_M_linear_bound}
\|\bM_U\|\le C \sqrt{\frac{\nu_{\star}\kappa L_{\star}}{\alpha nq}} + C\frac{\nu_{\star}\omega_*\|\bV^*\|_{2\to\infty}L_{\star}}{\alpha n\sigma_{\min}} ,
\end{equation}
with probability at least $1-\delta$. The second term is of lower order under the present
scaling assumptions. A completely symmetric argument gives the same bound for $\|\bM_V\|$.

On the other hand, it is easy to see that $\|\bM_{\rm quad}\| \le \|(\bar\bE_U,\bar\bE_V)\|_{\Fnw}^2$
Combining this with \eqref{eq_rect_bar_M_decomp},  and \eqref{eq_rect_bar_M_linear_bound}, we obtain
\begin{equation}\label{eq_rect_bar_M_bound_refined}
\|\bM(\bar\bU,\bar\bV)\|\le C\sqrt{\frac{\nu_\star\beta\kappa L_\star}{\alpha nq}} + C \frac{\nu_\star\omega_*^2L_\star}{\alpha m_\star\sigma_{\min}} + \|(\bar\bE_U,\bar\bE_V)\|_{\Fnw}^2 .
\end{equation}
For $p_\alpha^\natural(\bU,\bV) = \alpha nq\|\bM^*(\bU,\bV)\|_{\Fn}^2/4$, we know  $q^{-1}\nabla_{\bU}p_\alpha^\natural(\bar\bU,\bar\bV) = \alpha\{\bar\bU\bM(\bar\bU,\bar\bV)^\T + \bar\bE_U\bM(\bar\bU,\bar\bV)\}/2$ and $n^{-1}\nabla_{\bV}p_\alpha^\natural(\bar\bU,\bar\bV) = -\alpha\{\bar\bE_V\bM(\bar\bU,\bar\bV)^\T + \bar\bV\bM(\bar\bU,\bar\bV)\}/2$. We then conclude that
\begin{align}
\big\|\bS_Z^2\nabla_{\bZ}p_{\alpha}^{\natural}(\bar\bU,\bar\bV)\big\|_{\twinfw}&\le
\alpha\Big(\omega_*+\|(\bar\bE_U,\bar\bE_V)\|_{\twinfw} \Big)\|\bM(\bar\bU,\bar\bV)\|\nonumber\\
&\le C\alpha\Big(
\omega_*+\|(\bar\bE_U,\bar\bE_V)\|_{\twinfw} \Big)
\Bigg[\nu_{\star}\sqrt{\frac{\beta\kappa L_{\star}}{\alpha nq}}+
\|(\bar\bE_U,\bar\bE_V)\|_{\Fnw}^2
\Bigg].
\label{eq_rect_bar_penalty_bound_refined}
\end{align}
Combining \eqref{eq_rect_bar_grad_u}--\eqref{eq_rect_bar_cross_bound_u},
\eqref{eq_rect_bar_grad_remainder}, and \eqref{eq_rect_bar_penalty_bound_refined}, we obtain
 \begin{equation}\label{eq_lemma_rect_auxiliary_bar_score}
    \begin{aligned}
    R_{\bar Z}\le {}\,&C\nu_{\star}L_{\star}\sqrt{\frac{\beta }{\alpha n q\,\sigma_{\min}}} + C\frac{\nu_{\star}\beta\,\omega_*}{\alpha (n\wedge q)\sigma_{\min}} + C\frac{\nu_{\star}^2L_{\star}\omega_*}{\alpha\sigma_{\min}{nq}} + C\beta\omega_*\sqrt{\frac{\nu_{\star}\kappa L_{\star}}{\alpha nq}} +\frac{\nu_{\star}\beta\omega_*^3 L_{\star}}{\alpha nq\sigma_{\min}}\\
    &\quad+C\alpha\Big(\omega_*+\|(\bar\bE_U,\bar\bE_V)\|_{\twinfw}\Big)\Bigg[\sqrt{\frac{\nu_{\star}\kappa L_{\star}}{\alpha nq}} + \frac{\nu_{\star}\omega_*^2L_{\star}}{\alpha(n\wedge q)\sigma_{\min}} +\|(\bar\bE_U,\bar\bE_V)\|_{\Fnw}^2\Bigg]\\
    &\quad+ C (\alpha+\beta)\sqrt{\kappa\sigma_{\min}}\,\|(\bar\bE_U,\bar\bE_V)\|_{\Fnw}\,\|(\bar\bE_U,\bar\bE_V)\|_{\twinfw}\\
    &\quad + C(\alpha+\beta)\omega_*\,\|(\bar\bE_U,\bar\bE_V)\|_{\Fnw}^2.\end{aligned}
    \end{equation}

We next bound $(\bar\bE_U,\bar\bE_V)$. Recall that $\bar\bE_{U,i} = -\bH_{U,i}^{-1}\sum_{j=1}^q \xi_{ij}(\bV_{j,}^*)^\T$ and $\bar\bE_{V,j} = -\bH_{V,j}^{-1}\sum_{i=1}^n \xi_{ij}(\bU_{i,}^*)^\T$. With $\|\bH_{U,i}^{-1}\|\le (\alpha q\sigma_{\min})^{-1}$ and $\|\bH_{V,j}^{-1}\|\le (\alpha n\sigma_{\min})^{-1}$, we have $\|\bar\bE_U\|_{2\to\infty} \le \|\xi\bV^*\|_{2\to\infty}/(\alpha q\sigma_{\min})$ and $\|\bar\bE_V\|_{2\to\infty} \le \|\xi^\T\bU^*\|_{2\to\infty}/(\alpha n\sigma_{\min})$.
Consequently, one has
\begin{equation*}
\|(\bar\bE_U,\bar\bE_V)\|_{\twinfw}
\le \frac{1}{\alpha\sigma_{\min}} \left\{ q^{-1}\|\xi\bV^*\|_{2\to\infty} \vee n^{-1}\|\xi^\T\bU^*\|_{2\to\infty} \right\}.
\end{equation*}
Invoking item~\ref{item_bern_5}, we get
\begin{equation}
\label{eq_rect_bar_twinf_fixed}
\|(\bar\bE_U,\bar\bE_V)\|_{\twinfw} \le \frac{C\Delta_{\infty}(n,q,\delta)}{\alpha\sigma_{\min}} \omega_*.
\end{equation}
Similarly, by summing the row-wise bounds, there is
\begin{align*}
n^{-1}\|\bar\bE_U\|_{\Fn}^2 &= n^{-1}\sum_{i=1}^n \|\bar\bE_{U,i}\|^2 \le \frac{1}{\alpha^2 q^2\sigma_{\min}^2}\, n^{-1}\|\xi\bV^*\|_{\Fn}^2,\\
q^{-1}\|\bar\bE_V\|_{\Fn}^2 &= q^{-1}\sum_{j=1}^q \|\bar\bE_{V,j}\|^2 \le \frac{1}{\alpha^2 n^2\sigma_{\min}^2}\, q^{-1}\|\xi^\T\bU^*\|_{\Fn}^2.
\end{align*}
Therefore, we arrive at
\begin{equation}\label{eq_rect_bar_F_direct}
\|(\bar\bE_U,\bar\bE_V)\|_{\Fnw} \le
C\frac{\tau_*}{\alpha\sigma_{\min}} \left\{ \frac{\nu_\star}{n\wedge q} + \frac{\nu_\star^2\log((n+q)/\delta)}{nq} \right\}^{1/2} \le C\frac{\Delta_2(n,q,\delta)}{\alpha\sigma_{\min}}\tau_* ,
\end{equation}
which together with \eqref{eq_rect_bar_twinf_fixed} gives
\begin{equation}
\frac{\|(\bar\bE_U,\bar\bE_V)\|_{\twinfw}}{\omega_*}\le \frac{\Delta_{\infty}(n,q,\delta)}{\alpha\sigma_{\min}}, \qquad
\frac{\|(\bar\bE_U,\bar\bE_V)\|_{\Fnw}}{\tau_*}\le \frac{\Delta_{2}(n,q,\delta)}{\alpha\sigma_{\min}}.
\end{equation}
Next, substituting \eqref{eq_rect_bar_twinf_fixed} and \eqref{eq_rect_bar_F_direct} into the last three lines of \eqref{eq_lemma_rect_auxiliary_bar_score}, and using $\|\bU^*\|_{\twinf}\vee \|\bV^*\|_{\twinf}\le\omega_*$, one can check that with the scaling conditions~\eqref{eq_bern_prop_scaling}, the right side of \eqref{eq_lemma_rect_auxiliary_bar_score} can be further bounded by $C\Delta_{\infty}(n,q,\delta)\omega_*$, which therefore proves \eqref{eq_lemma_rect_auxiliary_bar_score_simple}.

We now show the second part of Lemma~\ref{lemma_rect_auxiliary_bar}.
Define $\bDelta = (\bDelta_U^\T,\bDelta_V^\T)^\T$ with $(\bDelta_U,\bDelta_V):=(\tilde\bU-\bar\bU,\tilde\bV-\bar\bV)$, and let $(\bU(s),\bV(s))=(\bar\bU,\bar\bV)+s(\bDelta_U,\bDelta_V)$ for $s\in[0,1]$. Applying the integral mean value theorem to
$\mathrm{vec}\{\nabla_{\bZ}h_{\alpha}^{\natural}(\cdot)^\T\}$ along this segment yields
\begin{equation}\label{eq_rect_bar_mvt_refined}
\bS_z\widehat{\cH}\bS_z\,\mathrm{vec}\big((\bS_Z^{-1}\bDelta)^\T\big) = \mathrm{vec}\Big( \bS_Z\{\nabla_{\bZ}h_{\alpha}^{\natural}(\tilde\bU,\tilde\bV) - \nabla_{\bZ}h_{\alpha}^{\natural}(\bar\bU,\bar\bV)\}^\T \Big),
\end{equation} where $\widehat{\cH} := \int_0^1 \nabla_{\bz}^2 h_{\alpha}^{\natural}(\bU(s),\bV(s))\,ds$ and recall that we have defined 
\begin{equation*}
\bS_Z=\begin{pmatrix}q^{-1/2}\bI_n&\zero\\ \zero&n^{-1/2}\bI_q\end{pmatrix},\qquad
\bS_z=\begin{pmatrix}q^{-1/2}\bI_{nr}&\zero\\ \zero&n^{-1/2}\bI_{qr}\end{pmatrix}.
\end{equation*}
Compared with a similar expansion \eqref{eq_rect_mvt_optimizer}, the terms corresponding to $\Gamma_2$ and $\Gamma_3$ are absorbed into $\widehat{\cH}$. In particular, as in Step~1 of the proof of Theorem~\ref{thm_general_theory_UV_noise}, Lemma~\ref{lemma_general_convex_asym} gives
\begin{equation*}
\lambda_{\min}\!\left(\bS_z \int_0^1\{\nabla_{\bz}^2 h_{\alpha}^{\natural}(\bU(s),\bV(s))-\cG_e(\bU(s),\bV(s))\}\,ds \bS_z \right) \ge \frac{\alpha\sigma_{\min}}{2}.
\end{equation*}
For the Bernoulli model, $\tilde\cG(\bU,\bV)=\bP^*-\bY$ is constant, hence $\int_0^1\{\cG_e(\bU(s),\bV(s))-\EE\cG_e(\bU(s),\bV(s))\}\,ds$ $= \tilde\cG_{\rm ave}$, with $\|\bS_z\tilde\cG_{\rm ave}\bS_z\|\le C\Delta_2(n,q,\delta)$, while $\int_0^1\EE\cG_e(\bU(s),\bV(s))\,ds = \bar\cG_{\rm ave}$ satisfies
\begin{equation*}
\|\bS_z\bar\cG_{\rm ave}\bS_z\| \le C\sqrt{\kappa\sigma_{\min}} \Big( \|(\bar\bE_U,\bar\bE_V)\|_{\Fnw} + \|(\bDelta_U,\bDelta_V)\|_{\Fnw} \Big).
\end{equation*}
Therefore, under the same small constant conditions as in
Theorem~\ref{thm_general_theory_UV_noise}, both terms are absorbed, and hence
\begin{equation}\label{eq_rect_bar_Hhat_lb_refined}
\lambda_{\min}(\bS_z\widehat{\cH}\bS_z)\ge \frac{\alpha\sigma_{\min}}{4}.
\end{equation}
Since $(\tilde\bU,\tilde\bV)$ minimizes $(\bU,\bV)\mapsto \|\bS_Z^2\nabla_{\bZ}h_{\alpha}^{\natural}(\bU,\bV)\|_{\twinfw}^2$ over $\cDde$, we have $\big\|\bS_Z^2\nabla_{\bZ}h_{\alpha}^{\natural}(\tilde\bU,\tilde\bV)\big\|_{\twinfw}$ $\le R_{\bar Z}$. We therefore arrive at
\begin{equation}\label{eq_rect_bar_scoreF_refined}
\begin{aligned}
\big\|\bS_Z\{\nabla_{\bZ}h_{\alpha}^{\natural}(\tilde\bU,\tilde\bV) - \nabla_{\bZ}h_{\alpha}^{\natural}(\bar\bU,\bar\bV)\}\big\|_{\Fn} &\le\big\|\bS_Z\nabla_{\bZ}h_{\alpha}^{\natural}(\tilde\bU,\tilde\bV)\big\|_{\Fn} + \big\|\bS_Z\nabla_{\bZ}h_{\alpha}^{\natural}(\bar\bU,\bar\bV)\big\|_{\Fn} \\
&\le 2\sqrt{2nq}\,R_{\bar Z}.
\end{aligned}
\end{equation}
Combining \eqref{eq_lemma_rect_auxiliary_bar_score_simple} \eqref{eq_lemma_rect_auxiliary_bar_score}, \eqref{eq_rect_bar_mvt_refined}, \eqref{eq_rect_bar_Hhat_lb_refined}, and
\eqref{eq_rect_bar_scoreF_refined}, we obtain \eqref{eq_lemma_rect_auxiliary_bar_close}.

Finally, because $\xi=\nu_{\star}(\bP^*-\bY)$, the two inequalities \eqref{eq_rect_bar_XiEV_bound2} and \eqref{eq_rect_bar_XtEU_bound2}, after replacing $\log(n/\delta)$ and $\log(q/\delta)$ by $\log((n+q)/\delta)$, using $(n\wedge q)^{-1}\ge n^{-1}\vee q^{-1}$ and invoking
scaling condition~\eqref{eq_bern_prop_scaling}, give
\eqref{eq_lemma_rect_auxiliary_bar_noise}.


\subsection{Proof of Lemma~\ref{lem_tentative_loo_gradient_mc}}\label{supp_sec_prove_lem_tentative_loo_gradient_mc}
\begin{proof}
Write $L_{\star} = r+\log\{(n+q)T_{\star}/\delta\}$. The Bernstein inequalities below are applied with failure probability of order $\delta/T_\star$ for each step $t\in\{1,\dots,T_{\star}\}$ and then union bounded over rows, columns. Let $\cL^{-i}(\bX)$ and $\cL^{-\ell}(\bX)$ denote the same empirical loss as $\cL(\bX)$, but constructed from $\bY^{-i}$ and $\bY^{-\ell}$, respectively, where the leave-one-out data matrices are defined in Theorem~\ref{Thm_example2}. Denote the corresponding gradient matrices by
\begin{equation}
\cG^{-i}(\bU,\bV):=\nabla_{\bX}\cL^{-i}(\bU\bV^\T),\quad \cG^{-\ell}(\bU,\bV):=\nabla_{\bX}\cL^{-\ell}(\bU\bV^\T).
\end{equation}

Starting from $(\bU^{0,-i},\bV^{0,-i})$, define the row wise leave-one-out gradient iterates by
\begin{equation*}
\bU^{t+1,-i} = \bU^{t,-i}-\frac{\eta}{q}\nabla_{\bU}\cL^{-i}(\bU^{t,-i}\bV^{t,-i\T}), \quad \bV^{t+1,-i} = \bV^{t,-i}-\frac{\eta}{n}\nabla_{\bV}\cL^{-i}(\bU^{t,-i}\bV^{t,-i\T}).
\end{equation*}
Starting from $(\bU^{0,-\ell},\bV^{0,-\ell})$, define the column wise leave-one-out gradient iterates similarly by
\begin{equation*}
\bU^{t+1,-\ell} = \bU^{t,-\ell}-\frac{\eta}{q}\nabla_{\bU}\cL^{-\ell}(\bU^{t,-\ell}\bV^{t,-\ell\T}), \qquad \bV^{t+1,-\ell} = \bV^{t,-\ell}-\frac{\eta}{n}\nabla_{\bV}\cL^{-\ell}(\bU^{t,-\ell}\bV^{t,-\ell\T}).
\end{equation*}
By construction, $(\bU^{0,-i},\bV^{0,-i})$ is independent of $\{\bY_{i\ell}\}_{\ell\in[q]}$, and the data matrix $\bY^{-i}$ has deterministic $i$th row equal to $\bP^*_{i,\cdot}$. Hence the whole trajectory $\{(\bU^{t,-i},\bV^{t,-i})\}_{t\ge 0}$ is independent of $\{\bY_{i\ell}\}_{\ell\in[q]}$.
Likewise, $\{(\bU^{t,-\ell},\bV^{t,-\ell})\}_{t\ge 0}$ is independent of $\{\bY_{k\ell}\}_{k\in[n]}$.

Next, recall that at Step~3 in the proof of Theorem~\ref{thm_general_theory_UV_noise}, we have defined $\bG_t^{\dagger}$ to be an optimal alignment matrix for the main iterate $(\bU^t,\bV^t)$ relative to $(\tilde\bU,\tilde\bV)$ with $\tilde\bU^t:=\bU^t\bG_t^{\dagger}$ and $\tilde\bV^t:=\bV^t(\bG_t^{\dagger})^{-\T}$. Analogously, for each $i\in[n]$, let $\bG_t^{-i}$ be an optimal alignment matrix for $(\bU^{t,-i},\bV^{t,-i})$ relative to $(\tilde\bU,\tilde\bV)$, i.e., $\bG_t^{-i}\in\argmin_{\bG\in GL(r)} \|(\bU^{t,-i}\bG-\tilde\bU,\bV^{t,-i}\bG^{-\T}-\tilde\bV)\|_{\Fnw}$, and define
\begin{equation*}
\tilde\bU^{t,-i}:=\bU^{t,-i}\bG_t^{-i}, \qquad \tilde\bV^{t,-i}:=\bV^{t,-i}(\bG_t^{-i})^{-\T}.
\end{equation*}
For each $\ell\in[q]$, let $\bG_t^{-\ell}$ be an optimal alignment matrix for $(\bU^{t,-\ell},\bV^{t,-\ell})$ relative to $(\tilde\bU,\tilde\bV)$, and define
\begin{equation*}
\tilde\bU^{t,-\ell}:=\bU^{t,-\ell}\bG_t^{-\ell}, \qquad \tilde\bV^{t,-\ell}:=\bV^{t,-\ell}(\bG_t^{-\ell})^{-\T}.
\end{equation*}
Write the error as
\begin{equation*}
\tilde\bE_U^{t,-i}:=\tilde\bU^{t,-i}-\bU^*, \;
\tilde\bE_V^{t,-i}:=\tilde\bV^{t,-i}-\bV^*, \;
\tilde\bE_U^{t,-\ell}:=\tilde\bU^{t,-\ell}-\bU^*, \;
\tilde\bE_V^{t,-\ell}:=\tilde\bV^{t,-\ell}-\bV^*.
\end{equation*}
Let $\bLambda_t^{\dagger}:=(\bG_t^{\dagger})^\T\bG_t^{\dagger}$, $\bLambda_t^{-i}:=(\bG_t^{-i})^\T\bG_t^{-i}$, and $\bLambda_t^{-\ell}:=(\bG_t^{-\ell})^\T\bG_t^{-\ell}$. 
By the same identity as in Step~3 of the proof of Theorem~\ref{thm_general_theory_uv}, the balancing penalty has vanishing gradient at each aligned iterate.  Therefore, the row-wise leave-one-out sequence admits
\begin{align*}
\bar\bU^{t+1,-i}:=&\, \bU^{t+1,-i}\bG_t^{-i}\\ =&\,
\tilde\bU^{t,-i} -\frac{\eta}{q}\nabla_{\bU}h_\alpha^{-i,\natural}(\tilde\bU^{t,-i},\tilde\bV^{t,-i}) -\frac{\eta}{q}\nabla_{\bU}h_\alpha^{-i,\natural}(\tilde\bU^{t,-i},\tilde\bV^{t,-i})(\bLambda_t^{-i}-\bI_r),\\
\bar\bV^{t+1,-i}:=&\,\bV^{t+1,-i}(\bG_t^{-i})^{-\T}\\ =&\, \tilde\bV^{t,-i}
-\frac{\eta}{n}\nabla_{\bV}h_\alpha^{-i,\natural}(\tilde\bU^{t,-i},\tilde\bV^{t,-i}) -\frac{\eta}{n}\nabla_{\bV}h_\alpha^{-i,\natural}(\tilde\bU^{t,-i},\tilde\bV^{t,-i})\big\{(\bLambda_t^{-i})^{-1}-\bI_r\big\},
\end{align*}
where we write $h_\alpha^{-i,\natural}(\bU,\bV):=\cL^{-i}(\bU\bV^\T)+p_\alpha^\natural(\bU,\bV)$.
Similarly, for $h_\alpha^{-\ell,\natural}(\bU,\bV):=\cL^{-\ell}(\bU\bV^\T)+p_\alpha^\natural(\bU,\bV)$, the column-wise leave-one-out sequence admits
\begin{align*}
\bar\bU^{t+1,-\ell}:= &\, \bU^{t+1,-\ell}\bG_t^{-\ell} \\=&\,
\tilde\bU^{t,-\ell} -\frac{\eta}{q}\nabla_{\bU}h_\alpha^{-\ell,\natural}(\tilde\bU^{t,-\ell},\tilde\bV^{t,-\ell}) -\frac{\eta}{q}\nabla_{\bU}h_\alpha^{-\ell,\natural}(\tilde\bU^{t,-\ell},\tilde\bV^{t,-\ell})(\bLambda_t^{-\ell}-\bI_r), \\
\bar\bV^{t+1,-\ell} :=&\, \bV^{t+1,-\ell}(\bG_t^{-\ell})^{-\T}\\
=&\, \tilde\bV^{t,-\ell} -\frac{\eta}{n}\nabla_{\bV}h_\alpha^{-\ell,\natural}(\tilde\bU^{t,-\ell},\tilde\bV^{t,-\ell}) -\frac{\eta}{n}\nabla_{\bV}h_\alpha^{-\ell,\natural}(\tilde\bU^{t,-\ell},\tilde\bV^{t,-\ell})\big\{(\bLambda_t^{-\ell})^{-1}-\bI_r\big\}.
\end{align*}
Recall that the full data gradient update can also be written as
\begin{align*}\bar\bU^{t+1} =&\,
\tilde\bU^t -\frac{\eta}{q}\nabla_{\bU}h_\alpha^\natural(\tilde\bU^t,\tilde\bV^t) -\frac{\eta}{q}\nabla_{\bU}h_\alpha^\natural(\tilde\bU^t,\tilde\bV^t)(\bLambda_t^{\dagger}-\bI_r),\\
\bar\bV^{t+1}= &\, \tilde\bV^t-\frac{\eta}{n}\nabla_{\bV}h_\alpha^\natural(\tilde\bU^t,\tilde\bV^t)-\frac{\eta}{n}\nabla_{\bV}h_\alpha^\natural(\tilde\bU^t,\tilde\bV^t)\big\{(\bLambda_t^{\dagger})^{-1}-\bI_r\big\}.\end{align*}

We focus on establishing the error bounds for the row-wise leave-one-out sequence. The column-wise argument can be similarly obtained by replacing $(n,\bU)$ with $(q,\bV)$. Let $\bDelta_U^{t,-i}:=\tilde\bU^{t,-i}-\tilde\bU^t$ and $\bDelta_V^{t,-i}:=\tilde\bV^{t,-i}-\tilde\bV^t$. These quantities would be much smaller than the error between the iterates and LOO iterates to the true $(\bU^*,\bV^*)$. Define
\begin{align*}
& a_t^{\mathrm{row}} := \max_{i\in[n]} \Big\{ n^{-1/2}\|\tilde\bU^{t,-i}-\tilde\bU^t\|_{\Fn} \vee q^{-1/2}\|\tilde\bV^{t,-i}-\tilde\bV^t\|_{\Fn} \Big\},\\
& a_t^{\mathrm{col}} := \max_{\ell\in[q]} \Big\{ n^{-1/2}\|\tilde\bU^{t,-\ell}-\tilde\bU^t\|_{\Fn} \vee q^{-1/2}\|\tilde\bV^{t,-\ell}-\tilde\bV^t\|_{\Fn} \Big\},\\
& b_t := \max_{i\in[n]}\|\tilde\bU^{t,-i}_{i,\cdot}-\bU^*_{i,\cdot}\| \vee \max_{\ell\in[q]}\|\tilde\bV^{t,-\ell}_{\ell,\cdot}-\bV^*_{\ell,\cdot}\|,\\
&c_t := \max_{i\in[n]} \Big( \|\bG_t^{-i}-\bR^0\| \vee \|(\bG_t^{-i})^{-\T}-\bR^0\| \Big) \vee \max_{\ell\in[q]} \Big( \|\bG_t^{-\ell}-\bR^0\| \vee \|(\bG_t^{-\ell})^{-\T}-\bR^0\| \Big).
\end{align*}
In addition, for the full data gradient descent sequence, we write
\begin{align*}
    d_t:=\| (\tilde\bU^t - \tilde\bU,&\,\tilde\bV^t - \tilde\bV)\|_{\Fnw}, \qquad r_t:=\| (\tilde\bU^t - \tilde\bU, \tilde\bV^t - \tilde\bV)\|_{\twinfw},\\&\,
    \vartheta_t := \|\bLambda_t^{\dagger}-\bI_r\| \vee \|(\bLambda_t^{\dagger})^{-1}-\bI_r\|.
\end{align*}
One can show by induction that, for all $0\le s\le t$,
\begin{align}
 a_t^{\rm row} &\le \rho^t a_0^{\rm row} +\frac{C}{\alpha\sigma_{\min}} \left\{ \tau_*\sqrt{\frac{\nu_{\star}\beta L_\star}{nq}} + \nu_{\star}\omega_*\frac{L_{\star}}{q\sqrt{n}} + \omega_*\sqrt{\nu_{\star}\frac{q\beta+\nu_{\star}L_\star}{n^2q}} \right\}, \label{eq_loo_pf_ind_row}\\
 a_t^{\rm col} &\le \rho^t a_0^{\rm col} +\frac{C}{\alpha\sigma_{\min}} \left\{ \tau_*\sqrt{\frac{\nu_{\star}\beta L_\star}{nq}} + \nu_{\star}\omega_*\frac{L_{\star}}{n\sqrt{q}} +\omega_*\sqrt{\nu_{\star}\frac{n\beta+L_\star}{nq^2}} \right\}, \label{eq_loo_pf_ind_col}\\
 b_t &\le\rho^tb_0+ C\frac{\Delta_{\infty}(n,q,\delta)}{\alpha\sigma_{\min}}\omega_*, \label{eq_loo_pf_ind_b}\\
 c_t &\le 2\iota_0\frac{\alpha}{\beta\kappa}, \label{eq_loo_pf_ind_c}\\
 \delta_{3,t}^{\dagger,U}\vee \delta_{3,t}^{\dagger,V} &\le \eta\big\{\frac{1}{8}\alpha\sigma_{\min}\rho^t\psi_{nq}^{\dagger} + C\Delta_{\infty}(n,q,\delta)\big\}\omega_*.\label{eq_loo_pf_ind_delta}
\end{align}
The quantities $d_t$, $r_t$ and $\vartheta_t$ satisfy the same local contraction arguments as in \ref{item_uv_ind_l2_dagger}--\ref{item_uv_ind_rot_dagger} from Step~3 of Section~\ref{supp_sec_prove_thm_general_theory_UV_noise}, where $\iota_0$ is the constant appearing in induction hypothesis~\ref{item_uv_ind_rot_dagger}. In particular, once \eqref{eq_loo_pf_ind_delta} is available, the same argument as in the proof of Theorem~\ref{thm_general_theory_UV_noise} yields the contractions for $d_{t+1}$, $r_{t+1}$, and $\vartheta_{t+1}$.

We therefore proceed as follows. First, assuming \eqref{eq_loo_pf_ind_delta} holds for the previous step, we obtain the bounds for $d_{t+1}$, $r_{t+1}$, and $\vartheta_{t+1}$ exactly as in the proof of Theorem~\ref{thm_general_theory_UV_noise}. Next, assuming that the local contraction bounds \ref{item_uv_ind_l2_dagger}--\ref{item_uv_ind_rot_dagger} hold, we verify \eqref{eq_loo_pf_ind_row}--\eqref{eq_loo_pf_ind_c}. Finally, we prove \eqref{eq_loo_pf_ind_delta}. This closes the induction. At time $t=0$, the bounds \eqref{eq_loo_pf_ind_row}--\eqref{eq_loo_pf_ind_c} follow from the initialization assumptions in Theorem~\ref{Thm_example2}.

For $a_t^{\mathrm{row}}$, we begin by bounding $\bar\bU^{t+1,-i}:=\bU^{t+1,-i}\bG_t^{-i}$ and $\bar\bV^{t+1,-i}:=\bV^{t+1,-i}(\bG_t^{-i})^{-\T}$. Recall that $\bar\bU^{t+1}:=\bU^{t+1}\bG_t^{\dagger}$ and $\bar\bV^{t+1}:=\bV^{t+1}(\bG_t^{\dagger})^{-\T}$. Let
\begin{equation}
\bar\bDelta_U^{t+1,-i}:=\bar\bU^{t+1,-i}-\bar\bU^{t+1}, \qquad \bar\bDelta_V^{t+1,-i}:=\bar\bV^{t+1,-i}-\bar\bV^{t+1}.
\end{equation} 
Using the aligned update identities for the main sequence and the leave-one-out sequence, we obtain
\begin{align}
\bar\bDelta_U^{t+1,-i} ={}&\, \bDelta_U^{t,-i} -\frac{\eta}{q}\Big\{ \nabla_{\bU}h_\alpha^{-i,\natural}(\tilde\bU^{t,-i},\tilde\bV^{t,-i}) - \nabla_{\bU}h_\alpha^\natural(\tilde\bU^t,\tilde\bV^t) \Big\} \nonumber\\
&\, -\frac{\eta}{q}\Big\{ \nabla_{\bU}h_\alpha^{-i,\natural}(\tilde\bU^{t,-i},\tilde\bV^{t,-i})(\bLambda_t^{-i}-\bI_r) - \nabla_{\bU}h_\alpha^\natural(\tilde\bU^t,\tilde\bV^t)(\bLambda_t^{\dagger}-\bI_r) \Big\}, \label{eq_loo_pf__du} \\
\bar\bDelta_V^{t+1,-i} ={}&\, \bDelta_V^{t,-i} -\frac{\eta}{n}\Big\{ \nabla_{\bV}h_\alpha^{-i,\natural}(\tilde\bU^{t,-i},\tilde\bV^{t,-i}) - \nabla_{\bV}h_\alpha^\natural(\tilde\bU^t,\tilde\bV^t) \Big\} \nonumber\\
&\, -\frac{\eta}{n}\Big\{ \nabla_{\bV}h_\alpha^{-i,\natural}(\tilde\bU^{t,-i},\tilde\bV^{t,-i})\big((\bLambda_t^{-i})^{-1}-\bI_r\big) - \nabla_{\bV}h_\alpha^\natural(\tilde\bU^t,\tilde\bV^t)\big((\bLambda_t^{\dagger})^{-1}-\bI_r\big) \Big\}. \label{eq_loo_pf__dv}
\end{align}
We first analyze the principal gradient difference $\nabla_{\bU}h_\alpha^{-i,\natural}(\tilde\bU^{t,-i},\tilde\bV^{t,-i}) - \nabla_{\bU}h_\alpha^\natural(\tilde\bU^t,\tilde\bV^t)$. Note that it can be decomposed as
\begin{align}
 \Big\{ \nabla_{\bU}h_\alpha^{-i,\natural}(\tilde\bU^{t,-i},\tilde\bV^{t,-i}) \!- \! \nabla_{\bU}h_\alpha^{\natural}(\tilde\bU^{t,-i},\tilde\bV^{t,-i}) \Big\}  - \Big\{\nabla_{\bU}h_\alpha^{\natural}(\tilde\bU^{t,-i},\tilde\bV^{t,-i})\! -  \!\nabla_{\bU}h_\alpha^\natural(\tilde\bU^t,\tilde\bV^t)\Big\}.
\label{eq_loo_pf_grad_split_u}
\end{align}
The same decomposition holds for the $\bV$ gradient. We first bound the second term.
Since $\ell_{\alpha_0}''(x;y)=\nu_{\star}\sigma(x)\{1-\sigma(x)\}$ does not depend on $y$, the Hessians of $h_\alpha^\natural$ and $h_\alpha^{-i,\natural}$ are identical. Hence, by the fundamental theorem of calculus (Theorem 4.2 in \citet{lang2012real}, Chapter XIII) along the segment $\big(\bU_{t,i}(s), \bV_{t,i}(s)\big) = \big(\tilde\bU^t+s\bDelta_U^{t,-i}, \tilde\bV^t+s\bDelta_V^{t,-i}\big)$ for $s\in[0,1]$, we have
\begin{align*}
&\,\begin{pmatrix}
\mathrm{vec}\Big(\big[ \nabla_{\bU}h_\alpha^\natural(\tilde\bU^{t,-i},\tilde\bV^{t,-i}) - \nabla_{\bU}h_\alpha^\natural(\tilde\bU^t,\tilde\bV^t) \big]^\T \Big)\\ 
\mathrm{vec}\Big( \big[ \nabla_{\bV}h_\alpha^\natural(\tilde\bU^{t,-i},\tilde\bV^{t,-i}) - \nabla_{\bV}h_\alpha^\natural(\tilde\bU^t,\tilde\bV^t) \big]^\T \Big) \end{pmatrix}
\\=&\, \int_0^1 \nabla_{\bz}^2 h_\alpha^\natural(\bU_{t,i}(s),\bV_{t,i}(s))\,ds \begin{pmatrix} \mathrm{vec}\big((\bDelta_U^{t,-i})^\T\big)\\ \mathrm{vec}\big((\bDelta_V^{t,-i})^\T\big)
\end{pmatrix}.
\end{align*}
One can check that by the induction hypotheses and the bounds for the sequence $(\bU^{t,-i},\bV^{t,-i})$, the whole segment lies in the same local region $\cDuvinf$. Therefore the same Hessian contraction argument as in Step~1 of the proof of Theorem~\ref{thm_general_theory_UV_noise} gives
\begin{equation}
\label{eq_row_loo_main_contr_new}
\Big\| \bS_z^{-1} \begin{pmatrix} \mathrm{vec}\big((\bDelta_U^{t,-i})^\T\big) \\ \mathrm{vec}\big((\bDelta_V^{t,-i})^\T\big)\end{pmatrix}
\!-\! \eta \bS_z \int_0^1 \nabla_{\bz}^2 h_\alpha^\natural(\bU_{t,i}(s),\bV_{t,i}(s))\,ds \,\bS_z \bw_t^{-i} \Big\| \le \Big(1-\frac{5}{8}\eta\alpha\sigma_{\min}\Big)\|\bw_t^{-i}\|,
\end{equation}
where we write $\bw_t^{-i} = \bS_z^{-1}\big(\mathrm{vec}\big\{(\bDelta_U^{t,-i})^\T\big\}^\T,\mathrm{vec}\big\{(\bDelta_V^{t,-i})^\T\big\}^\T\big)^\T$.

We now bound the first term in \eqref{eq_loo_pf_grad_split_u}.  Note that
\begin{align}
\label{eq_row_loo_pert_u_exact_new}
\nabla_{\bU}h_\alpha^\natural(\tilde\bU^{t,-i},\tilde\bV^{t,-i}) - \nabla_{\bU}h_\alpha^{-i,\natural}(\tilde\bU^{t,-i},\tilde\bV^{t,-i}) = \nu_{\star}\be_i\Big\{\be_i^\T(\bP^*-\bY)\tilde\bV^{t,-i}\Big\},\\
\nabla_{\bV}h_\alpha^\natural(\tilde\bU^{t,-i},\tilde\bV^{t,-i}) - \nabla_{\bV}h_\alpha^{-i,\natural}(\tilde\bU^{t,-i},\tilde\bV^{t,-i}) = \nu_{\star}(\bP^*_{i,\cdot}-\bY_{i,\cdot})^\T \tilde\bU^{t,-i}_{i,\cdot}. \label{eq_row_loo_pert_v_exact_new}
\end{align}
Now the perturbation is evaluated at $(\tilde\bU^{t,-i},\tilde\bV^{t,-i})$, which depends on the $i$th row of the data through $\bG_t^{-i}$. To control this dependence, we let 
\begin{equation*}
    \check\bG_t^{-i}\in\argmin_{\bG\in GL(r)} \|(\bU^{t,-i}\bG-\bU^*,\bV^{t,-i}\bG^{-\T}-\bV^*)\|_{\Fnw},
\end{equation*}and set $\check\bU^{t,-i}:=\bU^{t,-i}\check\bG_t^{-i}$, $\check\bV^{t,-i}:=\bV^{t,-i}(\check\bG_t^{-i})^{-\T}$. Then one can easily check with Lemma~\ref{lemma_alignment_near_rotation} that 
\begin{equation}\label{eq_loo_mc_reference_change}
    \big\|\check\bG_t^{-i}(\bG_t^{-i})^{-1} - \bI_r\big\|\vee \big\|\bG_t^{-i}(\check\bG_t^{-i})^{-1} - \bI_r\big\|\le C\big(\rho^t\phi_{nq} + \frac{\Delta_{2}(n,q,\delta)}{\alpha\sigma_{\min}}\big)\frac{\tau_*}{\sqrt{\sigma_{\min}}}.
\end{equation}

Subsequently, let $\cF_{-i}:=\sigma\!\big(\{Y_{k\ell}:k\neq i,\ \ell\in[q]\}\big)$. Then $(\check\bU^{t,-i},\check\bV^{t,-i})$ is $\cF_{-i}$-measurable, while $\{Y_{ij}\}_{j=1}^q$ is independent of $\cF_{-i}$. By definition, we know $\nu_{\star}P_{ij}^*(1-P_{ij}^*)\le \beta$. Conditional on $\cF_{-i}$, with probability at least $1-Ce^{-cL_{\star}}$,
\begin{equation*}
\Big\|\sum_{j=1}^q \nu_\star (Y_{ij}-P_{ij}^*)(\check\bV^{t,-i})_{j,}\Big\| \le
C\left\{ \sqrt{\nu_\star\beta L_{\star}}\,\|\check\bV^{t,-i}\|_{\Fn} + \nu_\star L_{\star}\,\|\check\bV^{t,-i}\|_{2\to\infty} \right\}.
\end{equation*}
Furthermore, with \eqref{eq_loo_mc_reference_change}, one can obtain that 
\begin{align*}
\big\|\be_i^\T(\bY-\bP^*)\tilde\bV^{t,-i}\big\| \le &\,\big\|\be_i^\T(\bY-\bP^*)\check\bV^{t,-i}\big\|\Big(1 + \big\|\bG_t^{-i}(\check\bG_t^{-i})^{-1} - \bI_r\big\|\Big) \\ \le &\, C\big\{\sqrt{\beta L_{\star}/\nu_{\star}}\,\|\check\bV^{t,-i}\|_{\Fn} + L_{\star}\|\check\bV^{t,-i}\|_{\twinf}\big\}.
\end{align*}
Next, because $\|(\bY_{i,\cdot}-\bP^*_{i,\cdot})^\T \tilde\bU^{t,-i}_{i,\cdot}\|_{\Fn} = \|\bY_{i,\cdot}-\bP^*_{i,\cdot}\|_2\, \|\tilde\bU^{t,-i}_{i,\cdot}\|$, Bernstein inequality gives, with probability at least $1-\delta/n$,
\begin{equation*}
\|\bY_{i,\cdot}-\bP^*_{i,\cdot}\|_2 \le C\sqrt{(q\beta+\nu_{\star}L_{\star})/{\nu_{\star}}}.
\end{equation*}
Therefore, using the induction hypotheses and the full sequence bounds, we get
\begin{equation*}
\|\tilde\bV^{t,-i}\|_{\Fn} \le \|\tilde\bV^t\|_{\Fn}+\|\bDelta_V^{t,-i}\|_{\Fn} \le C\sqrt q\,\tau_*,
\qquad \|\tilde\bU^{t,-i}_{i,\cdot}\| \le \|\bU^*_{i,\cdot}\|+b_t \le C\omega_*.
\end{equation*}
Substituting these estimates into \eqref{eq_row_loo_pert_u_exact_new} and \eqref{eq_row_loo_pert_v_exact_new}, we obtain uniformly over $i\in[n]$,
\begin{align}
\frac{1}{q\sqrt n} \Big\| \nabla_{\bU}h_\alpha^\natural(\tilde\bU^{t,-i},\tilde\bV^{t,-i}) - \nabla_{\bU}h_\alpha^{-i,\natural}(\tilde\bU^{t,-i},\tilde\bV^{t,-i}) \Big\|_{\Fn} &\le C\left\{ \tau_*\sqrt{\frac{\nu_\star\beta L_\star}{nq}} + \nu_\star\omega_*\frac{L_\star}{q\sqrt n} \right\}, \label{eq_row_loo_pert_u_final_new} \\
\frac{1}{n\sqrt q} \Big\| \nabla_{\bV}h_\alpha^\natural(\tilde\bU^{t,-i},\tilde\bV^{t,-i}) - \nabla_{\bV}h_\alpha^{-i,\natural}(\tilde\bU^{t,-i},\tilde\bV^{t,-i}) \Big\|_{\Fn}
&\le C\omega_* \sqrt{\frac{\nu_\star(q\beta+\nu_\star L_\star)}{n^2q}} .
\label{eq_row_loo_pert_v_final_new}
\end{align}
Now for the second term in \eqref{eq_loo_pf__du} and \eqref{eq_loo_pf__dv} that controls the balancing along the LOO iterates, similar to handling $\gamma_{4,t}^{\dagger}$ in the proof of Theorem~\ref{thm_general_theory_UV_noise}.
Using $\|\bLambda_t^{-i}-\bLambda_t^*\|
\vee \|(\bLambda_t^{-i})^{-1}-(\bLambda_t^*)^{-1}\| \le 6\iota_0\alpha/(\beta\kappa)$,
together with the bounds
\eqref{eq_row_loo_pert_u_final_new} and \eqref{eq_row_loo_pert_v_final_new}, we know that the second term in \eqref{eq_loo_pf__du} and \eqref{eq_loo_pf__dv} can be bounded similar to \eqref{eq_row_loo_pert_u_final_new} and \eqref{eq_row_loo_pert_v_final_new}.
Thus, the error arising from balancing can be absorbed into \eqref{eq_row_loo_pert_u_final_new} and \eqref{eq_row_loo_pert_v_final_new}.

Combining \eqref{eq_loo_pf__du}, \eqref{eq_loo_pf__dv},
\eqref{eq_row_loo_main_contr_new}, \eqref{eq_row_loo_pert_u_final_new},
and \eqref{eq_row_loo_pert_v_final_new}, we know $\bar a^{\mathrm{row}}_{t} =\max_{i\in[n]}n^{-1/2}\big\|\bar\bDelta_U^{t+1,-i}\big\|_{\Fn}\vee \max_{i\in[n]}q^{-1/2}\big\|\bar\bDelta_V^{t+1,-i}\big\|_{\Fn} $ can be controlled by 
\begin{align*}
\bar a^{\mathrm{row}}_{t}\le \Big(1-\frac{5}{8}\eta\alpha\sigma_{\min}\Big)a_t^{\mathrm{row}} + C\eta \left\{ \tau_*\sqrt{\frac{\nu_\star\beta L_\star}{nq}} + \nu_\star\omega_*\frac{L_\star}{q\sqrt n} + \omega_*\sqrt{\frac{\nu_\star(q\beta+\nu_\star L_\star)}{n^2q}} \right\}.\end{align*}
The column-wise version $a_{t}^{\rm col}$ follows by swapping $n,\bU$ and $q,\bV$.

We next bound $b_{t+1}$. The key point is that $\be_i^\T(\bY^{-i}-\bP^*)=0$.
In particular, following the same row-wise decomposition for $\tilde\bU^{t,-i}_{i,\cdot}-\bU^*_{i,\cdot}$ as in Step~3 of the proof of Theorem~\ref{thm_general_theory_UV_noise}, the analogue of the stochastic term $\delta_{3,t}^{\dagger,(U)}$ vanishes identically at row $i$. This reduces to the deterministic scheme. Therefore, a similar argument can yield
\begin{equation*}
\|\bar\bU_{i,\cdot}^{t+1,-i}-\bU^*_{i,\cdot}\| \le \Big(1-\frac12\eta\alpha\sigma_{\min}\Big) \|\tilde\bU_{i,\cdot}^{t,-i}-\bU^*_{i,\cdot}\| +C\eta{\Delta_{\infty}(n,q,\delta)}\omega_*.
\end{equation*}

Similar to the proof of Lemma~\ref{lemma_alignment_near_rotation}, we can show the following: suppose $(\bU_1,\bV_1)$ and $(\bU_2,\bV_2)$ lie within the same local neighborhood of the reference $(\tilde\bU,\tilde\bV)$, and let their optimal alignments to $(\tilde\bU,\tilde\bV)$ be $\bG_1$ and $\bG_2$. If
\begin{equation*}
\max\left\{ n^{-1/2}\|\bU_1\bG_0-\bU_2\bH_0\|_{\Fn}, q^{-1/2}\|\bV_1\bG_0^{-\T}-\bV_2\bH_0^{-\T}\|_{\Fn} \right\}\le \delta\end{equation*}
for two well-conditioned alignments $\bG_0,\bH_0$, then there is
\begin{equation*}
\|\bG_1-\bG_2\|\vee \|\bG_1^{-\T}-\bG_2^{-\T}\| \le C\frac{\delta}{\sqrt{\sigma_{\min}}}.\end{equation*}Then with $(\bU_1,\bV_1) = (\bar\bU^{t+1,-i}, \bar\bV^{t+1,-i})$, and $(\bU_2,\bV_2) = (\tilde\bU^{t+1}, \tilde\bV^{t+1})$ and we know 
\begin{equation*}
    \big\|\bG_{t+1}^{-i} - \bG_{t+1}^{\dagger}\big\|\vee \big\|(\bG_{t+1}^{-i})^{-\T} - (\bG_{t+1}^{\dagger})^{-\T}\big\|\le C\frac{\bar a_{t+1}^{\mathrm{row}}}{\sqrt{\sigma_{\min}}}.
\end{equation*}
The smallness of $\bar a_{t+1}^{\mathrm{row}}$ and the induction hypothesis for $\|\bG_s^{\dagger}-\bR^0\|\vee \|(\bG_s^{\dagger})^{-\T}-\bR^0\|$ yields \eqref{eq_loo_pf_ind_c}.

Now we are ready to bound $\|(\bY - \bP^*)\tilde\bE_V^{t}\|_{\twinf}$ and $\|(\bY - \bP^*)^\T\tilde\bE_U^{t}\|_{\twinf}$ for step $t$. We analyze $\be_i^\T(\bY - \bP^*)\tilde\bE_V^t = \be_i^\T(\bY - \bP^*)(\tilde\bV^t - \tilde\bV)$.
We write $\bW = \nu_{\star}(\bY  -\bP^*)$ and for each $i\in[n]$, decompose $\tilde\bE_{V}^{t} = (\tilde\bV^{t}-\tilde\bV^{t,-i}) + (\tilde\bV^{t,-i} - \check\bV^{t,-i}) + (\check\bV^{t,-i}-\bV^*) + (\bV^* - \bar\bV) + (\bar\bV - \tilde\bV)$. Then
\begin{equation}\label{eq_loo_mc_recentered_V_proof}
\begin{aligned}
q^{-1}\big\|\{\bW\tilde\bE_V^t\}_{i,}\big\| \le{}\;& q^{-1}\|\bW_{i,}(\tilde\bV^{t}-\tilde\bV^{t,-i})\| + q^{-1}\big\|\bW_{i,}(\tilde\bV^{t,-i}-\check\bV^{t,-i})\big\|\\&\; + q^{-1}\big\|\bW_{i,}(\check\bV^{t,-i}-\bV^*)\big\| + q^{-1}\big\|\bW_{i,}(\bV^* - \bar\bV)\big\| + q^{-1}\big\|\bW_{i,}(\bar\bV - \tilde\bV)\big\|\\ := &\, T_{1,t} + T_{2,t} + T_{3,t} + T_{4,t} + T_{5,t}.
\end{aligned}\end{equation}

The first term is bounded by Cauchy's inequality and the row-norm concentration for $\bY-\bP^*$.  Thus
\begin{equation*}
T_{1,t} \le q^{-1}\nu_{\star}\|\bY - \bP^*\|_{\twinf}\|\tilde\bV^{t} - \tilde\bV^{t, -i}\big\|_{\Fn}\le  C a_t^{\rm row}\sqrt{\frac{\nu_\star(q\beta+\nu_\star L_\star)}{q}}.
\end{equation*}
By the initialization assumptions, we further have
\begin{equation*}
    T_{1,t}\le C a_0^{\rm row}\Big\{\frac{\nu_\star(q\beta+\nu_\star L_\star)}{q}\Big\}^{1/2} \le c_0C\alpha\sigma_{\min}\psi_{nq}^{\dagger}\omega_* + C\Delta_\infty(n,q,\delta)\omega_*.
\end{equation*}
Here, the second inequality follows from \eqref{eq_loo_pf_ind_row} and the scaling conditions in \eqref{eq_bern_prop_scaling}.

For the second term, \eqref{eq_loo_mc_reference_change} and $\tilde\bV^{t,-i}-\check\bV^{t,-i} = \check\bV^{t,-i} \big\{ \check\bG_t^{-i}(\bG_t^{-i})^{-1}-\bI_r \big\}^{\T},$ imply
\begin{align*}
T_{2,i}\le &\,q^{-1}\nu_{\star}\|(\bY - \bP^*)_{i,}\check\bV^{t,-i}\|\times \|\bI_r - \check\bG_t^{-i}(\bG_t^{-i})^{-1}\| \\\le &\,C\left\{ \tau_*\sqrt{\frac{\nu_\star\beta L_\star}{q}} + \nu_\star\omega_*\frac{L_\star}{q} \right\}
\left( \rho^t\phi_{nq}^{\dagger} + \frac{\Delta_2(n,q,\delta)}{\alpha\sigma_{\min}} \right) \frac{\tau_*}{\sqrt{\sigma_{\min}}},
\end{align*}
where for $q^{-1}\|(\bY - \bP^*)_{i,}\check\bV^{t,-i}\|$ we applied the Bernstein inequality with $q^{-1/2}\|\check\bV^{t,-i}\|_{\Fn}\le C\tau_*$ and $\|\check\bV^{t,-i}\|_{\twinf}\le C\omega_*$, as $\check\bV^{t,-i}$ is $\cF_{-i}$-measurable. Using $d_t\le \rho^t\phi_{nq}^{\dagger}\tau_*$, $\tau_*\le \sqrt{r\kappa\sigma_{\min}}$, $\tau_*\le \sqrt r\,\omega_*$,  $\zeta_r\phi_{nq}^{\dagger}/\psi_{nq}^{\dagger}\le c_0$, and
\eqref{eq_bern_prop_scaling}, we obtain
\begin{equation*}
    T_{2,t} \le
    c_0\alpha\sigma_{\min}\rho^t\psi_{nq}^{\dagger}\omega_* + C\Delta_\infty(n,q,\delta)\omega_* .
\end{equation*}

For the third term, from the theorem's scaling condition and \eqref{eq_loo_mc_reference_change}, \eqref{eq_loo_pf_ind_row}, we use $\check\bV^{t,-i}-\bV^* = \check\bV^{t,-i} - \tilde\bV^{t} + \tilde\bV^{t} - \tilde\bV + \tilde\bV - \bV^*$ to obtain
\begin{align*}
    T_{3,i} \le&\, C\sqrt{\frac{\nu_\star\beta L_\star}{q}}\times q^{-1/2}\|\check\bV^{t,-i}-\bV^*\|_{\Fn} + C\frac{\nu_\star L_\star}{q}\times \|\check\bV^{t,-i}-\bV^*\|_{2\to\infty}\\\le &\,C\sqrt{\frac{\nu_\star\beta L_\star}{q}} \left\{ a_t^{\rm row} + d_t + \frac{\Delta_2(n,q,\delta)}{\alpha\sigma_{\min}}\tau_* \right\} +  C\frac{\nu_\star L_\star}{q}\omega_*.
\end{align*}
Since $\sqrt{{\nu_\star\beta L_\star}/{q}}\le  \big\{(\nu_{\star}q\beta + \nu_{\star}^2L_{\star})/q\big\}^{1/2}$ under \eqref{eq_bern_prop_scaling}, $a_t^{\rm row}$ can be controlled by \eqref{eq_loo_pf_ind_row}. The $d_t$ part is controlled by the same argument as in the treatment of $T_{2,t}$, and the remaining terms are absorbed into $C\Delta_\infty(n,q,\delta)\omega_*$ by
\eqref{eq_bern_prop_scaling}. Then we arrive at
\begin{equation*}
    T_{3,t} \le
    c_0\alpha\sigma_{\min}\rho^t\psi_{nq}^{\dagger}\omega_* + C\Delta_{\infty}(n,q,\delta)\omega_* .
\end{equation*}

The fourth term is exactly Lemma~\ref{lemma_rect_auxiliary_bar}, because changing the sign of the noise matrix does not affect the norm:
\begin{equation*}
T_{4,t}\le C\Delta_\infty(n,q,\delta)\omega_* .
\end{equation*}
For the fifth term, apply \eqref{eq_lemma_rect_auxiliary_bar_close}, Bernstein's row bound, and \eqref{eq_bern_prop_scaling} to get
\begin{align*}
T_{5,t}\le C\Delta_\infty(n,q,\delta)\omega_* .
\end{align*}

Combining the bounds for $T_{1,t}$--$T_{5,t}$ and shrinking the constants
in \eqref{eq_bern_prop_scaling}, we obtain uniformly over $i\in[n]$, we know $q^{-1}\|\{\bW\tilde\bE_V^t\}_{i,\cdot}\|\le \alpha\sigma_{\min}\rho^t\psi_{nq}^{\dagger}\omega_*/8 +C\Delta_\infty(n,q,\delta)\omega_* $.
Taking the maximum over $i\in[n]$ gives the upper bound for $\delta_{3,t}^{\dagger,U}$ in \eqref{eq_loo_mc_delta3_target}. The bound for $\delta_{3,t}^{\dagger,V}$ in \eqref{eq_loo_mc_delta3_target} can be obtained by symmetry.

Finally, replacing \eqref{eq_rect_step4_delta3} in Step~3 by
\eqref{eq_loo_pf_ind_delta} gives again
\[
\|(\tilde\bE_U^{t+1},\tilde\bE_V^{t+1})\|_{\twinfw} \le \rho\|(\tilde\bE_U^t,\tilde\bE_V^t)\|_{\twinfw} + C\eta\Delta_\infty(n,q,\delta)\omega_* .
\]
Following the same strategy, one can prove
\[
r_t \le \rho^t\psi_{nq}^{\dagger}\omega_* + C\frac{\Delta_\infty(n,q,\delta)}{\alpha\sigma_{\min}}\omega_* .
\]
The derivation of the induction hypothesis
\ref{item_uv_ind_rot_dagger} in Section~\ref{supp_sec_prove_thm_general_theory_UV_noise} is unchanged, since it only uses the Frobenius norm contraction. This closes the original induction and completes the proof of Lemma~\ref{lem_tentative_loo_gradient_mc}.

\end{proof}


\end{document}